\documentclass{article}
\usepackage[margin=1in]{geometry}
\usepackage[T1]{fontenc}
\usepackage[utf8]{inputenc}
\usepackage{microtype}
\usepackage[numbers,sort&compress]{natbib}
\usepackage{amsmath,amssymb,amsthm,mathtools,bm}
\usepackage[cal=euler]{mathalpha}
\usepackage{cochineal}
\usepackage{enumitem}
\usepackage{booktabs}
\usepackage{algorithm}
\usepackage[noend]{algpseudocode}
\usepackage{float}
\usepackage[hidelinks,bookmarksdepth=2]{hyperref}
\usepackage{url}
\usepackage{xcolor}
\usepackage{tikz}
\usepackage{longtable}
\usepackage{tabularx}

\mathtoolsset{showonlyrefs,showmanualtags}
\allowdisplaybreaks

\newtheorem{theorem}{Theorem}[section]
\newtheorem{lemma}[theorem]{Lemma}
\newtheorem{proposition}[theorem]{Proposition}
\newtheorem{corollary}[theorem]{Corollary}

\theoremstyle{remark}

\newcommand{\R}{\mathbb{R}}
\newcommand{\E}{\mathbb{E}}
\newcommand{\one}{\mathbf{1}}
\newcommand{\KL}{\mathrm{KL}}
\newcommand{\reg}{\mathrm{Reg}}
\newcommand{\Var}{\mathrm{Var}}
\newcommand{\Cov}{\mathrm{Cov}}
\newcommand{\TV}{\mathrm{TV}}

\newcommand{\Hc}{\mathrm{H}}
\newcommand{\cC}{\mathcal{C}}

\newcommand{\ip}[2]{\langle #1,#2\rangle}

\newcommand{\Env}[2]{\mathrm{Env}_{\Gamma}\!\left(#1;\,#2\right)}

\newcommand{\akshay}[1]{}
\newcommand{\vac}[1]{}

\makeatletter
\g@addto@macro\normalsize{%
  \setlength{\abovedisplayskip}{3pt plus 1pt minus 1pt}%
  \setlength{\belowdisplayskip}{3pt plus 1pt minus 1pt}%
  \setlength{\abovedisplayshortskip}{0pt plus 1pt}%
  \setlength{\belowdisplayshortskip}{2pt plus 1pt minus 1pt}%
}
\def\thm@space@setup{%
  \thm@preskip=5pt plus 1pt minus 1pt
  \thm@postskip=\thm@preskip
}
\makeatother
\title{Adaptive Bayes exactly tracks information over intrinsic time}
\author{Akshay Balsubramani \\ {\small \texttt{akshay@vac.bio}}}
\date{}

\makeatletter
\def\bvapx@title{}
\AtBeginDocument{\providecommand{\phantomsection}{}}
\let\bvapx@maketitle\maketitle
\def\maketitle{\global\let\bvapx@title\@title\bvapx@maketitle}
\newcommand{\appendixtitle}{%
  \clearpage
  \suppressfloats[t]%
  \phantomsection
  \begin{center}
    {\LARGE\bfseries Appendices\ifx\bvapx@title\@empty\else\space to ``\bvapx@title''\fi\par}%
  \end{center}
  \vspace{1.5\baselineskip}%
}
\makeatother

\providecommand{\akshay}[1]{}\renewcommand{\akshay}[1]{}%
\providecommand{\vac}[1]{}\renewcommand{\vac}[1]{}%
\begin{document}
\maketitle

\begin{abstract}
Bayesian and multiplicative-weights updates reweight experts, models, or actions from sequential feedback.
We show that the regret of any such update obeys an exact information-accounting identity.
On each round, the learner's excess loss to any chosen comparator is the sum of an immediate cost for the uncertainty exposed by the round and a reduction in the information distance from the learner's current weights to the comparator.
The cumulative cost defines a pathwise uncertainty clock, the \emph{intrinsic time} of the realized sequence.
Summing one-step balances yields two exact adaptive decompositions of cumulative regret, one for each natural way of composing the update across rounds.
Because the decompositions are exact, favorable stochastic or low-noise regimes appear as self-bounding properties of the realized intrinsic time.
The accounting also fixes a learning rate, inverse in the square root of intrinsic time.
That schedule is competitive with adaptive baselines in selected online-learning settings.
The same calculus covers Hedge, optimistic and side-information variants, continuous priors, boosting, online convex optimization, contextual bandits, and repeated games: the pathwise account is the same in every case.
\end{abstract}

\section{Introduction}

The main mechanisms for learning from sequential feedback are Bayesian and multiplicative-weights updates.
In the classical experts problem they reweight experts after each loss vector; in Bayesian model averaging they reweight models after each observation; and in many modern pipelines they reweight candidate actions, hypotheses, or responses after partial evidence.
The common basic technique is to keep a distribution over candidates, observe feedback, and move mass toward candidates that explain that feedback better.

On an arbitrary sequence, the regret of any such update equals an information quantity exactly, round by round.
Each round's excess loss splits into an immediate cost for the uncertainty the round exposes, plus the information transported toward the comparator.
Added together, these costs define a clock---the realized sequence's \emph{intrinsic time}---computed from quantities the update itself generates along the path.
Because the relation is an identity, an easy sequence registers directly as a slow realized clock.

The performance of these updates is therefore an exact ledger of information.
Summing the one-step balances produces two adaptive decompositions: one for a prior-retempered update that recomputes the current posterior from the original prior using cumulative scores, and one for a local update that moves only from the current weights using the current score vector.
In both cases, information is the common accounting currency: it measures comparator complexity, governs learning-rate effects, and records how difficult the realized sequence was.

\begin{figure}[t]
\centering
\includegraphics[width=\textwidth]{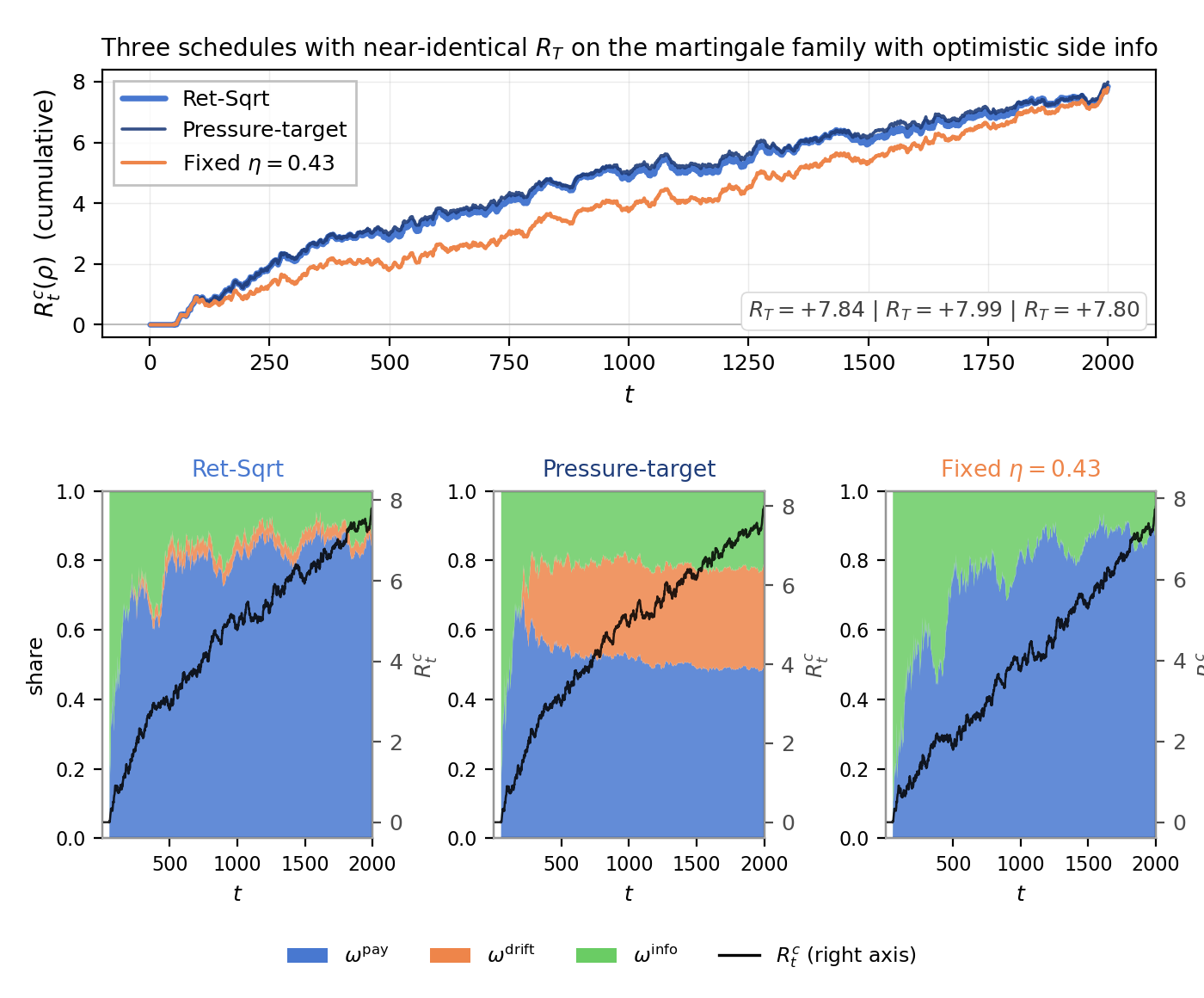}
\caption{\textbf{Three schedules reach the same regret by different routes, and the identity separates mechanisms that the terminal regret alone conceals.} The three learning-rate schedules run the same exponential-weights update---the intrinsic-time square-root schedule \textsc{Ret-Sqrt}, the pressure-target line search, and a fixed rate $\eta=0.43$---on one sequence family (martingale losses with optimistic side information; $K=8$, $T=2000$, $12$ seeds). \textbf{Top:} cumulative regret to the best fixed expert; the three terminal values agree to within $2.4\%$ ($R_T^c=+7.84,\,+7.99,\,+7.80$). \textbf{Bottom:} the identity splits each schedule's regret, round by round, into three shares summing to one---intrinsic-time loss $\omega^{\mathrm{pay}}$ (blue), temperature drift $\omega^{\mathrm{drift}}$ (orange), comparator information $\omega^{\mathrm{info}}$ (green)---with the signed regret as the thin black line on the secondary axis. \textsc{Ret-Sqrt} and the fixed rate pay chiefly through the clock ($\omega^{\mathrm{pay}}_T\approx 0.85$, $\omega^{\mathrm{info}}_T\approx 0.13$); the pressure-target schedule reaches the same regret by a different route ($\omega^{\mathrm{pay}}_T\approx 0.50$, $\omega^{\mathrm{drift}}_T\approx 0.30$, $\omega^{\mathrm{info}}_T\approx 0.21$) because its temperature varies locally with the data.}
\label{fig:same-regret-diff-decomp}
\end{figure}

The same accounting runs unchanged on real data: the round-by-round split closes exactly, and the three shares of Figure~\ref{fig:same-regret-diff-decomp} separate datasets by where their regret goes.
Benchmarked against adaptive online-learning baselines, the two schedules recommended here are competitive in the settings tested, and the decomposition attributes each gap to a named term (Section~\ref{sec:experiments}; Figure~\ref{fig:sota-comparison}).

This update template applies well beyond prediction with expert advice.
It encompasses Bayesian model averaging, optimistic online learning, bandit algorithms, boosting, online convex optimization over continuous domains, repeated-game play, contextual decision-making, multiscale aggregation, and softmax preference learning.
All are instances of a single Bayes-rule update with different choices of scores, side information, and learning rate. 

On round $t$ the learner chooses weights $p_t\in\Delta([K])$ over $K$ experts, observes expert losses $\ell_t\in\R^K$, and incurs mixture loss $\ip{p_t}{\ell_t}$.
For a comparator distribution $\rho\in\Delta([K])$, the cumulative regret up to time $T$ is
\[
R_T^\ell(\rho) := \sum_{t=1}^T \ip{p_t}{\ell_t}-\ip{\rho}{L_T}
\qquad \qquad
L_T := \sum_{t=1}^T \ell_t
\]
Point-mass comparators recover the usual regret to a single expert, while diffuse $\rho$ compare to arbitrary mixtures.
A comparator is the benchmark distribution we want to match in hindsight, and regret is the learner's extra cumulative loss relative to that benchmark.
Throughout, the loss sequence may be arbitrary, so every statement is pathwise, over the realized trajectory of the learner.

For most of the introduction the per-round loss can be read as the ordinary loss vector $\ell_t$: the classic Hedge/exponential-weights algorithm \citep{FreundSchapire97} is the special case where the score driving the update is $c_t=\ell_t$ (in the original Hedge parameterization, weights are multiplied by $\beta^{\ell_t(i)}$ with $\beta=e^{-\eta}$).
More generally, side information, optimistic forecasts, compensators, and several partial-information reductions are folded into a single \emph{composite loss} $c_t$ on which the same analysis runs unchanged \cite{HazanKale2010,ChiangEtAl2012,SteinhardtLiang2014}; the explicit reduction is developed in Section~\ref{sec:exact}, and until then no harm comes from reading $c_t$ as $\ell_t$.

This viewpoint is especially natural in settings where one repeatedly reweights proposals, actions, or hypotheses from sequential evidence.
It underlies optimistic repeated-game play \cite{HazanKale2010,ChiangEtAl2012}, adaptive step-size selection in online and stochastic optimization \cite{DuchiHazanSinger2011}, and, more recently, relative-entropy-regularized policy updates in reinforcement learning from human feedback \cite{rafailov2023neuripsb} and related post-training pipelines for language models. In those settings, the same softmax reweighting step governs how candidate responses are promoted or suppressed by sequential evidence.

Once written in this form, the prior plays two roles.
It initializes the update, and it also serves as a decoder for the comparator class.
For a point comparator concentrated on expert $i$, the description cost is $-\log\pi(i)$; for a mixed comparator $\rho$ it is $\KL(\rho\|\pi)$.
Difficulty is therefore measured relative to the realized sequence and the chosen decoder; no data-generating law is assumed \citep{grunwald2007minimum}.
This is close in spirit to the predictive-complexity and game-theoretic traditions of online learning, where one fixes the sequential prediction rule and asks how hard the realized path is for that rule \cite{Dawid1984,Vovk1998AA,MerhavFeder1998,Kalnishkan2002,vovk2005algorithmic,ShaferVovk2019}.

\subsection{The common mechanism: one-step information balance}
\label{sec:intro-mechanism}

The analysis rests on a single one-step identity.
Fix the \textbf{exponential-weights update}---equivalently the multiplicative-weights or one-step Bayesian update---which reweights the current distribution by the exponentiated score at temperature $\eta_t$, $q_{t+1}(i)\propto q_t(i)\,e^{-\eta_t c_t(i)}$.
The temperature is the learning rate of online learning; the two names are used interchangeably, the first emphasizing the Gibbs form of the update.
If $q_t$ is the current distribution, $q_{t+1}$ the distribution it produces on composite losses $c_t$, and $\rho$ is any comparator, then
\begin{equation}
\label{eq:intro-centered-step}
\ip{q_t}{c_t} - \ip{\rho}{c_t}
= \delta_t(c)
+ \frac{\KL(\rho\|q_t) - \KL(\rho\|q_{t+1})}{\eta_t}
\end{equation}
where $\delta_t(c)$ is the round's immediate cost: the gap between the learner's mean composite loss and the centered logarithmic mix loss for that round.
Read from left to right, the identity says that excess loss splits into what is paid now and what is converted into a reduction of the information gap to the comparator.
No approximation has yet been made.

This one-step balance is the paper's basic accounting statement, and it is special to exponential reweighting.
The cancellation that turns excess loss into a telescoping information gap is the algebraic identity the exponential-weights update satisfies; an arbitrary update rule admits no such cancellation.
Within that family it holds for every Bayesian update over experts, regardless of how the temperature is chosen.
At fixed temperature the relative-entropy term telescopes, giving the familiar fixed-rate identities.
When the temperature varies, the same one-step balance can be composed in two natural ways, depending on how the new temperature enters the update.

\subsection{Two ways to compose the same update}
\label{sec:intro-two-updates}

Once the temperature varies, two natural adaptive updates appear.
The first is the \textbf{prior-retempered update}, which recomputes the current posterior from the original prior at the current temperature:
\[
p_t(i)\propto \pi(i)e^{-\eta_t C_{t-1}(i)}
\qquad \qquad
C_t(i):=\sum_{s=1}^t c_s(i)
\]
Its decomposition (Theorem~\ref{thm:pacbayes-second-order}) says that for every comparator $\rho$ and time $T$,
\begin{equation}
\label{eq:intro-three-piece}
R_T^c(\rho)
=
\underbrace{\sum_{t=1}^T \eta_t Q_t(c)}_{\textup{intrinsic-time loss}}
+
\underbrace{D_T}_{\textup{drift from changing temperature}}
+
\underbrace{B_T(\rho)}_{\textup{terminal comparator information}}
\end{equation}
Here $Q_t(c)\ge 0$ is the round-difficulty increment revealed by the update itself.
The cumulative quantity $V_T(c):=\sum_{t=1}^T Q_t(c)$ is the \emph{intrinsic time}: the uncertainty clock that the realized path itself exposes.

The second is the \textbf{local update}, $p_{t+1}(i)\propto p_t(i)e^{-\eta_t c_t(i)}$,
which updates only from the current weights and the current score vector.
Its cumulative identity (Proposition~\ref{prop:pressure-cumulative}) is built from the same one-step loss term $Q_t(c)$, but the cumulative bookkeeping is different. Instead of a temperature-drift term, it produces a cumulative normalization identity, and for comparator classes it is naturally expressed through the terminal mass $p_{T+1}(E)$ assigned to an event or class $E$.

The practical distinction mirrors the algebra: the two updates call for different learning-rate schedules.
The prior-retempered update pairs naturally with a monotone second-order schedule driven by cumulative intrinsic time.
The local update pairs naturally with a ``pressure-target" line search that chooses the current temperature so that the one-step normalized loss hits a prescribed target.
At fixed temperature the two updates agree; under adaptive rates they diverge, even though the same one-step information balance drives both of them.
Crossing the two update geometries with two natural rate controllers gives a $2\times2$ family of adaptive Bayes algorithms, developed in Section~\ref{sec:schedules}.
The two main pairings are the prior-retempered update at a square-root schedule and the local update at a pressure target.

\subsection{The meaning of intrinsic time}
\label{sec:intro-intrinsic}

The decomposition therefore shifts attention from the size of regret to its composition: how the realized sequence splits it into online uncertainty $\sum_t \eta_t Q_t(c)$, schedule-induced distortion $D_T$, and comparator description cost $B_T(\rho)$. 
This variational form is kept intact for as long as possible: up through the schedule-dependent calculations, the main statements are identities or two-sided envelopes, and one-sided inequalities enter only where the exact terms are deliberately relaxed to simpler online surrogates. 

Two clocks must be kept apart: the exact intrinsic clock $V_t(c)$, on which the identities run, and a quadratic relaxation $W_t(c)$ used only as a practical controller (Section~\ref{sec:second-order}).

The intrinsic-time increment $Q_t(c)$ is therefore an exact quantity the update incurs. Proposition~\ref{prop:tilted-var} shows that it is an average of tilted variances generated by the Bayes-rule update on that round, verified numerically in Appendix~\ref{sec:eval-intrinsic-time-clock}; Section~\ref{sec:adaptive} takes up what this says about how the update reveals the realized path's structure.

Both decompositions are built from the same per-round information functional: the intrinsic-time increment $Q_t(c)$, read as a centered finite-temperature cumulant of the current played distribution. Only the prefactors and the cumulative bookkeeping differ. The two update families therefore share one mechanism---the round-by-round information cost is the same functional of $(p_t,c_t)$ in both---while the realized values of $Q_t(c)$ remain algorithm-dependent, since under a variable schedule the two recursions generally play different $p_t$ on the same loss path. They coincide round by round only when the temperature schedule is fixed, so that the played distributions agree.

\subsection{Choosing the update and the learning rate}
\label{sec:intro-which}

An adaptive Bayes algorithm is fixed by exactly two choices: whether the update is composed from the cumulative scores or from the current score alone, and how the learning rate is set.
Familiar algorithms occupy cells of this design space.
At any fixed temperature the two compositions coincide and reduce to classic Hedge \cite{FreundSchapire97}; AdaHedge plays the prior-retempered composition at its gap-implied rate \cite{derooij2014adahedge}; variance-tuned rates drive the same composition through a quadratic clock \cite{CesaBianchiMansourStoltz2007,gaillard2014secondorderhedge}; Safe Bayes runs it at a misspecification-calibrated constant temperature \cite{the2012safebayesian}.
Because the identity holds for every predictable schedule, the accounting covers all of these algorithms at once, and it tightens their classical analyses.
The familiar regret bounds are relaxations of the same ledger, obtained by discarding the terminal comparator information and replacing the exact intrinsic-time loss by a worst-case proxy (Appendix~\ref{sec:variational-sharpening}).

The recommended default is the prior-retempered update at the inverse-square-root intrinsic-time rate $\eta_t\propto 1/\sqrt{V_{t-1}(c)}$ of Section~\ref{sec:second-order}, with the constant set by a comparator-information budget.
Tunings of this square-root form differ only in which clock they certify.
Classic Hedge books every round at its worst-case range; a variance-bound tuning books the realized spread through its Bernstein-type certificate, and runs warmer once that spread falls sufficiently below its range ceiling; the intrinsic-time rate books the exact loss term, which neither certificate undercuts (Propositions~\ref{prop:range} and~\ref{prop:surrogate-tightness}).
The recommended rate is therefore slightly more aggressive than its variance-bound counterpart, just as variance-bound rates are more aggressive than the horizon-tuned rate of classic Hedge.
The identity makes that aggressiveness safe: the intrinsic-time loss is incurred at the clock itself, so no slack is reserved for an analysis.
The local update paired with the pressure-target line search is the specialist choice: it calibrates each round to a prescribed target instead of a cumulative budget, chases a moving leader that a cooling schedule tracks only with delay, and is the natural form in boosting-style applications (Section~\ref{sec:pressure}).

\subsection{Side information and a unified perspective}

The decomposition also relates several distinctions that are usually presented as separate tradeoffs---worst-case versus stochastic behavior, first- versus second-order bounds, point experts versus mixtures, fixed horizons versus adaptive schedules, full versus partial feedback. Many individual pieces of the picture have appeared before; this paper places them inside one exact accounting. Table~\ref{tab:tradeoffs} in the related-work section (Section~\ref{sec:discussion-related}) lays these out row by row, conventional manifestation against exact informational form, and is used there to organize the surrounding literature.

The composite-loss viewpoint is especially useful when some part of the next loss vector is predictable.
If $m_t\in\R^K$ is a forecast available before round $t$, taking $u_t = -m_t$ yields the optimistic update. 
The composite loss becomes $c_t=\ell_t-m_t$, so only the forecast residual enters the intrinsic-time cost. 
The original-loss regret then separates into an online term for the unpredictable residual and an explicit predictable mismatch term. 
This is the form used by optimistic Hedge in adaptive game playing, model-based repeated games, and other settings where one predicts the opponent or the environment before acting \cite{HazanKale2010,ChiangEtAl2012}. 
Good side information does not change the exact theorem; it changes the sequence to which the theorem is applied.

The main decomposition is distribution-free. 
A \emph{stochastic-luckiness} result adds a second, optional interpretation. 
If the realized sequence happens to satisfy a self-bounding relation---roughly, if the same intrinsic-time term that drives the adversarial bound is controlled by the comparator's own excess loss---then the typical $O(\sqrt{V_T})$ scaling of regret collapses to a constant or other fast rate. 
This covers adversarial robustness and stochastic easiness with a single statement about the observed sequence.
The algorithm remains the same, but the realized sequence turns out to be easy for the chosen decoder. 
Section~\ref{sec:luckiness} develops this comparator-centered fast-rate interpretation in PAC-Bayesian form. 

A useful historical reference point is the scaling-time question in \cite{Freund2016}, which asked for regret guarantees controlled by internal variance-like quantities in place of the horizon and the number of experts.
We derive the exact identity first and only then upper bound it. 
That makes the result simultaneous over distributional comparators, directly compatible with side information, and naturally interpretable as an individual-sequence complexity statement.

Adaptive online optimization methods choose step sizes from the realized geometry of the loss sequence \cite{Zinkevich2003,DuchiHazanSinger2011}.
The same philosophy appears here, with entropic geometry in place of Euclidean geometry. 
Two contrasting learning-rate families are studied in Section~\ref{sec:schedules}: we choose the temperature from the intrinsic-time process, and the other uses the current score geometry to hit a prescribed pressure target.
In that sense the accounting gives an entropic analogue of adaptive step-size selection: the update strength is chosen from the difficulty already revealed by the sequence.

Viewed this way, optimistic updates, control-variate corrections, compensators, and side priors are variants of one update: they change which part of the round's information is treated as predictable and which part remains to be learned online. In the composite-loss notation, all of these variants feed the same Bayes update a different residual sequence. When the look-ahead signal is accurate, that residual is genuinely easier, the intrinsic-time clock slows down, and the exact regret identity records the improvement directly. 
Lower regret in these variants is therefore a statement about reduced unexplained information.

\subsection{Contributions}

\paragraph{A generic Bayes-rule reduction.}
Section~\ref{sec:exact} converts multiplicative evidence, side priors, and predictable corrections into additive composite losses.
The reduction also identifies the residual information that still has to be learned online.
After the reduction, the rest of the analysis depends only on the induced sequence $c_t$ and the prior complexity term.

\paragraph{One shared one-step balance, two parallel three-piece cumulative decompositions.}
Section~\ref{sec:adaptive} proves the exact prefix and terminal identities for the prior-retempered update $p_t(i)\propto \pi(i)e^{-\eta_t C_{t-1}(i)}$, splitting regret into intrinsic-time loss, temperature-change drift, and terminal comparator-information terms.
Section~\ref{sec:pressure} proves a structurally parallel three-piece identity for the local update $p_{t+1}(i)\propto p_t(i)e^{-\eta_t c_t(i)}$ under pressure-target schedules: the same intrinsic-time loss, a transport drift, and a terminal-mass partition. Both updates use the same composite-loss reduction and the same per-round information functional (a centered cumulant of the current played distribution); only the cumulative bookkeeping differs, and at fixed temperature, where the two recursions play the same $p_t$, they collapse to one identity.

\paragraph{Schedule design from the same information clock.}
Section~\ref{sec:schedules} studies two natural controller families.
The prior-retempered update supports an intrinsic-time square-root schedule, driven either by the exact clock $V_t(c)$ or by the quadratic relaxation $W_t(c)$, yielding pathwise sampled-expert identities, high-probability and anytime consequences, upper law-of-the-iterated-logarithm behavior, and optimistic corollaries through predictable side information.
The local update supports a pressure-target schedule $\sum_i p_t(i)e^{-\eta_t(c_t(i)-a_t)}=1$, yielding exact weighted and class-conditioned identities that are especially natural in boosting-style applications.
The same accounting fixes the recommended default of Section~\ref{sec:intro-which}: the prior-retempered update at the inverse-square-root intrinsic-time rate, whose exact clock runs below every certificate of it and therefore supports a slightly warmer temperature than variance-bound tuning.

\paragraph{Comparator-rich and side-information extensions.}
Later sections return to the original losses and derive shifting-comparator and simultaneous-quantile theorems.
They take the comparator-centered low-noise condition from the fixed rate to the adaptive schedule, so that fast PAC-Bayes rates survive intrinsic-time tuning.
The calculus then transfers to repeated games, continuous-action online convex optimization (OCO), and boosting.

\paragraph{An empirical diagnostic viewpoint, with numerical evidence.}
Because the decomposition is exact, its shares can be tracked round by round, turning the identity into a diagnostic that reads a sequence's regime---stochastic, adversarial, or mixed---directly off how regret is spent. Section~\ref{sec:experiments} develops this diagnostic, confirms the predicted signatures on canonical sequence families and on real data streams, checks that the second-order envelope of Theorem~\ref{thm:scaling-time} binds as predicted, and benchmarks the paper's two schedules against adaptive online-learning baselines on synthetic and real streams, where they are competitive.

Taken together, these results describe how sequential information is spent and gained by Bayes/exponential weights. 
The composite-loss reduction determines what evidence is encoded. 
The two update families determine how that evidence is composed over time. 
The schedule determines how aggressively that evidence is trusted. 
And the theorem records exactly what is incurred immediately, what is shifted by changing the temperature, and what is ultimately attributed to comparator complexity.

\paragraph{Overlap with the game formulation.}
The one-step balance, the three-piece chain it composes into, and the two-sided envelope of Theorem~\ref{thm:scaling-time} are shared with \cite{balsubramani2026concentration}, which develops them as the value identity of a repeated zero-sum game between learner and nature.
The online-learning half of that correspondence is new here.
It comprises the composite-loss reduction that fixes what evidence the update encodes, the two update geometries and their parallel cumulative identities, and the empirical program that reads the decomposition as a diagnostic on realized paths.
It also comprises the schedule the shared clock determines, together with its sampled, high-probability, anytime and iterated-logarithm consequences.
Results already established in \cite{balsubramani2026concentration} are cited at their statements; this paper does not restate their proofs.

\paragraph{Guide to the paper.}
Appendix~\ref{sec:further} collects further extensions to weighted-entropy geometry, repeated games, continuum priors, continuous-action OCO, and boosting.

Throughout, the finite index set $[K]:=\{1,\dots,K\}$ may represent experts, actions, models, or hypotheses depending on the application.
For $q\in\Delta([K])$ and $x\in\R^K$, we write $\ip{q}{x}=\sum_{i=1}^K q(i)x(i)$.
The relative entropy is $\KL(q\|p):=\sum_{i=1}^K q(i)\log\frac{q(i)}{p(i)}$, with the usual convention that $\KL(q\|p)=+\infty$ unless $q$ is absolutely continuous with respect to $p$.
We use the standard total variation distance $\TV(q,p):=\frac12\sum_{i=1}^K |q(i)-p(i)|$.
If $\mu$ is a strictly positive measure on $[K]$, we write $\Hc(q,\mu):=-\sum_{i=1}^K q(i)\log \mu(i)$ for the generalized cross-entropy; when $\mu$ is a probability distribution, this is the usual cross-entropy.
We use PAC-Bayes terminology when convenient, but throughout the main online-learning statements the symbol $\rho$ denotes an arbitrary comparison distribution.
The main algorithms use the notation $p_t$ for the learner's weights; generic one-step identities use $q_t$; and temperature-indexed posteriors are written as $q_{t,\eta}$.

\section{Core identities}
\label{sec:exact-block}

Once side information has been absorbed into a composite loss, each round's excess loss equals an immediate centered cumulant cost plus a change in relative entropy to the comparator. Adaptive temperatures leave the one-step identity unchanged but alter the way those relative-entropy transport terms compose over time. For the prior-retempered update, they organize into a cumulative chain with an explicit drift term. For the local recursion, the natural cumulative object is different, and Section~\ref{sec:pressure} treats it separately together with the pressure-target schedule.

\subsection{The composite-loss reduction}
\label{sec:exact}

The mixed-coincidence identity of~\cite{balsubramani2026information} is the gateway from side information to an ordinary additive loss sequence; a self-contained proof of the finite form is included.
In static form it rewrites a geometric pool of measures as an entropy-corrected variational problem.
In sequential form it shows that multiplicative side factors, predictable corrections, and ordinary losses can all be folded into a single composite loss sequence.
Once that reduction is made, both adaptive updates studied later operate on the same objects.

\begin{theorem}[Finite mixed coincidence identity; {\cite{balsubramani2026information}}]\label{thm:finite-mixed}
Let $\mu_1,\dots,\mu_W$ be strictly positive measures on a finite set $\mathcal X$, and let $\alpha=(\alpha_1,\dots,\alpha_W)\in\R^W$.
Define
\[
Z(\alpha):=\sum_{x\in\mathcal X}\prod_{w=1}^W \mu_w(x)^{\alpha_w},
\qquad
p^\star_\alpha(x):=\frac{\prod_{w=1}^W \mu_w(x)^{\alpha_w}}{Z(\alpha)}
\]
Then for every $q\in\Delta(\mathcal X)$,
\begin{equation}\label{eq:finite-mixed}
-\log Z(\alpha)+\KL(q\|p^\star_\alpha)
=
\sum_{w=1}^W \alpha_w \Hc(q,\mu_w)-\Hc(q)
\end{equation}
where $\Hc(q):=-\sum_x q(x)\log q(x)$ is Shannon entropy.
If, in addition, each $\mu_w$ is a probability distribution and $\alpha\in\Delta([W])$, then
\[
-\log Z(\alpha)=\min_{q\in\Delta(\mathcal X)} \sum_{w=1}^W \alpha_w \KL(q\|\mu_w)
\]
and the unique minimizer is $q=p^\star_\alpha$.
\end{theorem}

To pass to the online setting, interpret one factor as the current posterior, one as the round-loss update, and the rest as side information.
The resulting step is exactly Bayes/exponential-weights updating with a predictable compensator.
This one-round identity yields the shared one-step information balance; Sections~\ref{sec:adaptive} and~\ref{sec:pressure} then compose it into the two distinct cumulative identities---retempered and local---that give the paper's two pathwise decompositions.

\begin{theorem}[Sequential mixed coincidence identity]\label{thm:seq-mixed}
Fix a sequence of learner distributions $(q_t)_{t\ge 1}$ on $[K]$, losses $\ell_t\in\R^K$, predictable learning rates $\eta_t>0$, and positive side factors $s_t\in(0,\infty)^K$.
Consider the update
\begin{equation}\label{eq:one-step-update}
q_{t+1}(i)
= \frac{q_t(i)e^{-\eta_t\ell_t(i)}s_t(i)}{Z_t},
\qquad
Z_t := \sum_{j=1}^K q_t(j)e^{-\eta_t\ell_t(j)}s_t(j)
\end{equation}
Define the side offset and the composite loss by $u_t(i):=-\eta_t^{-1}\log s_t(i)$ and $c_t(i):=\ell_t(i)+u_t(i)$.
Then for every posterior $\rho\in\Delta([K])$,
\begin{equation}\label{eq:seq-id}
-\frac1{\eta_t}\log Z_t+\frac1{\eta_t}\KL(\rho\|q_{t+1})
=
\frac1{\eta_t}\KL(\rho\|q_t)+\ip{\rho}{c_t}
\end{equation}
If the side factor decomposes as
\begin{equation}\label{eq:side-decomp}
s_t(i)=\prod_{w=1}^{J_t} \mu_{t,w}(i)^{\alpha_{t,w}}
\end{equation}
for strictly positive measures $\mu_{t,w}$ and coefficients $\alpha_{t,w}\in\R$ (with $w=1,\dots,J_t$), then
\begin{equation}\label{eq:seq-id-crossent}
-\frac1{\eta_t}\log Z_t+\frac1{\eta_t}\KL(\rho\|q_{t+1})
=
\frac1{\eta_t}\KL(\rho\|q_t)+\ip{\rho}{\ell_t}
+ \sum_{w=1}^{J_t} \frac{\alpha_{t,w}}{\eta_t}\Hc(\rho,\mu_{t,w})
\end{equation}
\end{theorem}

The next corollary rewrites the same identity in the mean-centered form that is most convenient for composing into cumulative regret; it is the form to which both cumulative decompositions are applied. The original form \eqref{eq:seq-id} records transport of comparator information; the centered form \eqref{eq:centered-seq} records regret.

\begin{corollary}[Centered one-step mixed coincidence identity]\label{cor:centered-seq}
In the setting of Theorem~\ref{thm:seq-mixed}, let $m_t:=-\eta_t^{-1}\log Z_t$ and $\delta_t(c):=\ip{q_t}{c_t}-m_t$.
Then for every posterior $\rho\in\Delta([K])$,
\begin{equation}\label{eq:centered-seq}
\ip{q_t}{c_t}-\ip{\rho}{c_t}
=
\delta_t(c)+\frac{\KL(\rho\|q_t)-\KL(\rho\|q_{t+1})}{\eta_t}
\end{equation}
Equivalently,
$
\ip{q_t}{\ell_t}-\ip{\rho}{\ell_t}
=
\delta_t(c)+\frac{\KL(\rho\|q_t)-\KL(\rho\|q_{t+1})}{\eta_t}
+\ip{\rho}{u_t}-\ip{q_t}{u_t}
$. 
If $\eta_t\equiv\eta$, then the composite-loss regret $R_T^c(\rho)$ (the analogue of $R_T^\ell(\rho)$ with $\ell_t$ replaced by $c_t$) satisfies
\begin{align*}
R_T^c(\rho)
=
\sum_{t=1}^T \delta_t(c)+\eta^{-1}\bigl(\KL(\rho\|q_1)-\KL(\rho\|q_{T+1})\bigr)
\end{align*}
\end{corollary}

Equation~\eqref{eq:centered-seq} is the one-step information balance that drives the whole paper.
It is the shared mechanism behind both adaptive updates.
Following \citep{derooij2014adahedge}, we refer to the round's immediate cost $\delta_t(c)$ as the centered mixability gap.
The relative-entropy difference is the amount of comparator information transported from $q_t$ to $q_{t+1}$.
When the learning rate is fixed, that transport telescopes exactly and one recovers the familiar fixed-rate identity below.
When the rate varies, the coefficients change and the two adaptive updates diverge: the prior-retempered update recomputes from the prior at each new temperature and tracks the resulting drift explicitly (Section~\ref{sec:adaptive}), while the local recursion compounds step-by-step and produces a cumulative normalization identity (Section~\ref{sec:pressure}).

The idea of tuning the learning rate according to the mixability gap has an important history.
AdaHedge \citep{derooij2014adahedge} and SafeBayes \citep{the2012safebayesian}, for example, both use the gap to calibrate how aggressively the learner should update.
Our approach is close in spirit but different in aim.
Where those methods design one specific schedule that controls an upper bound on the gap, we keep the decomposition exact for any predictable schedule and study how information flows under that accounting.

\begin{corollary}[Fixed-rate exact comparator identity]
\label{cor:fixed-rate}
Assume $\eta_t\equiv\eta>0$ in \eqref{eq:one-step-update}, and define the one-step mix loss $m_t:=-\eta^{-1}\log Z_t$.
Then for every posterior $\rho\in\Delta([K])$,
\begin{equation}
\label{eq:fixed-rate-exact}
\sum_{t=1}^T m_t+\eta^{-1}\KL(\rho\|q_{T+1})
=
\ip{\rho}{C_T}+\eta^{-1}\KL(\rho\|q_1)
\end{equation}
\end{corollary}

When the side factor has the decomposed form \eqref{eq:side-decomp}, the mismatch term admits a closed-form expression in terms of cross-entropies.

\begin{corollary}[Cross-entropy form of the side-factor mismatch]\label{cor:fixed-rate-mismatch}
Under the same assumptions as Corollary~\ref{cor:fixed-rate}, if \eqref{eq:side-decomp} holds, then the mismatch term is exactly
\begin{equation}\label{eq:fixed-rate-mismatch}
\sum_{t=1}^T \left(\ip{\rho}{u_t}-\ip{q_t}{u_t}\right)
=
\sum_{t=1}^T \sum_{w=1}^{J_t} \frac{\alpha_{t,w}}{\eta}
\left(\Hc(\rho,\mu_{t,w})-\Hc(q_t,\mu_{t,w})\right)
\end{equation}
\end{corollary}

\subsection{The one-step identity and the retempered chain}
\label{sec:adaptive}

For the prior-retempered update, the cumulative question is answered by recomputing the posterior from the original prior at the current temperature and keeping track of the resulting temperature-change drift.

We feed the composite losses into the standard variable-rate exponential-weights update
\begin{equation}\label{eq:retempered-hedge}
p_t(i)
= \frac{\pi(i) e^{-\eta_t C_{t-1}(i)}}{\sum_{j=1}^K \pi(j) e^{-\eta_t C_{t-1}(j)}}
\qquad
C_t(i) := \sum_{s=1}^t c_s(i)\end{equation}
where $\pi\in\Delta([K])$ is a fixed prior and the schedule $(\eta_t)$ is predictable. 
The current distribution is always recomputed from the prior at the current temperature, so the cumulative mix-loss analysis remains valid. 
This is the prior-retempered update of the framework; Section~\ref{sec:pressure} develops the second update---the local recursion $p_{t+1}(i)\propto p_t(i)e^{-\eta_t c_t(i)}$---which agrees with \eqref{eq:retempered-hedge} only at fixed temperature but produces its own parallel decomposition.
The identities below hold for \emph{every} predictable positive schedule; the second-order rule is one useful equalizer among the schedules they cover.

Define
\begin{align*}
m_t &:= -\eta_t^{-1}\log\!\left(\sum_{i=1}^K p_t(i)e^{-\eta_t c_t(i)}\right),\\
\delta_t(c) &:= \ip{p_t}{c_t}-m_t,\\
\psi_t(\lambda) &:= \log \E_{i\sim p_t} \left[ \exp\!\left(\lambda(c_t(i)-\ip{p_t}{c_t})\right) \right]
\end{align*}
Here $\psi_t$ is the cumulant generating function of the centered one-round loss $c_t(I_t)-\ip{p_t}{c_t}$, $I_t\sim p_t$.
We also use the scaled cumulant
\[
\phi_t(\eta) := \frac{\psi_t(-\eta)}{\eta}
= \eta^{-1} \log \E_{i\sim p_t} \left[ \exp\!\left(-\eta(c_t(i) - \ip{p_t}{c_t})\right) \right]
\]
At the played temperature, $\phi_t(\eta_t)=\delta_t(c)$, so $\phi_t(\eta_t)$ is the one-round mixability gap and $Q_t(c)=\eta_t^{-1}\phi_t(\eta_t)=\eta_t^{-2}\psi_t(-\eta_t)$ is its finite-rate intrinsic-time density.

Equation~\eqref{eq:retempered-hedge} is the formal version of the generic Bayes-rule template previewed in the introduction.
If $g_s(i)>0$ is any sequence of evidence factors and $u_s(i)$ any predictable side offset, then the updated posterior
\[
p_t(i)\propto \pi(i)\left(\prod_{s=1}^{t-1} g_s(i)\right)^{\eta_t}
\exp\!\left(-\eta_t\sum_{s=1}^{t-1} u_s(i)\right)
\]
is exactly of the form \eqref{eq:retempered-hedge} after setting $c_s(i):=-\log g_s(i)+u_s(i)$.
Thus the analysis below applies to any online algorithm that accumulates expertwise evidence by Bayes' rule and then chooses the strength of the update through a predictable learning rate.
Ordinary Hedge is the canonical special case $g_s(i)=e^{-\ell_s(i)}$, but the algebra itself is agnostic about whether the evidence came from losses, likelihood ratios, side information, or some other Bayesian score.
Once an easy-instance notion suggests a particular intrinsic-time functional,
the same calculus can be used both to analyze that Bayes-rule update and to design bespoke learning-rate schedules for it.

The PAC-Bayes variational reduction is most transparent in exact form. The usual change-of-measure inequality is recovered only after dropping a nonnegative relative-entropy remainder.

\begin{lemma}[Gibbs variational identity]\label{lem:dv}
Let $\pi\in\Delta([K])$ and $X\in\R^K$.
Define $q_X^\star(i) := \frac{\pi(i)e^{X(i)}}{\sum_{j=1}^K \pi(j)e^{X(j)}}$.
Then for every posterior $\rho\in\Delta([K])$ with $\rho\ll \pi$,
\begin{equation}\label{eq:dv}
\log\!\left(\sum_{i=1}^K \pi(i)e^{X(i)}\right)
=
\ip{\rho}{X}-\KL(\rho\|\pi)+\KL(\rho\|q_X^\star)\end{equation}
Equivalently, for every $\eta>0$ and every $x\in\R^K$, if $q_{\eta,x}^\star(i) := \frac{\pi(i)e^{-\eta x(i)}}{\sum_{j=1}^K \pi(j)e^{-\eta x(j)}}$,
then
\begin{equation}\label{eq:dv-eta}
-\eta^{-1}\log\!\left(\sum_{i=1}^K \pi(i)e^{-\eta x(i)}\right)
=
\ip{\rho}{x}+\eta^{-1}\KL(\rho\|\pi)-\eta^{-1}\KL(\rho\|q_{\eta,x}^\star)\end{equation}
In particular, the familiar PAC-Bayes inequality follows by dropping the last term, and equality is attained at $\rho=q_{\eta,x}^\star$.
\end{lemma}

The terminal potential is quantified by specializing the same variational identity to the cumulative loss.
It is recorded separately because it closes the adaptive proof and records how much comparator information remains in the final posterior.

\begin{lemma}[Terminal potential]
\label{lem:terminal-potential}
For $t\ge 0$ and $\eta>0$, define $A_t (\eta) := -\eta^{-1} \log\!\left( \sum_{i=1}^K \pi(i) e^{-\eta C_t(i)} \right)$ and $q_{t,\eta}(i) := \pi(i)e^{-\eta C_t(i)} \big/ \sum_{j=1}^K \pi(j)e^{-\eta C_t(j)}$.
Then for every $q\in\Delta([K])$,
\begin{equation}
\label{eq:At-exact}
A_t(\eta) + \eta^{-1} \KL(q \| q_{t,\eta})
= \ip{q}{C_t} + \eta^{-1} \KL(q \| \pi)
\end{equation}
In particular,
\begin{equation}
\label{eq:At-var}
A_t(\eta) = \min_{q \in \Delta([K])}\left\{\ip{q}{C_t} + \eta^{-1} \KL (q \| \pi) \right\}
\end{equation}
and $A_t(\eta)$ is nonincreasing in $\eta$.
\end{lemma}

\begin{theorem}[Exact cumulant decomposition for variable-temperature Bayes updating]
\label{thm:cumulant-chain}
Let $(\eta_t)_{t=1}^T$ be any predictable positive schedule, and let $(p_t)$ be generated by \eqref{eq:retempered-hedge}.
Then for every posterior $\rho\in\Delta([K])$ with $\rho\ll\pi$,
\begin{equation}\label{eq:exact-chain}
R_T^c(\rho)
=
\sum_{t=1}^T \phi_t(\eta_t)
+
\sum_{t=1}^{T-1}\left(A_t(\eta_t)-A_t(\eta_{t+1})\right)
+ \frac{\KL(\rho \| \pi) - \KL(\rho \| q_{T,\eta_T}) }{\eta_T}
\end{equation}
where $R_T^c(\rho) := \sum_{t=1}^T \ip{p_t}{c_t} - \ip{\rho}{C_T}$.
If, moreover, $(\eta_t)$ is nonincreasing, then the temperature-change drift satisfies
\begin{align*}
\label{eq:temp-change-drift}
\sum_{t=1}^{T-1} \left( A_t (\eta_t)-A_t (\eta_{t+1}) \right) \leq 0
\end{align*}
\end{theorem}

For fixed temperature, the adaptive chain collapses to a single terminal relative-entropy remainder.

\begin{corollary}[Fixed-rate exact PAC-Bayes identity]
\label{cor:fixed-cumulant}
If $\eta_t \equiv \eta > 0$, then for every posterior $\rho \in \Delta([K])$ with $\rho\ll\pi$,
\begin{equation}
\label{eq:fixed-cumulant}
R_T^c(\rho) =
\sum_{t=1}^T \frac{\psi_t(-\eta)}{\eta}
+ \frac{\KL(\rho\|\pi) - \KL(\rho \| p_{T+1})}{\eta}
\end{equation}
where $p_{T+1}(i) := \pi(i)e^{-\eta C_T(i)} \big/\sum_{j=1}^K \pi(j)e^{-\eta C_T(j)}$.
\end{corollary}

The chain becomes especially transparent once the cumulant itself is treated as the intrinsic quadratic increment.

\begin{theorem}[Exact PAC-Bayes second-order identity]
\label{thm:pacbayes-second-order}
Let $(\eta_t)_{t=1}^T$ be any predictable positive schedule, and let $(p_t)$ be generated by \eqref{eq:retempered-hedge}.
Define $Q_t(c) := \phi_t(\eta_t)/\eta_t$ and $V_T(c) := \sum_{t=1}^T Q_t(c)$,
and set $D_T := \sum_{t=1}^{T-1}\left(A_t(\eta_t)-A_t(\eta_{t+1})\right)$ and $B_T (\rho) := \left(\KL(\rho\|\pi)-\KL (\rho \| q_{T,\eta_T}) \right) / \eta_T$.
Then $Q_t(c)\ge 0$ for every $t$, and for every posterior $\rho\in\Delta([K])$ with $\rho\ll\pi$ (automatic when $\pi$ is strictly positive),
\begin{equation}\label{eq:main-pacbayes}
R_T^c(\rho)=D_T+B_T(\rho)+\sum_{t=1}^T \eta_t Q_t(c)
\end{equation}
If $\eta_t \equiv \eta$, then $D_T=0$ and
\begin{equation}
\label{eq:main-pacbayes-reg}
R_T^c(\rho) = B_T(\rho) + \eta V_T(c)
\end{equation}
\end{theorem}

The hypothesis $\rho\ll\pi$ is the right one and not merely a convenient one.
Both relative-entropy terms in $B_T(\rho)$ can diverge separately, so what the identity needs is that their difference be defined; the condition for that is $\rho\ll q_{T,\eta_T}$, and under the softmax update the terminal posterior and the prior share support, so the two conditions coincide.
A comparator outside it makes $B_T(\rho)$ infinite against a finite regret, which is exactly the additive ambiguity the hypothesis excludes.

\paragraph{Three structural pieces of the chain.}
Equation \eqref{eq:main-pacbayes} separates adaptive regret into three orthogonal objects for the retempered update; Section~\ref{sec:pressure} will show that the local update admits a parallel three-piece structure with different cumulative bookkeeping. 
First, the intrinsic time $V_T(c) = \sum_{t \le T} Q_t(c)$ is computed under the learner's own current distribution $p_t$, so it is the algorithm-defined variance scale of the realized trajectory. This per-round cost $Q_t(c)$ is computed by the same functional in both updates---it is the common per-round information functional in the one-step information balance, evaluated at whatever $p_t$ each recursion actually plays.
Second, the terminal posterior-mismatch term $B_T(\rho)$ is the unique place where structured comparator classes enter; if $\pi$ is read as a decoder, this term is also the remaining description cost of the comparator after the realized sequence has been processed. 
Quantile and other uniform quantifiers are obtained precisely by choosing $\rho$ to spread mass over a restricted set and thereby lowering the posterior complexity.
Third, the temperature-change drift $D_T$ is the exact bookkeeping term for arbitrary learning-rate changes.
Nothing is approximated there: for a nonmonotone schedule the drift is part of the identity, whereas for a nonincreasing schedule it becomes nonpositive and may be discarded.

\paragraph{An optional game-theoretic reading.}
The same three objects also admit a complementary game-theoretic interpretation, which the rest of the paper does not require. A Bayes/exponential-weights round can be viewed as a one-step game in which the learner commits to a distribution $p_t$ and a temperature $\eta_t$. In that game, nature reveals a centered score vector with a constrained one-scale cumulant generating function, and the terminal payoff is the support function of a relative-entropy ball around the prior. Under that reading, the per-round cumulant increment is the intrinsic-time increment $Q_t(c)$, the log-partition value function is the terminal potential $A_t(\eta)$, and the retempering drift $D_T$ is the cost of switching between fixed-scale games as the temperature changes. The present development can be read purely as the pathwise information-accounting calculus; the game formulation is an optional dual viewpoint, developed in \cite{balsubramani2026concentration}. There the one-round statement is made precise: by Donsker--Varadhan duality, the centered log-partition value $\eta_t^2 Q_t(c)$ is the maximum of a relative-entropy-regularized linear payoff over the simplex, attained uniquely at the one-step Gibbs update. The minimax theorem upgrades the maximum to a saddle point at which the Gibbs weights are the unique Bellman equalizer of a zero-sum game between learner and nature. Chaining the per-round games reproduces the three-part decomposition as the repeated game's value identity, so the intrinsic-time increment $Q_t(c)$ is the per-round loss term in both the primal information-accounting and the dual game framings.

\paragraph{Reading the scaled cumulant and the intrinsic time.}
Write $X_t := c_t(I_t) - \ip{p_t}{c_t}$ for $I_t\sim p_t$, and for a hypothetical temperature $\eta>0$ write $m_t(\eta) := -\eta^{-1}\log \sum_i p_t(i)e^{-\eta c_t(i)}$.
Then $\psi_t(\lambda)=\log \E[e^{\lambda X_t}]$ is the ordinary log-mgf of the centered one-round loss,
while $\phi_t(\eta)=\eta^{-1}\log \E[e^{-\eta X_t}]=\ip{p_t}{c_t}-m_t(\eta)$ is the mean-to-mix-loss gap that would be paid at temperature $\eta$ if the current distribution were frozen. 
Thus $Q_t(c)$ is the finite-rate curvature cost on round $t$.

The connection with the familiar quadratic variation appears only as a small-temperature interpretation.
Since $\psi_t(0)=\psi_t'(0)=0$ and $\psi_t''(0)=\Var_{i\sim p_t}(c_t(i))$, we have $\phi_t(\eta)=\frac{\eta}{2}\Var_{i\sim p_t}(c_t(i))+o(\eta)$ and therefore $Q_t(c)=\frac{1}{2} \Var_{i\sim p_t}(c_t(i))+o(1)$ as $\eta\downarrow0$.
The analysis below keeps $Q_t(c)=\eta_t^{-2}\psi_t(-\eta_t)$ itself throughout, so the variance picture is only a limit intuition and not a step in the proof.

\begin{proposition}[Tilted-variance representation of the intrinsic time \cite{balsubramani2026concentration}]
\label{prop:tilted-var}
Fix a round $t$ and a temperature $\eta>0$.
Let $\mu_t := \ip{p_t}{c_t}$ and, for $s\in[0,1]$, define the exponentially tilted distribution
\[
p_{t,s}^{(\eta)}(i)
 := \frac{p_t(i)\exp\!\left(-s\eta(c_t(i)-\mu_t)\right)}
{\sum_{j=1}^K p_t(j)\exp\!\left(-s\eta(c_t(j)-\mu_t)\right)}\]
Then
\[
\phi_t(\eta)=\eta\int_0^1 (1-s)\Var_{i\sim p_{t,s}^{(\eta)}}(c_t(i))\,ds,
\qquad
Q_t(c)=\int_0^1 (1-s)\Var_{i\sim p_{t,s}^{(\eta_t)}}(c_t(i))\,ds\]
In particular, $Q_t(c)\ge 0$, and $Q_t(c)=0$ if and only if $c_t(i)$ is $p_t$-almost surely constant.
\end{proposition}

\paragraph{Intrinsic time as revealed structure.}
Proposition~\ref{prop:tilted-var} explains why both updates literally track the intrinsic time of the realized path.
For each round, $Q_t(c)$ is an exact average of tilted variances generated by the actual Bayes-rule update; no external variance proxy has to be guessed. 
In that sense, the algorithm can be considered an unsupervised revealer of structure: once the update rule is fixed, the realized loss sequence itself determines the exact intrinsic-time path $V_T(c)=\sum_{t\le T}Q_t(c)$. 
This viewpoint also clarifies why changing the proposal distribution can matter so much in adaptive algorithms: it changes the variance geometry being measured. Improving a loose upper bound would leave that geometry untouched.

\section{Schedules and intrinsic time}\label{sec:schedules}

The cumulant chain becomes concrete pathwise control once the schedule is fixed.
Algorithm~\ref{alg:tempo-family} lays out the full $2\times2$ controller--update design space: one axis is how evidence is composed across rounds, from the cumulative scores or from the current score alone, and the other is how the learning rate is set.
With all four cells in view the analysis is modular. The two off-diagonal cells, in which each update is driven by the controller matched to the other update (Section~\ref{sec:offdiag-combinations}), separate the simplifications due to the update geometry from those due to the rate controller.
The two main pairings are the ones whose cumulative identities collapse most cleanly.
The first is the monotone square-root schedule for the prior-retempered update (Section~\ref{sec:second-order}), which interfaces directly with Theorem~\ref{thm:cumulant-chain} and yields pathwise, high-probability, and anytime second-order envelopes.
The second is the pressure-target line search for the local update (Section~\ref{sec:pressure}), whose reward is finer roundwise calibration in place of monotone drift control, and whose native cumulative identity is an exact terminal-mass partition.
Among the four cells, \textsc{Ret-Sqrt} driven by the exact clock is the recommended default (\S\ref{sec:intro-which}); \textsc{Loc-Press} is the specialist for round-calibrated and nonstationary play.

\begin{algorithm}[t]
\caption{The $2\times2$ adaptive Bayes design space. Rows choose the update geometry; columns choose the learning-rate controller.}
\label{alg:tempo-family}
\small
\begin{tabularx}{\linewidth}{@{}p{0.18\linewidth}XX@{}}
\toprule
& \textbf{Square-root clock} & \textbf{Pressure target} \\
\midrule
\textbf{Prior-retempered} &
\textsc{Ret-Sqrt}: play $p_t=q_{t-1,\eta_t}$ with
$\eta_t=1$ if $U_{t-1}=0$ and
$\eta_t=\min\{1,C\sqrt{\Gamma/U_{t-1}}\}$ otherwise, where $U_{t-1}\in\{V_{t-1},W_{t-1}\}$.
&
\textsc{Ret-Press}: play $p_t=q_{t-1,\eta_t}$; after observing $c_t$, choose the next temperature $\eta_{t+1}$ from the cumulative pressure equation
$A_t(\eta_{t+1})-A_{t-1}(\eta_{t+1})=a_t$, equivalently
$\sum_i q_{t-1,\eta_{t+1}}(i)e^{-\eta_{t+1}(c_t(i)-a_t)}=1$.
\\[2ex]
\textbf{Local} &
\textsc{Loc-Sqrt}: play the current recursive weights $p_t$; choose
$\eta_t=1$ if $U_{t-1}^{\rm loc}=0$ and
$\eta_t=\min\{1,C\sqrt{\Gamma/U_{t-1}^{\rm loc}}\}$ otherwise; update
$p_{t+1}(i)\propto p_t(i)e^{-\eta_t c_t(i)}$.
&
\textsc{Loc-Press}: play the current recursive weights $p_t$; after observing $c_t$, choose $\eta_t$ from
$\sum_i p_t(i)e^{-\eta_t(c_t(i)-a_t)}=1$; update
$p_{t+1}(i)\propto p_t(i)e^{-\eta_t c_t(i)}$.
\\
\bottomrule
\end{tabularx}
\[
q_{t,\eta}(i):=\frac{\pi(i)e^{-\eta C_t(i)}}{\sum_j\pi(j)e^{-\eta C_t(j)}}
\qquad \qquad
A_t(\eta):=-\eta^{-1}\log\sum_i\pi(i)e^{-\eta C_t(i)}
\]
\end{algorithm}

Here $\Gamma>0$ is a comparator-complexity budget and, unless a different constant is stated, the square-root controller uses the optimized value $C=1/\sqrt2$ from Theorem~\ref{thm:scaling-time}.
The exact intrinsic increment $Q_t(c)$ is defined in Section~\ref{sec:adaptive}; $V_t(c)$ is the exact cumulant clock, while $W_t(c)$ denotes a quadratic relaxation used only as an optional controller or concentration scale.
For the pressure cells, the target $a_t$ is a one-round free-energy level.
In the local row the line search is causal because $p_t$ is already fixed before $c_t$ is observed.
In the prior-retempered row the same line search naturally sets the next temperature $\eta_{t+1}$, since the current action distribution itself depends on the temperature.

The same lens organizes the Table~\ref{tab:tradeoffs} rows on ``parameter-free vs.\ adaptive'' and ``horizon'': both schedules replace the horizon $T$ by a quantity revealed by the realized path, but the retempered clock is a cumulative budget, while the pressure target is a one-round calibration.

\subsection{The two off-diagonal controller--update pairings}
\label{sec:offdiag-combinations}

The two axes in Algorithm~\ref{alg:tempo-family} are independent.
The square-root controller asks for a predictable global clock, while the pressure controller asks for a one-round free-energy target.
The prior-retempered update composes evidence by recomputing a Gibbs posterior from the original prior at the current temperature, while the local update composes evidence by multiplying the current weights by the current likelihood factor.
The main pairings, \textsc{Ret-Sqrt} and \textsc{Loc-Press} (analyzed in Sections~\ref{sec:second-order} and~\ref{sec:pressure}), are clean because the controller and the cumulative identity are matched.
The off-diagonal pairings drive the retempered update by the pressure target and the local update by the square-root clock. Both are legitimate algorithms, and both are useful diagnostics, because they show which part of the analysis is due to the controller and which part is due to the update geometry.

\paragraph{Pressure-targeted cumulative retempering.}
For the cumulative, prior-retempered geometry recall $q_{t,\eta}$ and $A_t(\eta)$ from Algorithm~\ref{alg:tempo-family}. After round $t$ has revealed $c_t$, the cumulative analogue of the local pressure curve (defined in Section~\ref{sec:pressure}) is
\[
m_t^{\rm ret}(\eta)
:=-\eta^{-1}\log\sum_i q_{t-1,\eta}(i)e^{-\eta c_t(i)}
=A_t(\eta)-A_{t-1}(\eta).
\]
A pressure-retempered schedule chooses the next temperature by solving
\begin{equation}
\label{eq:ret-press-target}
m_t^{\rm ret}(\eta_{t+1})=a_t,
\qquad\text{equivalently}\qquad
\sum_i q_{t-1,\eta_{t+1}}(i)e^{-\eta_{t+1}(c_t(i)-a_t)}=1,
\end{equation}
and then sets the next posterior to $p_{t+1}=q_{t,\eta_{t+1}}$. This is the precise cumulative-update version of the pressure-target rule. The one-step delay is important: using \eqref{eq:ret-press-target} to choose the already-played $p_t$ would require knowing $c_t$ before acting, while choosing $\eta_{t+1}$ after observing $c_t$ is causal.

Unlike the local pressure curve, $m_t^{\rm ret}(\eta)$ is not a fixed-distribution mix-loss curve, because the distribution $q_{t-1,\eta}$ also changes with $\eta$, so existence and uniqueness of the root are not automatic. They are settled exactly by the monotonicity of $m_t^{\rm ret}$, which an information identity makes explicit.

\begin{proposition}[Well-posedness of the pressure-retempered temperature]
\label{prop:ret-press-wellposed}
Let $\Delta_t(\eta):=\KL(q_{t,\eta}\|\pi)-\KL(q_{t-1,\eta}\|\pi)$ measure how much further absorbing $c_t$ moves the Gibbs posterior from the prior. The free energy satisfies $A_t'(\eta)=-\eta^{-2}\KL(q_{t,\eta}\|\pi)\le0$, so the retempered curve obeys the identity
\begin{equation}\label{eq:ret-press-deriv}
\frac{d}{d\eta}\,m_t^{\rm ret}(\eta)=-\frac{\Delta_t(\eta)}{\eta^{2}}
\end{equation}
with endpoints $m_t^{\rm ret}(0^+)=\ip{\pi}{c_t}$ and $m_t^{\rm ret}(\infty)=\min_i C_t(i)-\min_j C_{t-1}(j)$. Consequently $m_t^{\rm ret}$ is monotone exactly when $\Delta_t$ keeps one sign. The equation $m_t^{\rm ret}(\eta)=a_t$ has a positive root whenever $a_t$ lies strictly between the least and greatest values of $m_t^{\rm ret}$ on $(0,\infty)$; this range always contains the two endpoint values and coincides with their interval precisely when $m_t^{\rm ret}$ is monotone. The root is then unique. A sign change of $\Delta_t$ instead makes $m_t^{\rm ret}$ non-monotone, so its range overshoots one endpoint and several positive roots may occur.
\end{proposition}

The schedule therefore fixes a deterministic convention, taking the smallest positive root of $m_t^{\rm ret}(\eta)=a_t$ in the admissible range. Once the temperature path is selected, no new regret identity is needed: Theorem~\ref{thm:pacbayes-second-order} applies verbatim, with the resulting retempering drift $D_T$ left in the ledger.

An informationally principled way to choose $a_t$ is to select a dimensionless information target first and let the pressure level be induced. For a candidate temperature, define the current instance's cumulative Bayesian surprise
\begin{equation}
\label{eq:ret-info-target}
I_t^{\rm ret}(\eta)
:=\KL(q_{t,\eta}\|q_{t-1,\eta})
=\eta\left(m_t^{\rm ret}(\eta)-\ip{q_{t,\eta}}{c_t}\right).
\end{equation}
Choose an information quota $\beta_t\ge0$, solve $I_t^{\rm ret}(\eta_{t+1})=\beta_t$, and then set $a_t=m_t^{\rm ret}(\eta_{t+1})=\ip{q_{t,\eta_{t+1}}}{c_t}+\beta_t/\eta_{t+1}$.
This turns the target from a raw loss-scale number into an instance-wise information calculation: $\beta_t$ specifies how much posterior movement the current example is allowed to cause. A practical default is an information-fraction rule $\beta_t=\theta I_t^{\rm ret}(\bar\eta_t)$ with $0<\theta\le1$, where $\bar\eta_t$ is a reference scale such as the current temperature or the unit scale, optionally clipped by a maximum per-round information budget. This is the clean way to formalize a gap-retempered schedule: the observed mixability gap or Bayesian surprise determines the next cumulative temperature, and the free-energy target $a_t$ is then whatever pressure level realizes that information movement.

\paragraph{Square-root clock with the local update.}
The other off-diagonal pairing keeps the local recursion $p_{t+1}(i)\propto p_t(i)e^{-\eta_t c_t(i)}$ but drives it by the square-root clock
\[
\eta_t=1\quad\text{if }U_{t-1}^{\rm loc}=0,
\qquad
\eta_t=\min\{1,C\sqrt{\Gamma/U_{t-1}^{\rm loc}}\}\quad\text{otherwise},
\]
where $U_t^{\rm loc}$ may be the exact local cumulant clock $\sum_{s\le t}Q_s^{\rm loc}(c)$ or its quadratic relaxation. With $\mu_t=\ip{p_t}{c_t}$ and $Q_t^{\rm loc}(c):=\eta_t^{-2}\log\E_{i\sim p_t}e^{-\eta_t(c_t(i)-\mu_t)}$, the ordinary one-step identity gives, for every comparator $\rho$,
\begin{equation}
\label{eq:local-sqrt-abel}
R_T^c(\rho)
=\sum_{t=1}^T\eta_t Q_t^{\rm loc}(c)
+\frac{\KL(\rho\|p_1)}{\eta_1}
-\frac{\KL(\rho\|p_{T+1})}{\eta_T}
+\sum_{t=2}^T\KL(\rho\|p_t)\left(\frac1{\eta_t}-\frac1{\eta_{t-1}}\right).
\end{equation}
The same square-root envelope as Theorem~\ref{thm:scaling-time} controls only the intrinsic-time loss $\sum_t\eta_tQ_t^{\rm loc}(c)$. The extra $\sum_{t=2}^T\KL(\rho\|p_t)\left(\frac1{\eta_t}-\frac1{\eta_{t-1}}\right)$ term is the cost of applying a cooling controller to the local geometry. Since the square-root clock is typically nonincreasing, $1/\eta_t-1/\eta_{t-1}\ge0$, so this term is not a benign negative drift. If we have an independent compression bound $\KL(\rho\|p_t)\le\Lambda$ along the path, the extra variation is at most $\Lambda(1/\eta_T-1/\eta_1)$ and remains on the same order as the square-root term when $\Lambda$ is comparable to $\Gamma$; without such a bound it is an explicit additional cost. Thus \textsc{Loc-Sqrt} is a sensible baseline when we want a predictable local recursion, but it gives up both special simplifications: the nonpositive retempered drift of \textsc{Ret-Sqrt} and the exact terminal-mass normalization of \textsc{Loc-Press}.

\subsection{Second-order schedules and intrinsic-time consequences}
\label{sec:second-order}

We begin with the budgeted square-root rule from Algorithm~\ref{alg:tempo-family}. It is explicit, predictable, and nonincreasing, and it balances terminal comparator complexity against accumulated intrinsic time.

Fix $\Gamma>0$ and a constant $C>0$; the default optimized choice is $C=1/\sqrt2$. The \emph{second-order schedule} with budget $\Gamma$ and tuning $C$ is the predictable positive sequence
\begin{equation}\label{eq:budget-schedule-C}
\eta_t :=
\begin{cases}
1, & \text{if }V_{t-1}(c) = 0 \\[1ex]
\min\!\left\{1,\,C\sqrt{\frac{\Gamma}{V_{t-1}(c)}}\right\}
& \text{otherwise}
\end{cases}
\end{equation}
Because $V_{t-1}(c)$ is $\mathcal{F}_{t-1}$-measurable, each $\eta_t$ is predictable, and because $V_{t-1}(c)$ is nondecreasing in $t$ the schedule is nonincreasing.

The quadratic relaxation the algorithm may be driven by instead is $W_t(c):=\sum_{s=1}^t \tfrac{1}{2}\Var_{i\sim p_s}(c_s(i))$, which keeps only the ordinary second moments of the played losses; the two clocks are compared once the envelope is in hand.

\begin{proposition}[AdaHedge is the loss-clock member of the family]\label{prop:adahedge-instance}
Let $\pi$ be uniform on $[K]$.
Then the prior-retempered play at temperature $\eta_t$ coincides with exponential weights recomputed at $\eta_t$ on the cumulative loss, $p_t(i)\propto\pi(i)e^{-\eta_t C_{t-1}(i)}\propto e^{-\eta_t C_{t-1}(i)}$, and the round's loss term is its mixability gap, $\eta_t Q_t(c)=\ip{p_t}{c_t}-m_t(\eta_t)$.
Consequently the schedule $\eta_t=\Gamma/\sum_{s<t}\eta_s Q_s(c)$ at budget $\Gamma=\log K$ is AdaHedge \cite{derooij2014adahedge}, pathwise on every loss sequence.
\end{proposition}

The two rules therefore share an update and differ only in the clock that sets the temperature: the second-order schedule \eqref{eq:budget-schedule-C} reads the intrinsic time $V_{t-1}(c)=\sum_{s<t}Q_s(c)$, AdaHedge the cumulative intrinsic-time loss $\sum_{s<t}\eta_s Q_s(c)$.
The coincidence is a property of the uniform prior alone: under a non-uniform prior the retempered and recomputed plays separate, so a hybrid genuinely distinct from AdaHedge needs one.

The schedule \eqref{eq:budget-schedule-C} is the default because it balances comparator complexity against revealed intrinsic time in the right square-root scale. 
If one freezes a target intrinsic-time level $V>0$ and balances the proxy $\Gamma/\eta+\eta V$,
the minimizer is $\eta^\star(V)=\sqrt{\Gamma/V}$.
Thus the schedule $\eta_t\propto V_{t-1}(c)^{-1/2}$ is the predictable equalizer of terminal comparator complexity and cumulative centered cost.
The next theorem shows that, up to unavoidable discrete edge terms, this one-step equalizer survives along the whole realized path.

Along every realized path the algorithm tracks the exact intrinsic-time cost $\sum_{t=1}^T \eta_tQ_t(c)$, and that cost itself is trapped between matching square-root laws built from the same intrinsic time $V_T(c)$.

\begin{theorem}[Two-sided square-root envelope for the second-order schedule]
\label{thm:scaling-time}
Run the schedule \eqref{eq:budget-schedule-C}, and define $Q_*^T(c) := \max_{1\le t\le T}Q_t(c)$.
Then
\begin{equation}\label{eq:budgeted-two-sided}
2C\sqrt{\Gamma V_T(c)}-C^2\Gamma
\le
\sum_{t=1}^T \eta_t Q_t(c)
\le
Q_*^T(c)+2C\sqrt{\Gamma V_T(c)}\end{equation}
Consequently, for every posterior $\rho\in\Delta([K])$,
\begin{equation}
\label{eq:budgeted-two-sided-regret}
2C\sqrt{\Gamma V_T(c)}-C^2\Gamma
\le
R_T^c(\rho) - ( D_T + B_T(\rho) )
\le
Q_*^T(c)+2C\sqrt{\Gamma V_T(c)}\end{equation}
If moreover $\KL(\rho\|\pi)\le \Gamma$, then $B_T(\rho)\le \Gamma\eta_T^{-1}$ and
\begin{equation}\label{eq:budgeted-bound-C}
R_T^c(\rho)
\le
\Gamma+Q_*^T(c)+(2C+C^{-1})\sqrt{\Gamma V_T(c)}\end{equation}
In particular, if $C=1/\sqrt{2}$, then
\begin{equation}\label{eq:budgeted-bound-opt}
R_T^c(\rho)
\le
\Gamma+Q_*^T(c)+2\sqrt{2\Gamma V_T(c)}\end{equation}
\end{theorem}

Every later two-sided statement in this paper is \eqref{eq:budgeted-two-sided} displaced by the accounting terms of its own setting, so it is worth naming the bracket once.
At budget $\Gamma$, for a clock value $U\ge0$ and a largest-increment value $u_*\ge0$, write
\begin{equation}\label{eq:env-bracket}
\Env{U}{u_*}:=\Bigl[\,2C\sqrt{\Gamma U}-C^2\Gamma,\ \ u_*+2C\sqrt{\Gamma U}\,\Bigr]\end{equation}
and read $y\in x+\Env{U}{u_*}$ in the usual sense that $x$ displaces both endpoints.
Theorem~\ref{thm:scaling-time} is then $\sum_{t}\eta_tQ_t(c)\in\Env{V_T(c)}{Q_*^T(c)}$ and $R_T^c(\rho)\in D_T+B_T(\rho)+\Env{V_T(c)}{Q_*^T(c)}$.
Only the displacement $x$ and the driving clock change from setting to setting below.

Both slack terms are unavoidable, and each is attained by a one-round path.
They record the two ways a realized sum of increments departs from the smooth square-root law: the discreteness of the path supplies $Q_*^T(c)$, and the clip $\eta_t\le1$ over the opening rounds supplies $C^2\Gamma$.
Appendix~\ref{app:clipped-branch} collects the anatomy of that clipped branch---when it binds, why neither edge term can be dropped, and the quadrature reading in which the envelope is exactly the range of a two-sided defect.

\subsubsection{Exact and relaxed clocks}
\label{sec:relaxed-clocks}
The controller analyzed above uses the exact clock $V_t(c)=\sum_{s\le t}Q_s(c)$; its quadratic relaxation $W_t(c)$ keeps only the ordinary second moments of the played losses.
By the small-temperature expansion of Section~\ref{sec:adaptive}, $W_t(c)$ is the leading-order relaxation of the exact clock $V_t(c)$ (the factor $\tfrac12$ matches $Q_s(c)=\tfrac12\Var_{i\sim p_s}(c_s(i))+o(1)$); the un-halved sum $V^{\!\sqcup}_t=2W_t(c)$ is reserved for the Bernstein concentration clock of the iterated-logarithm statement.
We may therefore drive the square-root controller by either $V_{t-1}(c)$ or $W_{t-1}(c)$ in practice.
The theory below is stated for the exact clock, because that is the quantity that appears in the identity itself.

The choice of driving clock is also the choice of how aggressively the update trusts the revealed path, and it is where the recommendation of Section~\ref{sec:intro-which} lands.
For a common budget $\Gamma$, the certified controls of the increment can only enlarge the clock: $Q_t(c)\le\sigma_t^2\,(e^{\tau_t}-1-\tau_t)/\tau_t^2$ with $\tau_t=\eta_t R_t$ (Proposition~\ref{prop:surrogate-tightness}), and $Q_t(c)\le(b_t-a_t)^2/8$ (Proposition~\ref{prop:range}), so replacing $V_{t-1}(c)$ by a Bernstein variance-bound clock or by the worst-case range clock $\sum_{s<t}(b_s-a_s)^2/8$ only lowers the temperature.
Driving the rate by the exact clock is therefore slightly more aggressive than tuning to the variance bound---the certificate exceeds the exact increment by a factor $1+O(\tau_t)$ in the per-round product---exactly as variance-bound tuning is more aggressive than the horizon-tuned classic Hedge rate, which charges every round its full range.
Theorem~\ref{thm:scaling-time} above is why the extra aggressiveness is sound: it controls the intrinsic-time loss in terms of the same clock the schedule trusts, so nothing is lost between the tuning and the guarantee.
On realized paths the orderings are quantitative.
Against the empirical-Bernstein relaxation, the practical standard, the exact clock spends between $0.66$ and $0.89$ of the budget on the benchmark families (\S\ref{sec:exp-results-hoeffding-slack}); the much larger margin over the range-based Hoeffding clock is mostly the gap between range and variance.
The quadratic clock tracks it to within five percent, at a measured schedule cost below $10^{-2}$ regret when it stands in for the exact clock (\S\ref{sec:eval-intrinsic-time-clock}).

\begin{proposition}[Range-based bound for the intrinsic increment]\label{prop:range}
If $c_t(i) \in [a_t,b_t]$ for all $i$, then the standard bounded-range exponential-moment bound applied to $c_t(i)-\ip{p_t}{c_t}$ gives $Q_t(c)\le (b_t-a_t)^2/8$.
In particular, if $c_t(i)\in[0,1]$, then $Q_t(c)\le 1/8$ on every round.
Thus the intrinsic time $V_T(c)$ remains a genuine quadratic-variation object, now written directly in cumulant form.
\end{proposition}

\begin{proposition}[Second-moment control of the intrinsic increment]\label{prop:secondmoment}
Suppose $c_t$ is bounded below, and write $c_t^{\circ}:=c_t-\min_i c_t(i)$.
Then for every $\eta>0$,
\[
Q_t(c)\le\min\Bigl\{\tfrac12\,\E_{i\sim p_t}\bigl[c_t^{\circ}(i)^2\bigr],\ \eta^{-1}\ip{p_t}{c_t^{\circ}}\Bigr\}
\]
\end{proposition}

Neither branch uses an upper bound on $c_t$: the first branch is finite as soon as the score has a finite second moment under the play, and the second as soon as it has a finite first moment, which is all the increment itself requires of a score bounded below.
On bounded scores the two controls are not ordered: the range bound is the sharper worst-case statement, attained in the small-temperature limit by a two-point score vector that the play splits evenly, while the second-moment branch is far smaller on rounds whose scores sit well inside their range.

\begin{proposition}[Controlled surrogate-tightness of the quadratic clock]
\label{prop:surrogate-tightness}
Fix a round $t$ with played distribution $p_t$, score vector $c_t$, and temperature $\eta_t>0$. Let $\sigma_t^2:=\Var_{i\sim p_t}(c_t(i))$, let $R_t:=\max_i c_t(i)-\min_i c_t(i)$, and set $\tau_t:=\eta_t R_t$. Then
\[
Q_t(c)-\tfrac12\sigma_t^2=\int_0^1(1-s)\bigl(\Var_{i\sim p_{t,s}^{(\eta_t)}}(c_t)-\sigma_t^2\bigr)\,ds,
\]
and the exact clock $V_t(c)$ and its quadratic relaxation $W_t(c)$ obey the two-sided envelope
\[
\frac{\tau_t-1+e^{-\tau_t}}{\tau_t^2}\ \le\ \frac{Q_t(c)}{\sigma_t^2}\ \le\ \frac{e^{\tau_t}-1-\tau_t}{\tau_t^2}.
\]
Both bounds decrease to $\tfrac12$ as $\tau_t\downarrow0$, so $\bigl|Q_t(c)-\tfrac12\sigma_t^2\bigr|\le\sigma_t^2\bigl(\tfrac{e^{\tau_t}-1-\tau_t}{\tau_t^2}-\tfrac12\bigr)=\tfrac{\tau_t}{6}\sigma_t^2+O(\tau_t^2)\sigma_t^2$, and summing over rounds
\[
\bigl|V_T(c)-W_T(c)\bigr|\ \le\ \sum_{t=1}^T\sigma_t^2\Bigl(\frac{e^{\tau_t}-1-\tau_t}{\tau_t^2}-\tfrac12\Bigr).
\]
Thus $W_t(c)$ is the leading-order relaxation of $V_t(c)$ in the per-round products $\tau_t=\eta_t R_t$, with a fully explicit, played-distribution-computable error governed by the tilted-variance spread.
\end{proposition}

\begin{proposition}[Composite-loss transform of any variance process]
\label{prop:composite-variance-transform}
Fix a round $t$ and any variance-measuring distribution $\nu\in\Delta([K])$---the played distribution $p_t$, or a bespoke curvature distribution $q_t$ as in \cite{FreundEtAl2026}. For composite losses $c_t=\ell_t+u_t$ the second-order clock transforms by the exact bilinear identity
\[
\Var_{i\sim\nu}(c_t)=\Var_{i\sim\nu}(\ell_t)+\Var_{i\sim\nu}(u_t)+2\,\Cov_{i\sim\nu}(\ell_t,u_t).
\]
Hence the quadratic relaxation $W_T$ measured under any fixed $\nu$-process is the original-loss clock plus the side-term clock plus twice their cumulative cross-covariance: a predictable $u_t$ that is $\nu$-anticorrelated with $\ell_t$ strictly slows the clock, and $\nu$-almost-sure perfect prediction $u_t=-\ell_t$ drives the increment to zero. The choice of bespoke distribution enters only through which $\nu$ evaluates the covariance; the composite loss enters identically.
\end{proposition}

The display follows from the bilinear expansion of the covariance at $c_t=\ell_t+u_t$.

The exact clock obeys its own composite-loss law, in which the side term acts on the measuring distribution by tilting it.

\begin{proposition}[Composite-loss chain rule for the exact increment]\label{prop:composite-clock-chain}
Fix a round $t$ and a temperature $\eta>0$, let $c_t=\ell_t+u_t$, and let $p_t^{(u)}(i)\propto p_t(i)e^{-\eta u_t(i)}$ be the $u$-tilted play.
Write $Q_t^{[u]}(\ell)$ for the increment of $\ell_t$ measured at $p_t^{(u)}$.
Then
\[
Q_t(\ell+u)=Q_t(u)+Q_t^{[u]}(\ell)+\frac{1}{\eta}\Bigl(\ip{p_t}{\ell_t}-\ip{p_t^{(u)}}{\ell_t}\Bigr)
\]
The roles of $\ell_t$ and $u_t$ may be exchanged, and the same statement holds with $p_t$ replaced throughout by any variance-measuring distribution $\nu$.
\end{proposition}

The optimistic cancellation of Figure~\ref{fig:same-regret-diff-decomp} ($\nu=p_t$, $u_t=-m_t$) is the special case of the distribution-agnostic bilinear law of Proposition~\ref{prop:composite-variance-transform}.
The barrier of Section~\ref{subsec:lower-bounds} is respected because it is stated in whichever variance $V_T$ the reduction produces, which that proposition names exactly.
Two things remain.
First, the bilinearity is special to the quadratic clock: the exact increment $Q_t(c)=\int_0^1(1-s)\Var_{i\sim p_{t,s}^{(\eta_t)}}(c_t)\,ds$ of Proposition~\ref{prop:tilted-var} does not decompose additively under a fixed measure, since the tilting family $p_{t,s}^{(\eta_t)}$ is itself generated by $c_t$.
It obeys instead the chain rule of Proposition~\ref{prop:composite-clock-chain}, exact at every temperature and with no remainder, in which the finite-temperature counterpart of the covariance is a tilt of the measuring distribution.
The deviation of the exact increment from the fixed-measure law is first order in $\eta_t$, with leading coefficient $-\tfrac16\bigl(\kappa_3(\ell_t+u_t)-\kappa_3(\ell_t)-\kappa_3(u_t)\bigr)$ in the third central moments under $p_t$.
Second, a single-copy simultaneous-quantile guarantee for the curvature-distribution algorithm of \cite{FreundEtAl2026} fed the composite losses $c_t$ is not established here.
It would require re-running that potential analysis under composite losses, and is subject to the caution of Section~\ref{subsec:lower-bounds} against insisting on the classical $p$-variance verbatim together with full parameter-freeness.

\subsubsection{Pathwise sampled-expert PAC-Bayes regret}

The chain also gives the promised sampled-expert statement.
If the learner samples an expert $I_t\sim p_t$ on each round, then the sampled regret differs from the deterministic regret only by a martingale term.
The same randomized draw can also be implemented as a follow-the-perturbed-leader (FTPL) step \cite{kalai2005efficient} with i.i.d. Gumbel perturbations, so the pathwise identities below equally apply to that realization of the action sample (e.g. \cite[Theorem~3.9]{arora2012theory}).

\begin{theorem}[Pathwise sampled-expert PAC-Bayes identity]
\label{thm:sampled-pacbayes}
Assume that, conditionally on the past, the learner samples $I_t\sim p_t$ independently on each round. 
Define the sampled composite-loss regret $\widehat R_T^c(\rho) := \sum_{t=1}^T c_t(I_t)-\ip{\rho}{C_T}$ and the martingale increment sum $M_T^{\mathrm{sam}} := \sum_{t=1}^T \left(c_t(I_t)-\ip{p_t}{c_t}\right)$.
Then for every posterior $\rho\in\Delta([K])$,
\begin{equation}
\label{eq:sampled-exact}
\widehat R_T^c(\rho)
= M_T^{\mathrm{sam}} + D_T + B_T(\rho)+\sum_{t=1}^T \eta_t Q_t(c)
\end{equation}
For the original losses,
$\displaystyle \sum_{t=1}^T \ell_t(I_t)-\ip{\rho}{L_T}
= \widehat R_T^c(\rho)+\sum_{t=1}^T\left(\ip{\rho}{u_t}-u_t(I_t)\right) $. 
\end{theorem}

\begin{corollary}[Second-order bounds for the sampled-expert regret]
\label{cor:sampled-second-order}
Under the schedule \eqref{eq:budget-schedule-C},
\begin{equation}
\label{eq:sampled-two-sided}
\widehat R_T^c(\rho)\in M_T^{\mathrm{sam}}+D_T+B_T(\rho)+\Env{V_T(c)}{Q_*^T(c)}\end{equation}
If $\KL(\rho\|\pi)\le \Gamma$, then
\begin{equation}
\label{eq:sampled-second-order}
\widehat R_T^c(\rho)
\leq M_T^{\mathrm{sam}} + \Gamma+Q_*^T(c) + (2C+C^{-1}) \sqrt{\Gamma V_T(c)}
\end{equation}
\end{corollary}

\subsubsection{High-probability randomized prediction}

A Hedge regret theorem \citep{SteinhardtLiang2014} can be converted into a high-probability guarantee for the randomized prediction rule that samples an action from the current weights. In the present exact formalism the same conclusion is cleaner: Theorem~\ref{thm:sampled-pacbayes} isolates the randomness in a single martingale, so the deterministic intrinsic-time analysis survives unchanged and only one concentration step is needed.

\begin{theorem}[High-probability sampled-expert PAC-Bayes bound]
\label{thm:sampled-highprob}
Assume $c_t(i)\in[0,1]$ for all $t$ and $i$, and recall the un-halved concentration clock $V^{\!\sqcup}_T := \sum_{t=1}^T \Var_{i\sim p_t}(c_t(i))=2W_T(c)$.
Then for every $\delta\in(0,1)$, with probability at least $1-\delta$,
\begin{equation}\label{eq:sampled-highprob}
\widehat R_T^c(\rho)
\le
D_T+B_T(\rho)+\sum_{t=1}^T \eta_tQ_t(c)
+\sqrt{2V^{\!\sqcup}_T\log(1/\delta)}+\tfrac13\log(1/\delta)
\end{equation}
simultaneously for all posteriors $\rho\in\Delta([K])$.
Under the second-order schedule \eqref{eq:budget-schedule-C}, if $\KL(\rho\|\pi)\le \Gamma$, then on the same event
\begin{equation}
\label{eq:sampled-highprob-noshift}
\widehat R_T^c(\rho)
\le
\Gamma+Q_*^T(c)+(2C+C^{-1})\sqrt{\Gamma V_T(c)}
+\sqrt{2V^{\!\sqcup}_T\log(1/\delta)}+\tfrac13\log(1/\delta)\end{equation}
Since $V^{\!\sqcup}_T\le T/4$, we also have
\[
\widehat R_T^c(\rho)
\le
\Gamma+Q_*^T(c)+(2C+C^{-1})\sqrt{\Gamma V_T(c)}
+\sqrt{\tfrac{T}{2}\log(1/\delta)}+\tfrac13\log(1/\delta)\]
\end{theorem}

Thus the randomized-prediction extension costs only the additional concentration of the sampling martingale.
The intrinsic-time term itself is the same one that already appears in the deterministic PAC-Bayes theorem, and the bound is simultaneous over all posteriors because the martingale term does not depend on $\rho$.

The same martingale representation also yields an anytime form, with no union bound over times.

\begin{theorem}[Anytime sampled-expert confidence sequence]
\label{thm:sampled-anytime}
For $t\le T$, define the truncated sampled regret $\widehat R_t^c(\rho) := \sum_{s=1}^t c_s(I_s)-\ip{\rho}{C_t}$, the truncated drift $D_t := \sum_{s=1}^{t-1}(A_s(\eta_s)-A_s(\eta_{s+1}))$, the truncated terminal term $B_t(\rho):=(\KL(\rho\|\pi)-\KL(\rho\|q_{t,\eta_t}))/\eta_t$, and the truncated intrinsic-time objects $V_t(c) := \sum_{s=1}^t Q_s(c)$ and $Q_{*,t}(c) := \max_{s\le t}Q_s(c)$.
For every $\lambda>0$, the process $Z_t(\lambda) := \exp\!\left(\lambda M_t^{\mathrm{sam}}-\sum_{s=1}^t \psi_s(\lambda)\right)$ is a nonnegative martingale. 
Therefore, for every $\delta\in(0,1)$, with probability at least $1-\delta$,
\begin{equation}
\label{eq:anytime-confidence}
\widehat R_t^c(\rho)
\leq
D_t + B_t(\rho) + \sum_{s=1}^t \eta_s Q_s(c) + \frac{\log(1/\delta) + \sum_{s=1}^t \psi_s(\lambda)}{\lambda}
\end{equation}
simultaneously for all $t\le T$ and all posteriors $\rho \in \Delta([K])$. Under the second-order schedule,
if $\KL(\rho \| \pi) \leq \Gamma$, then on the same event
\begin{align}
\widehat R_t^c(\rho)
&\le \Gamma+Q_{*,t}(c)+(2C+C^{-1})\sqrt{\Gamma V_t(c)}  + \frac{\log(1/\delta)+\sum_{s=1}^t \psi_s(\lambda)}{\lambda}
\label{eq:anytime-confidence-noshift}
\end{align}
for all $t\le T$.
\end{theorem}

Mixing the martingales $Z_t(\lambda)$ over a density on $\lambda$ yields curved anytime boundaries, which tend to give tighter instance-specific bounds \cite{howard2020b}. Concretely, for an absolutely continuous mixing law with density $f$ on $(0,\infty)$ the mixture $\overline{Z}_t:=\int_0^\infty Z_t(\lambda)\,f(\lambda)\,d\lambda$ is again a nonnegative supermartingale, and Ville's inequality applied to $\overline{Z}_t$ gives a finite-time, all-$t$-simultaneous boundary. The boundary is recovered by a Laplace lower bound on the integrand near its peak, which contributes a $\log f(\lambda^\star)$ density term in place of a point mass.
The iterated-logarithm \emph{rate} for the sampled-expert martingale is the asymptotic counterpart, in the spirit of \cite{Balsubramani2014}: along a geometric grid of rates we obtain the sharp $\sqrt{2V^{\!\sqcup}_t\log\log V^{\!\sqcup}_t}$ growth constant, where $V^{\!\sqcup}_t=2W_t(c)$ is the predictable quadratic variation of the sampling martingale defined in the proposition below.

\begin{proposition}[Iterated-logarithm rate for the sampled-expert martingale]
\label{prop:pathwise-lil}
Assume $c_t(i) \in [0,1]$ for all $t$ and $i$. Let $X_t:=c_t(I_t) - \ip{p_t}{c_t}$ and $M_t^{\mathrm{sam}} := \sum_{s=1}^t X_s$, and recall the predictable quadratic variation $V^{\!\sqcup}_t := \sum_{s=1}^t \Var(X_s\mid \mathcal{F}_{s-1}) = \sum_{s=1}^t \Var_{i\sim p_s}(c_s(i))$. Then, almost surely on the event $\{V^{\!\sqcup}_t\to\infty\}$,
\begin{equation}
\label{eq:lil-limsup}
\limsup_{t\to\infty}\;
\frac{|M_t^{\mathrm{sam}}|}{\sqrt{2\,V^{\!\sqcup}_t\,\log\log V^{\!\sqcup}_t}}
\;\le\; 1 .
\end{equation}
Consequently, for every posterior $\rho\in\Delta([K])$, the sampled regret satisfies the pathwise asymptotic bound
\begin{equation}
\label{eq:finite-lil-regret}
\widehat R_t^c(\rho)
\leq
D_t + B_t(\rho) + \sum_{s=1}^t \eta_s Q_s(c) + \bigl(1+o(1)\bigr)\sqrt{2\,V^{\!\sqcup}_t\,\log\log V^{\!\sqcup}_t}
\end{equation}
almost surely on $\{V^{\!\sqcup}_t\to\infty\}$, simultaneously over all posteriors $\rho\in\Delta([K])$.
\end{proposition}

Following \cite{Balsubramani2014}, the growth constant in \eqref{eq:lil-limsup} is the sharp one: the iterated-logarithm scaling $\sqrt{V^{\!\sqcup}_t \log\log V^{\!\sqcup}_t}$ is attained. For the genuinely finite-time, all-$t$-simultaneous guarantee we use the anytime confidence sequence of Theorem~\ref{thm:sampled-anytime} (a single mixing rate $\lambda$, or its density mixture above); the iterated-logarithm constant is the asymptotic cost of the curvature of that boundary.

\paragraph{Bandit analogue.}
The same argument applies to the martingale term in a generic estimated-loss decomposition under partial feedback: once the estimator increments admit the same bounded-range or conditional sub-Gaussian mgf bound, the centered bandit regret inherits the same asymptotic iterated-logarithm rate after the explicit bias and intrinsic-time terms are subtracted.

\paragraph{Comparator-centered cumulants.}
For a fixed posterior $\rho$, define
\[
\psi_{t,\rho}(\lambda) := \log \E_{i\sim p_t}\exp\!\left(\lambda(c_t(i)-\ip{\rho}{c_t})\right),
\qquad
\phi_{t,\rho}(\eta) := \eta^{-1}\psi_{t,\rho}(-\eta)\]
and set $Q_t^\rho(c) := \left(\phi_t(\eta_t)-\phi_{t,\rho}(\eta_t)\right)/\eta_t$.
Then the one-round regret satisfies $\ip{p_t}{c_t}-\ip{\rho}{c_t} =\phi_t(\eta_t)-\phi_{t,\rho}(\eta_t) =\eta_t Q_t^\rho(c)$,
so $R_T^c(\rho)=\sum_{t=1}^T \eta_t Q_t^\rho(c)$. 
Section~\ref{sec:luckiness} compares this exact comparator-centered increment to the more familiar second-moment quantity $\Psi_t(\rho)$.

\subsection{Local updates and pressure-targeted schedules}
\label{sec:pressure}

The second-order schedule of Section~\ref{sec:second-order} calibrates $\eta_t$ against cumulative intrinsic time for the prior-retempered update.
The second update---the local recursion $p_{t+1}(i)\propto p_t(i)e^{-\eta_t c_t(i)}$---has its own natural controller.

\paragraph{Nonmonotone schedules under the local update.}
The exact chain of Section~\ref{sec:adaptive} already allows arbitrary predictable schedules.
Monotone cooling was singled out there only because the retempered drift $D_T=\sum_{t=1}^{T-1}\bigl(A_t(\eta_t)-A_t(\eta_{t+1})\bigr)$ is nonpositive and therefore easy to discard under cooling.
That monotonicity, however, simplifies the analysis without being a structural requirement.
Under the local update, each round's score is weighted directly by its own learning rate, and a schedule that can both decrease and increase is often the natural choice. Hard rounds may call for caution, later high-edge rounds may justify a larger coefficient, and a purely monotone controller cannot express that.
This is the setting in which pressure targets are most natural: instead of asking for one cumulative clock to keep cooling forever, they calibrate the current round to a chosen target and let the cumulative identity keep track of the result.

In its most general form the pressure-target rule takes $a_t$ to be an arbitrary scalar target between the current minimum and mean score, with $\eta_t$ chosen so that the one-step finite-temperature cumulant hits that target. The special interpretation of $a_t$ as a \emph{pressure} (normalized free energy) is natural but not necessary for the identities below; we may read $a_t$ purely as a scaled local CGF level, and all cumulative identities remain intact.

The parallel with the retempered update is worth making explicit. Both updates compose the shared one-step information balance cumulatively. The retempered chain produces a global cumulant decomposition with temperature-change drift; the local chain produces a cumulative normalization identity where the final weights satisfy $p_{T+1}(i)=p_1(i)e^{-G_T(i)}$ with $\sum_i p_1(i)e^{-G_T(i)}=1$. The same per-round information functional $Q_t(c)$ appears in both cases (its realized value tracking each recursion's own played $p_t$), but the cumulative bookkeeping---and therefore the natural schedule family---is different.

For a current distribution $p_t$ and score vector $c_t\in\R^K$, define the one-round mix-loss curve $m_t(\eta):=-\eta^{-1}\log\bigl(\sum_{i=1}^K p_t(i)e^{-\eta c_t(i)}\bigr)$ for $\eta>0$.
When $c_t$ is nonconstant, $m_t$ is continuous and strictly decreasing with
\[
\lim_{\eta\downarrow 0} m_t(\eta)=\ip{p_t}{c_t},
\qquad
\lim_{\eta\to\infty} m_t(\eta)=\min_i c_t(i)
\]
so every target between the current mean and the current minimum determines a unique positive temperature.

\begin{proposition}[Pressure-targeted normalization]\label{prop:pressure-normalized}
Fix $t$, assume $c_t$ is nonconstant, and choose a scalar $a_t$ such that $\min_i c_t(i) < a_t < \ip{p_t}{c_t}$.
Then there is a unique $\eta_t>0$ satisfying
\begin{equation}\label{eq:pressure-target}
-\eta_t^{-1}\log\!\left(\sum_{i=1}^K p_t(i)e^{-\eta_t c_t(i)}\right)=a_t\end{equation}
Equivalently, for the shifted scores $\tilde c_t(i):=c_t(i)-a_t$,
\begin{equation}\label{eq:pressure-target-root}
\sum_{i=1}^K p_t(i)e^{-\eta_t \tilde c_t(i)}=1\end{equation}
With the exponential-weights update $p_{t+1}(i)\propto p_t(i)e^{-\eta_t c_t(i)}$,
we have, for every posterior $\rho\in\Delta([K])$,
\begin{equation}\label{eq:pressure-target-kl}
\KL(\rho \| p_{t+1})-\KL(\rho \| p_t)=\eta_t\left(\ip{\rho}{c_t}-a_t\right)\end{equation}
If $c_t$ and $a_t$ are simultaneously scaled by a constant $b>0$, then the solution scales as $\eta_t/b$.
\end{proposition}

\begin{corollary}[Unit-potential normalization]\label{cor:pressure-unit}
The special choice $a_t=0$ is feasible whenever $\ip{p_t}{c_t}>0>\min_i c_t(i)$.
In that case the pressure-target root condition~\eqref{eq:pressure-target-root} becomes $\sum_{i=1}^K p_t(i)e^{-\eta_t c_t(i)}=1$,
which is the unit-potential or pressure-normalized schedule.
\end{corollary}

\paragraph{Thermodynamic interpretation and scale-freeness.}
The normalization viewpoint has thermodynamic content. Up to the usual sign convention, $m_t(\eta)=-\eta^{-1}\log \sum_i p_t(i)e^{-\eta c_t(i)}$ is the one-step free-energy density associated with the current score vector, while $\eta^{-1}\log \sum_i p_t(i)e^{-\eta(c_t(i)-a_t)}$ is the corresponding pressure relative to the target level $a_t$. Solving \eqref{eq:pressure-target} therefore enforces a prescribed free-energy level, or equivalently normalizes the shifted pressure to zero. This is why we call the rule \emph{pressure-targeted normalization}. It is also naturally scale-free: simultaneous rescaling $(c_t,a_t)\mapsto (bc_t,ba_t)$ sends the solution to $\eta_t/b$, so only the relative geometry of the current round matters.

\begin{corollary}[Pressure-target regret identity]
\label{cor:pressure-regret}
Under Proposition~\ref{prop:pressure-normalized}, for every posterior $\rho\in\Delta([K])$,
\begin{equation}
\label{eq:pressure-one-step-regret}
\ip{p_t}{c_t}-\ip{\rho}{c_t}
=
\left(\ip{p_t}{c_t}-a_t\right)
+
\frac{\KL(\rho \| p_t)-\KL(\rho \| p_{t+1})}{\eta_t}\end{equation}
Consequently,
\begin{equation}\label{eq:pressure-cum-regret}
R_T^c(\rho)
=
\sum_{t=1}^T \left(\ip{p_t}{c_t}-a_t\right)
+
\sum_{t=1}^T \frac{\KL(\rho \| p_t)-\KL(\rho \| p_{t+1})}{\eta_t}\end{equation}
Equivalently,
\begin{equation}\label{eq:pressure-cum-regret-abel}
R_T^c(\rho)
=
\sum_{t=1}^T \left(\ip{p_t}{c_t}-a_t\right)
+\frac{\KL(\rho \| p_1)}{\eta_1}
-\frac{\KL(\rho \| p_{T+1})}{\eta_T}
+\sum_{t=2}^T \KL(\rho \| p_t)\!\left(\frac{1}{\eta_t}-\frac{1}{\eta_{t-1}}\right)\end{equation}
If $(\eta_t)$ is nondecreasing, then
\begin{equation}\label{eq:pressure-monotone-upper}
R_T^c(\rho)\le \sum_{t=1}^T \left(\ip{p_t}{c_t}-a_t\right)+\frac{\KL(\rho \| p_1)}{\eta_1}\end{equation}
\end{corollary}

This bridges the fully analyzed second-order schedule and the earlier root-finding remark.
The second-order theorem may be read as an especially explicit way of choosing target levels $a_t=m_t(\eta_t)$ and then controlling the resulting cumulative gap $\sum_t (\ip{p_t}{c_t}-a_t)$ on a square-root scale.

Proposition~\ref{prop:pressure-normalized} defines a genuinely different adaptive algorithm from the prior-retempered update \eqref{eq:retempered-hedge}: the temperature is chosen by line search against the current distribution $p_t$ and the current score vector $c_t$, and the recursion then proceeds locally via $p_{t+1}(i) \propto p_t(i)e^{-\eta_t c_t(i)}$.
Thus the schedule does not retemper the entire cumulative history at each step. 
When the temperature is held fixed the local and retempered views coincide, but for variable rates the distinction is real. 
In particular, the pressure-target rule is genuinely local: once $p_t$, $c_t$, and the target $a_t$ are given, no separate cumulative budget such as $V_{t-1}(c)$ needs to be tracked.
All historical information is compressed into the current weights themselves.

This distinction changes the natural proof object.
For the prior-retempered update, the chain in Section~\ref{sec:adaptive} is organized by the cumulative free energy $A_t(\eta):=-\eta^{-1}\log\bigl(\sum_i \pi(i)e^{-\eta C_t(i)}\bigr)$, evaluated at the played temperature.
For the variable-rate local recursion, that retempered free energy is generally not the state variable of the algorithm.
The native cumulative object is instead the exact normalization of the realized weighted scores from Proposition~\ref{prop:pressure-cumulative}.
Consequently the primitive exact statement for the local schedule is a weighted-regret identity; ordinary regret is then recovered from Corollary~\ref{cor:pressure-regret} by summation by parts when we want a scalar unweighted bound.

Fixed-temperature variational forms of the local recursion, the terminal-mass and event-conditioned weighted-regret identities, and the centered-score and quadratic-envelope specializations are developed in Appendix~\ref{app:pressure-details}; the body uses the pressure-target rule and its cumulative-normalization identity above.

\section{The side-information and comparator layer}
\label{sec:return-luck}

Sections~\ref{sec:exact} and~\ref{sec:adaptive} produced identities for composite losses, and Section~\ref{sec:schedules} turned those identities into concrete rate rules. We now return to the application layer. Translating the chain back to the original losses records the exact side-factor mismatch. Varying the side information and the sufficient statistic fed to the Bayes update covers optimism, specialists, sparsity, and bounded-influence transforms. Enlarging the comparator class reaches shifting and quantile benchmarks, and overlaying stochastic structure yields fast-rate, comparator-centered Bernstein consequences.

In particular, the bounded-range/Hoeffding envelopes below are just the first relaxations of the exact identities, while the comparator-centered second-order results of Section~\ref{sec:luckiness} are a later Bernstein-style refinement of the same cumulant term within the same algorithm.

\subsection{Returning to the original losses}
\label{sec:return}

Those identities translate back to the original losses and the side factors from Section~\ref{sec:exact} by a purely algebraic step, which turns the reduction into application-level regret statements.

For every posterior $\rho$, the original-loss and composite-loss regrets are related pathwise by
\begin{equation}\label{eq:generic-reduction}
\sum_{t=1}^T \ip{p_t}{\ell_t} - \ip{\rho}{L_T}
=
R_T^c(\rho) + \sum_{t=1}^T \left(\ip{\rho}{u_t}-\ip{p_t}{u_t}\right)\end{equation}

\begin{corollary}[Exact original-loss regret identity]
\label{cor:actual-losses}

Let the composite losses be of the form $c_t(i)=\ell_t(i)+u_t(i)$ with $u_t(i):=-\eta_t^{-1}\log s_t(i)$.
Let $(p_t)$ be the variable-temperature Bayes sequence \eqref{eq:retempered-hedge}. 
Then for every posterior $\rho\in\Delta([K])$,
\begin{equation}\label{eq:actual-loss-bound}
\sum_{t=1}^T \ip{p_t}{\ell_t} - \ip{\rho}{L_T}
=
D_T+B_T(\rho)+\sum_{t=1}^T \eta_tQ_t(c)
+ \sum_{t=1}^T \left(\ip{\rho}{u_t}-\ip{p_t}{u_t}\right)\end{equation}
\end{corollary}

The fixed-rate specialization of the same identity yields the classical bounded-range regret bound as a direct corollary.

\begin{corollary}[Fixed-rate bounded-range regret bound]
\label{cor:fixed-rate-bounded}
Under the same assumptions as Corollary~\ref{cor:fixed-rate}, if $c_t(i)\in[a_t,b_t]$ for all $i$, then
\begin{equation}\label{eq:fixed-rate-comp-reg}
\sum_{t=1}^T \ip{q_t}{c_t} - \ip{\rho}{C_T}
\le \eta^{-1}\KL(\rho\|q_1) + \frac{\eta}{8}\sum_{t=1}^T (b_t-a_t)^2\end{equation}
and therefore
\begin{equation}
\label{eq:fixed-rate-ell-reg}
\sum_{t=1}^T \ip{q_t}{\ell_t} - \ip{\rho}{L_T}
\le \eta^{-1}\KL(\rho\|q_1) + \frac{\eta}{8}\sum_{t=1}^T (b_t-a_t)^2
+ \sum_{t=1}^T \left(\ip{\rho}{u_t}-\ip{q_t}{u_t}\right)\end{equation}
\end{corollary}

\begin{corollary}[Cross-entropy form of the adaptive mismatch]
\label{cor:adaptive-mismatch}
The cross-entropy form of the mismatch follows from the same calculation as Corollary~\ref{cor:fixed-rate-mismatch}, with the fixed rate $\eta$ replaced by the round-dependent $\eta_t$.
If
\begin{equation}
\label{eq:side-decomp-adaptive}
s_t(i)=\prod_{w=1}^{J_t} \mu_{t,w}(i)^{\alpha_{t,w}}
\end{equation}
then the mismatch term in \eqref{eq:actual-loss-bound} is exactly
\begin{equation}\label{eq:adaptive-mismatch-crossent}
\sum_{t=1}^T \left(\ip{\rho}{u_t}-\ip{p_t}{u_t}\right)
= \sum_{t=1}^T\sum_{w=1}^{J_t}\frac{\alpha_{t,w}}{\eta_t}
\left(\Hc(\rho,\mu_{t,w})-\Hc(p_t,\mu_{t,w})\right)\end{equation}
\end{corollary}

Substituting the two-sided envelope of Theorem~\ref{thm:scaling-time} into the exact identity yields an explicit second-order bound.

\begin{corollary}[Second-order original-loss regret bounds]\label{cor:actual-loss-bounds}
Under the schedule \eqref{eq:budget-schedule-C},
\begin{equation}\label{eq:actual-loss-two-sided}
\sum_{t=1}^T \ip{p_t}{\ell_t} - \ip{\rho}{L_T}
\in
D_T+B_T(\rho)+\sum_{t=1}^T \left(\ip{\rho}{u_t}-\ip{p_t}{u_t}\right)+\Env{V_T(c)}{Q_*^T(c)}\end{equation}
If $\KL(\rho \| \pi) \leq \Gamma$, then
\begin{equation}
\label{eq:actual-loss-budgeted}
\sum_{t=1}^T \ip{p_t}{\ell_t} - \ip{\rho}{L_T}
\leq \Gamma + Q_*^T(c) + (2C+C^{-1}) \sqrt{\Gamma V_T(c)}
+ \sum_{t=1}^T \left(\ip{\rho}{u_t}-\ip{p_t}{u_t} \right)
\end{equation}
For the optimized choice $C=1/\sqrt2$, this becomes
\[
\sum_{t=1}^T \ip{p_t}{\ell_t} - \ip{\rho}{L_T}
\le \frac32\Gamma+Q_*^T(c)+2\sqrt{2\Gamma V_T(c)}
+ \sum_{t=1}^T \left(\ip{\rho}{u_t}-\ip{p_t}{u_t}\right)
\]
\end{corollary}

\paragraph{Quantile corollaries.}
If the prior is uniform on $[K]$ and $\rho$ is uniform on a subset $A\subseteq[K]$, then $\KL(\rho\|\pi)=\log(K/|A|)$.
Every PAC-Bayes theorem above therefore immediately implies the corresponding quantile statement by choosing $A$ to be a hindsight set of good experts.
The only additional term in the mixed-prior setting is the explicit side-factor mismatch from \eqref{eq:adaptive-mismatch-crossent}.

The composite-score reduction also accommodates compensators and alternative sufficient-statistic choices for the side factor, and it extends to shifting comparators and sleeping experts; these are developed in Appendix~\ref{app:side-info}, each adding at most a named predictable term to the accounting.

\subsection{An outer controller over a dyadic grid}\label{sec:simul-quantile}

One more layer of exponential weights, run over a dyadic grid, turns a fixed-budget statement into one that holds simultaneously in $\varepsilon$.
The construction is simpler than a potential-specific single-copy algorithm, and it stays entirely within the exact Hedge template of this paper.
It is stated twice below because the grid can index two different things: complexity budgets, which gives simultaneity in the quantile level, and candidate side-priors, which gives an outer copy that learns pooling weights online.

\begin{theorem}[Simultaneous $\varepsilon$-quantile regret via a budget controller]
\label{thm:simul-quantile}
Assume $c_t(i)\in[0,1]$ for all $t$ and $i$, and let the base prior on experts be uniform on $[K]$.
Fix $J := \lceil\log_2 \log(\max\{e,K\})\rceil$ and budgets $\Gamma_j:=2^j$ for $j=0,\dots,J$.
For each $j$, run a worker copy of the second-order algorithm given by \eqref{eq:budget-schedule-C} with budget $\Gamma_j$, producing weights $p_t^{(j)}$.
Define the meta-losses $m_t(j) := \ip{p_t^{(j)}}{c_t}\in[0,1]$, run a controller copy of \eqref{eq:budget-schedule-C} on the $J+1$ meta-experts with uniform prior and budget $\Gamma^{\mathrm{ctl}} := \log(J+1)$, obtaining weights $\alpha_t\in\Delta([J+1])$, and play $\bar p_t := \sum_{j=0}^J \alpha_t(j) p_t^{(j)}$.
Let $Q_t^{\mathrm{ctl}},V_T^{\mathrm{ctl}},Q^{*,\mathrm{ctl}}_T$ denote the controller's intrinsic-time quantities for the meta-losses $m_t$, and let $Q_t^{(j)},V_T^{(j)},Q^{*,j}_T$ denote the corresponding quantities for worker $j$ on the original losses $c_t$.
Then for every $\varepsilon\in[K^{-1},e^{-1}]$ and every set $A\subseteq[K]$ with $|A|\ge \varepsilon K$, if $j(\varepsilon) := \min\{j: \Gamma_j\ge \log(1/\varepsilon)\}$, we have
\begin{align}
\sum_{t=1}^T \ip{\bar p_t}{c_t}-\frac1{|A|}\sum_{i\in A} C_T(i) 
&\leq \Gamma^{\mathrm{ctl}}+Q^{*,\mathrm{ctl}}_T
+(2C + C^{-1}) \sqrt{ \Gamma^{\mathrm{ctl}} V_T^{\mathrm{ctl}}} \notag\\
&\qquad\quad +\Gamma_{j (\varepsilon)} + Q^{*, j(\varepsilon)}_T
+(2C + C^{-1}) \sqrt{\Gamma_{j(\varepsilon)} V_T^{(j(\varepsilon))}}
\label{eq:simul-quantile}
\end{align}
Since $\Gamma_{j(\varepsilon)} \le 2 \log (1/\varepsilon)$, the right-hand side is the fixed-budget $\varepsilon$-quantile bound at scale $\log(1/\varepsilon)$ plus the extra controller term depending only on the $J+1$ budget copies.
\end{theorem}

This is the simplest exact route to a simultaneous quantile theorem in the present notation. 
It is not as sharp as a dedicated one-copy potential such as Squint or the newer NormalHedge variants \cite{koolen2015squint,Luo2026Squint,FreundEtAl2026}, but it already shows that the exact intrinsic-time identities of this paper can be upgraded from fixed-budget PAC-Bayes statements to a single run that covers all quantiles.
On synthetic and real online sequences the controller meets every quantile simultaneously within the predicted overhead, and incurs only a small fraction of it on real paths.

The same controller construction, run over candidate side-priors instead of budgets, gives an exact two-level meta-pooling identity in which an outer copy learns the pooling weights online.

\begin{theorem}[Exact two-level meta-pooling identity]
\label{thm:two-level}
Let $\pi_{t,1},\dots,\pi_{t,W}\in\Delta([K])$ be predictable candidate side-priors. Run an \emph{inner} retempered-Hedge copy \eqref{eq:retempered-hedge} with base prior $\pi$ and predictable schedule $(\eta_t)$ on the composite losses $c_t=\ell_t+u_t$, where $s_t(i)=\prod_{w=1}^W \pi_{t,w}(i)^{\alpha_{t,w}}$ and $u_t(i)=-\eta_t^{-1}\log s_t(i)$. Let an \emph{outer} retempered-Hedge copy with base prior $\beta\in\Delta([W])$ and predictable schedule $(\kappa_t)$ produce the pooling weights $\alpha_t\in\Delta([W])$ from the meta-losses $g_t(w):=\eta_t^{-1}\bigl(\Hc(\rho,\pi_{t,w})-\Hc(p_t,\pi_{t,w})\bigr)$, $w\in[W]$, which are $\mathcal F_{t-1}$-measurable because $p_t$ depends only on $C_{t-1}$ and not on $\alpha_t$. Write $D_T,B_T(\rho),V_T(c)$ for the inner chain quantities of Theorem~\ref{thm:pacbayes-second-order} and $D_T^{\mathrm o},B_T^{\mathrm o}(\sigma),V_T^{\mathrm o}(g)$ for the outer chain quantities computed from $(g_t,\kappa_t,\beta)$, with $G_T(w):=\sum_{t}g_t(w)$. Then for every action comparator $\rho\in\Delta([K])$ and every pooling comparator $\sigma\in\Delta([W])$,
\begin{equation}\label{eq:two-level}
R_T^{\ell}(\rho)
= \Bigl[D_T+B_T(\rho)+\sum_{t=1}^T\eta_tQ_t(c)\Bigr]
+\Bigl[D_T^{\mathrm o}+B_T^{\mathrm o}(\sigma)+\sum_{t=1}^T\kappa_tQ_t^{\mathrm o}(g)\Bigr]
+\ip{\sigma}{G_T},
\end{equation}
and $\ip{\sigma}{G_T}=\sum_{t=1}^T\sum_{w=1}^W\sigma_w\,\eta_t^{-1}\bigl(\Hc(\rho,\pi_{t,w})-\Hc(p_t,\pi_{t,w})\bigr)$ is the cross-entropy mismatch of the best fixed pool $\sigma$. Every term is exact; there are no unrealized constants. If both levels use the second-order schedule \eqref{eq:budget-schedule-C} with $\KL(\rho\|\pi)\le\Gamma$ and $\KL(\sigma\|\beta)\le\Gamma^{\mathrm o}$, then Theorem~\ref{thm:scaling-time} envelopes the two chains and
\[
R_T^{\ell}(\rho)\le (1{+}C^2)\Gamma+Q_*^T(c)+(2C{+}C^{-1})\sqrt{\Gamma V_T(c)}
+(1{+}C^2)\Gamma^{\mathrm o}+Q_*^{T,\mathrm o}+(2C{+}C^{-1})\sqrt{\Gamma^{\mathrm o} V_T^{\mathrm o}(g)}
+\ip{\sigma}{G_T}.
\]
\end{theorem}

A directly deployable variant of the two-level identity is available. The meta-loss $g_t$ contains the unobservable term $\Hc(\rho,\pi_{t,w})$, so \eqref{eq:two-level} is an accounting identity that cannot itself be run online. A directly runnable two-level system runs $W$ inner copies, copy $w$ on the single side-prior $\pi_{t,w}$, producing $p_t^{(w)}$; an outer second-order controller on the observable meta-losses $m_t(w)=\ip{p_t^{(w)}}{\ell_t}\in[0,1]$ with budget $\Gamma^{\mathrm{ctl}}=\log W$ plays the arithmetic mixture $\bar p_t=\sum_w\alpha_{t,w}p_t^{(w)}$. This is Theorem~\ref{thm:simul-quantile} with the dyadic budget grid replaced by the $W$ priors, and its proof gives a summed intrinsic-time regret bound (a controller term at $\Gamma^{\mathrm{ctl}}$ plus the best single-prior worker term).

\subsection{Stochastic luckiness and fast-rate PAC-Bayes control}
\label{sec:luckiness}

The pathwise theorems above make no stochastic assumptions: they are valid for every realized sequence. 
But we often want a second layer of interpretation when the data happen to come from a benign law. 
A stochastic-luckiness bound does exactly this. 
It identifies a regime in which the same quantity that drives the adversarial regret bound is itself controlled by the comparator's excess risk. 
Then the general square-root behavior collapses to a constant or other fast expected rate, without changing the algorithm and without sacrificing the original pathwise guarantee. 

In the present framework, the relevant pathwise quantity is the intrinsic-time cost $\sum_{t=1}^T \eta_t Q_t(c)$. To turn that cost into a fast-rate theorem, we need a self-bounding relation that makes $Q_t(c)$ comparable to the instantaneous gap to the comparator.
Because our regret identity is centered at a comparator distribution and not at a single best expert alone, the natural stochastic assumption is also comparator-centered: the roundwise fluctuation around the comparator average should be controlled by the comparator's mean excess loss.
This is why the resulting theorem is genuinely PAC-Bayesian.

This viewpoint is the online analogue of standard low-noise fast-rate conditions in statistical learning. 
Such conditions do not alter the worst-case guarantee; they say that on favorable problems the variance covers the bias, so slow rates accelerate to fast ones. 
Here the same idea is applied to the intrinsic-time clock of the Bayes/exponential-weights update. 
When the process is lucky in that sense, the comparator-centered intrinsic time largely covers its own cost, and the expected regret becomes constant in $T$.

The relation with \cite{PerezOrtizKoolen2022} is conceptually simple. 
Their multiscale analysis combines a second-order regret theorem with a stochastic self-bounding argument to obtain constant expected pseudoregret under a low-noise condition. In the normalized composite-loss regime studied here, the corresponding stochastic step is cleaner: one centers the one-round mixability-gap estimate at the comparator instead of at the learner mean, so the same self-bounding idea plugs directly into the exact cumulant identity.
The comparator-centered envelope and the fixed-rate consequence of that self-bounding step are established in \cite{balsubramani2026concentration} and are stated here for reference.
The transfer to the adaptive schedule comes after, and there the clock the rate reads is the same clock the low-noise condition controls.

\subsubsection{Fast learning under low comparator-centered noise}

\begin{proposition}[Comparator-centered second-order envelope \cite{balsubramani2026concentration}]
\label{prop:comparator-second-order}
Assume that $c_t(i)\in[0,1]$ and $\eta_t\in(0,1]$ for all $t\le T$ and $i\in[K]$, and define $\Psi_t(\rho) := \sum_{i=1}^K p_t(i)\left(c_t(i)-\ip{\rho}{c_t}\right)^2$.
Then the cumulant increment obeys
\begin{equation}\label{eq:Qt-psi-bound}
Q_t(c)\le (e-2)\Psi_t(\rho)
\qquad\text{for every }t\le T\end{equation}
If, in addition, $(\eta_t)_{t=1}^T$ is nonincreasing, then for every posterior $\rho\in\Delta([K])$,
\begin{equation}\label{eq:comparator-second-order}
R_T^c(\rho)
\le \frac{\KL(\rho\|\pi)}{\eta_T} + (e-2)\sum_{t=1}^T \eta_t \Psi_t(\rho)\end{equation}
\end{proposition}

The stochastic assumption we need for this result is the direct PAC-Bayes analogue of the point-comparator low-noise condition.

Let $c_1,c_2,\dots$ be i.i.d.\ random vectors in $[0,1]^K$ and let $\mu := \E[c_1]\in[0,1]^K$.
Fix $\rho\in\Delta([K])$. We say that $\rho$ satisfies the comparator-centered low-noise condition with constant $\kappa_\rho\in[0,\infty)$ if for every expert $i\in[K]$,
\begin{equation}\label{eq:rho-massart}
\E\bigl[(c_t(i)-\ip{\rho}{c_t})^2\bigr]
\le \kappa_\rho\left(\mu(i)-\ip{\rho}{\mu}\right)
\end{equation}

Condition \eqref{eq:rho-massart} is exactly the standard point-mass low-noise condition when $\rho=\delta_{k^*}$, but it is highly restrictive for diffuse posteriors.
If some expert in the support of $\rho$ has the same mean loss as the posterior average, then \eqref{eq:rho-massart} already forces $c_t(i)=\ip{\rho}{c_t}$ almost surely.
While it algebraically closes the exact comparator-centered second-order proof, this fast expected rate is therefore meaningful primarily for point-mass comparators (single best experts) or highly degenerate environments.
Constant regret for more general Bayes-optimal posteriors is nonetheless available by a different argument: a diffuse comparator is a strictly weaker benchmark than the single best expert, so its fast rate is inherited from the point-comparator result without invoking \eqref{eq:rho-massart}, as Proposition~\ref{prop:diffuse-domination} records.

\begin{theorem}[Fixed-rate stochastic luckiness \cite{balsubramani2026concentration}]
\label{thm:fixed-luckiness}
Assume that $c_1,c_2,\dots$ are i.i.d.\ random vectors in $[0,1]^K$ and that $\rho\in\Delta([K])$ satisfies Condition \eqref{eq:rho-massart} with constant $\kappa_\rho$. 
Run fixed-rate Hedge on the composite losses, $p_t(i) \propto \pi(i) e^{-\eta C_{t-1}(i)}, \qquad \eta \in (0,1]$.
Then for every $T\ge 1$ and every $\eta$ such that $(e-2)\kappa_\rho \eta < 1$,
\begin{equation}\label{eq:fixed-luckiness}
\E\bigl[R_T^c(\rho)\bigr]
\le \frac{\KL(\rho\|\pi)}{\eta\left(1-(e-2)\kappa_\rho\eta\right)}
\end{equation}
In particular, with the choice $\displaystyle \eta := \min\! \left\{ 1,\frac{1}{2(e-2) \kappa_\rho} \right\} $, 
we obtain the time-uniform bound
\begin{equation}
\label{eq:fixed-luckiness-explicit}
\E\bigl[R_T^c(\rho)\bigr]
\le 2\left(1 + 2 (e-2) \kappa_\rho \right) \KL(\rho\|\pi)
\end{equation}
\end{theorem}

The transfer to the adaptive schedule is the point of this subsection: the same fast rate survives when the temperature is driven by the realized clock, and the learner never needs the value of $\kappa_\rho$ to get it.

\begin{corollary}[Predictable-rate stochastic luckiness]
\label{cor:budget-luckiness}
Assume the same stochastic setting as in Theorem~\ref{thm:fixed-luckiness}, and run the second-order schedule \eqref{eq:budget-schedule-C} with budget $\Gamma>0$.
If $\KL(\rho\|\pi)\le \Gamma$, then for every $T\ge 1$,
\begin{equation}\label{eq:budget-luckiness-raw-C}
\E\bigl[R_T^c(\rho)\bigr]
\le 2\Gamma + 2\,\E\bigl[Q_*^T(c)\bigr]
+ (2C+C^{-1})^2(e-2)\kappa_\rho\Gamma\end{equation}
In particular, if $c_t(i)\in[0,1]$ almost surely for all $t,i$, then Proposition~\ref{prop:range} gives $Q_*^T(c)\le 1/8$ and therefore, for example, if $C=1/\sqrt{2}$, then $\E\bigl[R_T^c(\rho)\bigr] \le 2\Gamma + \frac{1}{4} + 8(e-2) \kappa_\rho \Gamma$.
So the same variable-rate algorithm enjoys constant expected regret in $T$ under the comparator-centered low-noise condition while retaining the cumulant intrinsic time in the pathwise theorem.
\end{corollary}

\begin{corollary}[Point-comparator low-noise and gap corollaries]
\label{cor:point-luckiness}
Let $k^*\in[K]$ and set $\rho=\delta_{k^*}$.

\begin{enumerate}[label=\textup{(\roman*)},leftmargin=2.4em]
\item 
If the ordinary point-comparator low-noise condition holds, namely
\begin{equation}\label{eq:point-massart}
\E\bigl[ (c_t(i)-c_t(k^*))^2 \bigr]
\le c_* \left( \mu(i)-\mu(k^*) \right)
\qquad \text{for all } i\in[K]\end{equation}
then Theorem~\ref{thm:fixed-luckiness} yields $\E\bigl[ R_T^c (k^*) \bigr]
\le 2\left( 1 + 2(e-2) c_* \right) \log \frac{1}{\pi(k^*)}$.

\item 
If, more concretely, there is a mean gap $d_{\min} := \min_{i\neq k^*}\left(\mu(i)-\mu(k^*)\right)>0$, then \eqref{eq:point-massart} holds with $c_*=1/d_{\min}$, and therefore $\E\bigl[R_T^c(k^*)\bigr]\le 2\left(1+\frac{2(e-2)}{d_{\min}}\right)\log\frac{1}{\pi(k^*)}$.
For the uniform prior this is $O\left((1+d_{\min}^{-1})\log K\right)$.
\end{enumerate}
\end{corollary}

The same point-comparator guarantee already delivers the diffuse case anticipated above, by a domination argument.

\begin{proposition}[Fast rates transfer to diffuse comparators by domination]
\label{prop:diffuse-domination}
Let $c_1,c_2,\dots$ be i.i.d.\ in $[0,1]^K$ with mean $\mu$, and let $k^\star\in\arg\min_i \mu(i)$. For every algorithm and every $\rho\in\Delta([K])$,
\[
\E\!\left[R_T^c(\rho)\right]=\E\!\left[R_T^c(k^\star)\right]-T\bigl(\ip{\rho}{\mu}-\mu(k^\star)\bigr)\ \le\ \E\!\left[R_T^c(k^\star)\right],
\]
since $\ip{\rho}{\mu}\ge \mu(k^\star)$. Consequently, whenever the point-comparator low-noise condition \eqref{eq:point-massart} holds at $k^\star$---in particular whenever there is a mean gap $d_{\min}>0$---Corollary~\ref{cor:point-luckiness} already yields, for every comparator $\rho$ (diffuse or not),
\[
\E\!\left[R_T^c(\rho)\right]\ \le\ 2\Bigl(1+\tfrac{2(e-2)}{d_{\min}}\Bigr)\log\tfrac{1}{\pi(k^\star)}=O(1).
\]
\end{proposition}

A diffuse comparator is a strictly weaker benchmark, so constant expected regret against it is inherited from the point-comparator result and never needs \eqref{eq:rho-massart}. Two points sharpen the picture. First, \eqref{eq:rho-massart} is more restrictive than a single equal-mean coincidence: because its left side is nonnegative, it is infeasible at \emph{every} expert with $\mu(i)\le\ip{\rho}{\mu}$ unless that expert's loss is deterministic, so for any diffuse $\rho$ not concentrated on one mean-level the condition forces determinism on all weakly-below-average experts. Second, when the means are distinct the Bayes-optimal posterior is a point mass, so a genuinely diffuse Bayes-optimal posterior arises only under a tie: on a tied optimal set $S$ we have $\E[R_T^c(\rho)]=\E[R_T^c(i)]$ for every $i\in S$, and constant regret follows from the gap $\min_{i\notin S}(\mu(i)-\mu^\star)>0$ to the strictly suboptimal experts. The sharper constant is a separate matter: the domination bound scales with $\log(1/\pi(k^\star))=\KL(\delta_{k^\star}\|\pi)$, whereas a diffuse comparator has strictly smaller complexity $\KL(\rho\|\pi)=\log K-\mathrm{H}(\rho)$; obtaining $\E[R_T^c(\rho)]=O(\KL(\rho\|\pi))$ is a fast-rate question for the central-condition machinery; the comparator-centered second-order envelope does not decide it.

The fixed-rate theorem is a clean theoretical way to understand the luckiness mechanism, but in practice the variable-rate algorithm from Sections~\ref{sec:adaptive} and~\ref{sec:second-order} admits an analogous conclusion, without requiring advance knowledge of the noise level parameter $\kappa_{\rho}$.

\subsubsection{Transfer back to the original losses}

Combining Corollary~\ref{cor:actual-losses} with any of the bounds above shows that the expected regret for the original losses $\ell_t$ equals the composite-loss constant plus the expected mismatch term $\E\Biggl[ \sum_{t=1}^T \left( \ip{\rho}{u_t} - \ip{p_t}{u_t} \right) \Biggr]$.
Hence any application in which this term is uniformly bounded above inherits a constant expected actual-loss regret bound as well. 
When the side factors are decomposed as in \eqref{eq:side-decomp-adaptive}, this mismatch is exactly the cumulative cross-entropy difference in \eqref{eq:adaptive-mismatch-crossent}.

The relation with \cite{PerezOrtizKoolen2022} is now transparent. 
Their multiscale theorems show that a carefully tuned multiscale entropy regularizer can have both scale-sensitive worst-case regret and constant expected pseudoregret under the same low-noise hypothesis. 
The relative-entropy-geometry results in the present section do \emph{not} by themselves subsume that scale-sensitive guarantee: standard Hedge on common-scale composite losses does not automatically deliver regret that tracks expert-dependent ranges, the guarantee the multiscale experts problem was introduced to provide \cite{bubeck2019onlinemultiscale}.
The accounting itself is not restricted to ordinary relative-entropy geometry, as the common-scale setting here shows and Appendix~\ref{sec:weighted-multiscale} shows more directly. 
The stochastic-luckiness step becomes much simpler and more PAC-Bayesian once one works with comparator-centered intrinsic time, while the weighted-entropy variant changes only the transport geometry and terminal potential. 
In that sense, Proposition~\ref{prop:comparator-second-order} is the common-scale counterpart of the more delicate regularizer-specific variance conversion carried out in \cite{PerezOrtizKoolen2022}.

\section{Applications and extensions}
\label{sec:applications-summary}

The composite-loss reduction of Section~\ref{sec:exact} makes the same information accounting portable: once side information, optimism, and partial feedback are folded into a composite score, each setting below is the one-step balance and its two cumulative forms, specialized. Each is developed in full, with proofs, in the appendices; here we record only what the accounting yields in each case.

In the full-information settings (Appendix~\ref{sec:further}), ordinary Hedge inherits an exact variational sharpening of the classical entropic regret bound, with the slack named and not discarded (Appendix~\ref{sec:variational-sharpening}). Weighted negative entropy and multiscale expert sets are absorbed into the same one-step cumulant (Appendix~\ref{sec:weighted-multiscale}), and multiscale forgetting arises from a continuum mixture of restart priors. Repeated games and robust opponent modeling read the decomposition as a value-of-side-information accounting. Continuum priors and hyperparameter mixtures yield the PAC-Bayes form directly (Appendix~\ref{sec:continuum}). Continuous-action online convex optimization is covered by a density-valued update (Appendix~\ref{sec:oco}), and boosting is the local update at a pressure target---the line search recovers the AdaBoost coefficient in closed form, and the cumulative normalization identity reproduces the exponential-loss decay exactly (Appendix~\ref{sec:boosting}).

Partial feedback is outside the scope of this paper. The learner there never observes the loss vector the accounting is written in, so the identity holds on an \emph{estimated} score and what the setting adds is a bounded and explicitly named set of terms relating those estimates to the losses actually incurred. That is a development of its own and is not undertaken here. Within full information the framework is unchanged across every application above; only the composite score differs.
\section{Related work}
\label{sec:discussion-related}

Table~\ref{tab:tradeoffs} is meant as a reading guide for the surrounding literature.
The online-learning literature is by now large and well surveyed \cite{hoi2021online}. Many of its papers appear to study the apparent tradeoffs of Table~\ref{tab:tradeoffs}, but they often differ mainly in which part of the same information balance is emphasized or relaxed.
We organize the discussion below by those themes.

\begingroup
\centering
\small
\renewcommand{\arraystretch}{1.35}

\begin{longtable}{@{}p{0.22\textwidth}|p{0.32\textwidth}|p{0.4\textwidth}@{}}
\caption{Supposed dichotomies of online learning, and their explanations through the lens of information.}
\label{tab:tradeoffs} \\

\toprule
\textbf{Apparent tradeoff} & \textbf{Conventional manifestation} & \textbf{Exact informational form} \\
\midrule
\endfirsthead

\multicolumn{3}{c}%
{{\bfseries Table \thetable\ continued from previous page}} \\
\toprule
\textbf{Apparent tradeoff} & \textbf{Conventional manifestation} & \textbf{informational form} \\
\midrule
\endhead

\bottomrule
\endfoot

\bottomrule
\endlastfoot

Hard vs.\ easy sequences &
Worst-case and stochastic analyses require separate theorems. &
Hardness is relative to a comparator: the same identity yields worst-case $O(\sqrt{V_T})$ or fast rates depending on how much the intrinsic time self-bounds. Pure noise remains hard across comparator choices. \\
Second-order vs.\ first-order bounds &
Second-order (variance-sensitive) bounds need extra assumptions or different algorithms. &
First-, second-, and higher-order bounds are successive Taylor relaxations of the same exact cumulant term; the order is chosen by the analyst while the algorithm is unchanged. \\
Different comparator classes &
Point experts, mixtures, and/or sleeping/shifting comparators need separate analyses. &
Every comparator $\rho$ incurs regret according to its information-theoretic cost $\KL(\rho\|\pi)$; richer classes have larger terminal terms in the same identity. \\
Parameter-free vs.\ adaptive &
``Parameter-free'' methods are presented as removing tuning, while adaptive methods still expose learning-rate or regularization choices. &
Budgets, temperatures, and targets are observable constraints on comparator complexity or revealed difficulty. Outer controllers may learn them online, but the same identity underlies both the tuned and simultaneous versions. \\
Algorithmic upper vs.\ lower envelopes &
Algorithmic penalty upper bounds and actual incurred penalties are usually analyzed separately, and the gap between them is hidden inside worst-case inequalities. &
For any realized path, the algorithm's internal intrinsic-time loss decomposes exactly. There is no slack at the identity level, and any later gap comes only from deliberately relaxing exact terms such as $Q_t(c)$ or $D_T$. \\
Variance definitions &
Multiple variance notions (empirical, conditional, quadratic) appear in different results. &
All approximate the exact intrinsic-time density: the algorithm's own play-conditional cumulant $Q_t(c)$. \\
Time horizon &
Bounds depend on the horizon $T$, which must be known or guessed. &
The intrinsic-time process $V_T(c)$ replaces the horizon; adaptive schedules track it online without knowing $T$. \\
Bounded vs.\ unbounded losses &
Bounded losses (e.g.\ $c_t(i)\in[0,1]$) are needed for low regret; boundedness and Taylor expansions go hand in hand to control the cumulant generating function. &
The identity requires only that the one-step log-normalizer is finite at the chosen temperature.  Boundedness enters only when relaxing the exact cumulant $Q_t(c)$ via a Taylor or range-based bound; without such relaxations, the pathwise identity itself remains meaningful whenever the exponential moment exists. \\
Full vs.\ partial feedback &
Bandit and partial-feedback settings require fundamentally different algorithms and analyses. &
The same identity applies to estimated losses; the only new terms are an explicit predictable bias correction and martingale terms from the estimation procedure. The intrinsic time is computed on the estimates, and graph or sparsity structure enters only through the observation ratios that govern those corrections. \\

\end{longtable}
\endgroup

One illustration makes the point of the table concrete. The three ledger terms, normalized at each prefix to shares $\omega^{\mathrm{pay}}_t,\omega^{\mathrm{drift}}_t,\omega^{\mathrm{info}}_t$ that sum to one, turn the decomposition into a diagnostic read directly off a run (Figure~\ref{fig:same-regret-diff-decomp}). Two sequences with the same terminal regret can spend it on very different terms, so the apparent trade-offs of Table~\ref{tab:tradeoffs} are different readings of the same three quantities.

\subsection{Difficulty measures and schedule design}

The first group of rows in Table~\ref{tab:tradeoffs} concerns how one measures the difficulty of a realized sequence.
The original weighted-majority and Hedge analyses yield the familiar $\widetilde O(\sqrt{T\log K})$ scale \cite{LW94,FreundSchapire97,CesaBianchiMansourStoltz2007}, and the standard monograph on prediction with expert advice collects this classical worst-case theory \cite{CesaBianchiLugosi2006}.
A long subsequent line of work replaced the horizon $T$ by more sensitive first- or second-order quantities, including quadratic excess losses, quantile-sensitive second-order terms, mixability-gap schedules, and parameter-free potentials \cite{gaillard2014secondorderhedge,derooij2014adahedge,koolen2015squint,luo2015achieving,orabona-pal}.
Freund's 2016 scaling-time problem sharpened this agenda by asking specifically for regret guarantees driven by the variance of the realized loss under the learner's own current distribution \cite{Freund2016}; the recent NormalHedge result of \cite{FreundEtAl2026} shows that a closely related curvature process can support simultaneous quantile adaptation.

A useful way to place these works side by side is to distinguish several common ``difficulty'' scales.
Let $e_t(i):=\ell_t(i)-\ell_t(k)$ denote excess loss relative to a fixed comparator expert $k$, and let $\mu_t:=\ip{p_t}{\ell_t}$.
For losses in $[0,1]$, the standard quantities satisfy
\begin{equation}
\label{eq:variancenotions}
\underbrace{\sum_{t=1}^T \Var_{i\sim p_t}(\ell_t(i))}_{\textup{variance across actions}}
\le
\underbrace{\sum_{t=1}^T \sum_{i=1}^K p_t(i)e_t(i)^2}_{\textup{quadratic excess loss}}
\le
\underbrace{\sum_{t=1}^T \sum_{i=1}^K p_t(i)|e_t(i)|}_{\textup{first-order magnitude}}
\end{equation}
Different algorithms and analyses are tuned to different points on this spectrum \cite{CesaBianchiMansourStoltz2007,gaillard2014secondorderhedge,koolen2015squint,luo2015achieving,FreundEtAl2026}.
AdaHedge and SafeBayes are especially relevant for the present paper because they tune the learning rate using the cumulative mixability gap itself \cite{derooij2014adahedge,the2012safebayesian}.

Our relation to this literature is twofold.
First, the quantity tracked here, $Q_t(c)$, is the finite-temperature loss term in the one-step identity; it is defined after the composite-loss reduction has already absorbed side information.
Proposition~\ref{prop:tilted-var} shows that it is an average of tilted variances, so the familiar quadratic terms arise only after an additional relaxation.
Second, the schedule question splits naturally into two geometries.
For the prior-retempered update, intrinsic-time schedules play the role of second-order, horizon-free clocks.
For the local update, pressure-target line searches stabilize a one-step normalization condition instead of a cumulative variance budget.
This is why the rows of Table~\ref{tab:tradeoffs} concerning hard versus easy sequences, first versus second order, variance definitions, and horizon dependence can be treated inside one exact accounting, with no switching between unrelated proofs.

The same perspective clarifies the bounded-versus-unbounded row of the table.
This is also why bounded-influence score transforms fit naturally beside recent work on heavy-tailed or unbounded losses \cite{neyman2023nologpooling}: the primitive object is still the log-normalizer, with no bounded-range assumption.

\paragraph{Safe Bayes as adaptive Bayes seen on the wrong clock.}
The Safe-Bayesian programme (\cite{the2012safebayesian} and the misspecification analysis of \cite{grunwald2017inconsistency}) tempers the likelihood by an exponent $\eta\in(0,1]$ to repair posterior consistency under model misspecification, choosing $\eta^*$ to track the relative-entropy-optimal concentration rate.
In the present framework the same exponent $\eta_t$ enters as the per-round inverse-step-size, so Safe Bayes can be read as Bayes computed on a clock that ticks slower than the wall clock.
The two halves of the discipline align cleanly: the cumulative clock $\tau_T:=\sum_t \eta_t^2 Q_t$, which at $\eta_t\equiv\eta^*$ equals $\eta^{*2}$ times the intrinsic time $V_T=\sum_t Q_t$ of \S\ref{sec:schedules}, is the integrated Safe-Bayesian clock, and the misspecification-driven prescription $\eta^*<1$ is the wall-clock face of a positive intrinsic-time growth rate.
Conversely, our identities recover the Safe-Bayesian guarantee at the choice $\eta_t=\eta^*$ as a corollary of the one-step decomposition, but generalize to time-varying schedules and to settings (online, adversarial, partial feedback) outside the i.i.d.\ misspecification frame Safe Bayes was originally designed for.
The structural distinction is that Safe Bayes asks ``which constant $\eta$ makes Bayes calibrated under misspecification?'' while the present framework asks ``what is the information cost of any chosen $\eta_t$, calibrated or not?''---and the former is recovered as the calibration-respecting subset of the latter.

\subsection{Comparator classes and the priors that encode them}

Another cluster of rows in Table~\ref{tab:tradeoffs} concerns the comparator.
Classical regret compares to the best single expert, but many later results compare to quantiles, mixtures, specialists, sleeping experts, or switching sequences.
Parameter-free hedging \cite{chaudhuri-freund-hsu}, Squint \cite{koolen2015squint}, AdaNormalHedge \cite{luo2015achieving}, and the new NormalHedge scaling-time result \cite{FreundEtAl2026} are central examples in the quantile or prior-mass direction.
Specialist and sleeping-expert methods allow abstention or confidence weighting \cite{FreundSchapireSingerWarmuth1997}, and the Bayesian reinterpretation of \cite{KoolenAdamskiyWarmuth2012} shows how naturally such structure fits with posterior averaging.
Tracking or switching regret compares to a sequence of experts with limited changes, in the tradition of fixed share \cite{herbster1998tracking}.
Internal and swap regret, especially important in game-theoretic applications, are classically obtained by reductions from external regret \cite{blum2007external}.

The present paper's viewpoint is that these are all changes of comparator class before they are changes of proof technique.
Once the comparator is encoded as a prior or a structured prior space, the same information balance applies.
This is explicit in the exact shifting-comparator theorem of Appendix~\ref{sec:shifting}, which writes dynamic regret exactly before any switch-count relaxation, and in the sleeping/quantile discussions of Appendix~\ref{sec:shifting} and Section~\ref{sec:simul-quantile}.
Conceptually, structured priors turn many ``new'' comparator notions into ordinary PAC-Bayesian regret on a richer space.

This also reframes the parameter-free versus adaptive row of Table~\ref{tab:tradeoffs}.
In the present notation, a budget $\Gamma$, a temperature schedule, or a pressure target $a_t$ is an explicit constraint or objective relative to which information is being measured, where the parameter-free literature treats such quantities as nuisance parameters to be removed.
We may still learn such quantities online.
Our dyadic controller in Section~\ref{sec:simul-quantile} is one example.
Single-copy simultaneous methods such as AdaNormalHedge, coin betting, and the NormalHedge scaling-time construction remain important comparison points \cite{luo2015achieving,orabona-pal,FreundEtAl2026}.
The difference is one of aim: this paper keeps the underlying information accounting exact, while those methods remove the need for manual tuning.

\subsection{Predictable structure and partial feedback in online optimization}

A separate row of Table~\ref{tab:tradeoffs} concerns predictable structure.
Optimistic and variation-based algorithms exploit the fact that part of the next loss vector is foreseeable from the past \cite{HazanKale2010,ChiangEtAl2012,SteinhardtLiang2014}.
The composite-loss reduction in Section~\ref{sec:exact} is a direct way of expressing that idea: predictable information is inserted through the side term $u_t$, and the exact intrinsic-time loss is then computed only on the residual sequence $c_t=\ell_t+u_t$.
In other words, optimism enters as one choice of what counts as already-explained information.

The full-versus-partial-feedback row is similar.
Bandit and feedback-graph algorithms introduce estimated losses, implicit exploration, and predictable bias corrections \cite{auer2002siam,neu2015neurips,alon2017journalb,RouyerEtAl2022FeedbackGraphs}; the standard references for the bandit setting collect this theory \cite{bubeck2012regret,lattimore2020bandit}.
From the present viewpoint, these are front-end changes to the score sequence; the back-end information identity is untouched.
Once we write down the estimated composite losses, the same one-step balance applies, with explicit martingale and bias terms tracking what the estimator added.
This is the reason the partial-information section can parallel the full-information one so closely.

\subsection{Local normalization and boosting coefficients}

The local pressure-target update belongs to a different related-work cluster from the retempered second-order schedules.
Its closest relatives are one-step coefficient choices in boosting and other multiplicative-normalization procedures.
Already in the original game-theoretic and boosting formulations of multiplicative weights \cite{freund1996game,FreundSchapire97}, the current round is summarized by a one-step normalizer, and choosing the coefficient amounts to deciding how aggressively to trade margin gain against normalization.
The textbook AdaBoost coefficient is exactly the minimizer of that one-step normalizer in the binary setting \cite{FreundSchapire97,SchapireFreund12BoostingBook}, and the confidence-rated generalization chooses each coefficient by directly minimizing that same normalization factor \cite{schapire1999improved}.
The drifting-games analysis of \cite{SteinhardtLiang2014} makes this normalization viewpoint especially explicit.

The local results in Section~\ref{sec:pressure} isolate the cumulative consequence of that viewpoint.
The schedule $\sum_i p_t(i)e^{-\eta_t(c_t(i)-a_t)}=1$ fixes a target level for the current normalized loss, and the cumulative identity records the resulting class-conditioned terminal mass exactly.
This is why the natural applications of the local update are different from those of the retempered update.
The retempered schedule is the right choice when we want anytime second-order control from a cumulative clock.
The local schedule is the right choice when we want to calibrate the present step itself, as in boosting or other procedures driven by one-step margin or loss targets.

This distinction is important conceptually.
Without it, the literature can look as if ``learning-rate tuning'' were one topic.
In fact, the rows of Table~\ref{tab:tradeoffs} split into two different questions: cumulative intrinsic-time equalization for the retempered update, and one-step pressure calibration for the local update.
Keeping these apart helps explain why the two adaptive updates in this paper share a one-step information cost but lead to different terminal objects.

\subsection{Individual-sequence prediction and pathwise PAC-Bayes}

The paper sits inside the individual-sequence traditions collected in \S\ref{sec:mirror-descent} below: prequential statistics, game-theoretic probability, predictive complexity, and universal coding.

Our use of PAC-Bayes language is therefore intentionally pathwise.
The quantity $\KL(\rho\|\pi)$ enters here as a comparator description cost, where the PAC-Bayes literature deploys it as an i.i.d.\ generalization penalty.
This is also why the row of Table~\ref{tab:tradeoffs} about algorithmic upper versus lower envelopes matters so much.

The concentration results fit neatly into that picture.
The equivalence theorems of \cite{RakhlinSridharan2017} show that deterministic regret inequalities and martingale tail bounds are often two views of the same statement.
Our sampled-expert and confidence-sequence theorems (Theorems~\ref{thm:sampled-pacbayes} and~\ref{thm:sampled-anytime}) follow the same pattern.
The deterministic part of the accounting is kept exact first; probabilistic control is added only to the explicit martingale term afterwards.
This is why the high-probability and anytime results in Section~\ref{sec:schedules} feel like genuine extensions of the decomposition.

\subsection{Scope of the scaling-time reduction}
\label{subsec:lower-bounds}

The most direct external comparison is with the scaling-time question of \cite{Freund2016} and its recent NormalHedge resolution in \cite{FreundEtAl2026}.
Our fixed-budget theorem gives an exact PAC-Bayes-style counterpart: for any chosen budget $\Gamma$, one run controls all comparators with $\KL(\rho\|\pi)\le \Gamma$, and the dyadic controller of Section~\ref{sec:simul-quantile} converts those fixed-budget guarantees into simultaneous quantile control with an explicit meta term.
This places the paper on the same conceptual map as the scaling-time literature, though simultaneous validity and simultaneous adaptivity are distinct.
A single run with fixed budget $\Gamma$ is valid for all simpler comparators, yet its learning rate is still anchored to that budget.
The strongest single-copy simultaneous adaptation results remain a separate benchmark \cite{luo2015achieving,FreundEtAl2026}.

The lower-bound story is unchanged by the composite-loss reduction.
Our setting strictly contains ordinary expert advice as the special case $u_t\equiv 0$, so every impossibility result for experts is inherited verbatim.
In particular, any future attempt to strengthen the scaling-time guarantees here must still confront the random-walk lower bound emphasized in \cite{FreundEtAl2026}:
\[
\operatorname{Reg}_{\epsilon}(T)
=
\Omega\!\left(\sqrt{\left(\sum_{t=1}^T \sigma_t^2\right)\log(1/\epsilon)}\right)
\]
and, when the intrinsic-time variable is the self-variance under the learner's own weights, the same discussion yields the sharper barrier
\[
\operatorname{Reg}_{\epsilon}(T)
=
\Omega\!\left(\sqrt{V_T\log(1/\epsilon)}\right)
\]
Mixed coincidence changes which loss sequence and which intrinsic-time functional are natural to analyze; it does not make the hard instances disappear.

\subsubsection{Exponential weights as a mirror-descent and FTRL step}
\label{sec:mirror-descent}

Exponential weights is simultaneously an entropic mirror-descent step and a follow-the-regularized-leader (FTRL) update, and the online-optimization literature made that equivalence central in both theory and practice \cite{arora2012theory,Zinkevich2003,DuchiHazanSinger2011,hazan2016introduction}.
Read as a Bayes update, the same step is an instance of the broader Bayesian learning rule, which casts natural-gradient posterior updates as mirror descent on exponential-family parameters through the underlying Legendre duality \cite{khan2023bayesian}.
Adaptive variants \citep{DuchiHazanSinger2011} make the same philosophical move as our predictable-rate schedules: the update magnitude is chosen from a cumulative geometry revealed by the realized data. 
In the present paper that geometry is the finite-temperature cumulant of the Bayes/exponential-weights update. 
We are therefore not proposing a new regularizer; rather, we isolate the information decomposition that any entropic exponential-weights update already satisfies before it is upper bounded. 
This is also why the same calculus survives the passage to continuous-action OCO in Appendix~\ref{sec:oco}.

Historically, the same update also entered the subject through early game-theoretic and boosting formulations \cite{freund1996game,freund1999games}. 
These ideas are rooted in the broader tradition of individual-sequence prediction, where algorithms are evaluated against the realized path with no assumed data-generating process.
The foundational work of \cite{CesaBianchiEtAl1997} established sharp worst-case bounds for prediction with expert advice, and \cite{HausslerKivinenWarmuth1998} extended the individual-sequence viewpoint to general loss functions with tight minimax characterizations. 
This has been largely unified from an information-theoretic perspective, connecting universal prediction to source coding and the minimum description length principle \citep{MerhavFeder1998}. 
A closely related probabilistic tradition is the prequential approach \cite{Dawid1984}, which argues that statistical procedures should be evaluated by their sequential predictive performance on each individual realization.
A long tradition of work in this spirit \citep{CesaBianchiLugosi1999,vovk2005algorithmic,ShaferVovk2019} has analyzed prediction of individual sequences through a minimax lens, characterizing the best achievable regret without distributional assumptions. 

Another nearby tradition is game-theoretic probability and defensive forecasting, which emphasizes sequential calibration, betting interpretations, and adversarial game structure over stochastic assumptions \citep{FosterVohra1998,vovk2005algorithmic,ShaferVovk2019}. 
The individual-sequence interpretation through comparator description length is also close to the predictive-complexity viewpoint developed around the aggregating algorithm \cite{Vovk1998AA,Kalnishkan2002}, and to the information-theoretic tradition of universal coding, where the difficulty of an individual sequence is measured by its code length relative to a model class \cite{Rissanen1984,Shtarkov1987,grunwald2007minimum}. The context-tree weighting method is a canonical sequential universal code of this kind, attaining the Rissanen bound over tree sources by a recursive mixture over models \cite{willems1995context}.
Cover's universal portfolio theory \cite{cover1991} applies the same philosophy to sequential investment, measuring individual-sequence wealth in place of expected growth.
A game-theoretic derivation in this vein obtains the exponential-weights update from the value of a repeated zero-sum game between learner and nature \cite{balsubramani2026concentration}.
Our use of PAC-Bayes language is different from these works, but the underlying aim is similar: keep the sequential game explicit in the absence of guiding structure of the sequence, and understand exactly how the learner's update transfers information from one round to the next.

\subsubsection{Status of the scaling-time question}

The 2016 question in \cite{Freund2016} asked for a second-order regret theorem parameterized by the \emph{variance across actions} and by the top $\varepsilon$ quantile of experts: $\reg_{\varepsilon}(T)\lesssim\sqrt{V_T\log(1/\varepsilon)}$ simultaneously for all $\varepsilon$, with $V_T$ a cumulative variance process of the instantaneous regrets under the learner's current weights.
That question is substantially resolved.
The NormalHedge scaling-time theorem of \cite{FreundEtAl2026} attains it up to an additive offset $t_0=\Theta(B^2\log N)$ contributing lower-order logarithmic terms, with a matching lower bound of order $\sqrt{V_T\log(1/\varepsilon)}$; their $V_T$ is computed under a curvature distribution induced by the NormalHedge potential, while ours is computed under the play distribution.

The present paper contributes a PAC-Bayes-flavored perspective on the same landscape.
For a uniform prior and $\rho_A$ uniform on a set of size $|A|=\varepsilon K$, Theorem~\ref{thm:pacbayes-second-order} gives the quantile identity
\begin{equation}\label{eq:quantile-known-V}
\sum_{t=1}^T \ip{p_t}{c_t} - \frac1{|A|}\sum_{i\in A} C_T(i)
=
D_T+B_T(\rho_A)+\sum_{t=1}^T \eta_t Q_t(c)\end{equation}
The second-order schedule at budget $\Gamma=\log(1/\varepsilon)$ therefore solves the fixed-budget analogue for ordinary exponential weights and the cumulant intrinsic time. Two edge terms intrinsic to the discrete path survive, the initialization cost $C^2\Gamma$ and the largest single-round jump $Q_*^T(c)$, and neither can be removed uniformly (Appendix~\ref{app:clipped-branch}).
A single run at budget $\Gamma$ is simultaneously \emph{valid} for every comparator with $\KL(\rho\|\pi)\le\Gamma$ but not simultaneously \emph{adaptive}, since $\eta_t$ stays anchored to $\Gamma$; Section~\ref{sec:simul-quantile} recovers simultaneity through a logarithmic-budget controller over dyadic $\Gamma$, at one extra meta-regret term of the same intrinsic-time form.
A genuine single-copy bound of order $\sqrt{V_T^\sharp(\rho)(\KL(\rho\|\pi)+1)}$ remains out of reach, and the lower-bound discussion in \cite{FreundEtAl2026} together with the contextual-bandit impossibility of \cite{MarinovZimmert2021} counsels against insisting on the classical $p$-variance verbatim alongside full parameter-freeness.

\section{Empirical diagnostics and numerical results}
\label{sec:experiments}

Because Theorem~\ref{thm:pacbayes-second-order} is an exact identity, the natural empirical object is the cumulative split of the regret into its three accounting terms. The same low terminal regret can be reached for very different reasons: a path may have been easy (small intrinsic time), the comparator may have been simple relative to the prior (small comparator-information term), or the schedule itself may have been forgiving (favorable temperature drift). The decomposition exposes which of these mechanisms each algorithm was actually using on a given run, and it computes on real data exactly as it does on synthetic families.

\subsection{Reading conventions and the prefix decomposition}
\label{sec:exp-conventions}\label{sec:exp-decomp}

Fix a horizon $T$ and a comparator $\rho\in\Delta([K])$. For the prior-retempered update of \S\ref{sec:adaptive}, define for every prefix $t\le T$
\begin{equation}
\label{eq:exp-prefix}
P_t(c) := \sum_{s=1}^t \eta_s Q_s(c)
\quad
D_t := \sum_{s=1}^{t-1}\bigl(A_s(\eta_s)-A_s(\eta_{s+1})\bigr)
\quad
B_t(\rho) := \frac{\KL(\rho\|\pi)-\KL(\rho\|q_{t,\eta_t})}{\eta_t}.
\end{equation}
The three terms are the prefix \emph{intrinsic-time loss}, \emph{temperature drift}, and \emph{terminal comparator information}. Theorem~\ref{thm:pacbayes-second-order} states that they sum exactly to the composite-loss prefix regret:
\begin{equation}
\label{eq:exp-prefix-decomp}
R_t^c(\rho)\;=\;P_t(c)\;+\;D_t\;+\;B_t(\rho).
\end{equation}
With side information $u_t$, the original-loss regret gains an additional predictable-mismatch term $M_t(\rho)=\sum_{s\le t}(\ip{\rho}{u_s}-\ip{p_s}{u_s})$. The normalized shares
\begin{equation}
\label{eq:exp-decomp-shares}
S_t(\rho):=P_t(c)+|D_t|+B_t(\rho),
\quad
\omega_t^{\mathrm{pay}}\!:=\!\frac{P_t(c)}{S_t(\rho)},
\;
\omega_t^{\mathrm{drift}}\!:=\!\frac{|D_t|}{S_t(\rho)},
\;
\omega_t^{\mathrm{info}}\!:=\!\frac{B_t(\rho)}{S_t(\rho)}
\end{equation}
lie in $[0,1]$ and sum to $1$ at every $t$. We plot them stacked on a primary axis and overlay the signed prefix regret $R_t^c(\rho)$ on a secondary axis to obtain the \emph{regret-decomposition plot} used throughout. We report seed-aggregated quantities with BCa $95\%$ bootstrap confidence intervals (1000 resamples); identity verifications report the maximum absolute residual over the grid. Figures use one fixed color system throughout: blue for the intrinsic-time loss, orange for the temperature drift, green for the comparator-information term, dark navy for realized trajectories, and, in comparison panels, blue for the paper's schedules against orange for external baselines. The local update of \S\ref{sec:pressure} satisfies a parallel identity (Corollary~\ref{cor:pressure-regret}) with $D_t$ replaced by a transport drift $D_t^{\mathrm{loc}}=\sum_{s\ge 2}\KL(\rho\|p_s)\,(1/\eta_s-1/\eta_{s-1})$ and the same numerical residual envelope; both forms underlie the figures below.

\subsection{Confirmation from the experiments}
\label{sec:exp-eight-properties}

The battery in Appendix~\ref{app:experiments} bears out the framework's claims across synthetic families and on real financial paths (\S\ref{sec:eval-olps-realdata}). The benchmark and real-data results establish the headline, and the identity checks stand behind them.

The schedules are competitive on benchmarks. Against current adaptive-online-learning baselines, each run from its own published update rule, the two schedules are competitive on the predictable families, improve on local-gap methods on cycling-adversarial paths with chaseable within-segment structure, and beat the best fixed expert on planted change-point sequences with no explicit tracking machinery (Figure~\ref{fig:sota-comparison}). That last advantage is a property of the construction and does not survive contact with real streams. On four of them---two power-demand series, a regional-demand benchmark, and an equity panel---no algorithm beats the best fixed forecaster, and the shifting comparator degenerates to the fixed one. The tracking gain, computable before any algorithm runs, is two hundred times smaller on the real streams than on the construction and accounts for the difference (\S\ref{sec:exp-results-shifting}). On the canonical universal-portfolio benchmark datasets read as open financial paths, matching the schedule's temperature to the information budget the identity exposes beats the worst-case $\log K$ wherever the worst case is mismatched and near-ties it where $\log K$ is already the right budget (\S\ref{sec:eval-olps-realdata}). Figure~\ref{fig:same-regret-realdata} shows the diagnostic itself on one such stream: three schedules finish with comparable regret while spending it in visibly different ways, and the identity closes along the whole path.

\begin{figure}[tbp]
\centering
\includegraphics[width=\textwidth]{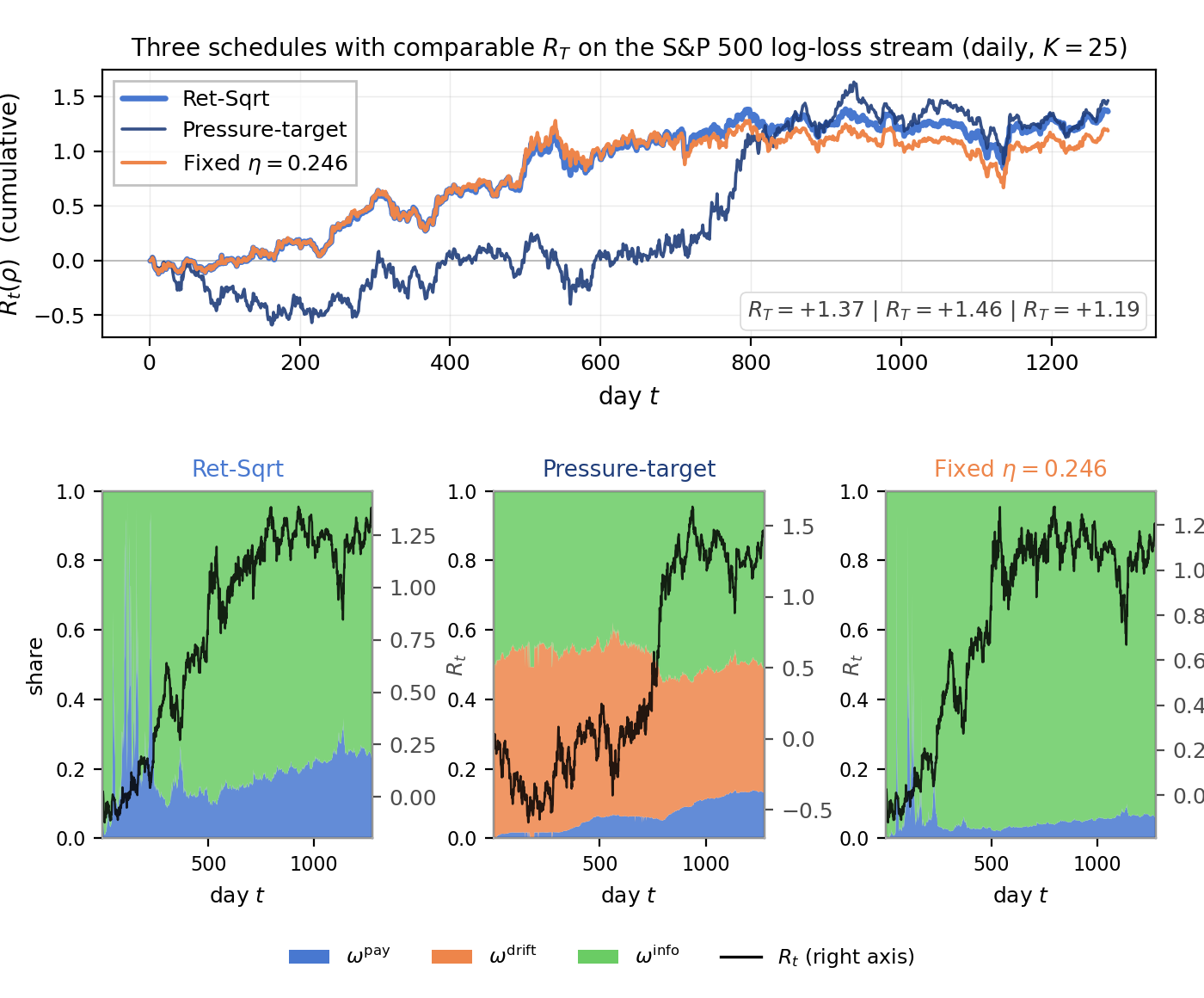}
\caption{\textbf{On a real financial stream the three schedules' decompositions differ far more than their regrets do.} The diagnostic of Figure~\ref{fig:same-regret-diff-decomp} runs on S\&P~500 constituents ($T=1275$ daily rounds, $K=25$; log-loss experts are single assets, comparator the best asset in hindsight). \textbf{Top:} cumulative regret for the intrinsic-time square-root schedule, the pressure-target line search, and a horizon-tuned fixed rate $\eta=0.246$; the three finish between $R_T=+1.19$ and $+1.46$. \textbf{Bottom:} the same three-share split. The fixed rate is comparator-information-dominated ($\omega^{\mathrm{info}}_T\approx 0.94$), the square-root schedule allocates a larger clock share ($\omega^{\mathrm{pay}}_T\approx 0.23$), and the pressure-target schedule runs a temperature-drift gain ($D_T<0$) that appears as the largest drift share ($\omega^{\mathrm{drift}}_T\approx 0.36$).}
\label{fig:same-regret-realdata}
\end{figure}

The identities hold, on real data as on synthetic. On every run the prefix decomposition closes, in both the prior-retempered form of Theorem~\ref{thm:pacbayes-second-order} and the local-update form of Corollary~\ref{cor:pressure-regret}, at every member $\rho_\alpha$ of the one-parameter comparator family \eqref{eq:exp-rho-alpha} including the best fixed expert; the same holds on the universal-portfolio return streams. The exactness survives the full range of alphabet sizes and horizons tested, the single-step Bellman-equalizer form, and the passage from a finite expert set to a continuous action set.
 The variance-driven schedules are a faithful leading-order relaxation of the exact intrinsic-time clock, agreeing with it to first order in the temperature.

The two-sided envelope is empirically tight. On cycling-adversarial paths the realized intrinsic-time loss approaches the lower envelope $2C\sqrt{\Gamma V_T(c)}$ from below as the intrinsic time accumulates, the finite-horizon gap accounted for by the envelope's bounded initialization tax (Appendix~\ref{app:clipped-branch}). On a single-spike construction the upper envelope of Theorem~\ref{thm:scaling-time} is approached as the maximal one-round loss term grows, confirming that the maximum-jump correction it includes is necessary.

The decomposition attributes difficulty correctly. The matrix-game side-information recipe moves mass off the intrinsic-time clock exactly when the opponent is genuinely predictable: as the forecast improves the intrinsic time falls smoothly toward zero, and the residual original-loss regret lands in the predictable-mismatch term.

Fast rates are visible at finite horizon. When the comparator-centered low-noise condition holds, the regret predicted by Theorem~\ref{thm:fixed-luckiness} and Corollary~\ref{cor:budget-luckiness} is essentially constant in the horizon, while the same algorithm reverts cleanly to the worst-case $\sqrt T$ rate on a misspecified family with no separation---one schedule tracking both regimes with no reconfiguration.

The budget is path-dependent where the range bound is uniform. The realized intrinsic time sits below the empirical-Bernstein clock on every family and far below the worst-case Hoeffding budget $T/8$, though the second margin is largely the range-against-variance gap that any variance-adaptive quantity also clears (\S\ref{sec:exp-results-hoeffding-slack}). The real financial streams are the extreme case: on the S\&P~500 stream the realized clock accumulates to $V_T=0.32$ against a worst-case budget near $159$. That path-dependent budget makes the full second-order story possible, and how far a run has moved along it decides whether the second-order envelope brackets it tightly or vacuously (Figure~\ref{fig:envelope-informative}).

Appendix~\ref{app:experiments} reports the battery in detail: the synthetic families and schedule panel, the side-information recipe, envelope tightness, the fast-rate and shifting-comparator checks, the head-to-head against current adaptive baselines, and the real-data evaluation on the open universal-portfolio suite.
The partial-information extensions run there on real substrates as well: supervised-to-bandit reductions and feedback graphs, drift-stream restart and stopping rules built on the ledger's computable terms, and the boosting, structured matrix-game, and sampling-estimator extensions.

\section{Discussion}

The exponential-weights update is an instance of a general identity about geometric mixtures and entropic duality---the dual-flat geometry of exponential families studied in information geometry \cite{amari2016informationb}.
The mixed-coincidence perspective determines exactly how several priors or side factors become composite losses.
The two branches of the variable-temperature update then compose the shared one-step information balance in structurally parallel ways: the retempered chain produces a global cumulant decomposition, the local chain produces a cumulative normalization identity.
These two different-looking objects share the same front end, the same per-round cumulant functional as their unit of account, and the same terminal comparator-information structure.

The cumulant theorem of Section~\ref{sec:adaptive} isolates three pieces of the retempered regret: the one-round centered cumulants, the temperature-change drift, and the terminal PAC-Bayes complexity term. 
The pressure-target identities of Section~\ref{sec:pressure} isolate a parallel three-piece structure for the local recursion: the one-round cumulant cost, the cumulative normalization (in place of drift), and the terminal partition (in place of retempered free energy).
Proposition~\ref{prop:tilted-var} shows that the intrinsic-time increment itself is an exact average of tilted variances in both cases, so both branches record the curvature revealed by the realized path. 
The familiar square-root intrinsic-time bound is then a corollary of the retempered chain, while the exact event-restricted partition identities are a corollary of the local chain.

The luckiness results sharpen that picture in the same-scale stochastic regime. 
For point comparators they give a short route to constant expected regret; for posterior comparators they isolate the exact additional stochastic condition needed by the comparator-centered argument. 
The sampled-expert theorem shows that the same chain persists pathwise once the martingale term is written explicitly.

The analysis is not specific to Hedge in the narrow sense. 
The Bayes-rule template of Section~\ref{sec:adaptive} only assumes that expertwise evidence is accumulated multiplicatively and then updated at a predictable temperature. 
That viewpoint keeps the cumulant costs visible, which variance-first Taylor-expansion analyses discard, and the same proof layer can be reused across a whole family of predictable schedules and Bayes-rule update designs.

Section~\ref{sec:applications-summary} and the appendices show that the same algebra is not confined to one theorem family: once the composite-loss reduction is in place, each of those settings becomes another instance of the same identity, and of the machinery around it.
For ordinary exponential weights with the cumulant intrinsic time, the fixed-budget PAC-Bayes counterpart of the scaling-time question is therefore solved here up to two unavoidable edge terms (Theorem~\ref{thm:scaling-time}), and the dyadic controller of Section~\ref{sec:simul-quantile} converts that fixed-budget guarantee into simultaneous quantile control with one explicit meta term. As discussed in Section~\ref{subsec:lower-bounds}, this is \emph{simultaneously valid} but not \emph{single-copy simultaneously adaptive}; the strongest single-copy results remain a separate benchmark \cite{luo2015achieving,FreundEtAl2026}. The composite-loss view, the pressure-target schedule, and the multi-prior side-factor reduction nonetheless give concrete routes to further algorithm design \cite{HazanKale2010,ChiangEtAl2012,SteinhardtLiang2014,PerezOrtizKoolen2022,RouyerEtAl2022FeedbackGraphs,FreundEtAl2026}, and the same accounting governs whatever new schedule, prior, or feedback model is plugged into the template.

The exact decompositions of Sections~\ref{sec:adaptive} and~\ref{sec:pressure} are telescoping identities in the finite-temperature log-normalizer, so they hold verbatim for unbounded losses whenever the one-step log-partition is finite; only the range control $Q_t\le\tfrac18$ of Proposition~\ref{prop:range} breaks, and there $Q_t$ may be unbounded per round. Proposition~\ref{prop:secondmoment} supplies the replacement: for scores bounded below the increment is controlled by the second moment of the score under the play in place of its range.
The second-order schedule---and with it scaling-time adaptation---therefore transfers to heavy-tailed and unbounded losses whenever that second moment is finite, with no clipping, no bounded-influence transform, and no predictable bias term. The envelope is then stated in the realized intrinsic time, and depends on $T$ only insofar as that clock is itself controlled. A bounded-influence transform of the raw score remains the route when even the second moment is unavailable: it keeps every identity exact, computes the intrinsic time from the robustified scores, and restores a finite per-round increment, at the cost of one predictable bias term. A calibration guarantee in the unclipped calibrated-expert regime does not follow, and cannot. The difficulty of \cite{neyman2023nologpooling} is an impossibility---no logarithmic pool is calibrated---orthogonal to schedule adaptivity, so the envelope holds while calibration is beyond the reach of any temperature schedule. For clipped experts the pooled log-loss is bounded and exp-concave, giving constant $\log N$ regret directly.

\subsection{Consequences of the exact accounting}
\label{subsec:new-connections}

\paragraph{Two-level composite-loss architecture for online meta-learning.}
Combining the multi-prior side-factor reduction (Section~\ref{sec:exact}) with the dyadic budget controller (Section~\ref{sec:simul-quantile}) suggests a natural two-level architecture. An inner algorithm runs retempered Hedge on composite losses $c_t=\ell_t+u_t$ where the side factor $s_t=\prod_w \pi_{t,w}^{\alpha_{t,w}}$ pools several candidate priors, and an outer controller updates the pooling weights $\alpha_t$ by another instance of the same identity. Because both levels inherit the same accounting, the end-to-end regret is an identity---a sum of two explicit intrinsic-time terms plus an explicit cross-entropy mismatch, without any unrealized constants (Theorem~\ref{thm:two-level}).

\paragraph{Pathwise regret as a variational dual.}
Theorem~\ref{thm:pacbayes-second-order} writes the pathwise regret as a sum of per-round equalizer values, where each round's payoff is the support function of a relative-entropy ball. This matches the strong-duality pattern of variational-form two-player games: the retempered chain is the \emph{primal} accounting of information spent on each round, and a corresponding min-max problem on relative-entropy-ball strategies is the \emph{dual} whose value equals that per-round loss term. This duality is made rigorous in \cite{balsubramani2026concentration}: Gibbs strategies are the equalizers of the per-round game, and the intrinsic-time increment is the per-round loss term in both the primal and dual framings. Repeating the same payoff across rounds turns the per-round statement into a repeated game whose value reproduces the accounting above.

\appendix
\appendixtitle

\section{The clipped branch of the second-order schedule}
\label{app:clipped-branch}

The second-order schedule \eqref{eq:budget-schedule-C} runs on two branches.
While the realized clock satisfies $V_{t-1}(c)<C^2\Gamma$ the square-root rule would ask for a temperature above one and the clip binds, so the schedule plays $\eta_t=1$; once $V_{t-1}(c)\ge C^2\Gamma$ the clip releases and the rate follows $C\sqrt{\Gamma/V_{t-1}(c)}$ for the rest of the path.
The clip is therefore an opening-rounds phenomenon on any path whose intrinsic time grows, and a permanent one on a path so easy that the clock never reaches $C^2\Gamma$.
Everything below concerns the two edge terms of Theorem~\ref{thm:scaling-time}; the identity itself, and the schedule's guarantees, are unaffected by which branch is active.

\paragraph{Necessity of both edge terms.}
For the upper side, the maximal jump $Q_*^T(c)$ is necessary: with $T=1$ and $Q_1(c)=q$, we have $\eta_1=1$, so $\sum_t \eta_t Q_t(c)=q$ and $V_T(c)=q$.
No bound of the form $\sum_t\eta_tQ_t(c)\le A+2C\sqrt{\Gamma V_T(c)}$ can hold for all $q$, because the right-hand side grows only like $\sqrt q$.
A correction of order the largest single jump is therefore indispensable.
For the lower side, the initialization tax $C^2\Gamma$ is forced by a second one-round example.
If $T=1$ and $Q_1(c)=C^2\Gamma$, then still $\eta_1=1$ and $\sum_t \eta_tQ_t(c)=C^2\Gamma$, whereas $2C\sqrt{\Gamma V_T(c)}=2C^2\Gamma$.
Therefore, any universal lower bound of the form $2C\sqrt{\Gamma V_T(c)}-A\le \sum_t \eta_tQ_t(c)$ must satisfy $A\ge C^2\Gamma$.
The tax is the price of entering the square-root regime from a clipped start, and it is paid once.

\paragraph{The envelope as a quadrature defect.}
Writing the schedule as a function of the clock, $\eta(v)=\min\{1,C\sqrt{\Gamma/v}\}$, exhibits the intrinsic-time loss as the left-endpoint Stieltjes sum $\sum_{t}\eta(V_{t-1})(V_t-V_{t-1})$ of a nonincreasing integrand against the increments of the intrinsic time.
That reading, and the resulting bracketing of the sum by its integral, are developed in \cite{balsubramani2026concentration}.
Since $\int_0^{v}C\sqrt{\Gamma/u}\,du=2C\sqrt{\Gamma v}$, the shared main term of \eqref{eq:budgeted-two-sided} is exactly the integral of the unclipped integrand, and
\[
\sum_{t=1}^T\eta_tQ_t(c)=2C\sqrt{\Gamma V_T(c)}-\mathcal D_T,
\qquad
\mathcal D_T:=2C\sqrt{\Gamma V_T(c)}-\sum_{t=1}^T\eta_tQ_t(c),
\]
is an identity in which $\mathcal D_T$ is the quadrature defect of the schedule.
On the square-root branch the integrand is decreasing, so the left-endpoint sum overshoots the integral by at most the largest increment $Q_*^T(c)$, which is the upper slack.
On the clipped branch the schedule plays $\eta=1<C\sqrt{\Gamma/V_{t-1}}$ and undershoots by at most $C^2\Gamma$, which is the lower slack.
The two branches together give $\mathcal D_T$ both signs, with endpoints $-Q_*^T(c)$ and $+C^2\Gamma$ each attained by a one-round path.
Hence $\sum_t\eta_tQ_t(c)$ is not any function of $V_T(c)$ alone, and \eqref{eq:budgeted-two-sided} is exactly the range of the quadrature defect around the exact integral.

\paragraph{Scale invariance off the clip.}
Scaling every score vector by $a>0$ scales $Q_t$ and $V_T$ by $a^2$ and rescales the second-order rate by $1/a$, so on the unclipped branch the schedule is scale invariant in the intrinsic-time sense recorded in Proposition~\ref{prop:regret-matching}.
The clipped rounds break that invariance, since $\eta_t=1$ is a fixed temperature in whatever units the scores arrive; they contribute the same $C^2\Gamma$ initialization tax at every scale.

\section{Local pressure-target schedules: further identities}
\label{app:pressure-details}

\subsection{Fixed-temperature and direct-adaptive variational forms}
\label{sec:pressure-local}

The fixed-temperature profile is the common-temperature specialization of the same update, while the direct-adaptive statements below are the native identities for the realized pressure-target path itself. 
No bounded-interval restriction on the temperatures is built into these formulas; we need only the relevant one-step log-partition to be finite at the chosen value.

\begin{proposition}[Fixed-temperature variational identity for the local recursion]
\label{prop:pressure-fixed-var}
Fix $\eta>0$ and run the local exponential-weights recursion $p_1^{(\eta)} := \pi$, $p_{t+1}^{(\eta)}(i)\propto p_t^{(\eta)}(i)e^{-\eta c_t(i)}$.
For each round define
\[
\psi_t^{(\eta)}(\lambda)
:=
\log \E_{i \sim p_t^{(\eta)}} \left[ \exp\!\left( \lambda\bigl(c_t(i) - \ip{p_t^{(\eta)}}{c_t}\bigr) \right) \right]
\]
Then for every posterior $\rho\in\Delta([K])$,
\begin{equation}
\label{eq:pressure-fixed-var}
R_T^{c,(\eta)}(\rho)
=
\sum_{t=1}^T \frac{\psi_t^{(\eta)}(-\eta)}{\eta}
+
\frac{\KL(\rho\|\pi)-\KL(\rho \| p_{T+1}^{(\eta)})}{\eta}
\end{equation}
where $R_T^{c,(\eta)}(\rho) := \sum_{t=1}^T \ip{p_t^{(\eta)}}{c_t}-\ip{\rho}{C_T}$.
\end{proposition}

\begin{corollary}[Self-correcting fixed-temperature envelopes]
\label{cor:pressure-self-correcting}
Under Proposition~\ref{prop:pressure-fixed-var}, let $v_t(\eta)$ be any quantities satisfying $v_t(\eta)\ge \psi_t^{(\eta)}(-\eta)/\eta$ for all $t$.
Then
\begin{equation}
\label{eq:pressure-self-correcting}
R_T^{c,(\eta)}(\rho)
=
\frac{\KL(\rho\|\pi)}{\eta}
+
\sum_{t=1}^T v_t(\eta)
-
\Delta_T^{\mathrm{loc}}(\rho,\eta)
\end{equation}
with the exact nonnegative gap
\[
\Delta_T^{\mathrm{loc}}(\rho,\eta)
:=
\frac{\KL(\rho \| p_{T+1}^{(\eta)})}{\eta}
+
\sum_{t=1}^T\left(
v_t(\eta)-\frac{\psi_t^{(\eta)}(-\eta)}{\eta}
\right)\ge 0
\]
\end{corollary}

This is the clean variational form of the fixed-temperature profile.
Any chosen surrogate for the centered cumulant, whether quadratic, self-normalized, or otherwise, differs from the truth by an explicit slack variable.
We may insert a bounded-range Hoeffding proxy when available.

\begin{proposition}[Cumulative normalization and direct weighted identity for the realized pressure-target path]
\label{prop:pressure-cumulative}
Under Proposition~\ref{prop:pressure-normalized}, define $G_T(i) := \sum_{t=1}^T \eta_t\bigl(c_t(i)-a_t\bigr)$.
Then
\begin{equation}
\label{eq:pressure-cumulative-normalization}
p_{T+1}(i)=p_1(i)e^{-G_T(i)}
\qquad\text{and}\qquad
\sum_{i=1}^K p_1(i)e^{-G_T(i)}=1
\end{equation}
Consequently, for every posterior $\rho\in\Delta([K])$,
\begin{equation}
\label{eq:pressure-cumulative-kl}
\sum_{t=1}^T \eta_t \bigl( \ip{\rho}{c_t} - a_t \bigr)
=
\KL(\rho \| p_{T+1})-\KL(\rho \| p_1)
\end{equation}
Equivalently, writing
\[
\mu_t := \ip{p_t}{c_t}
\qquad
\psi_t(\lambda)
:=
\log \E_{i\sim p_t} \left[ \exp\! \left( \lambda \bigl( c_t(i)-\mu_t \bigr) \right) \right]
\qquad
Q_t^{\mathrm{loc}}(c)
:=
\eta_t^{-2}\psi_t(-\eta_t)
=
\eta_t^{-1} \bigl( \mu_t - a_t \bigr)
\]
we have
\begin{equation}
\label{eq:pressure-cumulative-regret}
\sum_{t=1}^T \eta_t\bigl(\mu_t-\ip{\rho}{c_t}\bigr)
=
\sum_{t=1}^T \eta_t^2 Q_t^{\mathrm{loc}}(c)
+
\KL(\rho \| p_1)-\KL(\rho \| p_{T+1})
\end{equation}
\end{proposition}

Equation~\eqref{eq:pressure-cumulative-normalization} is the global effect of the local line search:
the schedule renormalizes the final exponentially weighted score exactly.
This is the local update's counterpart of the intrinsic-time loss in Theorem~\ref{thm:pacbayes-second-order}: both are exact information-tracking statements, but they track different cumulative objects.
The second-order schedule for the retempered update controls its cumulative intrinsic-time loss by a square-root law in a monotone budget,
whereas the pressure-target schedule for the local update enforces a chosen one-step target and therefore accumulates a cumulative normalization identity instead.
In both cases, the same exact-cumulant functional $Q_t(c)$ supplies the per-round information cost---evaluated at each recursion's own played distribution---and the terminal object measures how much comparator information remains in the final posterior.
This coincidence of the two updates' increments at a shared rate, and their controlled divergence under adaptive rates, are confirmed numerically in Appendix~\ref{sec:eval-intrinsic-time-clock}: the two updates' realized increments agree at a common temperature and separate monotonically with the total-variation gap between their played distributions.

The cumulative normalization of Proposition~\ref{prop:pressure-cumulative} is a set of exact fluctuation relations in the sense of stochastic thermodynamics \cite{jarzynski1997,crooks1999entropy,seifert2012stochastic}.
The setting there is a system prepared in equilibrium and then driven by a time-dependent protocol, each realization of the drive doing a random amount of work.
Such relations hold exactly at every driving speed, however far from equilibrium the protocol takes the system.
The dictionary to the present setting reads the expert index $i$ as the realization and the temperature schedule $(\eta_t)$ as the protocol.
The prior $\pi$ is then the initial equilibrium state, the terminal posterior $p_{T+1}$ the state the protocol leaves behind, an expert's accumulated tempered loss its work, and a log-normalizer a free energy.
With the expert-indexed \emph{work} $W(i):=\sum_{t=1}^T\eta_t c_t(i)$, the round normalizers $Z_t:=\sum_i p_t(i)e^{-\eta_t c_t(i)}$, and the free-energy difference $\Delta F:=-\log\sum_i\pi(i)e^{-W(i)}$, we have the Jarzynski identity $\prod_{t=1}^T Z_t=\E_{i\sim\pi}[e^{-W(i)}]=e^{-\Delta F}$ (with $\Delta F=\sum_t\eta_t a_t$ under the pressure-target schedule), the Crooks detailed relation $\pi(i)/p_{T+1}(i)=e^{W(i)-\Delta F}$, and the integral fluctuation theorem $\E_{i\sim\pi}[e^{-\Sigma}]=1$ with $\E_{i\sim\pi}[\Sigma]=\KL(\pi\|p_{T+1})\ge0$ for the entropy production $\Sigma(i):=W(i)-\Delta F$. Each follows immediately from $p_{T+1}(i)=\pi(i)e^{-W(i)+\Delta F}$.
The second law of thermodynamics is the integral fluctuation theorem read through Jensen's inequality: mean entropy production is nonnegative because it equals a relative entropy, and here that relative entropy is the exact separation $\KL(\pi\|p_{T+1})$ between the prior the run started from and the posterior it ended at.

\subsection{Terminal-mass identities and event-conditioned weighted regret}
\label{sec:pressure-event}

The local normalization also has a finite-scale coincidence interpretation.
Fix a round $t$ and set
\[
Z_t(\eta):=\sum_{i=1}^K p_t(i)e^{-\eta c_t(i)},
\qquad
\mu_t:=\ip{p_t}{c_t},
\qquad
r_t(i):=\frac{p_t(i)e^{-c_t(i)}}{Z_t(1)}
\]
Thus $r_t$ is the unit-temperature posterior formed from the current weights and the current score vector, whether or not the algorithm actually plays temperature $1$.

\begin{proposition}[One-round coincidence factorization of the local normalizer]
\label{prop:pressure-coincidence}
For every $\eta>0$,
\begin{equation}
\label{eq:pressure-coinc-factor}
Z_t(\eta)=Z_t(1)^\eta \cC_t(\eta),
\qquad
\cC_t(\eta):=\sum_{i=1}^K p_t(i)^{1-\eta}r_t(i)^\eta
\end{equation}
Consequently, with $\delta_t^{\mathrm{loc}}(\eta):=\mu_t+\eta^{-1}\log Z_t(\eta)$,
we have
\begin{equation}
\label{eq:pressure-coinc-gap}
\delta_t^{\mathrm{loc}}(\eta)
=
\delta_t^{\mathrm{loc}}(1)+\eta^{-1}\log \cC_t(\eta)
=
\delta_t^{\mathrm{loc}}(1)-\frac{1-\eta}{\eta}D_\eta(r_t\|p_t)
\end{equation}
where $D_\eta(r\|p):=(\eta-1)^{-1}\log\bigl(\sum_{i=1}^K r(i)^\eta p(i)^{1-\eta}\bigr)$ is the order-$\eta$ R\'enyi divergence.
For $0<\eta<1$, $\cC_t(\eta)\in[0,1]$ is the Hellinger/R\'enyi overlap between $p_t$ and $r_t$.
\end{proposition}

Thus the local finite-temperature loss term is governed by a coincidence factor in the one-step normalizer; the retempered free energy of Section~\ref{sec:adaptive} does not enter.
At any fixed round, changing the temperature changes the one-step loss term exactly through the overlap term $\cC_t(\eta)$, or equivalently through the R\'enyi divergence between the current weights and the unit-temperature posterior.

\begin{proposition}[Terminal mass and event-conditioned weighted regret]
\label{prop:pressure-event}
Under the notation of Proposition~\ref{prop:pressure-cumulative}, let $E\subseteq[K]$ be nonempty and define
\[
\Delta(E):=\{\rho\in\Delta([K]):\rho(E)=1\},
\qquad
\mathcal Z_T(E):=\sum_{i\in E} p_1(i)e^{-G_T(i)}=p_{T+1}(E)
\]
If $\mathcal Z_T(E)>0$, let $p_{T+1}^E$ denote the restriction of $p_{T+1}$ to $E$, namely $p_{T+1}^E(i):=p_{T+1}(i)/p_{T+1}(E)$ for $i\in E$.
Then for every posterior $\rho\in\Delta(E)$,
\begin{equation}
\label{eq:pressure-event-kl}
-\log \mathcal Z_T(E)+\KL(\rho\|p_{T+1}^E)
=
\sum_{t=1}^T \eta_t\bigl(\ip{\rho}{c_t}-a_t\bigr)+\KL(\rho\|p_1)
\end{equation}
Equivalently,
\begin{equation}
\label{eq:pressure-event-pointwise}
\sum_{t=1}^T \eta_t\bigl(\mu_t-\ip{\rho}{c_t}\bigr)-\KL(\rho\|p_1)
=
\sum_{t=1}^T \eta_t^2 Q_t^{\mathrm{loc}}(c)
+
\log p_{T+1}(E)
-
\KL(\rho\|p_{T+1}^E)
\end{equation}
Consequently,
\begin{equation}
\label{eq:pressure-event-var}
-\log \mathcal Z_T(E)
=
\min_{\rho\in\Delta(E)}
\left\{
\sum_{t=1}^T \eta_t\bigl(\ip{\rho}{c_t}-a_t\bigr)+\KL(\rho\|p_1)
\right\}
\end{equation}
and
\begin{equation}
\label{eq:pressure-event-sup}
\sup_{\rho\in\Delta(E)}
\left\{
\sum_{t=1}^T \eta_t\bigl(\mu_t-\ip{\rho}{c_t}\bigr)-\KL(\rho\|p_1)
\right\}
=
\sum_{t=1}^T \eta_t^2 Q_t^{\mathrm{loc}}(c)
+
\log p_{T+1}(E)
\end{equation}
\end{proposition}

When $E=[K]$, we have $\log p_{T+1}(E)=0$, so \eqref{eq:pressure-event-sup} reduces to Proposition~\ref{prop:pressure-cumulative}.
For smaller classes, the extra term $\log p_{T+1}(E)\le 0$ is an exact class correction: any class that ends with very small terminal mass under the final posterior incurs the cost of that scarcity.
Typical choices of $E$ include an $\epsilon$-quantile set of experts, a sparsity pattern, a model family, or an agreement set singled out by side information.
If we also want an ordinary unweighted regret bound, the same correction survives summation by parts: for every $\rho\in\Delta(E)$,
\[
R_T^c(\rho)
\le
\sum_{t=1}^T \bigl(\mu_t-a_t\bigr)
+\frac{\KL(\rho\|p_1)}{\eta_1}
+\frac{\log p_{T+1}(E)}{\eta_T}
+\sum_{t=2}^T \KL(\rho\|p_t)\!\left(\frac{1}{\eta_t}-\frac{1}{\eta_{t-1}}\right)
\]
by combining Corollary~\ref{cor:pressure-regret} with $\KL(\rho\|p_{T+1})=\KL(\rho\|p_{T+1}^E)-\log p_{T+1}(E)$.
Thus the terminal-mass term is genuinely part of the regret accounting; the only extra ingredient needed to pass from weighted to ordinary regret is the reciprocal-temperature variation term.

If the composite scores come from the mixed-coincidence reduction of Section~\ref{sec:exact}, namely $c_t(i)=\ell_t(i)-\eta_t^{-1}\log s_t(i)$ with $s_t$ of the form \eqref{eq:side-decomp}, then
\[
p_{T+1}(E)
=
\exp\!\Bigl(\sum_{t=1}^T \eta_t a_t\Bigr)
\sum_{i\in E} p_1(i)\exp\!\Bigl(-\sum_{t=1}^T \eta_t \ell_t(i)\Bigr)
\prod_{t=1}^T\prod_{w=1}^{J_t}\mu_{t,w}(i)^{\alpha_{t,w}}
\]
Thus the terminal mass of $E$ is, up to the explicit scalar factor $\exp(\sum_{t=1}^T \eta_t a_t)$, a restricted mixed-coincidence quantity.
The correction $\log p_{T+1}(E)$ in \eqref{eq:pressure-event-pointwise} therefore measures how much cumulative geometric agreement remains inside the class $E$.

\subsection{Centered scores and quadratic envelopes}

In many settings the observed scores are already centered excess losses or gains, with $\ip{p_t}{g_t}=0$. The zero free-energy target then collapses to the zero-temperature limit, so the natural control variable becomes a positive gain target; the following specialization records the resulting exact regret identity and its quadratic envelope.

\begin{corollary}[Centered-score specialization and exact adaptive regret theorem]
\label{cor:pressure-centered}
Assume that on each round one observes a score vector $g_t\in\R^K$ satisfying $\ip{p_t}{g_t}=0$.
If $g_t$ is nonconstant, choose a target $b_t$ with $0<b_t<\max_i g_t(i)$ and let $\eta_t>0$ be the unique solution of
\begin{equation}
\label{eq:pressure-centered-root}
\log\!\left(\sum_{i=1}^K p_t(i)e^{\eta_t g_t(i)}\right)=\eta_t b_t
\end{equation}
If $g_t\equiv 0$, set $b_t:=0$ and choose any $\eta_t>0$. Update $p_{t+1}(i)\propto p_t(i)e^{\eta_t g_t(i)}$.
Then for every posterior $\rho\in\Delta([K])$,
\begin{equation}
\label{eq:pressure-centered-potential}
\sum_{i=1}^K p_1(i)\exp\!\left(\sum_{t=1}^T \eta_t\bigl(g_t(i)-b_t\bigr)\right)=1
\end{equation}
and
\begin{equation}
\label{eq:pressure-centered-regret}
\sum_{t=1}^T \eta_t\bigl(\ip{\rho}{g_t}-b_t\bigr)
=
\KL(\rho \| p_1)-\KL(\rho \| p_{T+1})
\end{equation}
Hence
\begin{equation}
\label{eq:pressure-centered-upper}
\sum_{t=1}^T \eta_t\,\ip{\rho}{g_t}
\le
\KL(\rho \| p_1)+\sum_{t=1}^T \eta_t b_t
\end{equation}
If moreover $g_t(i)=\ip{p_t}{\ell_t}-\ell_t(i)$,
then $\ip{\rho}{g_t}=\ip{p_t}{\ell_t}-\ip{\rho}{\ell_t}$, so \eqref{eq:pressure-centered-regret} is an identity for the $\eta$-weighted regret $\sum_t\eta_t(\ip{p_t}{\ell_t}-\ip{\rho}{\ell_t})$ of the realized line-search schedule; Corollary~\ref{cor:pressure-regret} recovers the ordinary unweighted regret by summation by parts.
\end{corollary}

Combining Corollary~\ref{cor:pressure-centered} with Proposition~\ref{prop:pressure-event} and taking $c_t=-g_t$, $a_t=-b_t$ gives, for every nonempty $E\subseteq[K]$ with $p_{T+1}(E)>0$,
\[
\sup_{\rho\in\Delta(E)}
\left\{
\sum_{t=1}^T \eta_t \ip{\rho}{g_t} - \KL(\rho \| p_1)
\right\}
=
\sum_{t=1}^T \eta_t b_t+\log p_{T+1}(E)
\]
Thus the exact cost of forcing the comparator to lie in a gain-bearing class $E$ is again the negative logarithm of that class's final posterior mass.

\begin{corollary}[Exact quadratic-penalty representation of the local line search]
\label{cor:pressure-quadratic}
Under Corollary~\ref{cor:pressure-centered}, define
\[
\kappa_t
:=
\frac{b_t}{\eta_t}
=
\frac{1}{\eta_t^2}\log\!\left(\sum_{i=1}^K p_t(i)e^{\eta_t g_t(i)}\right)
\]
Then the target equation is equivalently
\begin{equation}
\label{eq:pressure-quadratic-root}
\sum_{i=1}^K p_t(i)\exp\!\left(\eta_t g_t(i)-\kappa_t\eta_t^2\right)=1
\end{equation}
and for every posterior $\rho\in\Delta([K])$,
\begin{equation}
\label{eq:pressure-quadratic-regret}
\sum_{t=1}^T \eta_t\,\ip{\rho}{g_t}
=
\sum_{t=1}^T \kappa_t\eta_t^2
+
\KL(\rho \| p_1)-\KL(\rho \| p_{T+1})
\end{equation}
Consequently, if $(\bar\kappa_t)$ is any predictable sequence with $\bar\kappa_t\ge \kappa_t$ for all $t$, then
\begin{equation}
\label{eq:pressure-quadratic-upper}
\sum_{t=1}^T \eta_t\,\ip{\rho}{g_t}
\le
\KL(\rho \| p_1)+\sum_{t=1}^T \bar\kappa_t\eta_t^2
\end{equation}
\end{corollary}

This is the clean repair of the conventional quadratic-penalty proof strategy.
The coefficient $\kappa_t$ is not an externally imposed surrogate but the realized finite-temperature cumulant coefficient selected by the line search itself.
Replacing it by any predictable upper envelope $\bar\kappa_t$ converts the exact identity into the familiar supermartingale-style inequality, but that surrogate is then a proof device; the definition of the algorithm is unchanged.

\paragraph{Small-temperature relation to second order.}
When the one-step cgf of $g_t$ is three times differentiable at the origin,
\[
\log \E_{i\sim p_t} e^{\eta g_t(i)}
=
\frac{\eta^2}{2}\Var_{i\sim p_t}(g_t(i)) + O(\eta^3)
\qquad
(\eta\downarrow 0)
\]
so the exact coefficient in Corollary~\ref{cor:pressure-quadratic} satisfies
\[
\kappa_t
=
\frac{1}{2}\Var_{i\sim p_t}(g_t(i)) + O(\eta_t),
\qquad
b_t
=
\frac{\eta_t}{2}\Var_{i\sim p_t}(g_t(i)) + O(\eta_t^2)
\]
Thus the usual quadratic-variation penalties are the small-temperature shadow of the exact pressure-target line search.

\paragraph{Choice of the target.} 
The target $a_t$ is problem dependent.
For general score vectors with positive current mean, the unit-potential rule $a_t=0$ is scale invariant and yields $Z_t=1$ exactly.
For centered excess scores with $\ip{p_t}{g_t}=0$, use instead the positive gain target $b_t=-a_t$ of Corollary~\ref{cor:pressure-centered}.
A fixed-fraction target $a_t=\lambda_t\ip{p_t}{c_t}$ with $0<\lambda_t<1$ keeps a chosen fraction of the current mean score at the mix-loss level.
In optimistic or model-based settings we can instead target a predictable benchmark, so that $\ip{p_t}{c_t}-a_t$ measures only the residual part that the learner failed to explain away at round $t$.
The information-targeted version of Section~\ref{sec:offdiag-combinations} can also be used locally: for a candidate $\eta$, let $p_{t+1}^{(\eta)}(i)\propto p_t(i)e^{-\eta c_t(i)}$ and set $I_t^{\rm loc}(\eta):=\KL(p_{t+1}^{(\eta)}\|p_t)=\eta(m_t(\eta)-\ip{p_{t+1}^{(\eta)}}{c_t})$; choosing a quota $\beta_t$ and solving $I_t^{\rm loc}(\eta_t)=\beta_t$ induces the pressure target $a_t=m_t(\eta_t)$.

\paragraph{Nonmonotone and other nonstandard rate schedules.}
The calibrated temperature in Proposition~\ref{prop:pressure-normalized} need not move monotonically in time.
Even at fixed weights $p_t=(1/2,1/2)$, the positive solutions of \eqref{eq:pressure-target-root} for the zero target $a_t=0$ and the score vectors $(-1,2)$,
$(-1,4)$, and $(-1,3/2)$ are approximately $0.48$, $0.66$, and $0.33$, respectively.
Thus an exact pressure-target schedule can increase and later decrease in response to the realized score geometry.
This is fully compatible with Theorem~\ref{thm:cumulant-chain}: the exact chain holds for every predictable positive schedule, while monotonicity matters only when we want to drop the reciprocal-temperature variation term and obtain a one-sided envelope.
On realized runs the drift takes both signs, which is the practical form of the same point: the identity stays exact while $D_T$ changes sign, so no one-signed-defect assumption is available to be made or needed.
Unlike the second-order schedule, the pressure-target family is defined implicitly and is designed to hit a local target; it maintains no monotone global budget. The cost of that flexibility is that we generally have to keep the reciprocal-temperature variation term in \eqref{eq:pressure-cum-regret-abel} unless additional monotonicity is available.

\subsection{The finite-scale regret spectrum}
\label{app:regret-spectrum}
For a fixed horizon $T$ and temperature $\eta$, the terminal weights $\mu_i:=p_{T+1}(i)$ define a simple R\'enyi profile $D_q(T):=-H_q(\mu)/\log T$, where $H_q$ is the order-$q$ R\'enyi entropy. Interpreted as a \emph{regret spectrum}, $D_0(T)$ measures the effective number of experts still in play, $D_1(T)$ is entropy per unit log-time, and curvature of $q\mapsto D_q(T)$ records how heterogeneous the loss landscape has become. Flat spectra correspond to near-indistinguishable experts; strongly curved spectra indicate a genuinely multiscale hierarchy of good experts.

For the uniform prior the shape of this spectrum has an exact closed form as a divided difference of the terminal free energy.

\begin{proposition}[Regret spectrum as a terminal free-energy divided difference]
\label{prop:regret-spectrum-freeenergy}
Take the uniform prior $\pi\equiv1/K$, fix a horizon $T$ and a terminal temperature $\eta>0$, and let $\mu:=p_{T+1}$ be the terminal weights $\mu_i\propto e^{-\eta C_T(i)}$. With $A_T(\eta):=-\eta^{-1}\log\sum_i\pi(i)e^{-\eta C_T(i)}$ and $\Lambda_T(\beta):=\log\sum_i e^{-\beta C_T(i)}$, the order-$q$ R\'enyi entropy is
\[
H_q(\mu)=\frac{\Lambda_T(q\eta)-q\,\Lambda_T(\eta)}{1-q}
=\log K+\frac{q\eta\bigl(A_T(\eta)-A_T(q\eta)\bigr)}{1-q},
\]
so the regret spectrum is the free-energy divided difference
\[
D_q(T)=-\frac{H_q(\mu)}{\log T}=\frac{1}{\log T}\left(q\eta\,\frac{A_T(\eta)-A_T(q\eta)}{q-1}-\log K\right).
\]
Since $\Lambda_T''(\beta)=\Var_{\nu_\beta}(C_T)$ with $\nu_\beta(i)\propto e^{-\beta C_T(i)}$, the curvature of $q\mapsto D_q(T)$ measures the cross-sectional dispersion (heat capacity) of the terminal cumulative losses $C_T(\cdot)$: the spectrum is flat iff all $C_T(i)$ coincide and strongly curved iff the $C_T(i)$ are heterogeneous.
\end{proposition}

The spectrum $D_q(T)$ depends on the terminal cumulative losses $C_T(\cdot)$ alone, whereas the intrinsic time $V_T(c)=\sum_t Q_t(c)$ is a pathwise accumulation of per-round tilted variances under the played $p_t$. The two are generically independent: paths reaching the same terminal $C_T$ share every $D_q(T)$ but may have arbitrarily different $V_T$ (a smooth path with tiny per-round variance versus a jagged path oscillating to the same endpoint). The exact relation is therefore between $D_q(T)$ and the terminal free energy $A_T$, with no functional dependence on $V_T$.

\section{Extensions via side information and compensators}
\label{app:side-info}

\subsection{Compensator and sufficient-statistic choices}
\label{sec:compensators}

With the original-loss identity in hand, we can now modify the composite-loss reduction in two structurally different ways. We may change the predictable side factors $s_t$, which alters the composite losses by an exact algebraic correction. Or we may keep the same comparator notion but replace the observed statistic by a transformed score before exponentiating it. The examples below use both viewpoints.

\paragraph{Vanilla Hedge.} Taking $s_t\equiv 1$ yields $u_t\equiv 0$. Corollary~\ref{cor:fixed-rate} reduces to the standard entropic regret bound, while Theorem~\ref{thm:scaling-time} gives a rigorous anytime second-order bound for ordinary Hedge in terms of $V_T(\ell)$.

\paragraph{Soft specialists and confidence-rated experts.} Let $\beta_t(i)\in(0,1]$ encode the confidence or availability of expert $i$ on round $t$, and set $s_t(i)=\beta_t(i)^{\lambda_t}$. 
Then $u_t(i) = -\frac{\lambda_t}{\eta_t}\log \beta_t(i)$, so experts with low confidence receive an additive penalty. 
This is a soft version of the specialist or sleeping-expert setup in which zeros are replaced by small positive masses to keep all information terms finite.

\paragraph{Predictable bonuses and variance penalties.} 
Suppose a predictable bonus $b_t(i)$ and a variance proxy $v_t(i)\ge 0$ are available. 
Taking $s_t(i) = \exp\left(\eta_t b_t(i)-\gamma_t v_t(i)\right)$ gives $c_t(i)=\ell_t(i)-b_t(i)+(\gamma_t/\eta_t)v_t(i)$.
Thus the same identity simultaneously handles optimism, pessimism, and second-order damping: the algorithm is exponential weights on the composite losses $c_t$.

\begin{corollary}[Optimistic Hedge from predictable side information]\label{cor:optimistic-hedge}
Let $m_t\in\R^K$ be any predictable vector and choose the side information $u_t(i):=-m_t(i)$, equivalently $s_t(i)=\exp(\eta_t m_t(i))$. 
Then $c_t(i)=\ell_t(i)-m_t(i)$, and for every posterior $\rho\in\Delta([K])$,
\[
\sum_{t=1}^T \ip{p_t-\rho}{\ell_t}
=
D_T+B_T(\rho)+\sum_{t=1}^T \eta_t Q_t(\ell-m)
+\sum_{t=1}^T \ip{p_t-\rho}{m_t}\]
where $Q_t (\ell-m)$ is the intrinsic increment computed from the residual vector $c_t=\ell_t-m_t$.
If $\operatorname{range}(\ell_t-m_t)\le \varepsilon_t$ for every $t$, then
\[
\sum_{t=1}^T \ip{p_t-\rho}{\ell_t}
\le
D_T+B_T(\rho)+\frac18\sum_{t=1}^T \eta_t\varepsilon_t^2
+\sum_{t=1}^T \ip{p_t-\rho}{m_t}\]
If $\eta_t\equiv\eta$, this simplifies to
\[
\sum_{t=1}^T \ip{p_t-\rho}{\ell_t}
\le
\eta^{-1}\KL(\rho\|\pi)+\frac{\eta}{8}\sum_{t=1}^T \varepsilon_t^2
+\sum_{t=1}^T \ip{p_t-\rho}{m_t}\]
Under the schedule \eqref{eq:budget-schedule-C} and the complexity condition $\KL(\rho\|\pi)\le \Gamma$, we also have
\[
\sum_{t=1}^T \ip{p_t-\rho}{\ell_t}
\le
\Gamma+Q_*^T(\ell-m)+(2C+C^{-1})\sqrt{\Gamma V_T(\ell-m)}
+\sum_{t=1}^T \ip{p_t-\rho}{m_t}\]
\end{corollary}

Both displays follow from Corollary~\ref{cor:actual-losses} and its budgeted form, evaluated at the side information $u_t=-m_t$.

This is the standard optimistic-Hedge mechanism, now written as a direct corollary of the side-information calculus. In repeated games or model-based play, we can take $m_t$ to be the loss vector predicted from a forecast of the opponent's move. Then only the forecast residuals $\ell_t-m_t$ contribute to the intrinsic-time cost, while the term $\sum_t\ip{p_t-\rho}{m_t}$ records how the comparator and learner evaluate the predictable baseline.

\paragraph{Perfect lookahead as a limiting case.}
In the degenerate case $m_t(i)=\ell_t(i)$, the residual losses vanish: $c_t\equiv 0$. Accordingly $Q_t(c)=0$ and $V_T(c)=0$, so the online part of the theorem disappears. Any remaining term in the original-loss decomposition is purely the predictable planning term $\sum_t\ip{p_t-\rho}{\ell_t}$. This is the right intuition: if the full loss vector is known before acting, there is no irreducible online uncertainty left and we should not expect to pay adversarial regret.

\paragraph{Slowly drifting losses.} A particularly simple choice is $m_t(i)=\ell_{t-1}(i)$ (with any convenient initialization for $t=1$), which yields $c_t(i)=\ell_t(i)-\ell_{t-1}(i)$.
If consecutive loss vectors satisfy $|\ell_t(i)-\ell_{t-1}(i)|\le \delta_t$ for all $i$, then again $Q_t(c)\le \delta_t^2/2$, so $V_T(c)\le \frac{1}{2} \sum_{t=1}^T \delta_t^2$.
The intrinsic-time bound therefore scales with squared temporal drift, which matches the intuition behind variation-based and gradual-variation regret bounds \cite{HazanKale2010,ChiangEtAl2012,SteinhardtLiang2014}.

\paragraph{Stochastic centering and compensators.} If $m_t(i)=\E[\ell_t(i)\mid \mathcal{F}_{t-1}]$ or a model-based conditional mean, then $c_t(i)=\ell_t(i)-m_t(i)$ is the martingale-difference residual. The identities therefore separate the predictable mean component $\sum_t \ip{p_t-\rho}{m_t}$ from the noise component $R_T^c(\rho)$. In favorable stochastic regimes the latter is controlled by the conditional noise scale, while the former is the genuine predictable excess-risk process. This is the martingale-compensator interpretation of the same coincidence calculus.

\paragraph{Positive-part sufficient statistics and sparsity.} Sparsity is orthogonal to the Bayesian prior calculus developed here.
The clean way to encode it is instead to change the observed sufficient statistic.
For a composite-loss sequence $c_t$, let $e_t(i) := \ip{p_t}{c_t}-c_t(i)$ be the one-round excess-loss vector, and write $e_t^+(i):=(e_t(i))_+$.
Then for every posterior $\rho$,
$R_T^c(\rho)=\sum_{t=1}^T \ip{\rho}{e_t}\le \sum_{t=1}^T \ip{\rho}{e_t^+}$.
Thus the positive part already upper-bounds the one-sided regret: only experts that beat the learner on round $t$ contribute.
Running the algorithm on $e_t^+$ (or, more generally, on shortfalls of the form $(e_t-b_t)_+$ or $(c_t-a_t)_+$) therefore gives a sparsity-seeking variant that deliberately forgets coordinates already on the favorable side of the comparison.
This construction is logically independent of the compensator choices above: it acts on the realized statistic after centering, leaving the prior untouched.
This is a conservative reduction: it keeps only the part of the observation relevant to one-sided excess loss, and it is no longer an exact re-expression of the original regret process.

At the level of the crude intrinsic-time envelopes used throughout the paper, truncation can only help.
Since the map $x\mapsto x_+$ is $1$-Lipschitz, $\operatorname{range}(e_t^+)\le \operatorname{range}(e_t)$,
and therefore every range-based bound such as $Q_t(e^+)\le \operatorname{range}(e_t^+)^2/8$ is no larger after truncation.

\paragraph{Bounded-influence score transforms.}
The same sufficient-statistic viewpoint also covers robust influence functions. Instead of exponentiating a raw residual $z_t$, we may exponentiate a transformed score $\psi(z_t)$ chosen to clip or smooth rare but very large observations. Positive-part truncation for sparsity and Catoni-type bounded-influence transforms are two instances of the same move: they change the sufficient statistic fed to the Bayes update and leave the exact information identity intact. The explicit bounded-influence formulas are especially useful for stabilizing importance-weighted estimates under partial feedback; conceptually the move belongs with the present discussion of statistic design.

\paragraph{Several side priors.} If $s_t(i)=\prod_w \pi_{t,w}(i)^{\alpha_{t,w}}$, then $u_t(i)=-\frac{1}{\eta_t}\sum_w \alpha_{t,w}\log \pi_{t,w}(i)$.
The regret corrections in \eqref{eq:fixed-rate-ell-reg}, \eqref{eq:generic-reduction}, and \eqref{eq:actual-loss-bound} therefore become explicit sums of cross-entropy gaps. In particular, the scaling-time quantity in Theorem~\ref{thm:scaling-time} depends only on the resulting composite losses, however many priors were used to construct them.

\subsection{Shifting comparators and sleeping experts}
\label{sec:shifting}

The same exact calculus also accommodates richer comparator objects.
A switching or tracking benchmark may be viewed either as a path of posteriors $(\rho_t)$ or, equivalently, as a structured prior on a parameter space of expert trajectories.

\begin{theorem}[Exact shifting-comparator identity]\label{thm:shifting}
Let $\rho_1,\dots,\rho_T\in\Delta([K])$ be any comparator path and define the dynamic composite-loss regret $R_T^{c,\mathrm{dyn}}(\rho_{1:T}) := \sum_{t=1}^T \ip{p_t-\rho_t}{c_t}$.
Then
\begin{equation}\label{eq:shifting-exact}
R_T^{c,\mathrm{dyn}}(\rho_{1:T})
=
D_T+B_T(\rho_T)+\sum_{t=1}^T \eta_tQ_t(c)
+\sum_{t=1}^{T-1}\ip{\rho_{t+1}-\rho_t}{C_t}\end{equation}
Under the schedule \eqref{eq:budget-schedule-C},
\begin{equation}\label{eq:shifting-two-sided}
R_T^{c,\mathrm{dyn}}(\rho_{1:T})
\in
D_T+B_T(\rho_T)+\sum_{t=1}^{T-1}\ip{\rho_{t+1}-\rho_t}{C_t}+\Env{V_T(c)}{Q_*^T(c)}\end{equation}
If moreover $c_t(i)\in[0,1]$ for all $t$ and $i$, then $\left|\sum_{t=1}^{T-1}\ip{\rho_{t+1}-\rho_t}{C_t}\right|
\le
2\sum_{t=1}^{T-1} t\,\TV(\rho_{t+1},\rho_t)$. 
Consequently, whenever $\KL(\rho_T\|\pi)\le \Gamma$,
\begin{equation}\label{eq:shifting-bound}
R_T^{c,\mathrm{dyn}}(\rho_{1:T})
\le
\Gamma+Q_*^T(c)+(2C+C^{-1})\sqrt{\Gamma V_T(c)}
+2\sum_{t=1}^{T-1} t\,\TV(\rho_{t+1},\rho_t)\end{equation}
\end{theorem}

The theorem contains several familiar special cases.
A comparator that switches among pure experts recovers a fixed-share-type tracking bound, in the same dynamic-regret spirit as \cite{herbster1998tracking}. A comparator supported only on currently active experts gives a soft sleeping-experts guarantee; and, because the additional term is exact before it is bounded by total variation, we can often exploit extra structure in $C_t$ beyond the crude total-variation bound.

\section{Full-information extensions and applications}
\label{sec:further}

Once the composite-loss reduction is in place, the same information accounting survives several changes of setting. 
The argument needs no new proof tricks. The same pattern reappears when one changes the geometry, the prior family, or the action domain: each update incurs an immediate information-based cost as regret, and transports the remaining information forward, simultaneously for any comparator until the algorithm runs out of information budget.

\subsection{Exact variational sharpenings of classical Hedge bounds}
\label{sec:variational-sharpening}

The usual fixed-rate Hedge proof hides two sources of slack: it drops the terminal relative-entropy remainder and replaces the exact intrinsic increment by a Hoeffding proxy.
Keeping both terms visible yields the following sharper pathwise decomposition.

\begin{corollary}[Classical bounded-range bound with exact remainder]
\label{cor:classical-gap}
Under the hypotheses of Corollary~\ref{cor:fixed-cumulant}, assume in addition that $c_t(i)\in[a_t,b_t]$ for all $i,t$. Define $S_T := \sum_{t=1}^T (b_t-a_t)^2$ and
\[
\Delta_T^{\mathrm{class}}(\rho,\eta)
:=
\frac{1}{\eta}\KL(\rho \| p_{T+1})
+
\eta\sum_{t=1}^T \left(\frac{(b_t-a_t)^2}{8}-Q_t(c)\right)
\]
Then $\Delta_T^{\mathrm{class}}(\rho,\eta)\ge 0$ and, for every posterior $\rho\in\Delta([K])$,
\begin{equation}
\label{eq:classical-gap}
R_T^c(\rho)
=
\frac{\KL(\rho\|\pi)}{\eta}
+
\frac{\eta S_T}{8}
-
\Delta_T^{\mathrm{class}}(\rho,\eta)
\end{equation}
Consequently,
\begin{equation}
\label{eq:classical-gap-ineq}
R_T^c(\rho)
\le
\frac{\KL(\rho\|\pi)}{\eta}
+
\frac{\eta S_T}{8}
\end{equation}
and the classical bound is tight only on extremal paths where both pieces of $\Delta_T^{\mathrm{class}}(\rho,\eta)$ vanish.
\end{corollary}

The gap $\Delta_T^{\mathrm{class}}(\rho,\eta)$ splits the usual pessimism into a posterior-mismatch penalty $\eta^{-1}\KL(\rho \| p_{T+1})$, which vanishes when the final Gibbs posterior already matches the comparator, and the shortfall of the realized intrinsic time $\sum_t Q_t(c)$ below its worst-case quadratic proxy $\sum_t (b_t-a_t)^2/8$.

\begin{corollary}[Square-root envelope with an exact gap]
\label{cor:classical-gap-sqrt}
Under the assumptions of Corollary~\ref{cor:classical-gap}, let $\Gamma>0$ satisfy $\Gamma\ge \KL(\rho\|\pi)$ and assume $S_T>0$. Set $\eta_\Gamma := \sqrt{8\Gamma/S_T}$.
Then
\begin{equation}
\label{eq:classical-gap-sqrt}
R_T^c(\rho)
\le
\sqrt{\frac{\Gamma S_T}{2}}
-
\left(
\frac{\Gamma-\KL(\rho\|\pi)}{\eta_\Gamma}
+
\Delta_T^{\mathrm{class}}(\rho,\eta_\Gamma)
\right)
\end{equation}
Thus the familiar $\sqrt{\Gamma S_T/2}$ fixed-budget Hedge envelope is exact up to a nonnegative gap.
Applying the same corollary to both $c_t$ and $-c_t$ gives the corresponding absolute-deviation versions as well.
\end{corollary}

The same point also has a direct adaptive analogue that refers only to the realized schedule, with no comparison to a hindsight-fixed temperature.

\begin{proposition}[Direct pathwise identity for a realized adaptive schedule]
\label{prop:direct-adaptive-pathwise}
Let $(q_t)$ be generated by the one-step update \eqref{eq:one-step-update} with any predictable positive schedule $(\eta_t)_{t=1}^T$.
Then for every posterior $\rho\in\Delta([K])$,
\begin{equation}
\label{eq:direct-adaptive-pathwise}
\sum_{t=1}^T \eta_t\bigl(\ip{q_t}{c_t}-\ip{\rho}{c_t}\bigr)
=
\sum_{t=1}^T \eta_t\delta_t(c)
+
\KL(\rho\|q_1)-\KL(\rho\|q_{T+1})
\end{equation}
Equivalently, if $Q_t^{\mathrm{step}}(c):=\delta_t(c)/\eta_t$, then
\begin{equation}
\label{eq:direct-adaptive-pathwise-Q}
\sum_{t=1}^T \eta_t\bigl(\ip{q_t}{c_t}-\ip{\rho}{c_t}\bigr)
=
\sum_{t=1}^T \eta_t^2 Q_t^{\mathrm{step}}(c)
+
\KL(\rho\|q_1)-\KL(\rho\|q_{T+1})
\end{equation}
\end{proposition}

Once the learning-rate path is fixed, the weighted regret already satisfies an identity, and the prior-anchored retempered theorem is the matching decomposition for the anchored update, whose extra drift term records the cost of changing temperature.

\subsection{Weighted negative entropy and multiscale experts}
\label{sec:weighted-multiscale}

The weighted negative entropy used for multiscale experts \cite{bubeck2019onlinemultiscale,PerezOrtizKoolen2022} instantiates the same accounting under a different mirror map.
The multiscale experts problem and this mirror map originate in \cite{bubeck2019onlinemultiscale}, where regret against each expert scales with that expert's own loss range; the luckiness analysis of \cite{PerezOrtizKoolen2022} works in the same geometry.
With per-expert scales $\sigma_1,\dots,\sigma_K>0$ and the weighted divergence $D_\sigma(q\|p) := \sum_{i=1}^K \sigma_i\left(q(i)\log\frac{q(i)}{p(i)}-q(i)+p(i)\right)$ in place of relative entropy, the corresponding weighted-entropy mirror-descent step obeys an exact one-step identity of the same form as \eqref{eq:centered-seq}. The immediate cost is the nonnegative weighted Bregman step $D_\sigma(p_t\|p_{t+1})/\eta_t$. The rest is transport of comparator information, measured in $D_\sigma$. The identity and the concentration-game bookkeeping in this geometry are recorded in \cite{balsubramani2026concentration}.
The same telescoping and potential arguments as in Sections~\ref{sec:exact}--\ref{sec:schedules} then yield fixed-rate and adaptive cumulative identities with the weighted terminal potential $A_t^\sigma (\eta) := \min_{q\in\Delta([K])} \left\{ \ip{q}{L_t} + \eta^{-1} D_\sigma (q \| \pi) \right\}$ in place of $A_t(\eta)$.
It is the sequential analogue of the variational geometric pool $p^\star_\alpha$ of Theorem~\ref{thm:finite-mixed}, with expert-specific scales fixing the geometry.

\subsection{Multiscale forgetting by continuum prior mixtures}

A prior over discount factors becomes a memory kernel over lags.
Let $\beta\in(0,1)$ be a forgetting rate.
Define the $\beta$-discounted cumulative loss and the corresponding exponential-weights prior 
\[
L_{t,\beta}(i):=\sum_{s=1}^{t-1}\beta^{t-1-s}\ell_s(i)
\qquad \qquad
\pi_{t,\beta}(i)\propto \exp(-\eta L_{t,\beta}(i))
\]
Now let $\alpha_t$ be a finite positive measure on $R\subset(0,1)$ such that $\int_R |\log \pi_{t,\beta}(i)|\,\alpha_t(d\beta)<\infty$ for each $i$. We may then define the geometric pool by $q_t(i)\propto \exp\bigl(\int_R \log \pi_{t,\beta}(i)\,\alpha_t(d\beta)\bigr)$.

\begin{proposition}[Mixtures of forgetting rates produce mixtures of exponential kernels]
\label{prop:forgetting}
With the above definitions,
\begin{equation}
q_t(i) \propto \exp\!\left(-\eta\sum_{s=1}^{t-1} k_t(t-1-s)\ell_s(i)\right)
\qquad \qquad
k_t(u) := \int_R \beta^u\,\alpha_t(d\beta)
\label{eq:forgetting-kernel}
\end{equation}
In particular, the effective memory kernel is a mixture of exponential decays.
\end{proposition}

Continuum pooling reads directly as learned memory design: choosing the rate measure is choosing the decay kernel.

\subsection{Repeated games and robust opponent modeling}

In repeated zero-sum games, consider row actions $i\in[K]$, column actions $j\in[L]$, and loss matrix $M\in[0,1]^{K\times L}$ for the row player; multiplicative-weights play against $\ell_t(i)=M_{i,j_t}$ gives no regret and a constructive minimax theorem \cite{freund1999games,arora2012theory}.

Candidate opponent models are accommodated in the prior before play begins.
Pooling several models' row-response priors multiplicatively is ordinary exponential response to the averaged opponent model, so models from different windows or discount scales give a multiscale fictitious-play prior.
The worst case over poolings is a robust center: by the finite mixed coincidence identity (Theorem~\ref{thm:finite-mixed}) and the minimax theorem, the maximal coincidence divergence over pooling weights equals the minimum worst-case reverse relative entropy to the model-conditioned priors. The optimizer is itself a geometric mixture---the robust object when no single opponent model is trusted.
Both statements, and the use of the robust center as a baseline for play on the residuals, are developed in \cite{balsubramani2026concentration}.

Section~\ref{sec:adaptive}'s second-order schedule already equalizes the complexity term $\Gamma/\eta$ against the intrinsic-time term $\eta V$.
In a repeated zero-sum game, that same balance becomes an entropic regret matcher.

\begin{proposition}[Exact second-order regret matching]
\label{prop:regret-matching}
Let $\ell_t\in[0,1]^K$ be the row player's loss vector and define the instantaneous regret vector $g_t(i) := \ell_t(i)-\ip{p_t}{\ell_t}$.
Then $\ip{p_t}{g_t}=0$. Run predictable-rate Hedge on the sequence $g_t$ with prior $\pi$ and schedule \eqref{eq:budget-schedule-C}. For every action $i\in[K]$,
\begin{equation}\label{eq:regret-matching-exact}
\sum_{t=1}^T \left(\ip{p_t}{\ell_t}-\ell_t(i)\right)
=
D_T^g+B_T^g(e_i)+\sum_{t=1}^T \eta_t Q_t(g)\end{equation}
where $D_T^g$ and $B_T^g$ are the drift and terminal terms for the centered losses $g_t$.
Hence, if $\KL(e_i\|\pi)\le \Gamma$,
\begin{equation}\label{eq:regret-matching-bound}
\sum_{t=1}^T \left(\ip{p_t}{\ell_t}-\ell_t(i)\right)
\le \Gamma+Q_*^T(g)+(2C+C^{-1})\sqrt{\Gamma V_T(g)}\end{equation}
Moreover, while $\eta_t=C\sqrt{\Gamma/V_{t-1}(g)}$, scaling every loss vector by a constant $a>0$ scales $Q_t(g)$ and $V_T(g)$ by $a^2$ and therefore rescales the second-order learning rate by $1/a$, so the schedule is scale invariant in that precise intrinsic-time sense.
\end{proposition}

\begin{corollary}[Self-play gap from intrinsic-time equalizers]
\label{cor:selfplay-gap}
Consider a repeated zero-sum matrix game, and let both players run the second-order update on their own instantaneous regret vectors.
If $\bar p_T$ and $\bar q_T$ are the average row and column plays, write $P_T^{\mathrm{row}}:=\sum_t\eta_t^{\mathrm{row}}Q_t^{\mathrm{row}}$ and $P_T^{\mathrm{col}}:=\sum_t\eta_t^{\mathrm{col}}Q_t^{\mathrm{col}}$ for the two intrinsic-time losses, and write $D_T^{\mathrm{row}},B_T^{\mathrm{row}}$ and $D_T^{\mathrm{col}},B_T^{\mathrm{col}}$ for the corresponding drift and terminal-information terms.
Then the exploitability gap has the informational form
\begin{align}
&T\left(\max_q\,\bar p_T^\top M q - \min_p\,p^\top M\bar q_T\right) \notag\\
&\qquad=
\max_j\left\{D_T^{\mathrm{col}}+B_T^{\mathrm{col}}(e_j)+P_T^{\mathrm{col}}\right\}
+
\max_i\left\{D_T^{\mathrm{row}}+B_T^{\mathrm{row}}(e_i)+P_T^{\mathrm{row}}\right\}.
\label{eq:selfplay-gap-exact}
\end{align}
Consequently,
\begin{equation}\label{eq:selfplay-gap}
\max_{q}\,\bar p_T^\top M q - \min_{p}\, p^\top M\bar q_T
\le \frac{U_T^{\mathrm{row}}+U_T^{\mathrm{col}}}{T},
\end{equation}
where $U_T^{\mathrm{row}}$ and $U_T^{\mathrm{col}}$ are the corresponding upper bounds from \eqref{eq:regret-matching-bound} for the two players.
\end{corollary}

The interpretation is simple: each player runs a scale-adaptive regret matcher on its own intrinsic time, so in self-play the exploitability gap is just the average of the two information budgets.

Where this accounting measures how fast repeated play converges toward equilibrium, a complementary line of work monitors the reverse question online,
using anytime-valid tests that bet against the equilibrium and stop the moment observed play departs from it~\cite{gauthier2026anytime}.

\subsection{Continuum expert classes and PAC-Bayes}
\label{sec:continuum}

The same information accounting also survives when the expert class is infinite.

The finite geometric-pooling identity extends whenever the logarithmic average is well defined.
Concretely, let $\{\pi_\theta\}_{\theta\in\Theta}$ be a measurable family of positive measures and let $\alpha$ be a finite positive measure on $\Theta$ such that $\int_\Theta |\log \pi_\theta(x)|\,\alpha(d\theta)<\infty \qquad\text{for each }x$.
Then we may define the pooled measure by
\[
p^\star_\alpha(x) \propto \exp\!\left(\int_\Theta \log \pi_\theta(x)\,\alpha(d\theta)\right)\]
This mixes over a continuum of temperatures, window lengths, forgetting rates, or model classes without discretizing first; Proposition~\ref{prop:forgetting} is the simplest example.

One immediate consequence is a multi-prior PAC-Bayes penalty.
If a variational formula or bound contains a single term $\KL(\rho\|\pi)$, replacing $\pi$ by a pooled prior $p^\star_\alpha$ and applying Theorem~\ref{thm:finite-mixed} gives a decomposition into several prior penalties minus a coincidence divergence.

\begin{proposition}[Multi-prior PAC-Bayes penalty]
\label{prop:pac-bayes}
Let $\mu$ be a $\sigma$-finite reference measure on a hypothesis space $\mathcal H$, let $\pi_1,\dots,\pi_W$ be probability measures with strictly positive $\mu$-densities $p_w:=d\pi_w/d\mu$, and let $\alpha\in\Delta([W])$. Define
\[
g_\alpha(h) := \prod_{w=1}^W p_w(h)^{\alpha_w},
\qquad
Z_\alpha := \int_{\mathcal H} g_\alpha(h)\,\mu(dh),
\qquad
\frac{dp^\star_\alpha}{d\mu}(h) := \frac{g_\alpha(h)}{Z_\alpha}\]
Then for every posterior $\rho$ such that $\rho\ll \pi_w$ for every $w$,
\begin{equation}
\KL(\rho \| p^\star_\alpha)=\sum_{w=1}^W \alpha_w\KL(\rho\|\pi_w)-\cC_\alpha(\pi_{1:W}),
\label{eq:pac-bayes-multi}
\end{equation}
where $\cC_\alpha(\pi_{1:W}):=-\log Z_\alpha$.
Hence every PAC-Bayes variational formula or inequality whose dependence on the prior enters only through the term $\KL(\rho\|\pi)$ admits an exact multi-prior version with penalty $\sum_w \alpha_w\KL(\rho\|\pi_w)-\cC_\alpha(\pi_{1:W})$.
\end{proposition}

The $\KL$ terms measure mismatch to each source, and the coincidence divergence $\cC_\alpha$ is the bonus the geometric pool earns over the weighted average of single-prior penalties: it vanishes at full prior overlap and grows with the sources' disagreement; Section~\ref{sec:adaptive} gives the dynamic analogue through side priors.

\subsection{Continuous-action online convex optimization}\label{sec:oco}

The same chain passes to compact convex action sets by replacing sums with integrals and playing the barycenter of the exponential-weights density \cite{SteinhardtLiang2014}.
Let $S\subset\R^d$ be compact and convex,
let $\pi$ be a probability measure on $S$ with full support, and let $f_t:S\to[0,1]$ be convex and continuous.
For predictable positive rates $\eta_t$, define
\[
p_t(dx) := \frac{e^{-\eta_t F_{t-1}(x)}\,\pi(dx)}{\int_S e^{-\eta_t F_{t-1}(u)}\,\pi(du)},
\qquad
F_t(x) := \sum_{s=1}^t f_s(x),
\qquad
x_t := \int_S x\,p_t(dx)\in S\]
Because $x_t$ is the barycenter of $p_t$, Jensen's inequality gives $f_t(x_t)\le \int_S f_t\,dp_t$ on each round.

The three-piece decomposition of Theorem~\ref{thm:pacbayes-second-order} carries over verbatim to the density regret $R_T^{\mathrm{dens}}(\rho):=\sum_{t=1}^T\int_S f_t\,dp_t-\int_S F_T\,d\rho$, and Jensen's inequality passes it to the ordinary OCO regret; a shrinking-comparator argument on the normalized Lebesgue prior recovers the point-comparator rate $\reg_T(x^\star)=O(\sqrt{dT\log T})$.
\begin{theorem}[Exact continuous PAC-Bayes chain and OCO corollary]
\label{thm:oco-exact}
For any posterior measure $\rho\ll\pi$, define the density regret $R_T^{\mathrm{dens}}(\rho) := \sum_{t=1}^T \int_S f_t(x)\,p_t(dx)-\int_S F_T(x)\,\rho(dx)$.
The continuous analogues of the finite-case objects of Theorem~\ref{thm:pacbayes-second-order} (integrals replacing sums, the played density $p_t$ in place of the played distribution) are $\displaystyle
A_t^{\mathrm{oco}}(\eta):=-\eta^{-1}\log\!\int_S e^{-\eta F_t(x)}\,\pi(dx),
\quad
q_{t,\eta}(dx) := \frac{e^{-\eta F_t(x)}\,\pi(dx)}{\int_S e^{-\eta F_t(u)}\,\pi(du)},
\quad
\psi_t^{\mathrm{oco}}(\lambda)
:= \log\!\int_S \exp\!\left(\lambda\left(f_t(x)-\int_S f_t\,dp_t\right)\right)\,p_t(dx)$,
and $\phi_t^{\mathrm{oco}}(\eta):= \frac{\psi_t^{\mathrm{oco}}(-\eta)}{\eta}$, $Q_t^{\mathrm{oco}} := \frac{\phi_t^{\mathrm{oco}}(\eta_t)}{\eta_t}$, $B_T^{\mathrm{oco}}(\rho) := \frac{\KL(\rho\|\pi)-\KL(\rho\|q_{T,\eta_T})}{\eta_T}$, $D_T^{\mathrm{oco}} := \sum_{t=1}^{T-1}\left(A_t^{\mathrm{oco}}(\eta_t)-A_t^{\mathrm{oco}}(\eta_{t+1})\right)$. With these, we have
\begin{equation}\label{eq:oco-exact}
R_T^{\mathrm{dens}}(\rho)=D_T^{\mathrm{oco}}+B_T^{\mathrm{oco}}(\rho)+\sum_{t=1}^T \eta_tQ_t^{\mathrm{oco}}\end{equation}
Consequently the ordinary OCO regret admits the identity
\begin{equation}\label{eq:oco-regret-posterior}
\sum_{t=1}^T f_t(x_t)-\int_S F_T(x)\,\rho(dx)
=
D_T^{\mathrm{oco}}+B_T^{\mathrm{oco}}(\rho)+\sum_{t=1}^T \eta_tQ_t^{\mathrm{oco}}-J_T^{\mathrm{oco}},\end{equation}
where $J_T^{\mathrm{oco}}:=\sum_{t=1}^T\bigl(\int_S f_t\,dp_t-f_t(x_t)\bigr)\ge 0$ is the Jensen gap of playing the barycenter $x_t=\int_S x\,p_t(dx)$ in place of the density $p_t$ (each $f_t$ convex); dropping this nonnegative term gives the ordinary-regret bound.

If the second-order schedule \eqref{eq:budget-schedule-C} is run with budget $\Gamma$ and $\KL(\rho\|\pi)\le \Gamma$, then with $V_T^{\mathrm{oco}} := \sum_{t=1}^T Q_t^{\mathrm{oco}}$ and $Q_{*,\mathrm{oco}}^T := \max_{t\le T}Q_t^{\mathrm{oco}}$, 
\begin{equation}\label{eq:oco-regret-noshift}
\sum_{t=1}^T f_t(x_t)-\int_S F_T\,d\rho
\le
\Gamma+Q_{*,\mathrm{oco}}^T+(2C+C^{-1})\sqrt{\Gamma V_T^{\mathrm{oco}}}\end{equation}
\end{theorem}

\begin{corollary}[Point-comparator regret by geometric shrinking]
\label{cor:oco-point}
Assume now that $S$ has nonempty interior and that $\pi$ is the normalized Lebesgue measure on $S$.
Let $x^\star\in\arg\min_{x\in S}F_T(x)$, fix $\epsilon\in(0,1]$, and write $\alpha := \epsilon^{1/d}$.
Then let $\rho_{\epsilon,x^\star}$ be the uniform measure on the shrunken copy $(1-\alpha)x^\star+\alpha S$.
Then $\KL(\rho_{\epsilon,x^\star}\|\pi)=\log(1/\epsilon)$ and
\begin{equation}\label{eq:oco-shrink}
\sum_{t=1}^T f_t(x_t)-F_T(x^\star)
\le
D_T^{\mathrm{oco}}+B_T^{\mathrm{oco}}(\rho_{\epsilon,x^\star})+\sum_{t=1}^T \eta_tQ_t^{\mathrm{oco}}
+T\epsilon^{1/d}\end{equation}
If $\reg_T(x^\star) := \sum_{t=1}^T f_t(x_t)-F_T(x^\star)$ and the second-order schedule is run with $\Gamma=\log(1/\epsilon)$, then
\begin{equation}\label{eq:oco-shrink-noshift}
\reg_T(x^\star)
\le
\log(1/\epsilon)+Q_{*,\mathrm{oco}}^T
+(2C+C^{-1})\sqrt{V_T^{\mathrm{oco}}\log(1/\epsilon)}+T\epsilon^{1/d}\end{equation}
In particular, taking $\epsilon=T^{-d}$ and $C=1/\sqrt2$ yields $\reg_T(x^\star)\le \frac32d\log T+Q_{*,\mathrm{oco}}^T+2\sqrt{2dV_T^{\mathrm{oco}}\log T}+1$.
Since $f_t\in[0,1]$, the standard bounded-range exponential-moment bound gives $Q_t^{\mathrm{oco}}\le 1/8$, so $V_T^{\mathrm{oco}}\le T/8$ and therefore $\reg_T(x^\star)=O(\sqrt{dT\log T})$.
\end{corollary}

The proofs are collected in Appendix~\ref{sec:proofs}.

This extension is structural: the information accounting survives the passage to continuous actions, but the exact density update and its barycenter are costly in high dimensions and the entropic geometry incurs a $\sqrt{d}$ penalty that projection-based methods avoid.

\subsection{Regret identities and exponential loss for boosting}
\label{sec:boosting}

Boosting fits the framework in two complementary ways.
The first treats examples as experts and turns regret identities into quantile-margin guarantees.
The second works directly with signed margins and recovers the usual exponential-loss variational formula.
The first viewpoint highlights intrinsic time; the second is closer to standard boosting analyses such as the drifting-games treatment in \cite{SteinhardtLiang2014}.

Consider binary-labeled data $(x_i,y_i)_{i=1}^N$ with $y_i\in\{-1,+1\}$.
At round $t$ the booster maintains weights $p_t\in\Delta([N])$ over examples, calls a weak learner, and receives a hypothesis $h_t$ with edge $\gamma_t := \frac{1}{2} \sum_{i=1}^N p_t(i)y_i h_t(x_i)$.
For the regret-based reduction we feed the learner the correctness score $\ell_t(i) := \one\{h_t(x_i)=y_i\}=(1+y_i h_t(x_i))/2\in\{0,1\}$ in place of a literal misclassification loss.
After $T$ rounds output the majority vote $H_T(x) := \mathrm{sign}\!\left(\sum_{t=1}^T h_t(x)\right)$.
The margin of example $i$ is $m_i:=T^{-1}\sum_{t=1}^T y_i h_t(x_i)$, so $L_T^\ell(i)=\frac{T}{2}(1+m_i)$.

\begin{theorem}[Generic regret-to-margin conversion]
\label{thm:boost-generic}
Let $\bar\gamma_T:=T^{-1}\sum_{t=1}^T \gamma_t$. For any posterior $\rho\in\Delta([N])$,
\begin{equation}\label{eq:boost-generic}
\ip{\rho}{m}=2\bar\gamma_T-\frac{2}{T}R_T^\ell(\rho)\end{equation}
where $R_T^\ell(\rho) := \sum_{t=1}^T\ip{p_t}{\ell_t}-\ip{\rho}{L_T^\ell}$.
Equivalently, any regret upper bound $R_T^\ell(\rho)\le U_T(\rho)$ immediately implies the average-margin lower bound $\ip{\rho}{m}\ge 2\bar\gamma_T-2U_T(\rho)/T$.

Now assume the prior is uniform on the examples and let $A_\epsilon\subset[N]$ be a set of $\lceil\epsilon N\rceil$ examples with the smallest margins.
If $m_{[\epsilon]}$ denotes the $\epsilon$-quantile margin, namely the largest margin inside $A_\epsilon$, then
\begin{equation}\label{eq:boost-quantile-generic}
m_{[\epsilon]}\ge 2\bar\gamma_T-\frac{2}{T}R_T^\ell(\rho_{A_\epsilon})\end{equation}
where $\rho_{A_\epsilon}$ is uniform on $A_\epsilon$.
Hence any bound $R_T^\ell(\rho_{A_\epsilon})\le U_T(\epsilon)$ yields
\begin{equation}\label{eq:boost-margin-from-regret}
m_{[\epsilon]}\ge 2\bar\gamma_T-\frac{2}{T}U_T(\epsilon)\end{equation}
Consequently, if $\theta<2\bar\gamma_T-2U_T(\epsilon)/T$, then at most an $\epsilon$ fraction of the training examples have margin at most $\theta$.
In particular, the training error of $H_T$ is at most $\epsilon$ whenever $2\bar\gamma_T>2U_T(\epsilon)/T$.
\end{theorem}

Any regret guarantee over examples is therefore automatically a quantile-margin guarantee.
Plugging in the intrinsic-time Hedge theorem gives the direct boosting counterpart of the main results of this paper.

\begin{corollary}[Intrinsic-time boosting via predictable-rate Hedge]
\label{cor:boost-noshift}
Run the predictable-rate Hedge update on the example losses $\ell_t$ with uniform prior and the second-order schedule \eqref{eq:budget-schedule-C}. Then
\begin{equation}\label{eq:boost-exact}
m_{[\epsilon]}
\ge
2\bar\gamma_T-\frac{2}{T}\!\left(D_T^\ell+B_T^\ell(\rho_{A_\epsilon})
+\sum_{t=1}^T \eta_t Q_t(\ell)\right)\end{equation}
where $D_T^\ell$, $B_T^\ell$, and $Q_t(\ell)$ are the objects from Section~\ref{sec:adaptive} evaluated on the correctness losses $\ell_t$.
Since $\KL(\rho_{A_\epsilon}\|u_N)=\log(N/|A_\epsilon|)\le \log(1/\epsilon)$, the second-order bound gives
\begin{equation}\label{eq:boost-noshift}
m_{[\epsilon]}
\ge
2\bar\gamma_T-\frac{2}{T}\!\left(\log(1/\epsilon)+Q_*^T(\ell)
+(2C+C^{-1})\sqrt{V_T(\ell)\log(1/\epsilon)}\right)\end{equation}
Because $\ell_t\in\{0,1\}$, we have $Q_t(\ell)\le 1/8$ and $V_T(\ell)\le T/8$.
Thus the margin shortfall is of order $\sqrt{\log(1/\epsilon)/T}$, and if the weak learner has a uniform edge lower bound $\gamma_t\ge \gamma>0$ then choosing $\epsilon=e^{-cT\gamma^2}$ with $c>0$ small enough forces the right-hand side above zero and therefore gives exponentially small training error.
\end{corollary}

The second-order schedule therefore yields a margin bound whose data-dependent part is the realized intrinsic time on the correctness scores.
That is this reduction's contribution here; the remaining boosting content is the exponential-loss instance of the concentration game, developed in \cite{balsubramani2026concentration}.

\paragraph{The exponential-loss instance.}
The complementary and more classical viewpoint works directly with the signed margins $g_t(i):=y_ih_t(x_i)\in[-1,1]$ and arbitrary coefficients $\alpha_t>0$, under $p_{t+1}(i)\propto p_t(i)e^{-\alpha_t g_t(i)}$ from $p_1=u_N$.
Write $M_T(i):=\sum_t\alpha_t g_t(i)$ for the cumulative weighted margin, $A_T:=\sum_t\alpha_t$, and $\mathcal L_T^{\exp}:=N^{-1}\sum_i e^{-M_T(i)}=\prod_t Z_t$ for the exponential training loss.
The same mixed-coincidence algebra used throughout gives, for every $\rho\in\Delta([N])$,
\begin{equation}\label{eq:boost-exp-exact}
-\log \mathcal L_T^{\exp} + \KL(\rho \| p_{T+1})
=
\KL(\rho\|u_N)+\ip{\rho}{M_T}\end{equation}
so $-\log \mathcal L_T^{\exp}$ is the Donsker--Varadhan value $\sup_{\rho}\{\ip{\rho}{M_T}-\KL(\rho\|u_N)\}$, attained at $\rho=p_{T+1}$.
Training error, margin tails, and quantile margins are all read off that one terminal functional, by the same worst-$\epsilon$-set substitution used for \eqref{eq:boost-quantile-generic}.
Choosing $\alpha_t$ by the local pressure rule of Section~\ref{sec:pressure} on the signed margins, $-\alpha_t^{-1}\log\sum_i p_t(i)e^{-\alpha_t g_t(i)}=a_t$, evaluates it exactly, at $-\log \mathcal L_T^{\exp}=\sum_t\alpha_t a_t$.
For binary weak hypotheses the coefficient minimizing the one-step normalizer is the classical AdaBoost weight $\alpha_t^\star=\tfrac12\log\{(1-\varepsilon_t)/\varepsilon_t\}$ with $\varepsilon_t=\Pr_{i\sim p_t}\{g_t(i)=-1\}$, at which the round normalizer is $2\sqrt{\varepsilon_t(1-\varepsilon_t)}$ and the exact ledger relaxes in one line to the classical rate $\mathrm{Err}_T\le e^{-2\gamma^2T}$ under a uniform edge $\gamma$.
These statements, their proofs, and the terminal-game reading that organizes them are in \cite{balsubramani2026concentration}.

The unit-potential choice $a_t=0$ gives an exact weighted-margin identity but keeps $\mathcal L_T^{\exp}=1$, so it is a scale-free calibration rule; positive targets $a_t>0$ turn the same algebra into exponential-loss descent.

In runs with decision-stump weak learners on multiclass classification substrates, the pressure-target recurrence recovers this classical AdaBoost coefficient exactly and the exponential loss sits on its theoretical decay envelope (\S\ref{sec:exp-results-boosting}).

Replacing the signed score $g_t(i)$ by a one-sided shortfall such as $(a_t-g_t(i))_+$ or $(-g_t(i))_+$ zeroes the instantaneous cost on examples already above target, so the update concentrates on margin violations and the example weights go sparse---the one-sided sufficient-statistic reduction, close in spirit to the sparse NH-Boost.DT behavior in \cite{SteinhardtLiang2014}.
Running the mixed-prior update of Section~\ref{sec:return} with composite example losses $c_t(i)=\ell_t(i)+u_t(i)$ gives
\[
\ip{\rho}{m}
=
2\bar\gamma_T-\frac{2}{T}\left(D_T+B_T(\rho)+\sum_{t=1}^T \eta_tQ_t(c)
+\sum_{t=1}^T \left(\ip{\rho}{u_t}-\ip{p_t}{u_t}\right)\right)\]
so curricula, example reliabilities, and domain-weighting side factors enter boosting through the same explicit mismatch terms used throughout.

\section{Proofs}
\label{sec:proofs}

We collect the proofs in one appendix, grouped by the section of the main text where each result is stated.

\subsection{Proofs for Section~\ref{sec:exact-block} (core identities)}

\begin{proof}[Proof of Theorem~\ref{thm:finite-mixed}.]
By definition,
$\log p^\star_\alpha(x)=\sum_{w=1}^W \alpha_w\log \mu_w(x)-\log Z(\alpha)$.
Therefore, for any $q\in\Delta(\mathcal X)$,
\begin{align*}
\KL(q\|p^\star_\alpha)
&= \sum_{x\in\mathcal X} q(x)\log\frac{q(x)}{p^\star_\alpha(x)} \\
&= \sum_{x\in\mathcal X} q(x)\log q(x)
   - \sum_{w=1}^W \alpha_w \sum_{x\in\mathcal X} q(x)\log \mu_w(x)
   + \log Z(\alpha) \\
&= -\Hc(q) + \sum_{w=1}^W \alpha_w\Hc(q,\mu_w) + \log Z(\alpha)\end{align*}
Rearranging proves \eqref{eq:finite-mixed}.

Now assume that each $\mu_w$ is a probability distribution and that $\alpha\in\Delta([W])$.
Because $\sum_w \alpha_w=1$,
$\sum_{w=1}^W \alpha_w \KL(q\|\mu_w) = \sum_{w=1}^W \alpha_w\Hc(q,\mu_w)-\Hc(q)$.
Combining this with \eqref{eq:finite-mixed} yields $\sum_{w=1}^W \alpha_w \KL(q\|\mu_w) = -\log Z(\alpha)+\KL(q\|p^\star_\alpha)$.
Since the relative entropy is nonnegative and vanishes only at $q=p^\star_\alpha$, the minimum is $-\log Z(\alpha)$ and the minimizer is unique.
\end{proof}

\begin{proof}[Proof of Theorem~\ref{thm:seq-mixed}.]
From \eqref{eq:one-step-update},
$q_{t+1}(i)=\frac{q_t(i)e^{-\eta_t\ell_t(i)}s_t(i)}{Z_t}$.
Hence
\begin{align*}
\KL(\rho\|q_{t+1})
&= \sum_{i=1}^K \rho(i)\log\frac{\rho(i)}{q_{t+1}(i)} \\
&= \sum_{i=1}^K \rho(i)\log\frac{\rho(i)Z_t}{q_t(i)e^{-\eta_t\ell_t(i)}s_t(i)} \\
&= \KL(\rho\|q_t) + \log Z_t + \eta_t\ip{\rho}{\ell_t} - \sum_{i=1}^K \rho(i)\log s_t(i)\end{align*}
Since $u_t(i)=-\eta_t^{-1}\log s_t(i)$ and $c_t(i)=\ell_t(i)+u_t(i)$, this is equivalent to \eqref{eq:seq-id}.

If \eqref{eq:side-decomp} holds, then
\[
-\sum_{i=1}^K \rho(i)\log s_t(i)
= -\sum_{i=1}^K \rho(i)\sum_{w=1}^{J_t} \alpha_{t,w}\log \mu_{t,w}(i)
= \sum_{w=1}^{J_t} \alpha_{t,w}\Hc(\rho,\mu_{t,w})\]
which yields \eqref{eq:seq-id-crossent}.
\end{proof}

\begin{proof}[Proof of Corollaries~\ref{cor:fixed-rate},~\ref{cor:fixed-rate-bounded}, and~\ref{cor:fixed-rate-mismatch}.]
When $\eta_t\equiv\eta$, summing \eqref{eq:seq-id} over $t=1,\dots,T$ gives
\[
\sum_{t=1}^T m_t + \eta^{-1}\KL(\rho\|q_{T+1})
= \eta^{-1}\KL(\rho\|q_1) + \sum_{t=1}^T \ip{\rho}{c_t}\]
which is exactly \eqref{eq:fixed-rate-exact}.
The displayed equivalent form in the statement is just $\sum_t \ip{\rho}{c_t}=\ip{\rho}{L_T}+\sum_t \ip{\rho}{u_t}$.

For \eqref{eq:fixed-rate-comp-reg}, let $I_t\sim q_t$ and let $X_t:=c_t(I_t)$.
Then $X_t\in[a_t,b_t]$ almost surely and $m_t = -\eta^{-1}\log \E[e^{-\eta X_t}]$.
The standard bounded-range exponential-moment bound implies
\[
-\eta^{-1}\log \E[e^{-\eta X_t}] \ge \E[X_t]-\frac{\eta}{8}(b_t-a_t)^2
= \ip{q_t}{c_t}-\frac{\eta}{8}(b_t-a_t)^2\]
Combining this with \eqref{eq:seq-id} (with constant $\eta$) gives
\[
\ip{q_t}{c_t}-\ip{\rho}{c_t}
\le \eta^{-1}\left(\KL(\rho\|q_t)-\KL(\rho\|q_{t+1})\right)
+ \frac{\eta}{8}(b_t-a_t)^2\]
Summing over $t$ proves \eqref{eq:fixed-rate-comp-reg}. Finally,
$\ip{q_t}{\ell_t}-\ip{\rho}{\ell_t} = \ip{q_t}{c_t}-\ip{\rho}{c_t} + \ip{\rho}{u_t}-\ip{q_t}{u_t}$,
and summing yields \eqref{eq:fixed-rate-ell-reg}.
The identity \eqref{eq:fixed-rate-mismatch} expands $u_t(i)=-\eta^{-1}\sum_w \alpha_{t,w}\log \mu_{t,w}(i)$.
\end{proof}

\begin{proof}[Proof of Lemma~\ref{lem:dv}.]
By definition of $q_X^\star$,
$\log\frac{q_X^\star(i)}{\pi(i)} = X(i)-\log\!\left(\sum_{j=1}^K \pi(j)e^{X(j)}\right)$.
Multiplying by $\rho(i)$, summing over $i$, and rearranging gives $\KL(\rho\|q_X^\star)=\KL(\rho\|\pi)-\ip{\rho}{X}+\log\!\left(\sum_{j=1}^K \pi(j)e^{X(j)}\right)$,
which is exactly \eqref{eq:dv}.
Substituting $X(i)=-\eta x(i)$ gives \eqref{eq:dv-eta} with $q_{\eta,x}^\star(i)\propto \pi(i)e^{-\eta x(i)}$.
\end{proof}

\begin{proof}[Proof of Lemma~\ref{lem:terminal-potential}.]
Fix $t$ and $\eta>0$, and define $q_\eta(i) := \frac{\pi(i)e^{-\eta C_t(i)}}{\sum_{j=1}^K \pi(j)e^{-\eta C_t(j)}}$.
Then for any $q\in\Delta([K])$,
\begin{align*}
\KL(q\|q_\eta)
&= \sum_{i=1}^K q(i)\log\frac{q(i)}{q_\eta(i)} \\
&= \KL(q\|\pi)+\eta\ip{q}{C_t}
   + \log\left(\sum_{j=1}^K \pi(j)e^{-\eta C_t(j)}\right) \\
&= \KL(q\|\pi)+\eta\ip{q}{C_t}-\eta A_t(\eta)\end{align*}
Rearranging gives $A_t(\eta)+\eta^{-1}\KL(q\|q_\eta) = \ip{q}{C_t}+\eta^{-1}\KL(q\|\pi)$.
Since the left-hand relative-entropy term is nonnegative and vanishes at $q=q_\eta$, this proves \eqref{eq:At-var}.

The map $\eta\mapsto A_t(\eta)$ is nonincreasing because for each fixed $q$ the expression $\ip{q}{C_t}+\eta^{-1}\KL(q\|\pi)$ is nonincreasing in $\eta$, and $A_t(\eta)$ is the pointwise minimum of those functions.
\end{proof}

\begin{proof}[Proof of Theorem~\ref{thm:cumulant-chain}.]
Let $H_T := \sum_{t=1}^T \ip{p_t}{c_t}, \qquad M_T := \sum_{t=1}^T m_t$.
Then $R_T^c(\rho)=H_T-\ip{\rho}{C_T}=(H_T-M_T)+(M_T-\ip{\rho}{C_T})$.

For the first term, let $\mu_t := \ip{p_t}{c_t}$. By definition of $\psi_t$,
\[
\psi_t(-\eta_t)
=
\log \E_{i\sim p_t}\exp\!\left(-\eta_t(c_t(i)-\mu_t)\right)
=
\eta_t(\mu_t-m_t)
=
\eta_t\delta_t(c)\]
Therefore
\begin{equation}
H_T-M_T=\sum_{t=1}^T \delta_t(c)=\sum_{t=1}^T \phi_t(\eta_t).
\label{eq:proof-chain-1}
\end{equation}

For the second term, \eqref{eq:retempered-hedge} gives
\[
m_t
=
-\eta_t^{-1}\log
\frac{\sum_{i=1}^K \pi(i)e^{-\eta_t C_t(i)}}{\sum_{j=1}^K \pi(j)e^{-\eta_t C_{t-1}(j)}}
=
A_t(\eta_t)-A_{t-1}(\eta_t)\]
Summing over $t$ and regrouping,
$M_T = A_T(\eta_T)+\sum_{t=1}^{T-1}\left(A_t(\eta_t)-A_t(\eta_{t+1})\right)$.
Applying \eqref{eq:At-exact} with $t=T$, $\eta=\eta_T$, and $q=\rho$ yields
\begin{equation}
A_T(\eta_T)-\ip{\rho}{C_T}
=
\frac{\KL(\rho\|\pi)-\KL(\rho\|q_{T,\eta_T})}{\eta_T}.
\label{eq:proof-chain-3}
\end{equation}
Substituting \eqref{eq:proof-chain-1}--\eqref{eq:proof-chain-3} into $R_T^c(\rho)=(H_T-M_T)+(M_T-\ip{\rho}{C_T})$ proves \eqref{eq:exact-chain}.

If $(\eta_t)$ is nonincreasing, then $\eta_t\ge \eta_{t+1}$ for every $t$.
Since $A_t(\eta)$ is nonincreasing in $\eta$ by Lemma~\ref{lem:terminal-potential},
$A_t(\eta_t)-A_t(\eta_{t+1})\le 0$,
and summing proves the result.
\end{proof}

\begin{proof}[Proof of Corollary~\ref{cor:fixed-cumulant}.]
If $\eta_t\equiv\eta$, the temperature-change drift in \eqref{eq:exact-chain} vanishes.
Moreover $q_{T,\eta}=p_{T+1}$ by definition.
Substituting these facts into Theorem~\ref{thm:cumulant-chain} gives \eqref{eq:fixed-cumulant}.
\end{proof}

\begin{proof}[Proof of Theorem~\ref{thm:pacbayes-second-order}.]
Theorem~\ref{thm:cumulant-chain} gives $R_T^c(\rho) = \sum_{t=1}^T \phi_t(\eta_t) + D_T+B_T(\rho)$.
By the definition of $Q_t(c)$,
$\phi_t(\eta_t)=\eta_t Q_t(c)$,
so \eqref{eq:main-pacbayes} follows immediately. Since $\phi_t(\eta_t)=\delta_t(c)=\ip{p_t}{c_t}-m_t$,
Jensen's inequality gives $m_t\le \ip{p_t}{c_t}$ and therefore $\phi_t(\eta_t)\ge 0$, so $Q_t(c)\ge 0$.
If $\eta_t\equiv\eta$, then $D_T=0$, and \eqref{eq:main-pacbayes-reg} follows from the same identity.
\end{proof}

\begin{proof}[Proof of Corollary~\ref{cor:centered-seq}.]
Subtract $\ip{q_t}{c_t}$ from both sides of \eqref{eq:seq-id} and use $m_t=-\eta_t^{-1}\log Z_t$.
The original-loss form follows from $c_t=\ell_t+u_t$.
When $\eta_t\equiv\eta$, sum \eqref{eq:centered-seq} over $t$.
\end{proof}

\subsection{Proofs for Section~\ref{sec:schedules} (schedules and intrinsic time)}

\begin{proof}[Proof of Proposition~\ref{prop:surrogate-tightness}.]
The identity is Proposition~\ref{prop:tilted-var} minus $\tfrac12\sigma_t^2=\int_0^1(1-s)\sigma_t^2\,ds$. Write $g(s):=\Var_{p_{t,s}^{(\eta_t)}}(c_t)=\psi_t''(-s\eta_t)$, so $g(0)=\sigma_t^2$ and $g'(s)=-\eta_t\,\kappa_{3,s}$ with $\kappa_{3,s}$ the third central moment of $c_t$ under $p_{t,s}^{(\eta_t)}$. Each tilt is supported on $\{c_t(i)\}$, of spread $R_t$, so $|\kappa_{3,s}|\le R_t\,g(s)$ and $|g'(s)|\le\tau_t g(s)$. Gr\"onwall's inequality gives $\sigma_t^2 e^{-\tau_t s}\le g(s)\le\sigma_t^2 e^{\tau_t s}$; integrating against $(1-s)$ and using $\int_0^1(1-s)e^{as}\,ds=(e^a-1-a)/a^2$ yields the envelope. The absolute bound follows from $\tfrac{e^\tau-1-\tau}{\tau^2}-\tfrac12\ge\tfrac12-\tfrac{\tau-1+e^{-\tau}}{\tau^2}$, i.e.\ $2(\cosh\tau-1)\ge\tau^2$.
\end{proof}

\begin{proof}[Proof of Proposition~\ref{prop:range}.]
The standard bounded-range exponential-moment bound \cite[Lemma~2.2]{CesaBianchiLugosi2006} states that for any random variable $Y$ supported on an interval of length $L$ with $\E[Y]=0$ and any $\eta\in\R$,
\[
\log \E[e^{\eta Y}]\le \eta^2 L^2/8
\]
Apply this to $Y=\ip{p_t}{c_t}-c_t(I_t)$ with $I_t\sim p_t$, which is supported on the interval $[\ip{p_t}{c_t}-b_t,\ip{p_t}{c_t}-a_t]$ of length $b_t-a_t$ and has mean zero. Then
\[
\psi_t(-\eta_t)=\log \E[e^{-\eta_t(c_t(I_t)-\ip{p_t}{c_t})}]=\log \E[e^{\eta_t Y}]\le \eta_t^2(b_t-a_t)^2/8
\]
so $Q_t(c)=\eta_t^{-2}\psi_t(-\eta_t)\le (b_t-a_t)^2/8$.
\end{proof}

\begin{proof}[Proof of Proposition~\ref{prop:secondmoment}.]
$Q_t$ is unchanged by adding a constant to $c_t$, so replace $c_t$ by $c_t^{\circ}\ge0$.
Then $\log\E_{i\sim p_t}e^{-\eta c_t^{\circ}}\le0$, which is the second branch.
For the first, $e^{-x}\le1-x+x^2/2$ on $x\ge0$ gives
\[
\E_{i\sim p_t}e^{-\eta c_t^{\circ}}\le 1-\eta\ip{p_t}{c_t^{\circ}}+\tfrac{\eta^2}{2}\E_{i\sim p_t}\bigl[c_t^{\circ}(i)^2\bigr]
\]
whose right side is at least $1-y+y^2/2\ge\tfrac12$ at $y=\eta\ip{p_t}{c_t^{\circ}}$, by Jensen's inequality.
Taking logarithms and applying $\log(1+z)\le z$,
\[
\eta^2 Q_t(c)=\eta\ip{p_t}{c_t^{\circ}}+\log\E_{i\sim p_t}e^{-\eta c_t^{\circ}}\le\tfrac{\eta^2}{2}\,\E_{i\sim p_t}\bigl[c_t^{\circ}(i)^2\bigr]
\]
\end{proof}

\begin{proof}[Proof of Proposition~\ref{prop:adahedge-instance}.]
A uniform prior contributes the constant $-\log K$ to the exponent $\log\pi(i)-\eta_t C_{t-1}(i)$, which the normalization removes, so the retempered play equals the recomputed exponential-weights play.
For the loss term, $\psi_t(-\eta)=\log\sum_i p_t(i)e^{-\eta(c_t(i)-\ip{p_t}{c_t})}=\eta\ip{p_t}{c_t}-\eta\,m_t(\eta)$, and $Q_t(c)=\eta^{-2}\psi_t(-\eta)$ gives $\eta Q_t(c)=\ip{p_t}{c_t}-m_t(\eta)$.
Summing identifies $\sum_{s<t}\eta_s Q_s(c)$ with the cumulative mixability gap that drives AdaHedge, and both rules start from an empty gap, so they agree round by round.
On a round with an empty accumulated gap the convention for the temperature is immaterial: every earlier round then had a score vector constant across experts, so $C_{t-1}$ is a multiple of $\one$ and every temperature returns the uniform play.
\end{proof}

\begin{proof}[Proof of Proposition~\ref{prop:composite-clock-chain}.]
With $Z_u=\sum_i p_t(i)e^{-\eta u_t(i)}$ we have $\sum_i p_t(i)e^{-\eta(\ell_t(i)+u_t(i))}=Z_u\sum_i p_t^{(u)}(i)e^{-\eta\ell_t(i)}$.
Take logarithms, add $\eta\ip{p_t}{\ell_t}+\eta\ip{p_t}{u_t}$, and divide by $\eta^2$.
Nothing in the argument uses that $p_t$ is the played distribution, so the same computation holds at any $\nu$.
\end{proof}

\begin{proof}[Proof of Theorem~\ref{thm:scaling-time}.]
Write
\[
q_t:=Q_t(c),
\qquad
S_t := \sum_{s=1}^t q_s,
\qquad
q_*:=Q_*^T(c)=\max_{1\le t\le T} q_t,
\qquad
G_T := \sum_{t=1}^T \eta_t q_t\]
Then $S_t=V_t(c)$ and $S_T=V_T(c)$. Also,
\[
\eta_t=f(S_{t-1}),
\qquad
f(x):=
\begin{cases}
1, & \text{if }x\le C^2\Gamma,\\[1ex]
C\sqrt{\Gamma/x}, & \text{if }x>C^2\Gamma,
\end{cases}
\]
where $f$ is nonincreasing on $[0,\infty)$.

For the lower bound, monotonicity of $f$ gives $q_t f(S_{t-1})=\int_{S_{t-1}}^{S_t} f(S_{t-1})\,dx \ge \int_{S_{t-1}}^{S_t} f(x)\,dx$.
Summing over $t$ yields $G_T\ge \int_0^{S_T} f(x)\,dx$.
If $S_T\le C^2\Gamma$, then $\int_0^{S_T} f(x)\,dx=S_T$, and $S_T\ge 2C\sqrt{\Gamma S_T}-C^2\Gamma$ by $(\sqrt{S_T}-C\sqrt{\Gamma})^2\ge 0$. If $S_T>C^2\Gamma$,
then
\[
\int_0^{S_T} f(x)\,dx
=
C^2\Gamma+C\sqrt{\Gamma}\int_{C^2\Gamma}^{S_T}x^{-1/2}\,dx
=
2C\sqrt{\Gamma S_T}-C^2\Gamma\]
So in all cases $G_T\ge 2C\sqrt{\Gamma S_T}-C^2\Gamma$.

For the upper bound, use again that $f$ is nonincreasing.
For every $x\in[S_{t-1},S_t]$ we have $x-q_t\le S_{t-1}$, so $f(S_{t-1})\le f\left((x-q_t)_+\right)\le f\left((x-q_*)_+\right)$.
Therefore
\[
q_t f(S_{t-1})
=
\int_{S_{t-1}}^{S_t} f(S_{t-1})\,dx
\le
\int_{S_{t-1}}^{S_t} f\left((x-q_*)_+\right)\,dx\]
Summing over $t$ gives
\[
G_T\le \int_0^{S_T} f\left((x-q_*)_+\right)\,dx
\le q_*f(0)+\int_0^{S_T} f(x)\,dx
=
q_*+\int_0^{S_T} f(x)\,dx\]
The lower-bound computation above evaluated $\int_0^{S_T} f$ exactly: it is $S_T$ when $S_T\le C^2\Gamma$, and $2C\sqrt{\Gamma S_T}-C^2\Gamma$ otherwise.
Both are at most $2C\sqrt{\Gamma S_T}$, the first because $S_T\le C^2\Gamma$ forces $\sqrt{S_T}\le C\sqrt{\Gamma}\le 2C\sqrt{\Gamma}$, and the second by discarding $-C^2\Gamma$.
Hence $G_T\le q_*+2C\sqrt{\Gamma S_T}$, which proves \eqref{eq:budgeted-two-sided}.

Equation \eqref{eq:budgeted-two-sided-regret} follows by adding $D_T+B_T(\rho)$ to \eqref{eq:budgeted-two-sided} and using \eqref{eq:main-pacbayes}.
If $\KL(\rho\|\pi)\le \Gamma$, then $B_T(\rho)\le \Gamma\eta_T^{-1}$.
Moreover,
\[
\eta_T^{-1}
=
\begin{cases}
1, & \text{if }V_{T-1}(c)=0,\\[1ex]
\max\!\left\{1,\,C^{-1}\sqrt{\frac{V_{T-1}(c)}{\Gamma}}\right\},
& \text{if }V_{T-1}(c)>0,
\end{cases}
\]
so $\Gamma\eta_T^{-1}\le \Gamma + C^{-1}\sqrt{\Gamma V_T(c)}$.
Combining this with the upper half of \eqref{eq:budgeted-two-sided-regret} and $D_T\le 0$ gives \eqref{eq:budgeted-bound-C}.
The case $C=1/\sqrt{2}$ is \eqref{eq:budgeted-bound-opt}, while $C=1$ is the corresponding unoptimized specialization of \eqref{eq:budgeted-bound-C}.
\end{proof}

\begin{proof}[Proof of Theorem~\ref{thm:sampled-pacbayes} and Corollary~\ref{cor:sampled-second-order}.]
By definition,
\[
\widehat R_T^c(\rho)
=
\sum_{t=1}^T \left(c_t(I_t)-\ip{p_t}{c_t}\right)
+
\sum_{t=1}^T \ip{p_t}{c_t}
-
\ip{\rho}{C_T}
=
M_T^{\mathrm{sam}} + R_T^c(\rho)\]
Substituting \eqref{eq:main-pacbayes} proves \eqref{eq:sampled-exact}. 
The two-sided estimate \eqref{eq:sampled-two-sided} is then immediate from Theorem~\ref{thm:scaling-time}. 
If $\KL(\rho\|\pi) \leq \Gamma$, the upper bound \eqref{eq:sampled-second-order} follows from \eqref{eq:budgeted-bound-C}. 
Finally,
\[ 
\sum_{t=1}^T \ell_t(I_t)-\ip{\rho}{L_T}
= \sum_{t=1}^T c_t(I_t)-\ip{\rho}{C_T} + \sum_{t=1}^T \left(\ip{\rho}{u_t}-u_t(I_t)\right) \]
which is exactly Theorem~\ref{thm:sampled-pacbayes} for the original losses.
\end{proof}

\begin{proof}[Proof of Proposition~\ref{prop:ret-press-wellposed}.]
Write $\Phi_t(\eta):=\sum_i\pi(i)e^{-\eta C_t(i)}$, so that $A_t(\eta)=-\eta^{-1}\log\Phi_t(\eta)$ and $q_{t,\eta}(i)=\pi(i)e^{-\eta C_t(i)}/\Phi_t(\eta)$. Differentiating, $\Phi_t'(\eta)=-\Phi_t(\eta)\ip{q_{t,\eta}}{C_t}$, hence
\[
A_t'(\eta)=\eta^{-2}\log\Phi_t(\eta)+\eta^{-1}\ip{q_{t,\eta}}{C_t}
\]
Since $\KL(q_{t,\eta}\|\pi)=\sum_i q_{t,\eta}(i)\bigl(-\eta C_t(i)-\log\Phi_t(\eta)\bigr)=-\eta\ip{q_{t,\eta}}{C_t}-\log\Phi_t(\eta)$, this equals $A_t'(\eta)=-\eta^{-2}\KL(q_{t,\eta}\|\pi)\le0$. Differencing the identity at $t$ and $t-1$ gives \eqref{eq:ret-press-deriv}. For the endpoints, $\log\Phi_t(\eta)=-\eta\ip{\pi}{C_t}+O(\eta^2)$ as $\eta\to0^+$ gives $A_t(\eta)\to\ip{\pi}{C_t}$, so $m_t^{\rm ret}(0^+)=\ip{\pi}{c_t}$; and $\eta^{-1}\log\Phi_t(\eta)\to-\min_i C_t(i)$ as $\eta\to\infty$ gives $m_t^{\rm ret}(\infty)=\min_i C_t(i)-\min_j C_{t-1}(j)$. The map $\eta\mapsto m_t^{\rm ret}(\eta)$ is continuously differentiable on $(0,\infty)$ with these limits, so its image is the interval between its infimum and supremum, and by the intermediate value theorem $m_t^{\rm ret}(\eta)=a_t$ has a root for every $a_t$ strictly inside that image. When $\Delta_t$ is single-signed the curve is strictly monotone, its image is the interval spanned by the two endpoint values. The root is unique. A sign change of $\Delta_t$ makes $m_t^{\rm ret}$ non-monotone, so its image overshoots one endpoint and a level set can contain several points, of which the schedule takes the smallest.
\end{proof}

\subsection{Proofs for Section~\ref{sec:return-luck} (side information, comparators, and luckiness)}

\begin{proof}[Proof of Corollaries~\ref{cor:actual-losses},~\ref{cor:adaptive-mismatch}, and~\ref{cor:actual-loss-bounds}.]
From $c_t=\ell_t+u_t$ we have, for each round $t$,
$\ip{p_t}{\ell_t} - \ip{\rho}{\ell_t} = \ip{p_t}{c_t} -\ip{\rho}{c_t} + \ip{\rho}{u_t} - \ip{p_t}{u_t}$.
Summing over $t$ and applying \eqref{eq:main-pacbayes} gives \eqref{eq:actual-loss-bound}.
If \eqref{eq:side-decomp-adaptive} holds, then $u_t(i)=-\eta_t^{-1}\sum_{w=1}^{J_t} \alpha_{t,w}\log \mu_{t,w}(i)$,
so
\begin{align*}
\ip{\rho}{u_t}-\ip{p_t}{u_t}
&= -\eta_t^{-1}\sum_{w=1}^{J_t}\alpha_{t,w}
\sum_{i=1}^K \left(\rho(i)-p_t(i)\right)\log \mu_{t,w}(i) \\
&= \sum_{w=1}^{J_t}\frac{\alpha_{t,w}}{\eta_t}
\left(\Hc(\rho,\mu_{t,w})-\Hc(p_t,\mu_{t,w})\right)\end{align*}
Summing over $t$ proves \eqref{eq:adaptive-mismatch-crossent}. 
The two-sided statement \eqref{eq:actual-loss-two-sided} follows by substituting \eqref{eq:budgeted-two-sided-regret} into \eqref{eq:actual-loss-bound}, and \eqref{eq:actual-loss-budgeted} follows from \eqref{eq:budgeted-bound-C}.
\end{proof}

\begin{proof}[Proof of Corollary~\ref{cor:point-luckiness}.]
Part \textup{(i)} is exactly Theorem~\ref{thm:fixed-luckiness} specialized to $\rho=\delta_{k^*}$, because then \eqref{eq:rho-massart} becomes \eqref{eq:point-massart} and $\KL(\delta_{k^*}\|\pi)=\log(1/\pi(k^*))$.

For part \textup{(ii)}, let $i\neq k^*$.
Since $c_t(i),c_t(k^*)\in[0,1]$, we have $|c_t(i)-c_t(k^*)|\le 1$ and therefore $(c_t(i)-c_t(k^*))^2\le 1$ almost surely. Therefore,
$\E\bigl[(c_t(i)-c_t(k^*))^2\bigr]\le 1 \le \frac{\mu(i)-\mu(k^*)}{d_{\min}}$.
For $i=k^*$ the inequality is trivial because the left-hand side is $0$.
Thus \eqref{eq:point-massart} holds with $c_*=1/d_{\min}$, and part \textup{(i)} yields the claimed bound.
\end{proof}

\begin{proof}[Proof of Proposition~\ref{prop:diffuse-domination}.]
Fix any algorithm, generating played weights $(p_t)_{t\le T}$, and any $\rho\in\Delta([K])$.
From the definition $R_T^c(\rho)=\sum_{t=1}^T\ip{p_t}{c_t}-\ip{\rho}{C_T}$, the played term is the same for every comparator, so pathwise
\[
R_T^c(\rho)-R_T^c(k^\star)
=\ip{\delta_{k^\star}}{C_T}-\ip{\rho}{C_T}
=C_T(k^\star)-\ip{\rho}{C_T}
\]
The right-hand side does not involve the algorithm's play.
Since the losses are i.i.d.\ with mean $\mu$ and $\rho$ is fixed, linearity of expectation gives $\E[C_T(k^\star)]=T\mu(k^\star)$ and $\E[\ip{\rho}{C_T}]=T\ip{\rho}{\mu}$, so that
\[
\E\bigl[R_T^c(\rho)\bigr]-\E\bigl[R_T^c(k^\star)\bigr]
=-T\bigl(\ip{\rho}{\mu}-\mu(k^\star)\bigr)
\]
Because $k^\star\in\arg\min_i\mu(i)$, every convex combination satisfies $\ip{\rho}{\mu}\ge\mu(k^\star)$, so the displayed identity and the inequality $\E[R_T^c(\rho)]\le\E[R_T^c(k^\star)]$ follow.

For the second display, assume the mean gap $d_{\min}=\min_{i\ne k^\star}(\mu(i)-\mu(k^\star))$ is positive.
Corollary~\ref{cor:point-luckiness}\textup{(ii)} applied at $k^\star$ bounds $\E[R_T^c(k^\star)]\le 2\bigl(1+\tfrac{2(e-2)}{d_{\min}}\bigr)\log\tfrac1{\pi(k^\star)}$, a constant independent of $T$, and the domination inequality extends the same bound to every $\rho\in\Delta([K])$.
Under \eqref{eq:point-massart} with a general constant $c_*$, the same argument runs through Corollary~\ref{cor:point-luckiness}\textup{(i)} with $c_*$ in place of $1/d_{\min}$.
\end{proof}

\begin{proof}[Proof of Corollary~\ref{cor:budget-luckiness}.]
Let $a:=e-2$ and write
\[
x_T := \E\bigl[R_T^c(\rho)\bigr],
\qquad
V_T:=V_T(c)=\sum_{t=1}^T Q_t(c),
\qquad
Q_*:=Q_*^T(c)\]
Fix $C>0$ and run the schedule \eqref{eq:budget-schedule-C}. By \eqref{eq:budgeted-bound-C},
$R_T^c(\rho)\le \Gamma + Q_* + (2C+C^{-1})\sqrt{\Gamma V_T}$.
Taking expectations and using Jensen's inequality for the concave square root gives
\begin{equation}\label{eq:lucky-budget-jensen}
x_T \le \Gamma + \E[Q_*] + (2C+C^{-1})\sqrt{\Gamma\,\E[V_T]}\end{equation}
Next, Proposition~\ref{prop:comparator-second-order} implies $Q_t(c)\le a\,\Psi_t(\rho) \qquad\text{pathwise for every }t$.
Because $p_t$ is $\mathcal F_{t-1}$-measurable while $c_t$ is independent of the past, conditioning on $\mathcal F_{t-1}$ and applying \eqref{eq:rho-massart} coordinatewise under $p_t$ gives
\[
\E[Q_t(c)\mid \mathcal F_{t-1}]
\le a\,\E[\Psi_t(\rho)\mid \mathcal F_{t-1}]
\le a\kappa_\rho\,\E\bigl[\ip{p_t}{c_t}-\ip{\rho}{c_t}\mid \mathcal F_{t-1}\bigr]\]
Summing over $t$ gives
\begin{equation}\label{eq:V-self-bounds}
\E[V_T] \le a\kappa_\rho x_T\end{equation}
Combining \eqref{eq:lucky-budget-jensen} and \eqref{eq:V-self-bounds} yields $x_T \le \Gamma + \E[Q_*] + (2C+C^{-1})\sqrt{a\kappa_\rho\Gamma x_T}$.
Applying $2uv\le u^2+v^2$ with $u=\sqrt{x_T}$ and $v=(2C+C^{-1})\sqrt{a\kappa_\rho\Gamma}$, we obtain $x_T \le \Gamma + \E[Q_*] + \frac{x_T}{2} + \frac{(2C+C^{-1})^2}{2}a\kappa_\rho\Gamma$.
Rearranging proves $x_T \le 2\Gamma + 2\,\E[Q_*] + (2C+C^{-1})^2 a\kappa_\rho\Gamma$,
which is \eqref{eq:budget-luckiness-raw-C}. The final displayed bound in Corollary~\ref{cor:budget-luckiness} then follows from Proposition~\ref{prop:range} and the choice $C=1/\sqrt{2}$.
\end{proof}

\begin{proof}[Proof of Corollary~\ref{cor:actual-loss-bounds}.]
Apply Corollary~\ref{cor:actual-losses} with $u_t=-m_t$. This gives
\[
\sum_{t=1}^T \ip{p_t}{\ell_t}-\ip{\rho}{L_T}
=
D_T+B_T(\rho)+\sum_{t=1}^T \eta_tQ_t(\ell-m)
+\sum_{t=1}^T \ip{p_t-\rho}{m_t}\]
which is the first display. The residual-range bound follows from Proposition~\ref{prop:range} applied to $c_t=\ell_t-m_t$.
The fixed-rate statement then uses $D_T=0$ and $B_T(\rho)\le \KL(\rho\|\pi)/\eta$, while the second-order-schedule statement is Theorem~\ref{thm:scaling-time} applied to the same residual sequence.
\end{proof}

\subsection{Proofs for Appendix~\ref{sec:further} (full-information extensions)}

\begin{proof}[Proof of Corollary~\ref{cor:classical-gap}.]
Corollary~\ref{cor:fixed-cumulant} gives
\[
R_T^c(\rho)
=
\eta\sum_{t=1}^T Q_t(c)
+
\frac{\KL(\rho\|\pi)-\KL(\rho \| p_{T+1})}{\eta}
\]
By Proposition~\ref{prop:range}, $Q_t(c)\le (b_t-a_t)^2/8$ on every round.
Adding and subtracting $\eta S_T/8$ gives \eqref{eq:classical-gap}, and \eqref{eq:classical-gap-ineq} is immediate.
\end{proof}

\begin{proof}[Proof of Corollary~\ref{cor:classical-gap-sqrt}.]
Substitute $\eta_\Gamma$ into \eqref{eq:classical-gap}.
Since $\Gamma/\eta_\Gamma+\eta_\Gamma S_T/8=\sqrt{\Gamma S_T/2}$ and $\KL(\rho\|\pi)\le \Gamma$, the displayed inequality follows.
\end{proof}

\begin{proof}[Proof of Proposition~\ref{prop:direct-adaptive-pathwise}.]
Multiply \eqref{eq:centered-seq} by $\eta_t$ and sum over $t$.
\end{proof}

\subsection{Proofs for Section~\ref{sec:pressure} (local updates and pressure)}

\begin{proof}[Proof of Proposition~\ref{prop:pressure-normalized}.]
Set $f(\eta) := \sum_{i=1}^K p_t(i)e^{-\eta(c_t(i)-a_t)}$.
Then $f$ is $C^\infty$ on $\R$, strictly convex because $c_t$ is nonconstant, and satisfies $f(0)=1$ and $f'(0)=-(\ip{p_t}{c_t}-a_t)<0$ by the hypothesis $a_t<\ip{p_t}{c_t}$.
Since $\min_i(c_t(i)-a_t)<0$ by the hypothesis $a_t>\min_i c_t(i)$, the exponent $-\eta(c_t(i_{\min})-a_t)\to+\infty$ as $\eta\to\infty$, so $f(\eta)\to\infty$.
A strictly convex $C^\infty$ function with $f(0)=1$, $f'(0)<0$, and $\lim_{\eta\to\infty} f(\eta)=+\infty$ has a unique positive minimizer $\eta^\star$ and a unique positive root $\eta_t$ on $(\eta^\star,\infty)$ at which $f(\eta_t)=1$.
This $\eta_t$ is exactly the solution of \eqref{eq:pressure-target-root}, and dividing its logarithm by $-\eta_t$ gives \eqref{eq:pressure-target}.

If $c_t$ and $a_t$ are simultaneously scaled by $b>0$, then $f_b(\eta)=\sum_i p_t(i)e^{-\eta(bc_t(i)-ba_t)}=f(b\eta)$,
so the unique positive root scales as $\eta_t/b$.

Finally, when \eqref{eq:pressure-target-root} holds, the update based on the shifted scores $c_t-a_t\mathbf 1$ is $p_{t+1}(i)=p_t(i)e^{-\eta_t(c_t(i)-a_t)}$.
Because \eqref{eq:pressure-target-root} implies $e^{\eta_t a_t}=1/\sum_j p_t(j)e^{-\eta_t c_t(j)}$, this is the same normalized update as $p_{t+1}(i)\propto p_t(i)e^{-\eta_t c_t(i)}$.
Therefore $\log\frac{p_t(i)}{p_{t+1}(i)}=\eta_t(c_t(i)-a_t)$,
and multiplying by $\rho(i)$ and summing over $i$ proves \eqref{eq:pressure-target-kl}.
\end{proof}

\begin{proof}[Proof of Corollary~\ref{cor:pressure-unit}.]
Immediate specialization of Proposition~\ref{prop:pressure-normalized} to $a_t=0$: the feasibility condition $\min_i c_t(i)<0<\ip{p_t}{c_t}$ is the hypothesis, and \eqref{eq:pressure-target-root} then reads $\sum_i p_t(i)e^{-\eta_t c_t(i)}=1$.
\end{proof}

\begin{proof}[Proof of Proposition~\ref{prop:pressure-fixed-var}.]
At a common temperature $\eta$, the local recursion $p_{t+1}^{(\eta)}(i)\propto p_t^{(\eta)}(i)e^{-\eta c_t(i)}$ with initial distribution $\pi$ unrolls to
\[
p_t^{(\eta)}(i)
=
\frac{\pi(i)e^{-\eta\sum_{s=1}^{t-1} c_s(i)}}{\sum_j \pi(j)e^{-\eta\sum_{s=1}^{t-1} c_s(j)}}
=
\frac{\pi(i)e^{-\eta C_{t-1}(i)}}{\sum_j \pi(j)e^{-\eta C_{t-1}(j)}}
\]
which is identical to the prior-retempered update \eqref{eq:retempered-hedge} at the same fixed $\eta$. Equation \eqref{eq:pressure-fixed-var} is then exactly Corollary~\ref{cor:fixed-cumulant} with $\psi_t$ replaced by $\psi_t^{(\eta)}$ (i.e., evaluated on $p_t^{(\eta)}$).
\end{proof}

\begin{proof}[Proof of Corollary~\ref{cor:pressure-self-correcting}.]
Start from \eqref{eq:pressure-fixed-var}, add and subtract $\sum_{t=1}^T v_t(\eta)$:
\[
R_T^{c,(\eta)}(\rho)
=
\sum_{t=1}^T v_t(\eta)
+\sum_{t=1}^T\left(\frac{\psi_t^{(\eta)}(-\eta)}{\eta}-v_t(\eta)\right)
+\frac{\KL(\rho\|\pi)}{\eta}-\frac{\KL(\rho\|p_{T+1}^{(\eta)})}{\eta}
\]
Rearranging and using the hypothesis $v_t(\eta)\ge \psi_t^{(\eta)}(-\eta)/\eta$ gives \eqref{eq:pressure-self-correcting} with nonnegative slack $\Delta_T^{\mathrm{loc}}(\rho,\eta)$ as displayed.
\end{proof}

\begin{proof}[Proof of Proposition~\ref{prop:pressure-cumulative}.]
Because \eqref{eq:pressure-target-root} holds at every round, $Z_t:=\sum_i p_t(i)e^{-\eta_t c_t(i)}=e^{-\eta_t a_t}$.
Unrolling the recursive update $p_{T+1}(i)=p_1(i)\prod_{t=1}^T e^{-\eta_t c_t(i)}/Z_t$ gives
\[
p_{T+1}(i)
=
p_1(i)\exp\!\left(-\sum_{t=1}^T \eta_t c_t(i)+\sum_{t=1}^T \eta_t a_t\right)
=
p_1(i)e^{-G_T(i)}
\]
which is the first part of \eqref{eq:pressure-cumulative-normalization}; the second part $\sum_i p_1(i)e^{-G_T(i)}=\sum_i p_{T+1}(i)=1$ follows from the fact that $p_{T+1}\in\Delta([K])$.
Taking logarithms of $p_{T+1}(i)/p_1(i)=e^{-G_T(i)}$, multiplying by $\rho(i)$, and summing over $i$ yields
\[
-\KL(\rho\|p_1)+\KL(\rho\|p_{T+1})
=
-\sum_i \rho(i)G_T(i)
=
-\sum_{t=1}^T \eta_t\bigl(\ip{\rho}{c_t}-a_t\bigr)
\]
which is \eqref{eq:pressure-cumulative-kl} after reversing signs.
For \eqref{eq:pressure-cumulative-regret}, note that $a_t=m_t(\eta_t)=\mu_t-\eta_t^{-1}\psi_t(-\eta_t)$, so $\eta_t(\mu_t-a_t)=\psi_t(-\eta_t)=\eta_t^2 Q_t^{\mathrm{loc}}(c)$. Summing and combining with \eqref{eq:pressure-cumulative-kl} proves the claim.
\end{proof}

\begin{proof}[Proof of Proposition~\ref{prop:regret-spectrum-freeenergy}.]
With uniform $\pi$, $\mu_i=e^{-\eta C_T(i)}/Z(\eta)$, $Z(\beta):=\sum_i e^{-\beta C_T(i)}$, so $\sum_i\mu_i^q=Z(q\eta)/Z(\eta)^q$ and $H_q(\mu)=\tfrac1{1-q}\log\sum_i\mu_i^q=(\Lambda_T(q\eta)-q\Lambda_T(\eta))/(1-q)$. Substituting $\Lambda_T(\beta)=-\beta A_T(\beta)+\log K$ gives the $A_T$ form and, via $D_q(T)=-H_q(\mu)/\log T$, the displayed spectrum. The curvature statement is $\Lambda_T''(\beta)=\Var_{\nu_\beta}(C_T)\ge0$, with equality iff $C_T$ is constant.
\end{proof}

\begin{proof}[Proof of Proposition~\ref{prop:pressure-coincidence}.]
For every $\eta>0$ and every $i$,
\[
p_t(i)^{1-\eta}r_t(i)^\eta
=
p_t(i)^{1-\eta}\left(\frac{p_t(i)e^{-c_t(i)}}{Z_t(1)}\right)^\eta
=
\frac{p_t(i)e^{-\eta c_t(i)}}{Z_t(1)^\eta}
\]
Summing over $i$ gives $\cC_t(\eta)=Z_t(\eta)/Z_t(1)^\eta$, which is \eqref{eq:pressure-coinc-factor}. Taking logarithms gives $\log Z_t(\eta)=\eta\log Z_t(1)+\log \cC_t(\eta)$, hence
\[
\delta_t^{\mathrm{loc}}(\eta)
=
\mu_t+\eta^{-1}\log Z_t(\eta)
=
\mu_t+\log Z_t(1)+\eta^{-1}\log \cC_t(\eta)
\]
In particular $\delta_t^{\mathrm{loc}}(1)=\mu_t+\log Z_t(1)$. Subtracting the last two displays gives the first equality in \eqref{eq:pressure-coinc-gap}.
For the second equality, $\log \cC_t(\eta)=(\eta-1)D_\eta(r_t\|p_t)$ follows directly from the definition of the order-$\eta$ R\'enyi divergence, and dividing by $\eta$ gives $\eta^{-1}\log \cC_t(\eta)=-(1-\eta)D_\eta(r_t\|p_t)/\eta$.
For $0<\eta<1$, the R\'enyi divergence $D_\eta(r_t\|p_t)\ge 0$, so $\log \cC_t(\eta)\le 0$, i.e., $\cC_t(\eta)\in(0,1]$.
\end{proof}

\begin{proof}[Proof of Proposition~\ref{prop:pressure-event}.]
From Proposition~\ref{prop:pressure-cumulative}, $p_{T+1}(E)=\sum_{i\in E}p_1(i)e^{-G_T(i)}=\mathcal Z_T(E)$.
For any $\rho\in\Delta(E)$, the identity
\[
\KL(\rho\|p_{T+1}^E)
=
\sum_{i\in E}\rho(i)\log\frac{\rho(i)p_{T+1}(E)}{p_{T+1}(i)}
=
\KL(\rho\|p_{T+1})+\log p_{T+1}(E)
\]
is immediate from $p_{T+1}^E(i)=p_{T+1}(i)/p_{T+1}(E)$ for $i\in E$.
Substituting \eqref{eq:pressure-cumulative-kl} and $\log p_{T+1}(E)=\log \mathcal Z_T(E)$ gives \eqref{eq:pressure-event-kl}.
Subtracting \eqref{eq:pressure-event-kl} from the identity $\sum_t \eta_t(\mu_t-a_t)=\sum_t \eta_t^2 Q_t^{\mathrm{loc}}(c)$ (which follows from $a_t=m_t(\eta_t)$) yields \eqref{eq:pressure-event-pointwise}.
Equations~\eqref{eq:pressure-event-var} and \eqref{eq:pressure-event-sup} follow by taking $\min_{\rho\in\Delta(E)}$ and $\sup_{\rho\in\Delta(E)}$ of \eqref{eq:pressure-event-kl} and \eqref{eq:pressure-event-pointwise} respectively. Nonnegativity of $\KL(\rho\|p_{T+1}^E)$, vanishing at $\rho=p_{T+1}^E$, identifies the extremizer in both cases.
\end{proof}

\begin{proof}[Proof of Corollary~\ref{cor:pressure-centered} and Corollary~\ref{cor:pressure-quadratic}.]
For Corollary~\ref{cor:pressure-centered}: apply Proposition~\ref{prop:pressure-normalized} to the centered scores $c_t=-g_t$ and target $a_t=-b_t$.
The feasibility condition $\min_i c_t(i)<a_t<\ip{p_t}{c_t}$ translates to $0<b_t<\max_i g_t(i)$ (using $\ip{p_t}{g_t}=0$), which is the hypothesis.
The pressure-target root equation becomes $\sum_i p_t(i)e^{\eta_t g_t(i)}=e^{\eta_t b_t}$, i.e., $\log\sum_i p_t(i)e^{\eta_t g_t(i)}=\eta_t b_t$, which is \eqref{eq:pressure-centered-root}.
Proposition~\ref{prop:pressure-cumulative} then gives $p_{T+1}(i)=p_1(i)\exp(\sum_t \eta_t(g_t(i)-b_t))$ with $\sum_i p_1(i)\exp(\sum_t\eta_t(g_t(i)-b_t))=1$, which is \eqref{eq:pressure-centered-potential}.
Equation~\eqref{eq:pressure-cumulative-kl} (with the same substitutions) is \eqref{eq:pressure-centered-regret}.
The upper bound \eqref{eq:pressure-centered-upper} follows from nonnegativity of $\KL(\rho\|p_{T+1})$.

For Corollary~\ref{cor:pressure-quadratic}: $\eta_t b_t=\log\sum_i p_t(i)e^{\eta_t g_t(i)}=\kappa_t\eta_t^2$ by definition of $\kappa_t$, so \eqref{eq:pressure-centered-root} is equivalent to \eqref{eq:pressure-quadratic-root}. Substituting $b_t=\kappa_t\eta_t$ into \eqref{eq:pressure-centered-regret} gives \eqref{eq:pressure-quadratic-regret}. If $\bar\kappa_t\ge \kappa_t$, then $\sum_t\kappa_t\eta_t^2\le \sum_t\bar\kappa_t\eta_t^2$, and dropping the nonnegative term $\KL(\rho\|p_{T+1})$ from \eqref{eq:pressure-quadratic-regret} gives \eqref{eq:pressure-quadratic-upper}.
\end{proof}

\begin{proof}[Proof of Corollary~\ref{cor:pressure-regret}.]
\eqref{eq:pressure-target-kl} rearranged gives
\[
\ip{p_t}{c_t}-\ip{\rho}{c_t}
=
\bigl(\ip{p_t}{c_t}-a_t\bigr)+\bigl(a_t-\ip{\rho}{c_t}\bigr)
=
\bigl(\ip{p_t}{c_t}-a_t\bigr)+\frac{\KL(\rho\|p_t)-\KL(\rho\|p_{t+1})}{\eta_t}
\]
which is \eqref{eq:pressure-one-step-regret}. Summing over $t$ yields \eqref{eq:pressure-cum-regret}.
Applying summation by parts $\sum_{t=1}^T a_t(K_t-K_{t+1})=a_1 K_1-a_T K_{T+1}+\sum_{t=2}^T K_t(a_t-a_{t-1})$ with $a_t=1/\eta_t$ and $K_t=\KL(\rho\|p_t)$ gives \eqref{eq:pressure-cum-regret-abel}.
If $(\eta_t)$ is nondecreasing, then $1/\eta_t-1/\eta_{t-1}\le 0$ and $\KL(\rho\|p_t)\ge 0$, so the last two terms in \eqref{eq:pressure-cum-regret-abel} are nonpositive, and dropping both (as well as the nonnegative $\KL(\rho\|p_{T+1})/\eta_T$) proves \eqref{eq:pressure-monotone-upper}.
\end{proof}

\subsection{Proofs for Section~\ref{sec:schedules} (sampled and anytime consequences)}

\begin{proof}[Proof of Theorem~\ref{thm:sampled-highprob}.]
Let $X_t:=c_t(I_t)-\ip{p_t}{c_t}$. Then $(X_t)_{t=1}^T$ is a martingale-difference sequence,
$X_t\le 1$ almost surely because $c_t\in[0,1]^K$, and $\sum_{t=1}^T \E[X_t^2\mid \mathcal{F}_{t-1}] = \sum_{t=1}^T \Var_{i\sim p_t}(c_t(i)) = V^{\!\sqcup}_T$.
Freedman's inequality therefore gives
\[
M_T^{\mathrm{sam}}=\sum_{t=1}^T X_t
\le \sqrt{2V^{\!\sqcup}_T\log(1/\delta)}+\tfrac13\log(1/\delta)
\]
with probability at least $1-\delta$.
On that event, \eqref{eq:sampled-exact} yields \eqref{eq:sampled-highprob} for every $\rho$.
Because the event depends only on $M_T^{\mathrm{sam}}$, the conclusion is simultaneous over all posteriors.
The second-order version \eqref{eq:sampled-highprob-noshift} follows by combining the same event with \eqref{eq:sampled-second-order}.
Finally, since a $[0,1]$-valued random variable has variance at most $1/4$, we have $V^{\!\sqcup}_T\le T/4$.
\end{proof}

\begin{proof}[Proof of Theorem~\ref{thm:sampled-anytime}.]
For each $t$, let $X_t:=c_t(I_t)-\ip{p_t}{c_t}$. Then $M_t^{\mathrm{sam}}=\sum_{s=1}^t X_s$ and,
conditionally on $\mathcal{F}_{t-1}$,
\[
\E\bigl[e^{\lambda X_t}\mid \mathcal{F}_{t-1}\bigr]
=
\sum_{i=1}^K p_t(i)e^{\lambda(c_t(i)-\ip{p_t}{c_t})}
=
\exp\!\left(\psi_t(\lambda)\right)\]
Therefore
\[
\E\bigl[Z_t(\lambda)\mid \mathcal{F}_{t-1}\bigr]
=
Z_{t-1}(\lambda)
\E\bigl[e^{\lambda X_t-\psi_t(\lambda)}\mid \mathcal{F}_{t-1}\bigr]
=
Z_{t-1}(\lambda)\]
so $Z_t(\lambda)$ is a nonnegative martingale. the nonnegative-supermartingale maximal inequality therefore yields $\Pr\!\left(\sup_{t\le T} Z_t(\lambda)\ge \delta^{-1}\right)\le \delta$.
Equivalently, with probability at least $1-\delta$,
\[
M_t^{\mathrm{sam}}
\le
\frac{\log(1/\delta)+\sum_{s=1}^t \psi_s(\lambda)}{\lambda}
\qquad\text{for all }t\le T\]
For each fixed $t$, Theorem~\ref{thm:sampled-pacbayes} applied to the truncated sequence $c_1,\dots,c_t$ gives $\widehat R_t^c(\rho) = M_t^{\mathrm{sam}}+D_t+B_t(\rho)+\sum_{s=1}^t \eta_sQ_s(c)$.
Combining the last two displays proves \eqref{eq:anytime-confidence}; because the high-probability event depends only on the sampling martingale, the conclusion is simultaneous over all posteriors.
Under the second-order schedule, Theorem~\ref{thm:scaling-time} applied at horizon $t$ gives $\sum_{s=1}^t \eta_sQ_s(c) \le Q_{*,t}(c)+2C\sqrt{\Gamma V_t(c)}$,
while $B_t(\rho)\le \Gamma\eta_t^{-1}\le \Gamma+C^{-1}\sqrt{\Gamma V_t(c)}$ whenever $\KL(\rho\|\pi)\le \Gamma$.
Substituting into \eqref{eq:anytime-confidence} yields \eqref{eq:anytime-confidence-noshift}.
\end{proof}

\begin{proof}[Proof of Proposition~\ref{prop:pathwise-lil}.]
Fix $\varepsilon\in(0,1)$. Choose $\gamma=\gamma_\varepsilon>0$ so small that $(1+\gamma)^{1/2}\le 1+\varepsilon/8$, and set $v_n:=(1+\gamma)^n$ for $n\ge 3$.
Let
\[
\delta_n := \frac{1}{n(\log n)^2},
\qquad
\alpha_n := \log\frac{1}{\delta_n},
\qquad
\lambda_n := \sqrt{\frac{2\alpha_n}{v_{n+1}}}\]
Since $\sum_{n\ge 3}\delta_n<\infty$ and $\alpha_n=o(v_{n+1})$, we have $\lambda_n\downarrow 0$.
For each $n\ge 3$, let
\[
F_n := \left\{M_t^{\mathrm{sam}}
\le
\frac{\alpha_n+\sum_{s=1}^t \psi_s(\lambda_n)}{\lambda_n}
\text{ for all }t\ge 1\right\}\]
By the same nonnegative-martingale argument used in the proof of Theorem~\ref{thm:sampled-anytime},
$\Pr(F_n^c)\le \delta_n$. Therefore Borel--Cantelli implies that the event $\Omega_\varepsilon := \{F_n\text{ holds for all but finitely many }n\}$ has probability one.

Fix a sample path in $\Omega_\varepsilon\cap\{V^{\!\sqcup}_t\to\infty\}$, and let $t$ be large enough that $V^{\!\sqcup}_t\ge v_3$.
Choose the unique $n=n(t)\ge 3$ such that $v_n\le V^{\!\sqcup}_t<v_{n+1}$.
For all sufficiently large such $t$, the event $F_n$ holds and $\lambda_n<1$.
Since $X_s\in[-1,1]$ and $\E[X_s\mid \mathcal{F}_{s-1}]=0$, the scalar inequality $e^x\le 1+x+x^2/(2(1-x/3))$ for $x<3$ gives $\psi_s(\lambda_n) \le \frac{\lambda_n^2}{2(1-\lambda_n/3)}\Var(X_s\mid \mathcal{F}_{s-1})$.
Summing over $s\le t$ and using $V^{\!\sqcup}_t<v_{n+1}$ yields
\[
M_t^{\mathrm{sam}}
\le \frac{\alpha_n}{\lambda_n}
+ \frac{\lambda_n}{2(1-\lambda_n/3)}V^{\!\sqcup}_t
\le \sqrt{\frac{\alpha_n v_{n+1}}{2}}
\left(1+\frac{1}{1-\lambda_n/3}\right)\]
Because $\lambda_n\to 0$, for all large $n$ this is at most
\[
\left(1+\frac{\varepsilon}{8}\right)\sqrt{2\alpha_n v_{n+1}}
\le \left(1+\frac{\varepsilon}{8}\right)\sqrt{2(1+\gamma)\alpha_n V^{\!\sqcup}_t}\]
Now $\alpha_n=\log n+2\log\log n$ and $\log\log V^{\!\sqcup}_t \ge \log\log v_n = \log\left(n\log(1+\gamma)\right) = \log n + O(1)$,
so $\alpha_n/\log\log V^{\!\sqcup}_t\to 1$ as $t\to\infty$.
Therefore, for all large $t$ on the chosen path,
$M_t^{\mathrm{sam}} \le (1+\varepsilon)\sqrt{2V^{\!\sqcup}_t\log\log V^{\!\sqcup}_t}$.
Hence for each fixed $\varepsilon>0$ there is a probability-one event $E_\varepsilon$ on which the last display holds eventually.
Defining $E_0 := \bigcap_{m\ge 1}E_{1/m}$ gives a probability-one event such that on $E_0\cap\{V^{\!\sqcup}_t\to\infty\}$,
$\limsup_{t\to\infty} \frac{M_t^{\mathrm{sam}}}{\sqrt{2V^{\!\sqcup}_t\log\log V^{\!\sqcup}_t}} \le 1$,
which is \eqref{eq:lil-limsup} after applying the same argument to $-X_t$ on another probability-one event $E_0^-$ to obtain the two-sided absolute-value form.
The regret consequence \eqref{eq:finite-lil-regret} then follows from the identity $\widehat R_t^c(\rho)-D_t-B_t(\rho)-\sum_{s=1}^t \eta_sQ_s(c)=M_t^{\mathrm{sam}}$ of Theorem~\ref{thm:sampled-pacbayes}.
\end{proof}

\begin{proof}[Proof of Theorem~\ref{thm:shifting} and Theorem~\ref{thm:simul-quantile}.]
For Theorem~\ref{thm:shifting}, observe that
\[
\ip{\rho_T}{C_T}-\sum_{t=1}^T \ip{\rho_t}{c_t}
=
\sum_{t=1}^T \ip{\rho_T-\rho_t}{c_t}
=
\sum_{t=1}^{T-1}\sum_{s=1}^t \ip{\rho_{t+1}-\rho_t}{c_s}
=
\sum_{t=1}^{T-1}\ip{\rho_{t+1}-\rho_t}{C_t}\]
Therefore
\[
R_T^{c,\mathrm{dyn}}(\rho_{1:T})
=
\sum_{t=1}^T \ip{p_t}{c_t}-\ip{\rho_T}{C_T}
+\sum_{t=1}^{T-1}\ip{\rho_{t+1}-\rho_t}{C_t}
=
R_T^c(\rho_T)+\sum_{t=1}^{T-1}\ip{\rho_{t+1}-\rho_t}{C_t}\]
Substituting \eqref{eq:main-pacbayes} gives \eqref{eq:shifting-exact}, and substituting the two-sided bound \eqref{eq:budgeted-two-sided-regret} gives \eqref{eq:shifting-two-sided}.
If $c_t(i)\in[0,1]$, then $0\le C_t(i)\le t$ for every $i$, so $\left|\ip{\rho_{t+1}-\rho_t}{C_t}\right| \le 2t\,\TV(\rho_{t+1},\rho_t)$.
Summing over $t$ proves $\left|\sum_{t=1}^{T-1}\ip{\rho_{t+1}-\rho_t}{C_t}\right|
\le
2\sum_{t=1}^{T-1} t\,\TV(\rho_{t+1},\rho_t)$, and combining this with
\eqref{eq:budgeted-bound-C} proves \eqref{eq:shifting-bound}.

For Theorem~\ref{thm:simul-quantile}, let $j:=j(\varepsilon)$ and let $\rho_A$ be the uniform distribution on $A$.
Since $\bar p_t=\sum_{m=0}^J \alpha_t(m)p_t^{(m)}$,
$\ip{\bar p_t}{c_t}=\sum_{m=0}^J \alpha_t(m)\ip{p_t^{(m)}}{c_t}=\ip{\alpha_t}{m_t}$.
Therefore
\[
\sum_{t=1}^T \ip{\bar p_t}{c_t}-\ip{\rho_A}{C_T}
=
\sum_{t=1}^T \ip{\alpha_t-e_j}{m_t}
+\sum_{t=1}^T \ip{p_t^{(j)}}{c_t}-\ip{\rho_A}{C_T}\]
Apply \eqref{eq:budgeted-bound-C} to the controller problem on the meta-losses $m_t$ with comparator $e_j$.
The controller prior is uniform on $[J+1]$, so $\KL(e_j\|u_{J+1})=\log(J+1)=\Gamma^{\mathrm{ctl}}$, which gives
\[
\sum_{t=1}^T \ip{\alpha_t-e_j}{m_t}
\le
\Gamma^{\mathrm{ctl}}+Q_{*,\mathrm{ctl}}^T
+(2C+C^{-1})\sqrt{\Gamma^{\mathrm{ctl}}V_T^{\mathrm{ctl}}}\]
Apply the same theorem to worker $j$ with comparator $\rho_A$.
Because the worker prior is uniform on $[K]$ and $|A|\ge \varepsilon K$,
$\KL(\rho_A\|u_K)=\log\frac{K}{|A|}\le \log(1/\varepsilon)\le \Gamma_j$.
Hence
\[
\sum_{t=1}^T \ip{p_t^{(j)}}{c_t}-\ip{\rho_A}{C_T}
\le
\Gamma_j+Q_{*,j}^T+(2C+C^{-1})\sqrt{\Gamma_jV_T^{(j)}}\]
Adding the last two displays proves \eqref{eq:simul-quantile}.
Finally, because $\varepsilon\le e^{-1}$ and $\Gamma_j=2^j$ is the first dyadic budget at least $\log(1/\varepsilon)$,
we have $\Gamma_j\le 2\log(1/\varepsilon)$.
\end{proof}

\begin{proof}[Proof of Theorem~\ref{thm:two-level}.]
The inner chain is Theorem~\ref{thm:pacbayes-second-order} applied to $c_t$: $R_T^c(\rho)=D_T+B_T(\rho)+\sum_t\eta_tQ_t(c)$. By Corollary~\ref{cor:centered-seq}, $R_T^{\ell}(\rho)=R_T^c(\rho)+\sum_t(\ip{\rho}{u_t}-\ip{p_t}{u_t})$. Expanding $u_t(i)=\eta_t^{-1}\sum_w\alpha_{t,w}(-\log\pi_{t,w}(i))$ gives $\sum_t(\ip{\rho}{u_t}-\ip{p_t}{u_t})=\sum_t\sum_w\alpha_{t,w}g_t(w)=\sum_t\ip{\alpha_t}{g_t}$. Applying Theorem~\ref{thm:pacbayes-second-order} to the outer copy on the losses $g_t$ yields $\sum_t\ip{\alpha_t}{g_t}=\ip{\sigma}{G_T}+D_T^{\mathrm o}+B_T^{\mathrm o}(\sigma)+\sum_t\kappa_tQ_t^{\mathrm o}(g)$. Substituting proves \eqref{eq:two-level}; the envelope follows from Theorem~\ref{thm:scaling-time} at each level.
\end{proof}

\begin{proof}[Proof of Proposition~\ref{prop:regret-matching}, Corollary~\ref{cor:selfplay-gap}.]
For Proposition~\ref{prop:regret-matching}, the centered losses satisfy $\ip{p_t}{g_t}=\ip{p_t}{\ell_t}-\ip{p_t}{\ell_t}=0$.
Applying Theorem~\ref{thm:pacbayes-second-order} to the sequence $g_t$ and the comparator $e_i$ gives
\[
\sum_{t=1}^T \left(\ip{p_t}{g_t}-g_t(i)\right)
=
D_T^g+B_T^g(e_i)+\sum_{t=1}^T \eta_tQ_t(g)\]
Since $\ip{p_t}{g_t}=0$ and $-g_t(i)=\ip{p_t}{\ell_t}-\ell_t(i)$, this is \eqref{eq:regret-matching-exact}.
The upper bound \eqref{eq:regret-matching-bound} is the specialization of \eqref{eq:budgeted-bound-C}. In the unclipped regime,
$Q_t(ag)=a^2Q_t(g)$ and $V_t(ag)=a^2V_t(g)$ after replacing $\eta_t$ by $\eta_t/a$, which is the stated scale-invariance relation.

For Corollary~\ref{cor:selfplay-gap}, let $R_T^{\mathrm{row}}(i):=\sum_{t=1}^T p_t^\top M q_t-\sum_{t=1}^T e_i^\top M q_t$ and $R_T^{\mathrm{col}}(j):=\sum_{t=1}^T p_t^\top M e_j-\sum_{t=1}^T p_t^\top M q_t$, where the column player is viewed as minimizing the negated payoff.
The identity from Proposition~\ref{prop:regret-matching} gives
\[
R_T^{\mathrm{row}}(i)=D_T^{\mathrm{row}}+B_T^{\mathrm{row}}(e_i)+P_T^{\mathrm{row}},
\qquad
R_T^{\mathrm{col}}(j)=D_T^{\mathrm{col}}+B_T^{\mathrm{col}}(e_j)+P_T^{\mathrm{col}}.
\]
Since linear optimization over each simplex is attained at an extreme point,
\begin{align*}
T\left(\max_q \bar p_T^\top M q - \min_p p^\top M\bar q_T\right)
&=
\max_j \sum_{t=1}^T p_t^\top M e_j - \min_i \sum_{t=1}^T e_i^\top M q_t \\
&=
\max_j R_T^{\mathrm{col}}(j)+\max_i R_T^{\mathrm{row}}(i),
\end{align*}
which is \eqref{eq:selfplay-gap-exact}.
Dividing by $T$ and bounding the two maxima by $U_T^{\mathrm{row}}$ and $U_T^{\mathrm{col}}$ proves \eqref{eq:selfplay-gap}.
\end{proof}

\begin{proof}[Proof of Theorem~\ref{thm:oco-exact} and Corollary~\ref{cor:oco-point}]
For any $\eta>0$ and any posterior measure $\rho\ll\pi$,
\[
\KL(\rho\|q_{t,\eta})
=
\int_S \log\!\left(\frac{d\rho}{d\pi}(x)\frac{\int_S e^{-\eta F_t(u)}\,\pi(du)}{e^{-\eta F_t(x)}}\right)\rho(dx)
=
\KL(\rho\|\pi)+\eta\!\int_S F_t\,d\rho+\log\!\int_S e^{-\eta F_t}\,d\pi\]
Rearranging gives the continuous analogue of \eqref{eq:At-exact},
\[
A_t^{\mathrm{oco}}(\eta)+\eta^{-1}\KL(\rho\|q_{t,\eta})
=
\int_S F_t\,d\rho+\eta^{-1}\KL(\rho\|\pi).
\tag*{\(\star\)}
\]
Next, by definition of $p_t$,
\[
A_t^{\mathrm{oco}}(\eta_t)-A_{t-1}^{\mathrm{oco}}(\eta_t)
=
-\eta_t^{-1}\log\!\int_S e^{-\eta_t f_t(x)}\,p_t(dx)
=
\int_S f_t\,dp_t-\eta_t Q_t^{\mathrm{oco}}\]
Summing over $t$ and rewriting the varying-temperature telescope gives
\[
\sum_{t=1}^T \int_S f_t\,dp_t
=
A_T^{\mathrm{oco}}(\eta_T)+D_T^{\mathrm{oco}}+\sum_{t=1}^T \eta_t Q_t^{\mathrm{oco}}\]
Subtracting $\int_S F_T\,d\rho$ and using \((\star)\) at time $T$ yields \eqref{eq:oco-exact}.
Also, \((\star)\) shows that $A_t^{\mathrm{oco}}(\eta)$ is the pointwise infimum over $\rho\ll\pi$ of $\int_S F_t\,d\rho+\eta^{-1}\KL(\rho\|\pi)$, so it is nonincreasing in $\eta$ and therefore $D_T^{\mathrm{oco}}\le 0$ for every nonincreasing schedule.
Since $x_t=\int_S x\,p_t(dx)$ and each $f_t$ is convex, the per-round Jensen gap $\int_S f_t\,dp_t-f_t(x_t)\ge 0$, so $J_T^{\mathrm{oco}}\ge 0$, and subtracting it from the density identity \eqref{eq:oco-exact} yields \eqref{eq:oco-regret-posterior}.
Dropping the nonnegative $J_T^{\mathrm{oco}}$ and then repeating the proof of Theorem~\ref{thm:scaling-time} verbatim with $Q_t^{\mathrm{oco}}$ in place of $Q_t(c)$ and $A_t^{\mathrm{oco}}$ in place of $A_t$ gives the second-order bound \eqref{eq:oco-regret-noshift}.

For Corollary~\ref{cor:oco-point}, let $T_\alpha(y):=(1-\alpha)x^\star+\alpha y$ with $\alpha=\epsilon^{1/d}$, and let $\rho_{\epsilon,x^\star}:=T_{\alpha\#}\pi$. Because $S$ is convex,
$T_\alpha(S)\subseteq S$, and because $\pi$ is the normalized Lebesgue measure on $S$, the Jacobian of $T_\alpha$ is $\alpha^d=\epsilon$, so
\[
\frac{d\rho_{\epsilon,x^\star}}{d\pi}(x)=\epsilon^{-1}\one\{x\in T_\alpha(S)\},
\qquad
\KL(\rho_{\epsilon,x^\star}\|\pi)=\log(1/\epsilon)\]
Also, for every $y\in S$ and every round $t$, convexity and boundedness of $f_t$ give $f_t(T_\alpha(y)) \le (1-\alpha)f_t(x^\star)+\alpha f_t(y) \le f_t(x^\star)+\alpha$.
Integrating with respect to $\pi(dy)$ and summing over $t$ yields $\int_S F_T (x) \, \rho_{\epsilon,x^\star}(dx)\le F_T (x^\star) + T \epsilon^{1/d}$.
Applying \eqref{eq:oco-regret-posterior} (dropping the nonnegative $J_T^{\mathrm{oco}}$) with $\rho=\rho_{\epsilon,x^\star}$ proves \eqref{eq:oco-shrink}.
The second-order bound \eqref{eq:oco-shrink-noshift} follows from \eqref{eq:oco-regret-noshift} and the identity $\KL(\rho_{\epsilon,x^\star}\|\pi)=\log(1/\epsilon)$.
Setting $\epsilon=T^{-d}$ gives the displayed specialization, and $Q_t^{\mathrm{oco}}\le 1/8$ follows from Proposition~\ref{prop:range} because each $f_t$ takes values in $[0,1]$.
\end{proof}

\begin{proof}[Proof of Theorem~\ref{thm:boost-generic} and Corollary~\ref{cor:boost-noshift}.]
Because $\ell_t(i)=(1+y_i h_t(x_i))/2$,
$\ip{p_t}{\ell_t} = \frac{1}{2} +\frac{1}{2} \sum_{i=1}^N p_t(i)y_i h_t(x_i) = \frac{1}{2} +\gamma_t$.
Likewise, since $m_i=T^{-1}\sum_{t=1}^T y_i h_t(x_i)$,
$L_T^\ell(i)=\sum_{t=1}^T \ell_t(i)=\frac{T}{2}(1+m_i)$.
Therefore, for any posterior $\rho$,
\[
R_T^\ell(\rho)
=
\sum_{t=1}^T \ip{p_t}{\ell_t}-\ip{\rho}{L_T^\ell}
=
T\!\left(\frac{1}{2} +\bar\gamma_T\right)-\frac{T}{2}\left(1+\ip{\rho}{m}\right)
=
T\bar\gamma_T-\frac{T}{2}\ip{\rho}{m}\]
which is exactly \eqref{eq:boost-generic}.

Now let $A_\epsilon$ be a set of $\lceil\epsilon N\rceil$ examples with the smallest margins, and let $m_{[\epsilon]}$ be the largest margin inside that set.
Since $\rho_{A_\epsilon}$ is uniform on $A_\epsilon$, its average margin is at most that largest value:
$\ip{\rho_{A_\epsilon}}{m}\le m_{[\epsilon]}$.
Combining this with \eqref{eq:boost-generic} proves \eqref{eq:boost-quantile-generic}, and \eqref{eq:boost-margin-from-regret} is immediate from any bound $R_T^\ell(\rho_{A_\epsilon})\le U_T(\epsilon)$.
If more than an $\epsilon$ fraction of the examples had margin at most $\theta$, then necessarily $m_{[\epsilon]}\le \theta$; the contrapositive gives the margin-distribution claim.
Any example misclassified by $H_T$ has margin at most $0$, so the training-error statement follows by setting $\theta=0$.

For Corollary~\ref{cor:boost-noshift}, apply Theorem~\ref{thm:pacbayes-second-order} to the example losses $\ell_t$ and the posterior $\rho_{A_\epsilon}$ to get \eqref{eq:boost-exact}.
Since the prior is uniform, $\KL(\rho_{A_\epsilon}\|u_N)=\log(N/|A_\epsilon|)\le \log(1/\epsilon)$, so Theorem~\ref{thm:scaling-time} gives \eqref{eq:boost-noshift}. Finally,
$Q_t(\ell)\le 1/8$ follows from Proposition~\ref{prop:range}.
Hence $m_{[\epsilon]}\ge 2\gamma - O\!\left(\sqrt{\frac{\log(1/\epsilon)}{T}} +\frac{\log(1/\epsilon)}{T}\right)$ whenever $\gamma_t\ge\gamma$.
Choosing $\epsilon=e^{-cT\gamma^2}$ with $c>0$ sufficiently small makes the right-hand side positive, and then the training error is at most $\epsilon$.

The variational identity \eqref{eq:boost-exp-exact} follows by unrolling the update to $p_{T+1}(i)=e^{-M_T(i)}/\sum_j e^{-M_T(j)}$, so that $\mathcal L_T^{\exp}=\prod_t Z_t$ and $\log\{\rho(i)/p_{T+1}(i)\}=\log\{\rho(i)/u_N(i)\}+M_T(i)+\log \mathcal L_T^{\exp}$; averaging over $i\sim\rho$ gives the display.
\end{proof}

\begin{proof}[Proof of Proposition~\ref{prop:forgetting}.]
For each $\beta$,
$\log \pi_{t,\beta}(i)=\mathrm{const}_t(\beta)-\eta\sum_{s=1}^{t-1}\beta^{t-1-s}\ell_s(i)$,
where the additive constant depends on $(t,\beta)$ but not on $i$. Integrating against $\alpha_t$ and exponentiating gives
\[
q_t(i)\propto \exp\!\left(-\eta\int_R\sum_{s=1}^{t-1}\beta^{t-1-s}\ell_s(i)\,\alpha_t(d\beta)\right)\]
Fubini's theorem yields
\[
\int_R\sum_{s=1}^{t-1}\beta^{t-1-s}\ell_s(i)\,\alpha_t(d\beta)
=
\sum_{s=1}^{t-1}\left(\int_R\beta^{t-1-s}\,\alpha_t(d\beta)\right)\ell_s(i)\]
which is exactly \eqref{eq:forgetting-kernel}.
\end{proof}

\begin{proof}[Proof of Proposition~\ref{prop:pac-bayes}.]
Let $r:=d\rho/d\mu$. Since $\rho\ll \pi_w$ for every $w$, the quantities below are finite whenever the corresponding relative entropies are finite.
By definition of $p^\star_\alpha$,
$\log\frac{r(h)}{dp^\star_\alpha/d\mu(h)} = \log r(h)-\sum_{w=1}^W \alpha_w \log p_w(h)+\log Z_\alpha$.
Integrating with respect to $\rho$ gives
\begin{align*}
\KL(\rho \| p^\star_\alpha)
&= \int_{\mathcal H} \log r\,d\rho
   -\sum_{w=1}^W \alpha_w \int_{\mathcal H}\log p_w\,d\rho
   +\log Z_\alpha \\
&= \sum_{w=1}^W \alpha_w
   \int_{\mathcal H}\log\frac{r}{p_w}\,d\rho
   +\log Z_\alpha \\
&= \sum_{w=1}^W \alpha_w\KL(\rho\|\pi_w)-\cC_\alpha(\pi_{1:W})\end{align*}
because $\sum_w \alpha_w=1$ and $\cC_\alpha(\pi_{1:W})=-\log Z_\alpha$.
This is \eqref{eq:pac-bayes-multi}.
\end{proof}

\section{Experimental details}\label{app:experiments}

The full empirical battery summarized in \S\ref{sec:experiments} is reported here in detail. The reading conventions of \S\ref{sec:exp-conventions} hold throughout: seed-aggregated quantities report BCa $95\%$ bootstrap confidence intervals, identity checks report the maximum absolute residual over the grid, and the figures use the one fixed color system defined there. Exact grids, per-configuration residuals, normalization conventions, and the code regenerating every number and figure below are in the accompanying code release.

\subsection{Setup and protocols}

\subsubsection{Sequence families and schedule panel}\label{sec:exp-protocol}\label{sec:exp-families}

We benchmark the framework on four sequence families chosen to expose genuinely different sources of difficulty. All four use $K=8$ experts, horizon $T=2000$, uniform prior $\pi=(1/K,\ldots,1/K)$, and bounded losses $\ell_t(i)\in[0,1]$. Throughout we set $\Gamma=\log K$ and $C=1/\sqrt 2$ as the second-order schedule constants.

\paragraph{(a) I.i.d.~stochastic losses.}
We take $\ell_t(i)=\mu_i+\xi_t(i)$ clipped to $[0,1]$ with a fixed mean vector $\mu\in[0.15,0.7]^K$ (expert $0$ optimal at $\mu_0=0.15$) and i.i.d.~Gaussian noise $\xi_t(i)\sim\mathcal{N}(0,0.15^2)$. The only difficulty lies in identifying the leading expert; once identified, the remaining intrinsic time accumulates slowly.

\paragraph{(b) Martingale losses.}
We take $\ell_t(i)=m_t(i)+\xi_t(i)$ with predictable $m_t(i)=\mu_i+0.3\sin(2\pi t/500)$ and $\E[\xi_t\mid\mathcal{F}_{t-1}]=0$. This is the canonical setting for side information: with $u_t=-m_t$, the composite loss is the noise residual $c_t=\ell_t-m_t$. Running both ordinary and optimistic Hedge tests whether the decomposition correctly transfers mass from the intrinsic-time term onto the predictable mismatch $M_t(\rho)$ when the compensator is accurate.

\paragraph{(c) Cycling adversarial losses.}
We rotate the best expert every $50$ rounds: the leader of the current segment has loss $\sim 0.1\pm 0.1$ while the rest sit at $\sim 0.6\pm 0.1$. The sequence destroys predictability by design and keeps experts nearly tied on long horizons.

\paragraph{(d) Mixed-character sequence.}
We concatenate four blocks of length $T/4=500$: i.i.d.~stochastic with expert $0$ optimal, cycling adversarial, i.i.d.~stochastic with expert $3$ optimal, cycling adversarial. This is the clearest pathwise test, since a terminal regret number would hide the transitions.

\paragraph{Schedule panel.}
We pair the algorithms in Algorithm~\ref{alg:tempo-family} with two natural fixed-temperature baselines, giving six schedules grouped into three families of two:

\smallskip
\noindent
\begin{tabular}{p{0.32\textwidth}p{0.32\textwidth}p{0.32\textwidth}}
\textbf{Fixed temperatures.} & \textbf{Retempered.} & \textbf{Pressure-target.} \\
$\eta\equiv 0.1$ (low) and $\eta\equiv 1.0$ (high), local update at fixed $\eta$. & Prior-retempered posterior with the second-order rule $\eta_t=\min\{1,C\sqrt{\Gamma/V_{t-1}}\}$, computed from either the exact intrinsic-time clock $V_{t-1}=\sum_{s<t}Q_s(c)$ (\emph{intrinsic}) or its quadratic relaxation $V^{\mathrm{var}}_{t-1}=\sum_{s<t}\Var_{p_s}(c_s)/2$ (\emph{variance}). & Local update with the line-search schedule $\sum_i p_t(i)e^{-\eta_t(c_t(i)-a_t)}=1$. The \emph{intrinsic} variant chooses $a_t$ as the mix-loss at the gap-implied rate $\eta_{\mathrm{gap}}=\log K/\Delta_{t-1}$; the \emph{variance} variant uses the closed-form $\eta_t=\min\{1,C\sqrt{\Gamma/V^{\mathrm{var}}_{t-1}}\}$ with the same local update.
\end{tabular}
\smallskip

\noindent The two columns within each pair share an algorithm class, so any difference comes from the schedule rule alone. The two pressure-target variants differ sharply in temperature dynamics: the variance schedule is bounded by $1$ and sits at that cap on easy paths, while the intrinsic schedule uses the gap-implied rate $\eta_{\mathrm{gap}}=\log K / \Delta_{t-1}$, which grows unbounded as $\Delta_{t-1}\to 0^+$ (capped for numerical safety at $\eta_{\max}=50$) and so ramps far faster.

\paragraph{Comparator panel.}
For each family we evaluate three comparator types: the best single expert $\rho=\delta_{i^\ast}$, a softmax-quantile comparator $\rho_i\propto\exp(-L_T(i)/\sqrt T)$, and the uniform mixture. For stochastic families we aggregate terminal metrics over multiple random seeds and report means with inter-quartile-range (IQR) bands; for deterministic adversarial families we report representative pathwise traces alongside multi-seed summaries.

\subsubsection{A comparator-informed recipe for side information}\label{sec:exp-side-info}

A general recipe for side information is to encode structure shared by the comparator class into a predictable baseline. Identify a predictable common component $m_t(i)$ of the experts' losses and set $u_t(i):=-m_t(i)$, so the composite losses become the residuals $c_t(i)=\ell_t(i)-m_t(i)$. If the comparator class is well explained by that common component, the intrinsic-time clock $V_T(c)$ is small and the theorem yields low regret for the right reason. If the prediction is poor, the same identity remains valid and reports that the sequence was not easy in that way: the bookkeeping does not lose its diagnostic value when $m_t$ fails.

A canonical instance is optimistic play in a repeated matrix game. Let expert $i$ be the pure action $i$, let the loss matrix be $G$, and let $y_t$ be the opponent's mixed action on round $t$, so that $\ell_t(i)=e_i^\top G y_t$. Suppose a forecast $\hat y_t$ of the next opponent move is available from past play---a moving average, a state-space model, or any opponent-model output. Setting $m_t(i):=e_i^\top G\hat y_t$ and $u_t(i):=-m_t(i)$ makes $c_t(i)=e_i^\top G(y_t-\hat y_t)$ the forecast residual, and the original-loss regret against any fixed mixed action $u\in\Delta([K])$ becomes
\begin{equation}
\label{eq:exp-game-recipe}
R_t^\ell(u)\;=\;P_t(c)+D_t+B_t(u)+\sum_{s=1}^t \ip{p_s-u}{m_s},
\qquad c_s(i)=e_i^\top G(y_s-\hat y_s).
\end{equation}
This is precisely the optimistic-Hedge decomposition used in adaptive game play \cite{HazanKale2010,ChiangEtAl2012}. The side information does not change the algebra; it changes the sequence whose difficulty the prefix decomposition measures.

\subsection{The exact decomposition}

\subsubsection{Regret-decomposition shares across families}\label{sec:exp-results-shares}

The regret-decomposition plot of \S\ref{sec:exp-conventions} produces a family-specific signature on each of the four sequence families, tracking $\omega^{\mathrm{pay}}$, $\omega^{\mathrm{drift}}$, $\omega^{\mathrm{info}}$ against the signed prefix regret $R_t^c(\rho)$ at the best-single-expert comparator across the whole schedule panel. Figure~\ref{fig:same-regret-diff-decomp} shows the martingale family under optimistic side information. The prefix identity \eqref{eq:exp-prefix-decomp} holds on every schedule and family.

The predicted decomposition signatures appear cleanly on all four families. On i.i.d.\ paths the comparator-information share $\omega^{\mathrm{info}}$ dominates the first $50$--$200$ rounds and then plateaus, the only difficulty being identification of the leader. The run finishes comparator-information-dominated ($\omega^{\mathrm{info}}\approx 0.86$, $\omega^{\mathrm{pay}}\approx 0.14$, $\omega^{\mathrm{drift}}<0.001$); on cycling-adversarial paths $\omega^{\mathrm{pay}}$ stays high throughout as $V_T(c)$ grows linearly in $T$, finishing intrinsic-time-loss-dominated at $\omega^{\mathrm{pay}}\approx 0.67$; and on the mixed-character sequence the dominant share switches at every block boundary, information-dominated on the i.i.d.\ blocks and loss-dominated on the adversarial blocks. The martingale family of Figure~\ref{fig:same-regret-diff-decomp} is the side-information test. The intrinsic-time share is small under every schedule---the family finishes information-dominated with a small $\omega^{\mathrm{pay}}$ share and a modest drift share---and the predictable mass is carried off-diagram by the mismatch term $M_t(\rho)$ of \eqref{eq:exp-game-recipe}. A companion run confirms that the optimistic version's composite-loss regret goes deeply negative relative to ordinary Hedge as that predictable component is absorbed into the side information. One apparent anomaly is instructive: the pressure-target (intrinsic) panels show a substantial drift share even on easy paths, because each ramp of the gap-implied rate feeds the telescoping drift $D_t^{\mathrm{loc}}=\sum_{s\ge 2}\KL(\rho\|p_s)(1/\eta_s-1/\eta_{s-1})$, whereas the variance variant holds $\eta_t$ near the cap and shows almost none. Only the choice of target $a_t$ separates the two.

At fixed temperature the drift share is identically zero on every family. Read across the regimes, the $\omega^{\mathrm{pay}}$ share rises monotonically, $0.14$ (i.i.d.)\ $\to 0.47$ (shifting) $\to 0.67$ (adversarial), with the drift signature spiking at regime boundaries on the shifting and adversarial sequences. Two qualifications bound the reading. The martingale predictable component is approximated by a $50$-round rolling mean in place of the exact sinusoid $m_t$, so the residual $c-m$ retains small predictable structure. Because the synthetic generators are softer than the paper's worst case, the strict share thresholds (for example $\omega^{\mathrm{drift}}>0.05$ on i.i.d.)\ are not all attained, though the qualitative ordering is exactly as predicted.

\subsubsection{Comparator-by-comparator exactness of the prefix identity}\label{sec:exp-results-comparator}

All other panels in this section, and most of the literature on PAC-Bayes regret bounds, evaluate at the comparator $\rho=\delta_{i^*}$ that puts unit mass on the best fixed expert. Theorem~\ref{thm:pacbayes-second-order} demands more: that the identity close uniformly over $\rho\in\Delta([K])$, and that the $\rho$-dependent piece $B_T(\rho)$---the only piece of the decomposition that depends on $\rho$---track the closed form $(\KL(\rho\|\pi)-\KL(\rho\|q_{T,\eta_T}))/\eta_T$, whose ingredients $q_{T,\eta_T}$ and $\eta_T$ the run already emits.

We sweep the one-parameter comparator family
\begin{equation}
\label{eq:exp-rho-alpha}
\rho_\alpha\;:=\;\alpha\,\delta_{i^*}+\frac{1-\alpha}{K-1}\bigl(\mathbf{1}-\delta_{i^*}\bigr),
\qquad \alpha\in[1/K,\,1],
\end{equation}
which interpolates monotonically from the uniform prior ($\alpha=1/K$, $\KL(\rho\|\pi)=0$) to the point-mass on $i^*$ ($\alpha=1$, $\KL(\rho\|\pi)=\log K$), on prior-retempered runs of horizon $T=1500$ over twelve seeds, on i.i.d.~stochastic and cycling-adversarial paths.

The check confirms the theorem as a $\rho$-uniform statement. Every $(\alpha,\mathrm{seed})$ pair places $B_T(\rho_\alpha)$ on the value the theorem prescribes; the two $\rho$-independent pieces $P_T$ and $D_T$ are constant in $\alpha$, and $R^c_T-B_T\equiv P_T+D_T$ throughout. The full prefix identity closes on both families, so the decomposition accounts for the realized regret exactly, with no residual budget unassigned. The boundary terms are evaluated through their log-space closed forms; for the retempered chain, $B_T(\rho)=A_T(\eta_T)-\ip{\rho}{C_T}$. Forming the terminal weights and taking a divergence instead would fail here, since its off-leader coordinates underflow once the weights collapse onto the leader.

The $\rho$-dependent piece $B_T$ is, perhaps counterintuitively, often \emph{negative} when $q_{T,\eta_T}$ has concentrated faster than $\rho_\alpha$ has.
On i.i.d.~paths $q_{T,\eta_T}$ collapses onto $\delta_{i^*}$ (the cap binds, $\eta_T=1$, and $T=1500$ rounds is enough for the leader to dominate), so for any $\rho_\alpha$ that retains mass on non-leaders we have $\KL(\rho_\alpha\|q_{T,\eta_T})\gg\KL(\rho_\alpha\|\pi)$ and $B_T(\rho_\alpha)<0$.
This is consistent with the identity: the algorithm has already absorbed comparator information beyond what $\rho_\alpha$ itself encodes, and the negative $B_T$ records the gain.
On cycling-adversarial paths, $q_{T,\eta_T}$ stays diffuse (no fixed expert dominates), $\KL(\rho_\alpha\|q_{T,\eta_T})$ stays small, and $B_T(\rho_\alpha)$ is close to the loose upper bound $\KL(\rho_\alpha\|\pi)/\eta_T$ commonly quoted in PAC-Bayes analyses, which the theorem also implies.

\paragraph{Local update.}
The pressure-targeted line-search rule of Section~\ref{sec:pressure} is a different algorithm with a different cumulative bookkeeping, and Corollary~\ref{cor:pressure-regret} states its analogous prefix identity, whose loss term coincides with the prior-retempered $P_T$ on common-rate runs. Running the same $\rho_\alpha$ sweep on pressure-target runs with those local pieces in place of $P_T,D_T,B_T$, the identity closes at the same floor over the same $25\times 8 = 200$ $(\alpha,\mathrm{seed})$ pairs.

\subsubsection{The exact clock against its practical relaxation}\label{sec:exp-results-hoeffding-slack}

Two standard relaxations replace the exact intrinsic increment $Q_t(c)$ by a coarser statistic, and they are not equally informative.
Proposition~\ref{prop:range} gives the range relaxation $Q_t(c)\le (b_t-a_t)^2/8$, which for losses in $[0,1]$ reads $Q_t(c)\le 1/8$ and cumulates to the worst-case Hoeffding clock $V_T(c)\le T/8$.
The Bennett-form relaxation keeps the realized variance of the played distribution,
\begin{equation}
\label{eq:exp-bernstein-clock}
Q_t(c)\;\le\;\Var_{i\sim p_t}\bigl(c_t(i)\bigr)\,\frac{e^{\eta_t}-1-\eta_t}{\eta_t^{2}}
\end{equation}
with a coefficient falling from $e-2\approx 0.718$ at $\eta_t=1$ to $1/2$ as $\eta_t\to0$, so the small-rate limit is the variance proxy $W_t(c)$ of \S\ref{sec:exp-results-varproxy}.
The empirical-Bernstein form \cite{audibert2007c} is this relaxation, and a practitioner deploys it.
The range bound ignores the played distribution entirely, so it is loosest exactly where the learner has concentrated, which is the regime the algorithm is built to reach.
The informative comparison is therefore against \eqref{eq:exp-bernstein-clock}, reported first below.

The measurement runs retempered play at horizon $T=4000$ on the four families, eight seeds each.
Against the empirical-Bernstein clock the exact clock spends $0.66$, $0.85$, $0.89$ and $0.68$ of the budget on the i.i.d.\ stochastic, martingale, cycling-adversarial and mixed-character families.
Exactness therefore tightens the practical standard by a further eleven to thirty-four percent, and the temperature sets how much.
On the concentrating i.i.d.\ and mixed-character paths $V_T(c)$ stays small, the schedule sits at the cap $\eta_t=1$, and the Bernstein coefficient is at its largest $e-2$; those are the two families where exactness gains most.
On the martingale and cycling-adversarial paths the clock accumulates to $V_{4000}\approx 15$ and $\approx 30$, the schedule cools, the coefficient falls toward $1/2$, and the two clocks converge.
The relaxation is tightest where the schedule has cooled, which is also where the envelope of \S\ref{sec:exp-results-envelope} becomes informative.

The range relaxation is a different matter, and the familiar way of quoting it overstates what it shows in a way worth naming.
It sits between $16$ and $1076$ times above the empirical-Bernstein clock on these same runs.
So the observation that the realized clock falls three orders of magnitude below $T/8$ is largely a statement about range against variance, and very little about the exact accounting: any variance-adaptive quantity clears the same margin.
The per-round ratio $Q_t(c)/(1/8)$ concentrates below $0.05$ at the mode on the i.i.d., martingale and mixed-character families, with maxima of $0.21$, $0.17$ and $0.20$.
On cycling-adversarial paths its upper tail is heavier, reaching $0.34$, the rotation keeping $p_t$ from concentrating so that the variance stays substantial.
No family approaches the worst-case $Q_t=1/8$ on the instances tested.

\subsubsection{Forecast-accuracy ablation on a matrix game}\label{sec:exp-results-forecast}

The matrix-game recipe of \S\ref{sec:exp-side-info} predicts that as the forecast $\hat y_t$ improves, the composite loss $c_t(i)=e_i^\top G(y_t-\hat y_t)$ shrinks, the intrinsic-time clock $V_T(c)$ falls, and the diagnostic shifts mass from the on-diagram intrinsic-time loss $P_T$ to the off-diagram predictable mismatch $M_T$. In the other direction, an uninformative forecast leaves a residual at least as hard as the original loss, and on a truly unpredictable opponent, subtracting a forecast uncorrelated with the realized play makes the composite loss harder still, so the intrinsic-time loss rises. The ablation sweeps the forecast from perfect lookahead to a useless constant by mixing the true opponent move with the uniform action, $\hat y_t:=(1-\sigma)\,y_t+\sigma\,\mathrm{Unif}$ with $\sigma\in[0,1]$, so that $\sigma=0$ is the perfect-lookahead extreme discussed in \S\ref{sec:compensators} and $\sigma=1$ is a constant uniform forecast that contains no opponent information.

With perfect lookahead $\sigma=0$ the composite loss collapses to $c_t\approx 0$, so every $Q_t(c)\approx 0$ and the intrinsic-time loss vanishes. Against the no-forecast baseline, for each $\sigma$ the optimistic run incurs no more on the on-diagram axis than the unsupplemented run, and strictly less for $\sigma$ smaller than approximately $0.85$. As anticipated by the discussion of perfect lookahead in \S\ref{sec:compensators}, even a perfect forecast does not reduce the original-loss regret to zero on this game. The algorithm sees $c_t\approx 0$, makes essentially no update from the prior $\pi$, and accordingly accumulates an original-loss gap to the best fixed expert that is captured by the predictable-mismatch term $M_T$. Tuning the optimistic version's prior to be informed by the same forecast (so that $\pi$ already concentrates on the predicted opponent's best response) recovers the gap, but that is a comparator-design question and is orthogonal to the side-information attribution being tested here.


\subsubsection{Envelope tightness: lower side on adversarial paths, upper side on a single-spike construction}\label{sec:exp-results-envelope}

Theorem~\ref{thm:scaling-time} bounds the realized intrinsic-time loss of Algorithm~\ref{alg:tempo-family} on both sides:
\begin{equation}
\label{eq:exp-envelope}
2C\sqrt{\Gamma\,V_T(c)} - C^2\Gamma \;\le\; P_T(c) \;\le\; Q_*^T(c) + 2C\sqrt{\Gamma\,V_T(c)},
\qquad Q_*^T(c) := \max_{1\le t\le T} Q_t(c).
\end{equation}
The envelope is interesting only on paths where $V_T(c)$ accumulates fast enough that the leading $2C\sqrt{\Gamma V_T(c)}$ term dominates the lower-order $C^2\Gamma$ slack. Synthetic stochastic, martingale, and mixed paths in our setup all sit at small $V_T(c)\le 0.3$, so the lower bound is loose by construction (the scheduler is at the cap, $\eta_t=1$, and the realized intrinsic-time loss is essentially $\sum_t Q_t$, which is well-controlled but does not stress the bound). The cycling-adversarial family is the one that exercises the envelope, and the validation is focused there.

The realized intrinsic-time loss reaches $\approx 0.93$ of the envelope's leading term $2C\sqrt{\Gamma V_T(c)}$ at $T=8000$ across seeds, equivalently an unnormalized ratio $P_T/\sqrt{\Gamma V_T(c)}$ approaching $2C\approx 1.41$. The trend across $T\in\{500,\ldots,16000\}$ is increasing and monotone; the shortfall from $1$ at finite $T$ is the lower-order $C^2\Gamma$ initial tax together with the $\eta_{\max}=1$ cap binding on the early rounds. Cycle length and seed variation do not change the picture: across cycle lengths $\{10,25,50,100,200\}$ the realized intrinsic-time loss stays inside the envelope and its gap to the lower bound stays below the $C^2\Gamma$ initial tax in every regime tested. The terminal numbers for the predictable families appear in \S\ref{sec:exp-results-baselines}.

The upper side of the envelope is qualitatively different: it includes the maximal one-round loss term $Q_*^T(c)$, and as the discussion after Theorem~\ref{thm:scaling-time} shows, this maximum-jump correction is forced by a one-round example with $Q_1(c)=q$ at any $q\in(0,1/8]$. On the cycling-adversarial paths above $Q_t(c)$ stays close to its mean throughout the run, so $Q_*^T(c)$ does not exceed the typical $Q_t$ by much and the upper bound is loose. To probe the upper side numerically we instead use the construction the proof points to directly: a single-spike sequence whose $Q_*^T(c)$ is the dominant term. Lifting the loss range from $[0,1]$ to $[0,B]$ (still bounded but allowed to be larger) makes $Q_1$ as large as desired, since $Q_t\le (b_t-a_t)^2/8$ scales with $B^2$.

\paragraph{A single-spike construction.}
Take $K=2$ experts with uniform prior $\pi=(1/2,1/2)$ and $\Gamma=\log 2$. The sequence is $\ell_1=(0,B)$ on round one and $\ell_t=(0,0)$ on every round thereafter, so $Q_t(c)=0$ for $t\ge 2$ and the run reduces to a single non-trivial step. The cap binds at $\eta_1=1$ because $V_0=0$, giving
\[
Q_1\;=\;\frac{B}{2}\;-\;\log\!\left[\tfrac{1}{2}\!\left(1+e^{-B}\right)\right]\;\xrightarrow{B\to\infty}\;\frac{B}{2}-\log 2,\qquad P_T=Q_1,\qquad V_T=Q_*^T=Q_1.
\]
The upper bound at this $V_T$ becomes $Q_1+2C\sqrt{\Gamma Q_1}$, and the tightness ratio is the closed form
\[
\frac{P_T}{\textup{upper}}\;=\;\frac{Q_1}{Q_1+2C\sqrt{\Gamma Q_1}}\;\xrightarrow{Q_1\to\infty}\;1.
\]

The approach is slow: the ratio reaches $0.86$ only at $Q_1\approx 49$ and $0.95$ at $Q_1\approx 499$. The lower bound $\max(0,2C\sqrt{\Gamma V_T}-C^2\Gamma)$ saturates at zero for $V_T<C^2\Gamma=\log 2/2\approx 0.35$ and only catches up to $P_T$ at moderate $Q_1$. The maximum-jump correction $Q_*^T(c)$ is therefore necessary on its own merits: omitting it would leave a $Q_1$-dominant upper bound below the realized $P_T$ for any sufficiently large $Q_1$; the term protects the upper envelope on heavy-tailed-jump paths.

The information spent decides whether that envelope says anything about a given run (Figure~\ref{fig:envelope-informative}). Its width is $C^2\Gamma+Q_*^T(c)$, which has no dependence on $V_T(c)$ at all, so the bracket closes relative to an intrinsic-time loss that grows. Below $V_T=\Gamma/8$ the lower bound is identically zero, at $V_T=C^2\Gamma$ the cap releases and the intrinsic-time loss meets the lower bound exactly. Beyond it the two converge. Measured across twenty-three real-stream games spanning four decades of $V_T/(C^2\Gamma)$---three asset panels, four forecaster families, and a penalty sweep over phase-of-day specialists on two electricity streams---the realized intrinsic-time loss follows one closed form in $x$ alone. On the fifteen games that release the cap it tracks $P_T/(C^2\Gamma)=2\sqrt{x}-1$ to within a tenth of a percent. On the eight that do not, the cap holds $\eta_t=1$ throughout and the loss is the clock itself, $P_T=V_T$. The tightness ratio runs from $0.22$ to $0.99$. The split is sharp and interpretable: every game in which one expert comes to dominate stalls at a small clock and a vacuous bracket, because a learner that has found its expert stops accumulating; the bracket is informative exactly where the learner must keep accumulating.

\begin{figure}[tbp]
\centering
\includegraphics[width=\textwidth]{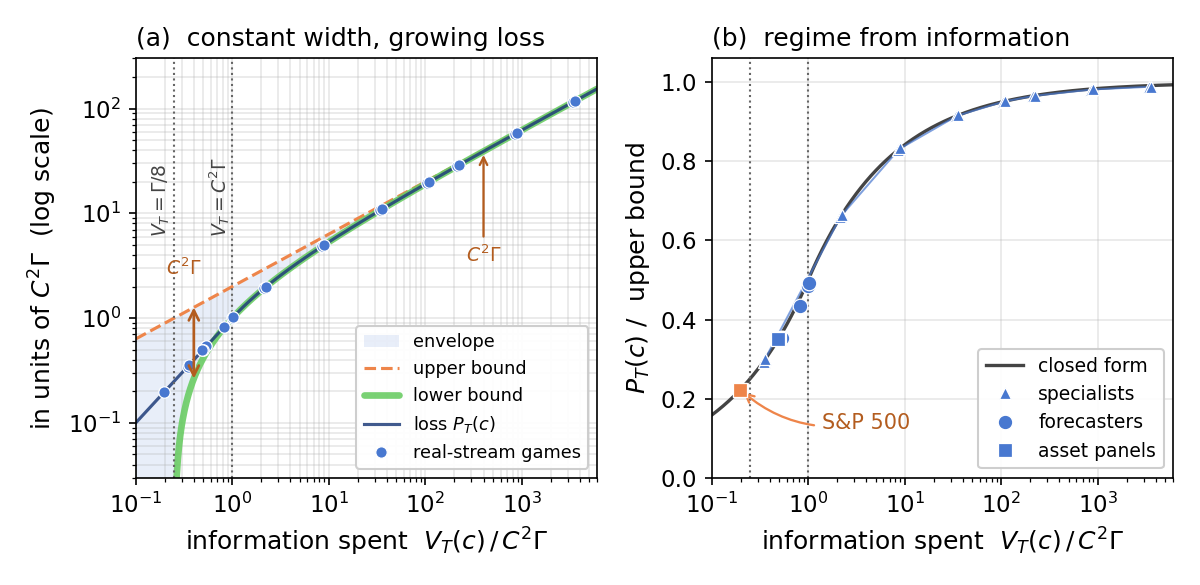}
\caption{\textbf{The information spent decides whether the envelope is informative.} \textbf{(a)} In units of the comparator budget $C^2\Gamma$, the bracket has constant width while the intrinsic-time loss grows, so it closes on a run that keeps accumulating. The arrows mark that one width at both ends of the axis, where it spans a sixth of the panel and where it has closed to a hairline. The dotted lines, named here and drawn bare in (b), mark $V_T=\Gamma/8$, where the lower bound leaves zero, and $V_T=C^2\Gamma$, where the temperature cap releases and the intrinsic-time loss meets the lower bound. \textbf{(b)} Twenty-three real-stream games collapse onto one closed form in $x=V_T(c)/C^2\Gamma$ alone, free of alphabet size, horizon and expert family. That form is $\sqrt{x}/2$ while the cap binds and $(2\sqrt{x}-1)/(2\sqrt{x})$ once it releases, with tightness running from $0.22$ to $0.99$. The games are three asset panels ($K=23,25,36$), four forecaster families, and a penalty sweep over phase-of-day specialists on two electricity streams. Both panels draw the curves at $Q_*^T(c)=0$. The marked point at the far left is the S\&P~500 path: its clock stalls below the activation threshold, which is why its bracket is four and a half times wider than the intrinsic-time loss it brackets. The penalty sweep is an instrument for traversing the axis, and no separate result is claimed from it.}
\label{fig:envelope-informative}
\end{figure}


\subsection{The intrinsic-time clock and the two updates}

\subsubsection{Variance proxy as a Taylor relaxation of $Q_t(c)$}\label{sec:exp-results-varproxy}

Two of the six schedules in the panel of \S\ref{sec:exp-protocol} are driven by the variance proxy $W_t(c)=\sum_{s\le t}\tfrac{1}{2}\Var_{i\sim p_s}(c_s(i))$ in place of the exact intrinsic-time clock $V_t(c)=\sum_s Q_s(c)$. The reading-off of these schedules in the decomposition figures relies implicitly on the small-$\eta$ expansion recorded after Theorem~\ref{thm:pacbayes-second-order}: since $\psi_t(0)=\psi_t'(0)=0$ and $\psi_t''(0)=\Var_{i\sim p_t}(c_t(i))$, the Taylor expansion of the cumulant gives
\begin{equation}
\label{eq:exp-varproxy-taylor}
Q_t(c)\;=\;\frac{1}{2}\Var_{i\sim p_t}(c_t(i))\;+\;O(\eta_t),
\end{equation}
so $V_t(c)$ and $W_t(c)$ agree at leading order with a residual that is linear in $\eta$. The variance schedule replaces an algorithm-defined exact quantity with its leading order; no upper-bound relation is claimed, and Theorem~\ref{thm:pacbayes-second-order} continues to hold along the resulting path with whatever $V_t(c)$ the variance-driven schedule actually realizes.

The relation \eqref{eq:exp-varproxy-taylor} is verified numerically on the four families. On a representative prior-retempered weight trajectory, the residual $\bigl|\eta^{-2}\psi_t(-\eta)-\tfrac12\Var_{p_t}(c_t)\bigr|$ is evaluated at each recorded $(p_t,c_t)$ pair over a logarithmic grid of hypothetical temperatures $\eta$, decoupled from the actually-played $\eta_t$. The round-averaged residual scales linearly in $\eta$, as \eqref{eq:exp-varproxy-taylor} predicts. The empirical log--log slope against $\eta$, over eight seeds, is $0.989$ on i.i.d.~stochastic, $0.985$ on martingale, $1.013$ on cycling adversarial, and $0.992$ on mixed-character paths, all within $\pm 0.02$ of the predicted exponent of $1$, with the cycling-adversarial family the highest because larger range translates into a larger third-cumulant correction. The absolute scale of the residual at $\eta=1$ is $\sim 10^{-5}$ on stochastic and martingale paths and $\sim 4\times 10^{-4}$ on adversarial paths, one to four decades below the variance itself, so the schedule rule $\eta_t=\min\{1,C\sqrt{\Gamma/W_{t-1}}\}$ deviates from the exact $\eta_t=\min\{1,C\sqrt{\Gamma/V_{t-1}}\}$ only in one or two trailing digits of $\eta_t$. The linear scaling is a per-round statement, and cumulative drift can compound when the schedule pushes $\eta_t$ above $1$ on long predictable paths; the cap $\eta_t\le 1$ in Algorithm~\ref{alg:tempo-family} confines the residual to its small-$\eta$ regime by construction.


\subsubsection{The tilted-variance reading of the intrinsic-time increment}\label{sec:eval-intrinsic-time-clock}

Proposition~\ref{prop:tilted-var} states that the intrinsic-time increment is an exact integral of tilted variances,
$Q_t(c)=\int_0^1(1-s)\,\Var_{i\sim p_{t,s}^{(\eta_t)}}\!\big(c_t(i)\big)\,ds$,
where $p_{t,s}^{(\eta)}$ is the $s$-tilted interpolation from the played distribution to its $\eta$-exponential tilt.
The configurations span low-spread (near-flat integrand) and high-spread / bimodal (curved integrand) score configurations, sweeping the temperature $\eta$, the support size $K$, and three score families, with five seeds at each ($180$ in all); the integral is evaluated by Gauss quadrature and compared against the closed form $Q_t(c)=\phi_t(\eta_t)/\eta_t$.

The closed form and the quadrature evaluation agree throughout, and the curvature summary cleanly separates the two regimes.
In the low-spread family the ratio $Q_t/W^{(1)}_t$ of the increment to its first-order variance proxy is $1.000$ to three decimals as $\eta_t\to 0$,
where the integrand is flat and the tilted variances coincide.
In the high-spread and bimodal families the ratio departs from $1$ at $\eta_t=1$, exactly the regime where the average-of-tilted-variances structure differs from a single variance term.
The intrinsic-time increment is therefore the realized average of the tilted variances the Bayes-rule update generates on that round, confirming Proposition~\ref{prop:tilted-var}.

That curvature is operative on real online sequences, and the same diagnostic run on three real-style streams --- a UCI digits log-loss stream, an intraday-electricity-like stream, and a planted-shift stream --- locates it at the round level.
The proxy is accurate in aggregate---the mean per-stream ratio of proxy to exact increment is between $0.98$ and $1.02$---yet it departs by more than five percent from the exact increment on a majority of individual rounds.
This happens on $51\%$, $22\%$, and $93\%$ of rounds on the three streams, $56\%$ on average, with the largest departures on the planted-shift stream.
The curvature distinction is a round-level phenomenon that averages out: the first-order proxy is a good cumulative summary but not a faithful per-round substitute, confirming that the tilted-variance reading of Proposition~\ref{prop:tilted-var} holds on real data.

In magnitude, and in what it costs to drive a schedule by it, the proxy is nonetheless tight; \S\ref{sec:exp-results-varproxy} verifies that it is exact to leading order in the rate.
The full-magnitude surrogate tightness, together with the realized-regret cost of actually driving the schedule by the proxy, is measured on an i.i.d.\
stochastic and a cycling-adversarial family at $K=8$, $T=400$.
The cumulative ratio $W_T/V_T$ stays in $[0.982,\,1.05]$, so the proxy reads the exact clock to within about five percent in magnitude across both families.
Driving the schedule by the proxy instead of the exact clock costs no measurable extra regret on the i.i.d.\ family (regret difference $0$) and a small, separated penalty on the cycling-adversarial family (regret difference $5.0\times10^{-3}$),
where the proxy slightly under-reads the clock and the late-round rate is held below one.
The proxy is thus a faithful relaxation of the exact clock whose only realized cost is a small regret penalty on adversarial paths, precisely the regime in which the second-order curvature term separates from its linear surrogate.

The paper claims that the retempered Bayes update and the local pressure-target update are driven by a single round-by-round information functional $Q_t(c)$, evaluated at each recursion's own played distribution.
The realized $Q_t$ paths of the two updates are compared on the same two families at $K=8$, $T=400$.
At a common temperature the two updates' increments coincide throughout. Under adaptive,
differing rates they diverge, the divergence monotone in the total-variation gap between the two played distributions (Spearman rank correlation $0.83$ on the cycling-adversarial family) and collapsing to coincidence on the i.i.d.\ family,
where the played distributions stay close.
The two updates are thus the same functional evaluated along two played-distribution paths.

\subsection{Schedules and controllers at scale}

\subsubsection{The exact identity across the $K$-grid and horizon sweep}\label{sec:exp-results-identity-scale}

%
The comparator-by-comparator check of \S\ref{sec:exp-results-comparator} certifies the prefix identity on $K=8$, $T=1500$. The grid-level extension checks that the same control holds across the algorithm-relevant range of $(K,T,\text{regime})$ combinations, so that the certification is not a small-$K$ artifact. The residual $\Delta_T:=R_T^c(\rho)-(P_T+D_T+B_T(\rho))$ is evaluated over $960$ configurations spanning $K\in\{8,16,32,64\}$, $T\in\{10^3,10^4\}$, two schedules, twenty seeds, and three sequence families: the i.i.d.\ stochastic and cycling-adversarial families of \S\ref{sec:exp-protocol} together with a segment-shifting family. The headline uses $\rho=\delta_{i^*}$, the single best expert in hindsight; identical control is verified for the diffuse comparators (top-$k/2$ uniform mixture, Laplace-smoothed frequency) on the $K\in\{8,16,32\}$ subgrid. The identity closes across the grid. It has no slack to close to, since the theorem leaves none; the sweep therefore tests the implementation.


%
\subsubsection{The exact identity on real language-model log-loss}\label{sec:eval-llm-ensemble-identity}

The synthetic grid of \S\ref{sec:exp-results-identity-scale} certifies the three-piece decomposition on constructed loss families.
The same exact identity governs genuine, unbounded log-loss data.
Four small open-weights language models (Pythia-70M, Pythia-160M, GPT-2, and Pythia-410M) are scored forward-only over $50\,000$ tokens of natural text, giving a per-token negative-log-likelihood matrix whose columns have mean losses $4.22$, $3.66$, $3.64$, and $3.15$ nats, monotone in model scale.
Treating the models as experts and their per-token surprisals as the loss vectors, the prefix decomposition $R_T^c(\rho)=P_T+D_T+B_T(\rho)$ is evaluated on this real matrix under five predictable schedules: the retempered square-root rule at the best-single and uniform comparators, and fixed temperatures $\eta\in\{0.1,0.5,1.0\}$.
Against a cumulative log-loss of order $1.6\times 10^{5}$ nats, the decomposition closes under all five schedules: the identity is a property of real language-model log-loss.
The retempered ensemble tracks the strongest single model to within $13.9$ nats over the $50\,000$ tokens ($3.1494$ against $3.1491$ nats per token). The static uniform mixture-of-experts predictor reaches $2.47$ nats per token, well below any single model, a genuine diversity gain that the tracking regret and the mixture log-loss measure as distinct quantities.

\subsubsection{The fixed-scale equalizer}\label{sec:exp-results-equalizer}

Let $W_\eta(S):=\eta^{-1}\log\sum_i\pi(i)e^{\eta S(i)}$ be the closed-form Bellman potential of the fixed-scale game. The pointwise increment identity $W_\eta(S_t)-W_\eta(S_{t-1})=\eta Q_t(c)$ is stronger than the variance-proxy bound: the equalizer is the exact log-partition potential and the increment identity holds on every realized path with no slack, providing a complementary single-step verification to the global decomposition check of \S\ref{sec:exp-results-comparator}. Across the $(K,T,\eta,\text{seed})$ grid of \S\ref{sec:exp-results-identity-scale} the increment identity closes throughout. The functional form $\eta^{-1}\log\sum\pi(\cdot)e^{\eta\,\cdot}$ is invariant under the choice of base measure in any exponential family that contains $\pi$; switching base measures shifts $W_\eta$ by an additive function of $\eta$ alone (the centered cumulant generating function of the reference law).


%
%
%
\subsubsection{The anatomy of the empirical-Bernstein bound's slack}\label{sec:eval-bound-slack}

The exact chain says more about the standard analysis when the two are subtracted.
At a fixed rate the identity reads $R_T^c(u)=B_T(u)+\eta V_T$ with $B_T(u)=(\KL(u\|\pi)-\KL(u\|q_{T,\eta}))/\eta$.
The reference bound is the empirical-Bernstein one \cite{audibert2007c}, $\KL(u\|\pi)/\eta+\eta^{-1}(e^{\eta}-1-\eta)\sum_t\Var_{p_t}(c_t)$, which relaxes each increment through the realized variance of the played distribution.
A range bound would serve here too, and would make every share below look better, but it is not what a practitioner would deploy: it ignores the played distribution and is loosest exactly where the learner has concentrated.
Their difference --- the bound's slack over the realized regret --- splits exactly into two nonnegative pieces.
The first is the terminal comparator-information residual $\KL(u\|q_{T,\eta})/\eta$, the part of the prior budget the final weights still hold;
the second is the variance overhang $\eta^{-1}(e^{\eta}-1-\eta)\sum_t\Var_{p_t}(c_t)-\eta V_T$, the amount by which the variance relaxation over-bounds the realized intrinsic-time loss, nonnegative round by round by the Bennett cumulant bound.
Which piece accounts for the slack is charted first on a synthetic full-information grid spanning $320$ configurations:
two loss families (i.i.d.\ Gaussian and two-cluster), rates $\eta\in\{0.01,0.03,0.1,0.3,1\}$, horizons $T\in\{200,1000,5000,20000\}$, several spreads and expert counts, ten seeds at each.
The two shares reconstruct the slack across the grid,
and the map is cleanly monotone.
The terminal residual dominates the slow, unconverged corner and decays as either knob grows.
Its grid-averaged share falls from $0.78$ at $\eta=0.01$ to $0.001$ at $\eta=1$, and from $0.64$ at $T=200$ to $0.026$ at $T=20{,}000$, with the variance overhang taking the complement.
Against the looser $\eta^2\Var_{p_t}(c_t)$ relaxation the same two averages run $0.60$ to $0.001$ and $0.54$ to $0.001$.
Sharpening the reference therefore lifts the terminal share on every one of the $320$ cells, and lifts it most where the rate is smallest, which steepens both decays without turning either around.
The direction runs opposite to the natural conjecture, which has the terminal share growing with the rate;
the measured map shows the residual belongs to the regime where the learner has not yet converged, at small $\eta$ and short $T$.
In the slow, short regime the recoverable slack therefore sits in the terminal residual, so exact-budget methods pay off there,
while in the fast, long regime it sits in the variance overhang, which the realized intrinsic time replaces.
Both loss families show the same structure.

The same two-way split computed round by round on five real streams under importance-weighted bandit feedback, via a supervised-to-bandit reduction, transfers the accounting from controlled data to real limited-feedback play unchanged.
The streams are three stationary (\texttt{digits}, $K=10$; \texttt{covtype}, $K=7$; \texttt{letter}, $K=26$) and two concept-drift (\texttt{elec2}, $K=2$; \texttt{covtype-drift}, $K=7$),
with a synthetic piecewise-stationary control and the full-information grid as bridges.
The split closes pathwise on every stream.
On every real stream the variance overhang accounts for almost the whole slack (share $0.86$ to $1.00$; Figure~\ref{fig:eval-bound-slack-realstreams}):
on real data the loose part of the standard analysis is the variance relaxation, and neither the terminal residual nor the realized intrinsic-time loss.
The terminal residual is small on the stationary streams ($0.001$ to $0.113$) and rises on the drift stream ($0.141$ on \texttt{covtype-drift}) ---
measured on a real drifting trajectory, this is the spike the restart trigger of \S\ref{sec:eval-slack-restart-drift-streams} actuates on.
Sharpening the reference bound moves these shares in the expected direction without changing the reading: against the looser $\eta^2\Var_{p_t}(c_t)$ relaxation the same five streams give overhang shares of $0.89$ to $1.00$ and terminal residuals of $0.000$ to $0.110$. The conclusion therefore survives the harder of the two tests.
The loss share of the realized regret spans $0.77$ to $2.47$, exceeding one where the terminal term is a gain,
which is the drift signature the three-share diagnostic reads.

\begin{figure}[tbp]
\centering
\includegraphics[width=\textwidth]{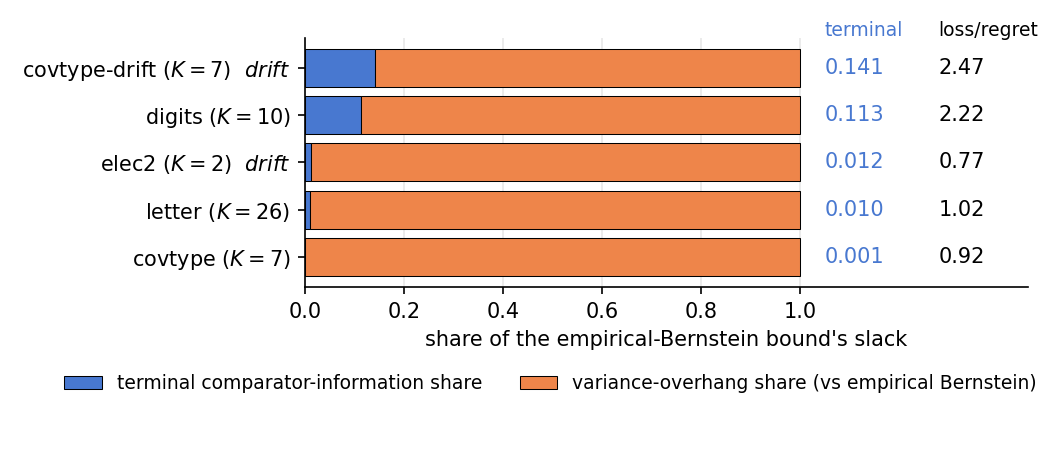}
\caption{\textbf{The variance overhang accounts for essentially the whole slack on every real stream, and the terminal comparator-information share rises under drift.} Trajectory-averaged shares of the empirical-Bernstein bound's slack on real limited-feedback streams, with the loss share of the realized regret at the right. The reference is the sharp Bennett relaxation $\Var_{p_t}(c_t)(e^{\eta}-1-\eta)/\eta$; no range bound enters anywhere, so this is the harder of the two tests. The two shares partition the slack, so each bar sums to one; the terminal share is printed because it is the small quantity the reading depends on. It is near zero on the stationary streams and largest on \texttt{covtype-drift}, the drift stream that the restart trigger of \S\ref{sec:eval-slack-restart-drift-streams} actuates on. The loss share exceeds one exactly where the terminal term is a gain.}
\label{fig:eval-bound-slack-realstreams}
\end{figure}

The terminal residual $\KL(u\|q_{t,\eta})/\eta$ needs only the current weights and the candidate comparator,
so it is computable online at every round --- the classical analysis discards it, and the exact chain accounts for it.
A $27$-point grid checks whether monitoring it stops a run better than the standard empirical-regret-plateau heuristic: three loss families (i.i.d.\ Gaussian, drift, and near-converged-then-shift) crossed with rates and structural parameters, twelve seeds at each, with Hedge run to $T=20{,}000$.
On each realized trajectory three rules are scored by the excess regret left at their stopping time and by the rounds spent past the oracle stopping time:
stop when the computable residual falls below a threshold, stop at the first empirical-regret plateau, and a fixed-horizon reference.
The residual rule leaves lower mean excess regret than the plateau heuristic at all twenty-seven (grid means $0.040$ against $0.066$, with the paired difference negative on average).
The advantage is cleanest on the drift and i.i.d.\ families, where the empirical-regret curve plateaus prematurely while the residual keeps falling,
so the residual-monitored run continues until the comparator is genuinely tracked.
The cost is conservatism: the residual rule spends more rounds past the oracle stopping time ($42$ on average against the plateau rule's $8$),
because it waits for a certified quantity where the heuristic reacts to the first flat stretch.
The computable residual is a usable stopping signal that attains lower final regret at the cost of later stopping.

\subsection{Fast rates and shifting comparators}

\subsubsection{Stochastic luckiness and the fast rate}\label{sec:exp-results-luckiness}

Theorem~\ref{thm:fixed-luckiness} predicts that, when the comparator-centered low-noise condition \eqref{eq:rho-massart} holds with constant $\kappa_\rho$, fixed-rate Hedge with $\eta=\min\{1,1/(2(e-2)\kappa_\rho)\}$ enjoys expected composite-loss regret bounded by $2(1+2(e-2)\kappa_\rho)\KL(\rho\|\pi)$, which is \emph{constant in $T$}. Corollary~\ref{cor:budget-luckiness} extends the conclusion to the second-order retempered schedule of Algorithm~\ref{alg:tempo-family} without requiring advance knowledge of $\kappa_\rho$, at the cost of an additional bounded $C^2\Gamma+Q_*^T(c)$ initialization tax.

The \emph{well-specified} family takes $K=8$ experts with mean vector $\mu_0=0.2$ and $\mu_i\in\{0.5,0.55,\ldots,0.8\}$ for $i\ge 1$, additive Gaussian noise $\xi_t(i)\sim\mathcal{N}(0,0.10^2)$, and clipping to $[0,1]$.
The minimum gap is $d_{\min}=0.30$, and Corollary~\ref{cor:point-luckiness} gives the sufficient condition $\kappa_\rho\le 1/d_{\min}\approx 3.34$ for the comparator $\rho=\delta_0$.
Numerically, the empirical maximum $\widehat\kappa_\rho := \max_{i\ne 0}\, \widehat{\E}\bigl[(c_t(i)-c_t(0))^2\bigr]/(\widehat\mu_i-\widehat\mu_0)$ on a $20{,}000$-round realization is $\widehat\kappa_\rho\approx 0.63$, well below the $1/d_{\min}$ ceiling and giving a theorem bound of $2(1+2(e-2)\widehat\kappa_\rho)\log K\approx 7.92$ for the constant-regret prediction.
The \emph{misspecified} family has $\mu_i\equiv 0.5$ for every $i$; with no separation the per-expert ratio diverges (the denominator $\mu_i-\mu_{k^\ast}$ vanishes at population level), so the low-noise hypothesis fails.
We run both families at horizons $T\in\{500,1000,2000,4000,8000,16000\}$ over eight seeds, with two schedules each: fixed-rate Hedge at the theorem-prescribed $\eta$, and the second-order retempered schedule with $\Gamma=\log K$.

\begin{figure}[tbp]
\centering
\includegraphics[width=\textwidth]{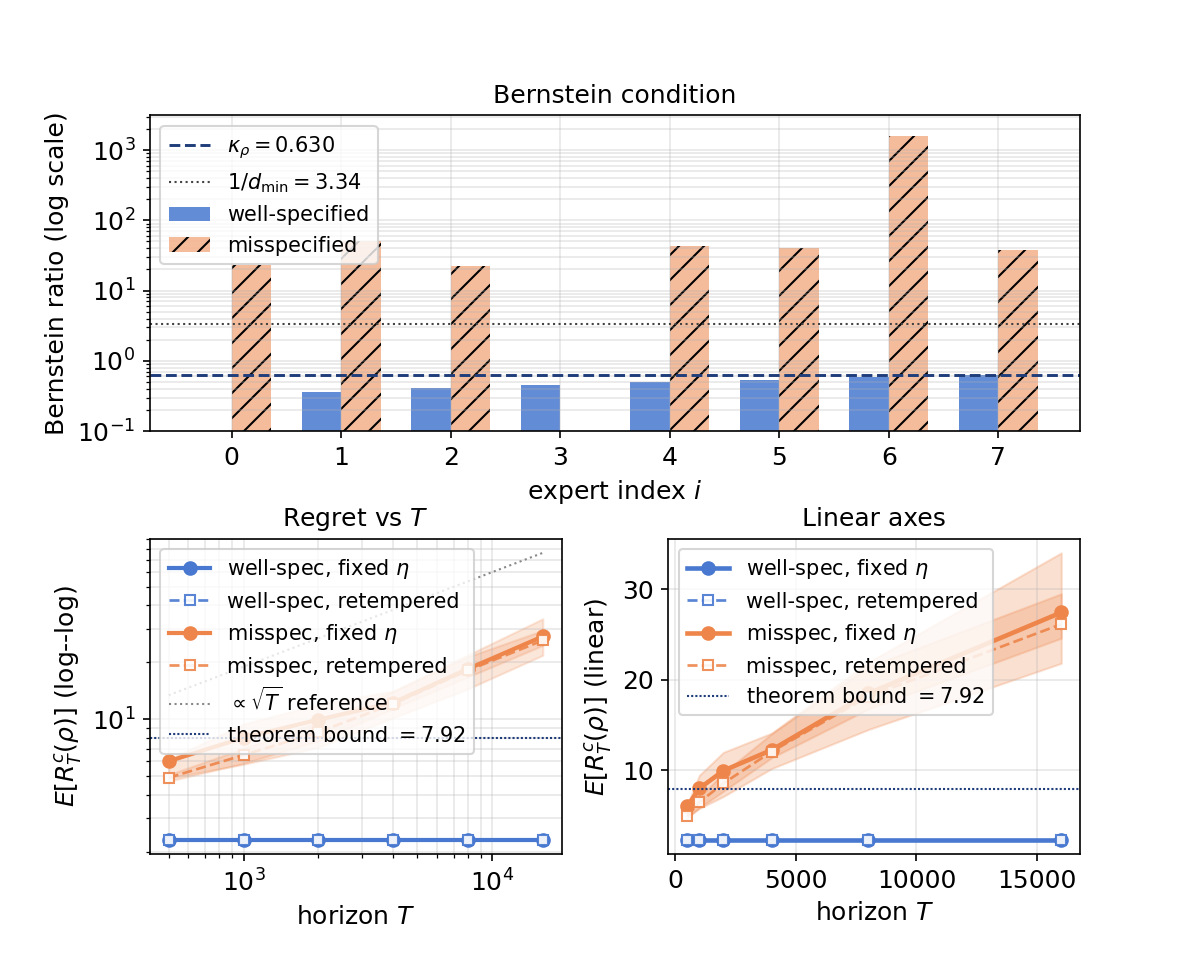}
\caption{\textbf{Regret stays flat on the well-specified family and grows at the $\sqrt{T}$ rate once the Bernstein condition fails.} The panels verify Theorem~\ref{thm:fixed-luckiness} and Corollary~\ref{cor:budget-luckiness}. \textbf{Top:} per-expert Bernstein ratios $\widehat\E[(c_t(i)-c_t(k^\ast))^2]/(\widehat\mu_i-\widehat\mu_{k^\ast})$ on a $20{,}000$-round realization, log $y$-axis. Well-specified family (blue) sits uniformly below $\widehat\kappa_\rho\approx 0.63$ (dashed navy) and below the gap-derived ceiling $1/d_{\min}\approx 3.34$ (dotted gray); misspecified family (hatched orange, $\mu_i\equiv\mu_{k^\ast}$) shows ratios one to four orders of magnitude higher because the gap is essentially zero. \textbf{Bottom left:} regret vs $T$ on log--log axes; the well-specified curves are flat (empirical slopes $\approx 0.00$ for both schedules) while the misspecified curves grow with empirical slopes near the predicted $\sqrt T$ value $0.5$ ($+0.43$ fixed and $+0.49$ retempered in this run; a replication reports $+0.58$). The dashed horizontal line is the theorem-prescribed bound on the well-specified family. \textbf{Bottom right:} same data on a linear scale with inter-quartile bands across seeds. The two schedules are visually indistinguishable on the well-specified family (both at the cap $\eta=1$ throughout); on the misspecified family they coincide because the retempered schedule also sits at the cap until $V_t$ accumulates.}
\label{fig:bernstein-luckiness}
\end{figure}

The well-specified family produces regret $2.29\pm 0.02$ at every $T$ tested, an empirical slope of $\approx 0.00$ on log--log axes, and a value comfortably inside the theorem bound of $7.92$ (Figure~\ref{fig:bernstein-luckiness}). The misspecified family produces a regret that grows from $4.8$ at $T=500$ to $30.4$ at $T=16{,}000$, with empirical slope $0.43$--$0.58$ across runs, consistent with the predicted $\sqrt T$ rate.

The well-specified empirical regret of $2.29$ is approximately $0.29\KL(\delta_0\|\pi)/\widehat\kappa_\rho$ at the realized $\widehat\kappa_\rho$, well inside the theorem's pre-factor $2(1+2(e-2)\widehat\kappa_\rho)\approx 3.81$, confirming that the fixed-rate bound is the right shape but not asymptotically tight. The second-order retempered schedule attains the same fast rate as the fixed-rate Hedge without ever using the value of $\widehat\kappa_\rho$, exactly as Corollary~\ref{cor:budget-luckiness} predicts: on this construction $\widehat\kappa_\rho$ is small enough that the retempered rule sits at the cap $\eta_t=1$ throughout. The cap-binding regime is precisely where the retempered and fixed-rate updates coincide. The same retempered schedule reverts to the $\sqrt T$ rate on the misspecified family, so one algorithm tracks both regimes without configuration changes.

The retempered schedule's fast-rate behavior on this family does not exhibit the additional $2\Gamma+2\,\E[Q_*^T(c)]$ tax that Corollary~\ref{cor:budget-luckiness} adds to the fixed-rate bound, because $V_T(c)$ never reaches the threshold at which the schedule rule starts cooling.


\subsubsection{Sleeping experts and shifting comparators}\label{sec:exp-results-shifting}

Section~\ref{sec:return-luck} extends the prefix identity to comparators that shift across $k$ segments, replacing $\KL(\rho\|\pi)$ by a sum of segmentwise relative-entropy terms. The guarantee is exact; the empirical question is whether real sequences supply a shifting comparator worth competing with. We test that on four real streams---household and regional electricity demand, transformer load, and an equity panel. The setup fixes one family of eight standard forecasters (persistence, daily and weekly seasonal naive, rolling means, phase-of-day climatology, drift extrapolation), losses normalized to $[0,1]$ by a scale fixed on an excluded warm-up prefix, and $k=4$ contiguous equal segments. The paper's schedules, a fixed-share wrapper \cite{herbster1998tracking} at $\alpha=0.02$, and the adaptive baselines of \S\ref{sec:exp-results-baselines} compete against the best fixed forecaster and the per-segment leader. A planted-change-point construction ($k=4$ segments of length $500$, $T=2000$, $K=8$, per-segment leader at loss mean $0.15$ against off-leader means in $[0.45,0.75]$) runs through the same harness as a positive control.

\begin{figure}[tbp]
\centering
\includegraphics[width=\textwidth]{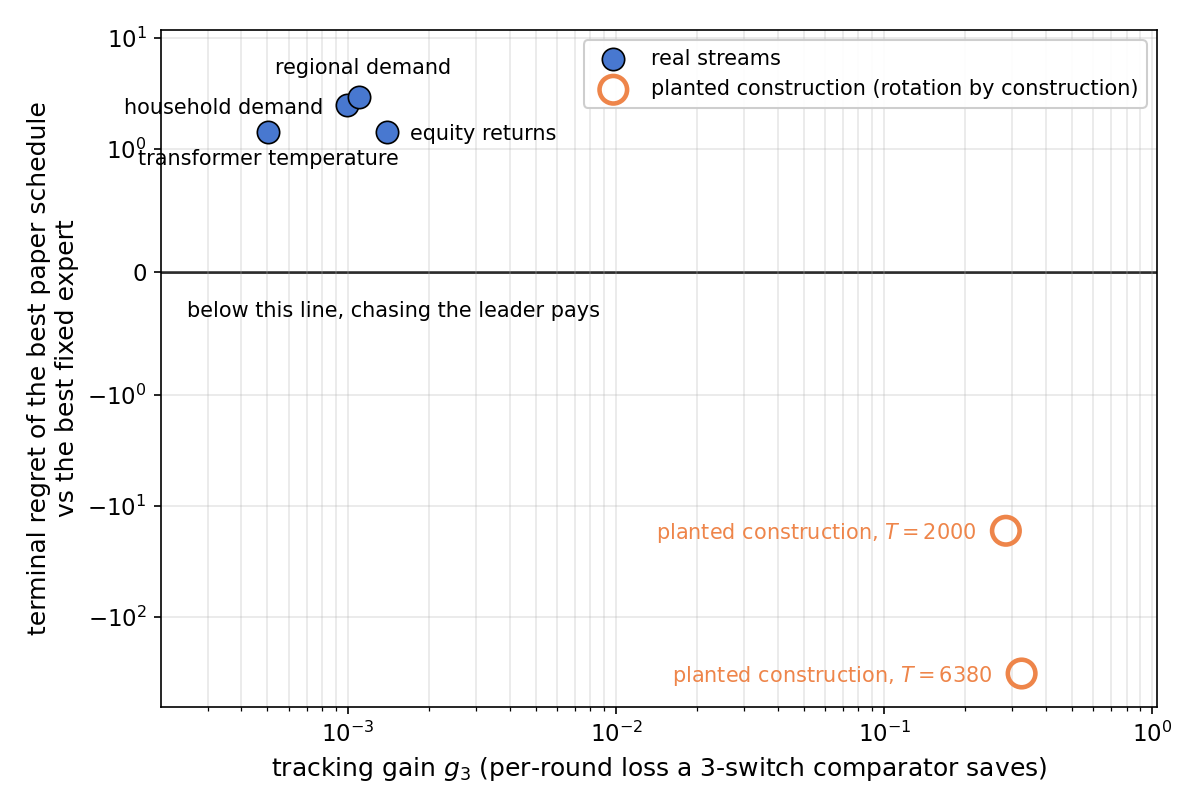}
\caption{\textbf{Real streams and the planted construction fall on opposite sides of the zero line.} Each point is one stream: the terminal regret of the paper's best schedule against the best fixed expert (vertical, symmetric-log) against the tracking gain $g_3$ that the loss matrix reports before any algorithm is run (horizontal, log). The four real streams sit at $g_3\le 1.4\times 10^{-3}$ and finish above the zero line, so no schedule beats the best fixed forecaster on any of them. The planted change-point construction, whose leaders rotate by design, sits two and a half decades to the right and far below the line at both horizons. The separation is a regime gap, and $g_3$ reports which side a stream falls on in advance.}
\label{fig:tracking-gain-scale-gap}
\end{figure}

\textbf{No algorithm beats the best fixed forecaster on any of the four real streams.} The best a paper schedule attains is $+1.41$, and the parameter-free hedges finish in the same band; the prefix identities close throughout. Two measurements account for the outcome. The four-segment shifting comparator is \emph{degenerate} on every real stream---one forecaster leads all four segments, so it coincides with the best fixed forecaster, and the regret to the two agrees. And the \emph{tracking gain} $g_m$, the per-round loss that the best comparator making at most $m$ switches saves over the best fixed expert, computable from the loss matrix before any algorithm is run, is $0.284$ on the planted construction against $0.0005$ to $0.0014$ on all four streams (Figure~\ref{fig:tracking-gain-scale-gap}). The gap is not a shortage of rotation. Lifting the switch budget entirely raises the real streams to $0.09$--$0.23$, comparable to what the construction extracts with three switches, but reaching it takes between five and twenty-seven thousand switches. Real streams lack \emph{persistent} regime structure, so a shifting comparator on any realistic budget has correspondingly little to compete for. The same harness on the construction reaches $-16.7$ at $T=2000$ and $-323.8$ at $T=6380$, so the null is a property of the data.

Leader-change counts do not stand in for $g_m$: the equity stream's trailing-window leader changes $432$ times per thousand rounds and the regional-demand stream's changes $15$ times, with comparable tracking gains and no advantage on either. Nor does designing the experts to rotate change the outcome. A specialist family in which one of eight phase-of-day experts is awake per round---the soft-specialist reduction of \S\ref{sec:compensators}, which keeps every information term finite and the schedule predictable---again leaves every algorithm above the best fixed expert, its rotation likewise spread over thousands of switches. That arm does separate the two updates sharply, the local pressure-target schedule finishing at $+1.75$ against the retempered chain's $+61.3$, which is the behavior \S\ref{sec:schedules} predicts for round-calibrated play. The negative result licenses one reading: leader-chasing gains shown on planted change-point sequences are not evidence about streams whose segments were not planted, and $g_m$ distinguishes the two cases in advance.

\subsubsection{The comparator-information term as a restart trigger on real drift streams}\label{sec:eval-slack-restart-drift-streams}

The terminal comparator-information term of the exact chain is computable online for any chosen comparator.
Against a trailing-window comparator its increments spike when the environment shifts and the converged weights go stale,
so a CUSUM on those increments is a drift trigger with one scale-free threshold.
Restarting the weights toward uniform when it fires gives a tracking policy for the shifting-comparator setting of Section~\ref{sec:return-luck} under bandit feedback.
This evaluation scores that policy on two real concept-drift streams cast as importance-weighted bandits ---
\texttt{covtype-drift} ($K=7$, elevation-ordered so the concept drifts) and the \texttt{elec2} electricity stream ($K=2$) ---
plus a synthetic piecewise-stationary control.
The baselines are sliding-window UCB, discounted UCB, periodic-restart EXP3, and a no-restart ablation that isolates the trigger's value,
with the threshold swept over $\{1,2,4,8\}$, fifteen seeds at each, and $T=20{,}000$.
The scoring comparator is the per-segment-oracle dynamic regret:
against a single fixed arm every algorithm posts negative regret on a drifting stream, each beating the one stale arm, so the fixed comparator cannot rank the policies.

On \texttt{covtype-drift} the slack-triggered restart attains the lowest segmented dynamic regret across the whole threshold sweep:
$542$ on average (range $532$ to $551$), against $753$ for sliding-window UCB, $1564$ for the no-restart ablation, $2217$ for discounted UCB, and $2727$ for periodic-restart EXP3.
It wins on precision: about $1.5$ restarts per run against the periodic policy's $7.0$,
with false restarts at zero in eleven of the twelve configurations (at worst $0.07$ per run).
Its recall is low by design ($0.03$ to $0.23$ on the real streams); the advantage comes from placing few restarts well.
The balance closes pathwise through the restarts throughout, so the restart logic leaves the accounting the trigger is built on intact.
The other streams bound the claim.
On \texttt{elec2} the no-restart ablation sits lower ($432$ to $449$ against the restart policy's $495$ to $550$),
and the window baselines post negative segmented regret there by adapting within segments.
On the synthetic control, where change points are dense and easy (recall $0.64$ to $0.91$), sliding-window UCB is lowest.
The trigger earns its keep where change points are sparse and expensive to miss, and hands the advantage back where a window already suffices.

\subsection{Applications}

\subsubsection{Benchmark traces: the universal-portfolio suite and intraday electricity}\label{sec:exp-results-realdata}

The synthetic core (\S\ref{sec:exp-results-shares}--\S\ref{sec:exp-results-shifting}) controls the difficulty regime to expose specific predictions; the two evaluations below run on genuinely real paths, an open universal-portfolio suite and an intraday-electricity stream.

\subsubsection{The exact decomposition and the exact budget on real financial paths}\label{sec:eval-olps-realdata}

The decomposition figures of \S\ref{sec:exp-conventions} report the prefix identity \eqref{eq:exp-prefix-decomp} on matched synthetic families; we verify it on genuinely real financial return streams, and then isolate the exact-budget mechanism against a softer comparator.
The five canonical universal-portfolio benchmark datasets --- NYSE(O), NYSE(N), SP500, DJIA, and MSCI --- are read as open per-period price series and normalized to genuine per-period relative prices.
Experts are single assets with log-loss $\ell_t(i)=-\log r_{t,i}$; the learner plays the prior-retempered update $p_t(i)\propto\pi(i)e^{-\eta_t C_{t-1}(i)}$ with uniform prior.
This is the framework's native log-loss game, distinct from the portfolio-wealth game of universal-portfolio selection.

The prefix identity holds on the real streams, across the five admitted datasets, three schedules, and two comparators.
The same identity closes for the retempered pressure-target schedule, so the local-update form of Corollary~\ref{cor:pressure-regret} and the retempered form of Theorem~\ref{thm:pacbayes-second-order} are both properties of the realized financial paths.
As on the long i.i.d.\ synthetic paths, the boundary term must be evaluated by its stable closed form $A_T(\eta_T)-\ip{\rho}{C_T}$.
On these multi-thousand-round real streams the played posterior collapses onto a leader distinct from the best-in-hindsight expert, so a naive final-posterior divergence underflows and reports a spurious infinite residual on a minority of runs.
The three-share decomposition varies interpretably across the real datasets: the intrinsic-time-loss share dominates on the high-asset-count NYSE series ($0.59$ and $0.50$), while the shorter series spread the budget more evenly across payment, temperature drift, and terminal-comparator information (DJIA $0.35/0.20/0.46$, MSCI $0.38/0.28/0.34$).
The decomposition is diagnostic: it separates datasets on which the algorithm's regret is dominated by the cost of learning a persistent leader from those on which it is dominated by the residual information the final posterior still retains about the comparator.

In the schedule head-to-head, measured by cumulative log-loss regret per round against the best single asset, the second-order square-root schedule is run at the worst-case budget $\Gamma=\log K$ this comparator prescribes, with no temperature cap.
It attains lower per-round regret than the parameter-free AdaHedge reference \citep{derooij2014adahedge} on three of the five datasets (SP500, DJIA,
and MSCI) and near-ties on the two long-horizon NYSE series (NYSE(O) $0.000565$ vs $0.000552$; NYSE(N) $0.000412$ vs $0.000408$).
The advantage is concentrated on the short, warm-optimal paths: on DJIA and MSCI the properly-tuned schedule reaches $0.000475$ and $0.000245$ against AdaHedge's $0.000532$ and $0.000254$, and the budget $\Gamma=\log K$ flips SP500. A fixed under-tuned rate does not.
Bootstrap confidence intervals over rolling-window starts accompany the square-root regret rate.
Against the best single asset the exact terminal-comparator budget equals the worst-case $\log K$, so no smaller penalty is available and the outcome depends on temperature alone.
AdaHedge's gap-implied rate under-warms on the three short paths; on the two long, well-conditioned NYSE series its adaptive temperature is already near-optimal and the uncapped clock slightly over-warms.

The exact identity's mechanism is therefore visible only against a comparator that holds \emph{fewer} than $\log K$ nats.
We sweep the pooled comparator $\rho_m$ that places uniform mass on the $m$ best-in-hindsight assets: its information budget is $\KL(\rho_m\|\pi)=\log(K/m)=:B_m$, tracing the budget down from $\log K$ (at $m=1$, the best single asset) as $m$ grows.
Regret is measured against that same $\rho_m$ for AdaHedge (agnostic $\log K/\Delta$
\cite{derooij2014adahedge}), an exact-budget schedule that spends $B_m$ instead of $\log K$, and the paper-native retempered pressure-target rule run with cumulative information spend $B_m$.

Against the softer comparators, replacing the worst-case constant $\log K$ with the exact budget $B_m<\log K$ attains lower per-round regret than AdaHedge on the two long-horizon, high-asset-count NYSE series, and the advantage grows \emph{monotonically} as the budget shrinks.
Figure~\ref{fig:locpress-vs-adahedge-penalty-ratio} reports AdaHedge's per-round regret ratio to the exact-budget schedule per comparator:
on the NYSE series the ratio keeps rising with $m$, reaching roughly $1.7\times$ at the softest comparator, the behavior the identity predicts: a smaller penalty helps more as the comparator softens.
At $m=1$ the two schedules coincide, since $B_1=\log K$ makes the exact budget equal to the worst-case constant.

\begin{figure}[tbp]
\centering
\includegraphics[width=\textwidth]{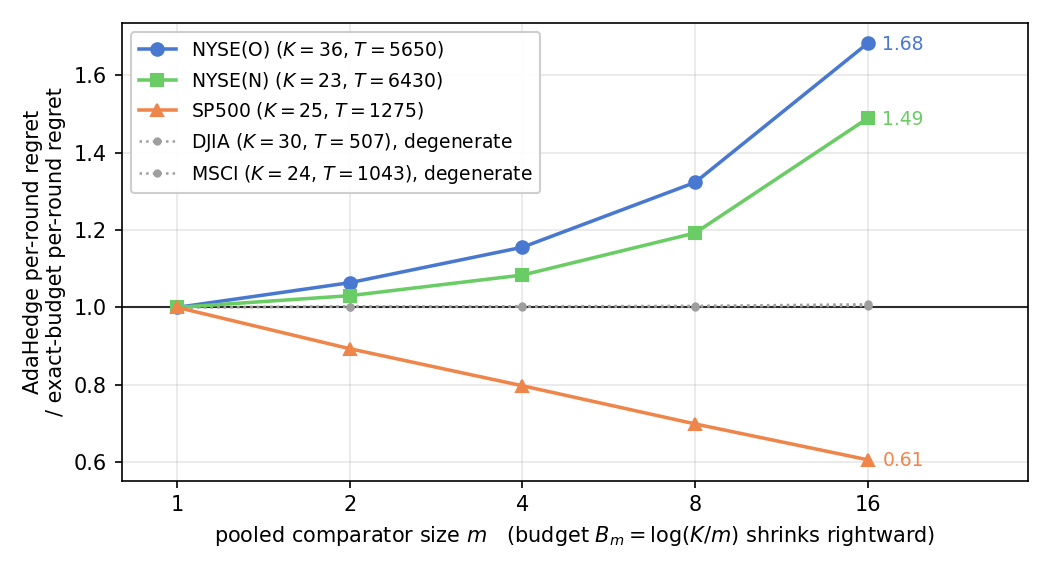}
\caption{\textbf{Shrinking the comparator budget separates the two long NYSE series from the short SP500 series, which reverses.} Ratio of AdaHedge per-round regret to the exact-budget schedule's per-round regret, against the pooled comparator $\rho_m$ (uniform over the $m$ best assets, budget $B_m=\log(K/m)$), on the open universal-portfolio benchmark suite; a ratio above $1$ means the exact budget wins, and the budget shrinks to the right. On the two long-horizon large-$K$ NYSE series the advantage is substantive and grows monotonically as the budget shrinks, reaching $1.68$ and $1.49$ at $m=16$. On the short SP500 series it reverses, falling to $0.61$. The very short DJIA and MSCI series have a degenerate pooled comparator for $m\ge 2$---their regret to $\rho_m$ turns negative, so the ratio is uninformative there---and are drawn in gray.}
\label{fig:locpress-vs-adahedge-penalty-ratio}
\end{figure}

The effect is regime-dependent.
On the short SP500 series the exact-budget schedule loses to AdaHedge and worsens as the comparator softens.
The diagnosis is a temperature-scale mismatch.
AdaHedge runs a warm schedule there (mean temperature $\approx 9.5$), which is the right regime for this short path, whereas the smaller exact budget $B_m<\log K$ structurally prescribes a colder square-root clock, reaching a mean temperature of only $\approx 3.1$ even when the temperature cap is removed entirely.
The reversal is thus a cold-schedule artifact of a warm-optimal short horizon.
No single global adjustment (raising the temperature cap, or normalizing by the loss scale) recovers SP500 without disturbing the NYSE series, because the three small-scale datasets share a similar loss scale and a scale-aware rule cannot single out SP500.
Every modified schedule keeps the exact prefix identity closed.
The worst-case constant $\log K$ is therefore not improved upon on the short,
warm-optimal path, and the exact budget helps precisely when the standard temperature is not already over-warm.

%
\subsubsection{Variance-proxy tightness and schedule ordering on the real electricity trace}\label{sec:eval-varproxy-ordering-electricity}

The variance-proxy relaxation \eqref{eq:exp-varproxy-taylor} and the schedule ordering are checked here on a genuine intraday-electricity trace (the ELEC2 stream, $45\,000$ rounds), replacing the matched-profile synthetic surrogate used previously.
Eight experts predict the up-or-down price move from the real market features under a zero-one loss; the retempered play is run to obtain the weight trajectory, and $Q_t(\eta)=\eta^{-2}\psi_t(-\eta)$ is compared with its relaxation $\tfrac12\Var_{p_t}(c_t)$ across a temperature grid.
On the rounds where the play retains genuine dispersion the ratio $Q_t(\eta)/\tfrac12\Var_{p_t}(c_t)$ rises monotonically toward $1$ as $\eta\to 0$, reaching $0.998$ at $\eta=0.0125$, and the round-averaged residual $|Q_t-\tfrac12\Var_{p_t}(c_t)|$ has empirical log--log slope $0.997$ in $\eta$, matching the predicted exponent of $1$ and the synthetic-family slopes of \S\ref{sec:exp-results-varproxy} to within their spread.
The capped square-root play concentrates quickly on this predictable trace, so the same diagnostic on its own trajectory rests on fewer dispersed rounds but returns the same slope, $0.994$.
Head-to-head on the real trace, with every baseline run from its own update rule, the terminal regret to the best fixed expert orders as NormalHedge \cite{chaudhuri-freund-hsu} at $2.14$, AdaHedge \cite{derooij2014adahedge} at $3.09$, the second-order square-root schedule at $3.69$, and Squint \cite{koolen2015squint} at $12.7$. This ordering places the paper's schedule alongside the parameter-free hedges and Squint as the outlier under its $\eta\le\tfrac12$ constraint.

\subsubsection{Exact recovery of the AdaBoost step from pressure-target boosting}\label{sec:exp-results-boosting}

The pressure-target recurrence of Section~\ref{sec:pressure} reproduces the classical AdaBoost coefficient $\alpha_t^*=\tfrac12\log\{(1-\varepsilon_t)/\varepsilon_t\}$ when the weak hypothesis has binary outputs, and the exponential loss decays by the factor $2\sqrt{\varepsilon_t(1-\varepsilon_t)}$ per round \cite{FreundSchapire97,SchapireFreund12BoostingBook}.
Runs with decision-stump weak learners on multiclass synthetic substrates confirm both statements: the pressure-recurrence coefficient agrees with the closed-form AdaBoost value, and the realized exponential loss sits on the $\prod_t 2\sqrt{\varepsilon_t(1-\varepsilon_t)}$ envelope.
The regret decomposition on the signed-margin state exhibits no retempering drift, since the pressure game is Nature-first and the one-step CGF constraint is active.

\subsubsection{Sampling estimators and the prior--posterior-ratio martingale}\label{sec:exp-results-ts}

The prior-posterior-ratio (PPR) $R_t(\theta):=\pi_0(\theta)/\pi_t(\theta)$ between the initial prior $\pi_0$ and the round-$t$ posterior $\pi_t$ on a candidate parameter $\theta$ is a nonnegative martingale. The intrinsic-time clock of a sampling estimator of a given update tracks the clock of the exact update, with a relative gap of order $T^{-1/2}$ in the horizon and at the same rate in the number of samples drawn per round. The coupling requires a precondition: the played distribution must retain dispersion, since a clock that saturates leaves a bounded denominator under a growing numerator. Posterior Thompson sampling \cite{RussoVanRoy14,RussoEtAl18} compared with exponential weights is a different object---two update rules instead of an update and a sample of it, with generally different stationary dispersions---and the same decay is not predicted for it.

A mutation--selection substrate holds the dispersion precondition fixed while the horizon and the population size vary. The PPR martingale is separately checked against Ville's bound $\Pr(\sup_t R_t\ge 1/\alpha)\le \alpha$ on the grid $\alpha\in\{0.5,0.2,0.1,0.05,0.02\}$ at $T=2000$ over $K\in\{4,8,16\}$, giving $15$ configurations over $500$ seeds, a check sensitive to tail mass.

\subsubsection{Ville bound on the prior-posterior-ratio martingale (\S\ref{sec:exp-results-ts})}\label{sec:eval-ville-bound-ts-martingale}

Ville's bound holds at every level of the $\alpha$ grid, with the conservative Wilson upper bound certifying $12$ of the $15$, the Wilson $95\%$ upper bound being conservative because of Monte-Carlo noise at small $\alpha$. At $\alpha=0.05$ and $\alpha=0.02$ the Wilson upper bound exceeds the nominal level by that noise while the bound itself is satisfied with positive slack, the residual noise reflecting the $500$-seed Monte-Carlo resolution at small $\alpha$.

\subsubsection{Unit-potential wealth as a calibrated anytime-valid test}\label{sec:eval-unit-potential-wealth-detection}

At the unit-potential root of Corollary~\ref{cor:pressure-unit} the one-round normalizer satisfies $\log Z_t=0$ exactly, so each round is a fair bet under the played weights. That certificate has an operational consequence:
under the null the learner's cumulative wealth is an e-process, so thresholding it at $1/\alpha$ gives an anytime-valid sequential test by the same Ville argument as \S\ref{sec:exp-results-ts}.
This evaluation cashes the certificate out as a change detector on three UCI tabular streams cast as expert streams (\texttt{breast\_cancer}, \texttt{digits}, \texttt{wine}).
Each stream runs in a calibration mode (permutation-null streams, measuring the false-alarm rate against nominal levels $\alpha\in\{0.05,0.1\}$) and a delay mode (streams with an injected distribution shift, measuring detection rate and delay against a generic sequential test and a CUSUM baseline).
The injected-shift streams substitute for the concept-drift streams named at the outset.

Calibration holds everywhere:
the wealth test's empirical false-alarm rate sits at or below nominal on every stream ($0.013$ at $\alpha=0.05$ and $0.04$ at $\alpha=0.1$ on \texttt{breast\_cancer}, with the same pattern on the other two),
and the fair-bet identity holds on the exactly-normalized rounds.
Two boundaries qualify the result.
The exact root is rare on these real streams: the damped fallback that the rule permits fires on $97\%$ to $99\%$ of rounds, and the Type-I control survives that dominance.
Validity costs detection power:
the wealth test detects the injected shift on $17\%$ to $67\%$ of runs across the three streams,
against $79\%$ to $89\%$ for the generic sequential test (the CUSUM baseline detects at most $14\%$ on this setup).
The wealth test is a genuinely calibrated anytime-valid detector whose sensitivity trails a test that spends no validity budget;
a restart rule that needs few, well-placed triggers is the setting where that trade wins, and \S\ref{sec:eval-slack-restart-drift-streams} runs it.

\subsection{Comparison with prior methods}

\subsubsection{Comparison with current adaptive online learning}\label{sec:exp-results-baselines}

A benchmark against current adaptive online learning baselines places the paper's schedules on the same axis the field has been using.
We compare the retempered schedule, the pressure-target schedule, and a hybrid \emph{gap-retempered} schedule against the state-of-the-art (SOTA) baselines AdaHedge \citep{derooij2014adahedge}, NormalHedge \citep{chaudhuri-freund-hsu}, AdaNormalHedge \citep{luo2015achieving}, the follow-the-regularized-leader (FTRL) algorithm with $1/2$-Tsallis entropy \citep{zimmert2021jmlr}, and Squint \citep{koolen2015squint}, on the four families with $12$ random seeds.
The gap-retempered schedule takes prior-retempered posteriors at the AdaHedge gap-implied rate; at the uniform prior used here Proposition~\ref{prop:adahedge-instance} makes it AdaHedge itself, so its row is an implementation check.
Every baseline is run from its own published update rule.
Figure~\ref{fig:sota-comparison} reports the mean and inter-quartile range (IQR).
The baselines span very different dynamic ranges: AdaHedge and AdaNormalHedge sit at $\approx 0.5$ regret on stochastic paths, while Squint sits at $\approx 31$.
Each panel therefore uses a symmetric-log (symlog) $y$-axis, which keeps every algorithm legible across the three-order range without compressing the paper's schedules.

\begin{figure}[tbp]
\centering
\includegraphics[width=\textwidth]{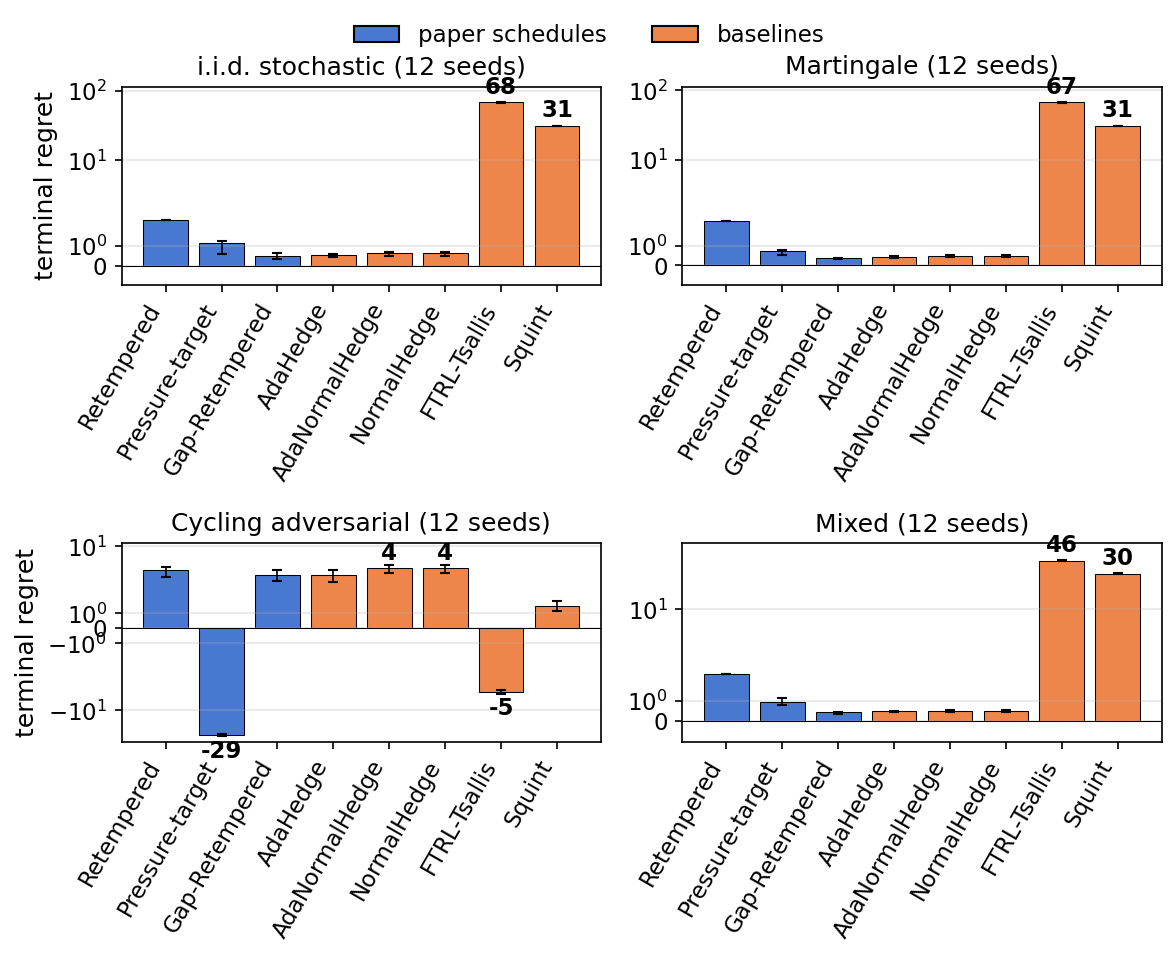}
\caption{\textbf{On the predictable families the paper's retempered schedules cluster with the parameter-free hedges, while Squint and FTRL-Tsallis sit orders of magnitude higher.} Terminal regret to the best fixed expert ($T=2000$, $K=8$, $12$ seeds, mean with IQR). Blue bars: paper schedules; orange bars: SOTA adaptive baselines. Per-panel symmetric-log (symlog) $y$-axis, so the FTRL-Tsallis / Squint outliers do not compress the paper algorithms.}
\label{fig:sota-comparison}
\end{figure}

On the predictable families (i.i.d., martingale, mixed) the gap-retempered schedule reproduces AdaHedge exactly---at a uniform prior the two are one algorithm (Proposition~\ref{prop:adahedge-instance}), so their bars coincide identically---and AdaNormalHedge and NormalHedge sit with them in a tight cluster at terminal regret $0.4$--$0.65$. The plain retempered schedule sits at $\approx 2.3$, the difference matching the width $C^2\Gamma+Q_*^T(c)$ of the envelope's bracket in the cap-binding regime: with $\Gamma=\log 8\approx 2.08$ and $C=1/\sqrt 2$, we have $C^2\Gamma\approx 1.04$ and $Q_*^T(c)\approx 1.2$ on these paths, summing to roughly the observed $2.3$ gap. The retempered schedule accepts the upfront cost in exchange for the worst-case envelope; on these easy paths the tax is wasted but bounded.

On the cycling-adversarial family the pressure-target schedule attains substantially negative regret to the best fixed expert ($-28.9 \pm 3.6$). The reason is geometric: the local update with the line-search rule chases the rotating leader as soon as a switch happens. The cumulative loss of the resulting weight sequence is much smaller than the cumulative loss of any single fixed expert (which has been the leader for only a quarter of the rounds). AdaHedge and AdaNormalHedge cannot do this---their gap-implied rate ramps up only after the cumulative gap accumulates, so they enter each new segment at a rate that lags the change-point by hundreds of rounds.

Squint incurs $\approx 31$ in regret on the i.i.d., martingale, and mixed paths. Squint's analysis requires $\eta\le 1/2$ to control the second-order moment, and on well-separated stochastic paths the optimal fixed temperature is well above $1/2$. The paper's intrinsic-time schedule sidesteps this constraint at the cost of an explicit $C^2\Gamma$ initialization tax: under the cap $\eta_t\le 1$, we incur a known additive term and recover the high-temperature regime that Squint's design forbids. This is the empirical content of the cap design choice.

\paragraph{Cap binding and $\Gamma$ sensitivity.}
The cap $\eta_t \le 1$ in Algorithm~\ref{alg:tempo-family} binds in practice. Figure~\ref{fig:diagnostics} (left) compares the capped Algorithm~\ref{alg:tempo-family} schedule with the uncapped second-order rule $\eta_t=C\sqrt{\Gamma/V_{t-1}}$ on i.i.d.~stochastic and on cycling-adversarial paths. On the stochastic family the uncapped rule settles well above the cap, at $\eta\approx 3.5$; on the adversarial family both rules eventually decay below $1$, with the capped version delayed because it sits at the cap until $V_t$ accumulates past the threshold. Figure~\ref{fig:diagnostics} (right) sweeps $\Gamma$ on three families.
Terminal regret on the i.i.d.\ and mixed families is insensitive to $\Gamma$ across the whole sweep, placing the default $\Gamma=\log K$ well inside their flat regime.
On the cycling-adversarial family the default lies on the shoulder of a curve still rising toward its plateau near $\Gamma\approx 8$.
The schedule is therefore robust to the choice of $\Gamma$ on predictable paths and only mildly $\Gamma$-sensitive on the hardest ones.

\begin{figure}[tbp]
\centering
\includegraphics[width=\textwidth]{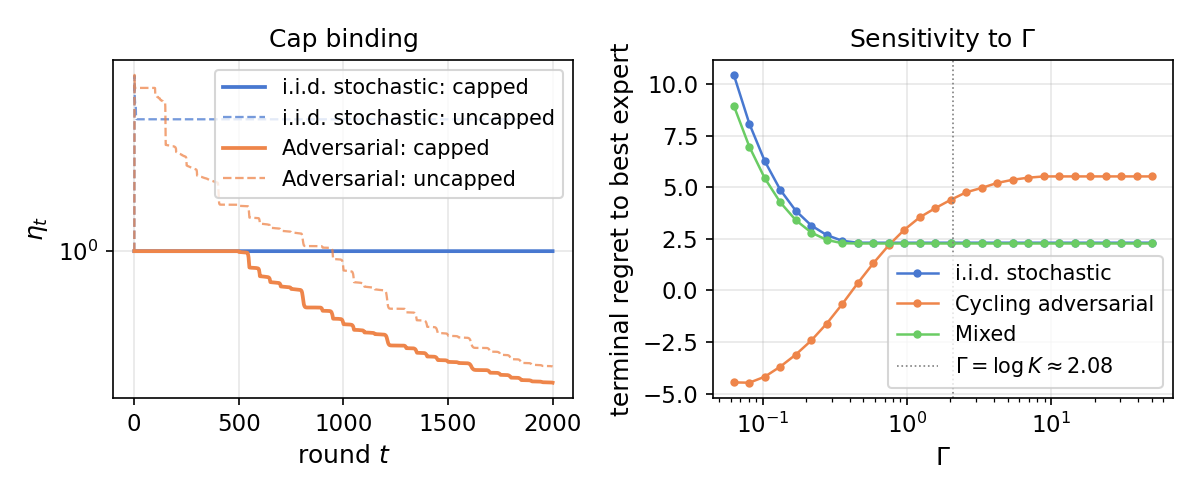}
\caption{\textbf{The temperature cap binds in practice, and the schedule is robust to $\Gamma$ on predictable paths and only mildly sensitive on the hardest ones.} \textbf{Left:} cap-binding diagnostic. Solid lines are the capped Algorithm~\ref{alg:tempo-family} rule $\eta_t=\min\{1,C\sqrt{\Gamma/V_{t-1}}\}$; dashed are the uncapped rule. Symmetric-log $y$-axis to keep the uncapped i.i.d.~trajectory legible. \textbf{Right:} sensitivity of the retempered schedule to $\Gamma$ on three families. The default $\Gamma=\log K$ is robust over two decades.}
\label{fig:diagnostics}
\end{figure}


%
\subsubsection{Head-to-head against faithful baseline implementations}\label{sec:eval-sota-true-baselines}

The comparison of \S\ref{sec:exp-results-baselines} is repeated with an independent set of implementations, each built directly from its baseline's published update rule.
Each implementation includes a unit-check that reproduces its own design behavior: the data-adaptive parameter-free hedges (AdaHedge \cite{derooij2014adahedge}, NormalHedge \cite{chaudhuri-freund-hsu}, AdaNormalHedge \cite{luo2015achieving}) attain bounded regret on a realizable instance and sublinear regret on well-separated stochastic paths. Squint \cite{koolen2015squint} reproduces the higher regret its $\eta\le\tfrac12$ constraint forces on well-separated paths, and the non-adaptive FTRL rule with $\tfrac12$-Tsallis entropy \cite{zimmert2021jmlr} follows its worst-case $\sqrt{TK}$ rate.
All unit-checks pass.

\begin{table}[tbp]
\centering
\small
\setlength{\tabcolsep}{4pt}
\caption{Terminal regret to the best fixed expert ($T=2000$, $K=8$, mean over $12$ seeds with the inter-quartile range in brackets), every baseline run from its own update rule. Rows marked (paper) are the paper's schedules and the remainder are the faithful baseline implementations, the shared row being the gap-retempered schedule and AdaHedge, which coincide at a uniform prior by Proposition~\ref{prop:adahedge-instance}. On the predictable families the retempered paper schedules cluster with the parameter-free hedges; Squint and FTRL-Tsallis are the high-regret outliers, as the panel of Figure~\ref{fig:sota-comparison} also shows.}
\label{tab:sota-true-baselines}
\begin{tabular}{lrrrr}
\toprule
Method & i.i.d. & martingale & cycling & mixed \\
\midrule
Retempered square-root (paper) & $2.42$ $[2.37, 2.45]$ & $3.69$ $[3.20, 3.84]$ & $3.49$ $[2.99, 3.85]$ & $2.94$ $[2.47, 3.22]$ \\
Pressure-target, local (paper) & $481$ $[478, 485]$ & $206$ $[176, 223]$ & $3.30$ $[2.34, 4.37]$ & $238$ $[234, 241]$ \\
AdaHedge $=$ gap-retempered (paper) & $0.93$ $[0.74, 1.02]$ & $2.04$ $[1.59, 2.13]$ & $2.75$ $[2.39, 3.12]$ & $1.49$ $[0.92, 1.82]$ \\
NormalHedge & $0.74$ $[0.45, 0.92]$ & $1.74$ $[1.39, 1.89]$ & $2.66$ $[2.02, 3.35]$ & $1.21$ $[0.57, 1.55]$ \\
AdaNormalHedge & $1.45$ $[1.35, 1.54]$ & $16.2$ $[13.9, 19.2]$ & $4.81$ $[3.51, 6.42]$ & $5.20$ $[1.50, 7.72]$ \\
FTRL, $\tfrac12$-Tsallis & $71.2$ $[70.5, 71.6]$ & $49.7$ $[39.5, 66.2]$ & $-52.8$ $[-53.7, -51.9]$ & $53.8$ $[51.0, 55.9]$ \\
Squint & $10.5$ $[10.49, 10.54]$ & $41.6$ $[36.6, 45.7]$ & $11.8$ $[10.6, 13.2]$ & $13.3$ $[10.3, 16.7]$ \\
\bottomrule
\end{tabular}
\end{table}

The faithful baselines confirm the qualitative reading of \S\ref{sec:exp-results-baselines} (Table~\ref{tab:sota-true-baselines}, terminal regret to the best fixed expert over twelve seeds).
On the predictable families the gap-retempered schedule reproduces AdaHedge exactly, which Proposition~\ref{prop:adahedge-instance} explains: at a uniform prior the retempered play and the exponential-weights play recomputed at the same temperature are the same distribution.
At the gap-implied rate the two are therefore one algorithm, and the shared row is an implementation check.
NormalHedge is marginally lower, and the plain retempered square-root schedule incurs its worst-case initialization overhead of about two nats before settling.
Squint and the non-adaptive FTRL rule are the high-regret outliers.
The local pressure-target is the leader-tracking variant: on stationary predictable paths its running-product weights collapse onto a stale leader, which is why the retempered schedules are used there, and its advantage is confined to chaseable non-stationary structure.
On this independently-constructed cycling family the local pressure-target does not reach the negative-regret leader-chasing regime that the paper's own planted shifting-comparator construction exhibits (\S\ref{sec:exp-results-shifting}); the aggressive fixed-temperature FTRL rule does chase it, reaching $-52.8$.
The negative-regret regime is thus present but family-dependent.

\section*{Acknowledgments}

Large language models were used in producing this work, for which the author is solely responsible.

\bibliographystyle{plainnat}
\bibliography{online_learning}

\end{document}